\documentclass[
  BCOR=15mm,
  headinclude=true,
  footinclude=true,
  mpinclude=false,
  paper=a4,
  twoside=true,
  titlepage=true,
  abstract=true,
  footnotes=multiple,
  headings=normal,
  captions=bottom,
  bibliography=oldstyle, 
  headwidth=438pt,
  toc=index,
  toc=bibliography,
  appendixprefix=true,
  chapterprefix=true,
  numbers=noenddot,
  parskip]{scrbook}

\usepackage[utf8]{inputenc}
\usepackage[ngerman,american]{babel}
\usepackage{graphicx}
\usepackage{subcaption}
\usepackage{amsfonts}
\usepackage{amssymb}
\usepackage{amsmath}
\usepackage{amsthm}
\usepackage{grffile}
\usepackage{array}
\usepackage[hyphens]{url}
\usepackage[pdftex,hidelinks]{hyperref}
\usepackage[strict]{changepage}

\usepackage{geometry}

\usepackage[usenames,dvipsnames]{color}
\usepackage{booktabs}
\usepackage{makeidx}
\usepackage{longtable}
\usepackage[table]{xcolor}
\usepackage{marvosym} 
\usepackage{wasysym} 
\usepackage{multirow}
\usepackage{multicol}
\usepackage{sidecap}
\usepackage{setspace}
\usepackage[all]{xy}
\usepackage{rotating}
\usepackage{enumitem}
\usepackage{pdflscape}
\usepackage{nomencl}
\usepackage{graphicx}

\usepackage{listings}
\lstnewenvironment{lstcode}%
  {\lstset{basicstyle=\fontsize{11pt}{11}\usefont{OT1}{ptm}{m}{n}}}%
  {}

\usepackage{alphalph}

\usepackage{algorithmicx, algorithm}
\algnewcommand\algorithmicinput{\textbf{Input:}}
\algnewcommand\algorithmicoutput{\textbf{Output:}}
\algnewcommand\Input{\item[\algorithmicinput]}
\algnewcommand\Output{\item[\algorithmicoutput]}

\usepackage{algorithm}
\usepackage{makecell}
\usepackage{wrapfig,lipsum,booktabs}
\DeclareOldFontCommand{\bf}{\normalfont\bfseries}{\mathbf}

\newcommand{\interp}{\mathcal{I}}     
\newcommand{\interpDom}{\Delta^\mathcal{I}}    
\newcommand{\interpMap}{\cdot^{\mathcal{I}}}  
\newcommand{\interpw}{\mathcal{I}_w}     
\newcommand{\interpDomw}{\Delta^{\interpw}}    
\newcommand{\interpMapw}{\cdot^{\interpw}}  
\newcommand{\abox}{\mathcal{A}}     
     
\newcommand{\tbox}{\mathcal{T}}     
            
\renewcommand{\Box}{\operatorname{Box}}
\newcommand{\Vol}{\operatorname{MVol}}

\newcommand{\contains}{\operatorname{Disjoint}}
\newcommand{\concept}{\textsf}
\newcommand{\rel}{\textsf}
\newcommand{\insta}{\textsf}

\usepackage{url}
\usepackage{graphicx}
\usepackage{makecell}
\usepackage{amsfonts}
\usepackage{amsmath,amsthm}
\usepackage{multirow}
\usepackage[disable]{todonotes}
\usepackage{algorithm}
\usepackage{algpseudocode}
\usepackage{hyperref}
\usepackage[capitalise]{cleveref}
\usepackage{wrapfig,lipsum,booktabs}
\usepackage{bbm}

\usepackage{bm}   

\PassOptionsToPackage{%
  backend=bibtex8,bibencoding=ascii,%
  language=auto,%
  style=numeric-comp,%
  sorting=nyt, 
  maxbibnames=10, 
  natbib=true 
}{biblatex}
\usepackage{biblatex}
\usepackage[dottedtoc]{classicthesis}
\titleformat{\section}
  {\relax\fontsize{12}{15}}{\textsc{\MakeTextLowercase{\thesection}}}{1em}{\spacedlowsmallcaps}
\hypersetup{
  colorlinks  = true,  
  urlcolor    = black, 
  linkcolor   = black, 
  citecolor   = black  
}

\renewcommand{\Box}{\operatorname{Box}}
\newcommand{\QA}{\operatorname{QA}}

\usepackage{mathtools}
\newcommand{\Con}{\operatorname{\textbf{Con}}}

\newcommand{\set}[1]{\mathcal{#1}}
\providecommand{\sT}{\ensuremath{\set{T}}}
\providecommand{\sR}{\ensuremath{\set{R}}}
\providecommand{\sV}{\ensuremath{\set{V}}}
\providecommand{\sTT}{\ensuremath{\widehat{\set{T}}}}

\renewcommand{\vec}[1]{{\bf{#1}}}

\providecommand{\rhat}{\ensuremath{\widehat{r}}}
\providecommand{\Rhat}{\ensuremath{\widehat{\set{R}}}}

\providecommand{\Ghat}{\ensuremath{\widehat{G}}}

\usepackage{sidenotes}

\makeindex

\makenomenclature
\graphicspath{{./figs/}{./imgs/}{./plot-scripts/plt-wiki-talk/}{./eval/role/plots/}{./eval/trust/}{./plot-scripts/}{./sim/plots/}{./sim/degree-growth/}}

\newcommand{\submissiondate}{Sep. 2, 2024}

\newtheorem{theorem}{Theorem}[]
\newtheorem{proposition}{Proposition}[]
\newtheorem{corollary}{Corollary}[]
\newtheorem{lemma}{Lemma}[]
\newtheorem{definition}{Definition}[]

\usepackage{amsthm}

\numberwithin{equation}{chapter}
\numberwithin{figure}{chapter}
\counterwithin{table}{chapter}

\begin{document}

\storeareas\ClassicThesisDefault
\newgeometry{outer=2.8cm, inner=3.7cm, top=68pt}
\KOMAoptions{headwidth=412pt}

\begin{titlepage}
  \begin{center}
    \large
    \hfill
    \vfill
    \begingroup
    \color{Maroon}\spacedallcaps{ Geometric Relational Embeddings } \\
    \bigskip
    \endgroup

    \vfill
    Von der Fakultät Informatik, Elektrotechnik und Informationstechnik der Universität Stuttgart
    zur Erlangung der Würde eines Doktors der
    Naturwissenschaften (Dr. rer. nat.) genehmigte Abhandlung

    \vfill
    Vorgelegt von

    \spacedlowsmallcaps{Bo Xiong}

    aus Nanchang, China

    \vfill
    \begin{tabular}{ll}
      Hauptberichter:             & Prof. Dr. Steffen Staab \\
      Mitberichter:               & Prof. Dr. Mathias Niepert \\
                                  & Prof. Dr. Isabel Valera  \\
      & \\
      Tag der mündlichen Prüfung: & 24.07.2024
    \end{tabular}

    \vfill
    Institut für Künstliche Intelligenz (KI)
    der Universität Stuttgart

    2024
    \vfill
  \end{center}
\end{titlepage}

\cleardoublepage

\pagenumbering{roman}
\pagestyle{plain}

\pdfbookmark[1]{Abstract}{Abstract}
\begingroup
\let\clearpage\relax
\let\cleardoublepage\relax
\let\cleardoublepage\relax

\chapter*{Abstract}

In classical AI, symbolic knowledge is typically represented as relational data within a graph-structured framework, a.k.a., relational knowledge bases (KBs). 
Relational KBs suffer from incompleteness and numerous efforts have been dedicated to KB completion. 
One prevalent approach involves mapping relational data into continuous representations within a low-dimensional vector space, referred to as relational representation learning. This facilitates the preservation of relational structures, allowing for effective inference of missing knowledge from the embedding space.
Nevertheless, existing methods employ pure-vector embeddings and map each relational object, such as entities, concepts, or relations, as a simple point in a vector space (typically Euclidean $\mathbb{R}$). While these pure-vector embeddings are simple and adept at capturing object similarities, they fall short in capturing various discrete and symbolic properties inherent in relational data.

This thesis surpasses conventional vector embeddings by embracing geometric embeddings to more effectively capture the relational structures and underlying discrete semantics of relational data. Geometric embeddings map data objects as geometric elements, such as points in hyperbolic space with constant negative curvature or convex regions (e.g., boxes, disks) in Euclidean vector space, offering superior modeling of discrete properties present in relational data.
Specifically, this dissertation introduces various geometric relational embedding models capable of capturing: 1) complex structured patterns like hierarchies and cycles in networks and knowledge graphs; 2) intricate relational/logical patterns in knowledge graphs; 3) logical structures in ontologies and logical constraints applicable for constraining machine learning model outputs; and 4) high-order complex relationships between entities and relations. 

Our results obtained from benchmark and real-world datasets demonstrate the efficacy of geometric relational embeddings in adeptly capturing these discrete, symbolic, and structured properties inherent in relational data, which leads to performance improvements over various relational reasoning tasks.


\endgroup
\cleardoublepage

\begingroup
\begin{otherlanguage}{ngerman}
  \pdfbookmark[1]{Zusammenfassung}{Zusammenfassung}
  \chapter*{Zusammenfassung}

In der klassischen Künstlichen Intelligenz wird symbolisches Wissen in der Regel als relationale Daten in einem graphenstrukturierten Rahmen repräsentiert, auch als relationale Wissensdatenbanken (KBs) bekannt. Relationale KBs leiden unter Unvollständigkeit und Rauschen, und zahlreiche Bemühungen wurden der Vervollständigung von KBs gewidmet. Ein verbreiteter Ansatz besteht darin, relationale Daten in kontinuierliche Repräsentationen in einem niederdimensionalen Vektorraum abzubilden, bekannt als relationales Repräsentationslernen. Dies erleichtert die Bewahrung relationaler Strukturen und ermöglicht eine direkte Inferenz von fehlendem Wissen aus dem Einbettungsraum.

Dennoch verwenden bestehende Methoden reine Vektor-Einbettungen und ordnen jedes relationale Objekt, wie Entitäten, Konzepte oder Relationen, als einfachen Punkt in einem Vektorraum (typischerweise euklidisch $\mathbb{R}$) zu. Obwohl diese reinen Vektor-Einbettungen einfach sind und Objektähnlichkeiten gut erfassen können, sind sie weniger geeignet, um verschiedene diskrete und symbolische Eigenschaften, die in KBs vorhanden sind, zu erfassen.

Diese Dissertation übertrifft herkömmliche Vektoreinbettungen, indem sie geometrische Einbettungen annimmt, um die relationalen Strukturen und die zugrunde liegende diskrete Semantik relationaler Daten effektiver zu erfassen. Geometrische Einbettungen ordnen Datenobjekte als geometrische Elemente zu, wie Punkte im hyperbolischen Raum mit konstanter negativer Krümmung oder konvexe Bereiche (z. B. Boxen, Scheiben) im euklidischen Vektorraum, und bieten damit eine überlegene Modellierung diskreter Eigenschaften, die in relationalen Daten vorhanden sind.

Insbesondere führt diese Dissertation verschiedene Modelle geometrischer relationer Einbettungen ein, die in der Lage sind, folgendes zu erfassen: 1) komplexe strukturierte Muster wie Hierarchien und Zyklen in Netzwerken und Wissensgraphen; 2) komplexe relationale/logische Muster in Wissensgraphen; 3) logische Strukturen in Ontologien und logische Einschränkungen, die zur Einschränkung von Ausgaben von maschinellem Lernen verwendet werden können; und 4) komplexe Beziehungen höherer Ordnung zwischen Entitäten und Relationen.

Unsere Ergebnisse, die aus Benchmark- und realen Datensätzen gewonnen wurden, zeigen die Wirksamkeit geometrischer relationer Einbettungen bei der geschickten Erfassung dieser diskreten, symbolischen und strukturierten Eigenschaften, die in relationalen Daten vorhanden sind.

\end{otherlanguage}
\endgroup

\vfill

\cleardoublepage

\pdfbookmark[1]{Acknowledgments}{acknowledgments}


\bigskip

\begingroup
\let\clearpage\relax
\let\cleardoublepage\relax
\let\cleardoublepage\relax
\chapter*{Acknowledgments}

It would not be possible without the collaboration with my co-authors, fellow colleagues, supervisors, and the support from my family.

Foremost, I would like to thank Steffen Staab, my primary supervisor, who consistently sets the highest standards for "what is good research". In the early stages of my research, Prof. Staab displayed exceptional patience and kindness, especially when my English-speaking and scientific presentation skills were less proficient. His insightful feedback on my research and critical thinking has been invaluable, and he has provided hands-on guidance in paper writing, teaching me how to conduct impactful and meaningful research.

Meanwhile, I am equally grateful to Mathias Niepert and Isabel Valera, who served as members of my Thesis Advisory Committee (TAC). Their constructive comments and discussions during my TAC meetings greatly shaped the quality of the research.

I would like to thank some fellow co-authors whose guidance greatly influenced my research: Michael Cochez for discussions on learning with hyperbolic geometry; Shirui Pan for discussions on graph neural networks in various geometric spaces; Nico Potyka, a former postdoc in my research group, for discussions on Description Logic and ontology embeddings.

I also extend my heartfelt thanks to many collaborators who have significantly contributed to the development of my work. I acknowledge the valuable contributions of Mojtaba Nayyeri, Shichao Zhu, Trung-Kien Tran, Carl Yang, Daniel Daza, Zifeng Ding, Chengjin Xu, Chuan Zhou, Jiaying Lu, Ming Jing, Linhao Luo, Yuqicheng Zhu, Yunjie He, Zihao Wang, Özge Erten, Cosimo Gregucci, Daniel Hernandez, and many others. Their collaboration and insights have enriched my academic pursuits.

My Ph.D. project was mainly supported by two parties: 1) KnowGraphs (Knowledge Graphs at Scale), in which I worked as the ESR 5 (Early Stage Researcher) for three years; 2) the International Max Planck Research School for Intelligent Systems (IMPRS-IS), from which my Ph.D. thesis advisory commitment (TAC) team was formed.  I would like to express my sincere gratitude to KnowGraphs and IMPRS-IS for providing essential support and instruments in facilitating my research journey. 

In addition, my research journey was enriched by two secondments. I am deeply thankful to Michel Dumontier of Maastricht University in the Netherlands, who generously supported and supervised my exploration of ontology embeddings for Cell ontology. I am equally appreciative of Roberto Navigli at Babelscape in Rome, who served as an excellent host during my secondment. My time in Rome was not only academically enriching but also personally fulfilling, as I had the opportunity to make new friends and cherish the experiences that contributed to my growth. 

Last but not least, I would like to thank my family. I would not be able to finish my Ph.D. without their truly love.

\endgroup

\cleardoublepage

\pagestyle{scrheadings}
\pdfbookmark[1]{\contentsname}{tableofcontents}
\setcounter{tocdepth}{2} 
\setcounter{secnumdepth}{3} 
\manualmark
\markboth{\spacedlowsmallcaps{\contentsname}}{\spacedlowsmallcaps{\contentsname}}
\tableofcontents
\automark[section]{chapter}
\renewcommand{\chaptermark}[1]{\markboth{\spacedlowsmallcaps{#1}}{\spacedlowsmallcaps{#1}}}
\renewcommand{\sectionmark}[1]{\markright{\textsc{\thesection}\enspace\spacedlowsmallcaps{#1}}}

\cleardoublepage

\let\OOOchapter\chapter
\renewcommand{\chapter}[1]{\OOOchapter{\texorpdfstring{#1}{\thechapter~#1}}}
\let\OOOsection\section
\renewcommand{\section}[1]{\OOOsection{\texorpdfstring{#1}{\thesection~#1}}}
\newcommand{\sectionX}[2]{\OOOsection[\texorpdfstring{#1}{\thesection~#1}]{#2}}

\pagenumbering{arabic}
\pagestyle{scrheadings}

\ClassicThesisDefault
\KOMAoptions{headwidth=438pt}

\cleardoublepage
\chapter{Introduction}
\label{chap_intro}

\section{Background, Motivation, and Challenges }

Representation learning plays a pivotal role in modern machine learning, providing the capability to acquire compact, continuous, and lower-dimensional representations for a variety of real-world data, including images \cite{DBLP:conf/iccv/HuaBW07}, words \cite{DBLP:conf/emnlp/ToutanovaCPPCG15}, and documents \cite{DBLP:conf/bibm/0001DCBRD20}. These distributional representations allow for the discrimination of relevant distinctions while disregarding irrelevant variations among these objects. They serve as valuable inputs for various machine learning tasks, such as image recognition \cite{DBLP:conf/iccv/HuaBW07}, text categorization \cite{DBLP:conf/emnlp/ToutanovaCPPCG15}, and disease diagnosis \cite{DBLP:conf/bibm/0001DCBRD20}.

The predominant approach in many current works maps objects into a low-dimensional vector space, typically represented in the Euclidean space $\mathbb{R}^d$. We refer to this representation as a plain vector embedding. The rationale behind using plain vector embeddings is their ability to preserve 'similarities' between objects through pairwise distances or inner products in the vector space. For example, images from the same categories or words occurring in similar linguistic contexts are mapped to vectors that are 'near' in the embedding space.

Unlike the approaches prevalent in modern machine learning, classical AI, such as knowledge representations and reasoning \cite{DBLP:conf/ecai/LakemeyerN92} and statistical relational learning \cite{popescul2003statistical}, relies on symbolic knowledge. Symbolic knowledge is often represented as structured and relational data that delineate semantic relationships among entities and/or concepts. Typically, this representation takes the form of a set of factual statements, each encapsulating a fact that describes a semantic relationship involving two or more entities and/or classes. Additionally, with the aid of complex mathematical constructs, symbolic knowledge can also be described as a set of logical statements, each describing a logical relationship among various entities and/or concepts. These factual and logical statements together form a symbolic knowledge base, storing relational knowledge over a specific domain. Such symbolic knowledge plays a crucial role in various applications, including biomedical \cite{rector1996galen,gene2015gene} and intelligent systems \cite{stephanopoulos1996intelligent}.

\begin{figure}[t!]
\begin{center}
\centerline{\includegraphics[width=\columnwidth]{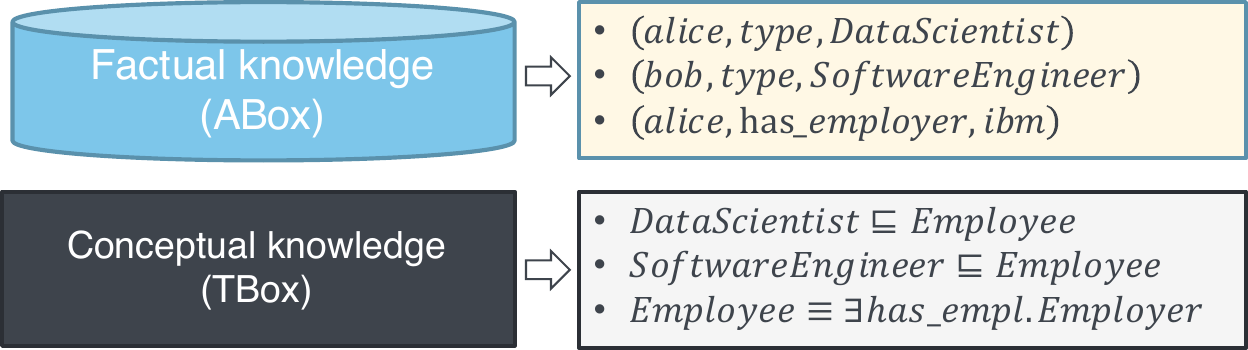}}
\caption{A schematic illustration of a symbolic knowledge base. It consists of a ABox describing the factual knowledge over entities and a TBox describing the conceptual knowledge over concepts.}
\label{fig:kb}
\end{center}
\vskip -0.3in
\end{figure}

The relational knowledge expressed in a symbolic knowledge base can be categorized into two categories as shown in Fig \ref{fig:kb}: 

\begin{itemize}[leftmargin=0.5cm]

    \item \textbf{Factual knowledge} describes relationships among different entities. 
    This is typically described in a knowledge graph, where factual knowledge is represented by a set of factual statements in the form of triples \((h, r, t)\), with \(h, t\) being the entities and \(r\) being the label of a binary relation between them. Beyond binary relations, a knowledge graph can be extended to represent high-order or multi-fold relations among multiple entities, such as the co-authorship relationships in a scientific network. 
    
    \item \textbf{Conceptual knowledge} describes relationships among different concepts. This is typically described in concept ontologies, where  concepts are organized in a concept hierarchy with hierarchical and conceptual relationships such as "is\_a" or "has\_part". By applying complex mathematical constructs such as intersection, existential, and universal quantifications, ontologies can be extended to also describe logical relationships among concepts.
\end{itemize}

Following the conventions of the semantic web community, we call the part of the factual knowledge an ABox (assertional knowledge) and the part of conceptual knowledge a TBox (terminological knowledge).

Relational representation learning acts as a bridge, transforming discrete and symbolic relational data into continuous and low-dimensional representations. This approach connects classical knowledge representation with modern vector-based machine learning. The benefits of using vector representations for relational data are twofold: 1) Vector representations capture not only the relational structure suitable for classical reasoning but also the similarity and analogical structure between entities/concepts, facilitating analogical and similarity-based reasoning. 2) The learned vector representations are robust to incomplete and noisy data, making them more suitable for reasoning tasks in real-world scenarios. In the real world, many relational datasets, such as Freebase \cite{DBLP:conf/sigmod/BollackerEPST08} and DBpedia \cite{DBLP:conf/semweb/AuerBKLCI07}, are curated through human efforts. Despite the already substantial volume of this data, it remains incomplete and noisy.

Vector representations have proven beneficial for various instances of relational data, spanning graphs, knowledge or multi-relational graphs, and ontologies:
\begin{itemize}[leftmargin=0.5cm]
    \item \textbf{Graph embeddings} map nodes in graphs into vectors while preserving the graph structure \cite{kipf2016semi}. Typical approaches include random walk-based approaches \cite{DBLP:journals/pvldb/FangKLWFLYC23} and graph neural networks \cite{chen2020simple}. These learned embeddings can be used to predict missing node types (i.e.,m node classification), missing edges between nodes (i.e., link prediction), and graph-level properties (i.e., graph classifications).

    \item \textbf{Knowledge or multi-relational graph embeddings} map both entities and relations into vector space while preserving their relational structures in the embedding space \cite{DBLP:conf/nips/Kazemi018}. This is achieved by modeling relations between entities as functions (i.e., functional methods) \cite{DBLP:conf/nips/BordesUGWY13} or by modeling the plausibility of a fact as a three-way interaction (i.e., semantic matching methods) \cite{NickelTK11}. Knowledge graph embeddings can effectively infer missing relational facts, even in incomplete knowledge graphs, which is a process known as knowledge graph completion.

    \item \textbf{Ontology embeddings} map concepts in ontologies into vectors, while preserving the ontological structure \cite{DBLP:conf/kr/Gutierrez-Basulto18}. Unlike knowledge graph embeddings that focus on factual knowledge, ontology embeddings consider ontological knowledge. Encoding the logical structure in a standard embedding space is challenging but crucial for ontology embeddings.

\end{itemize}

\begin{figure}
    \centering
    \includegraphics[width=\columnwidth]{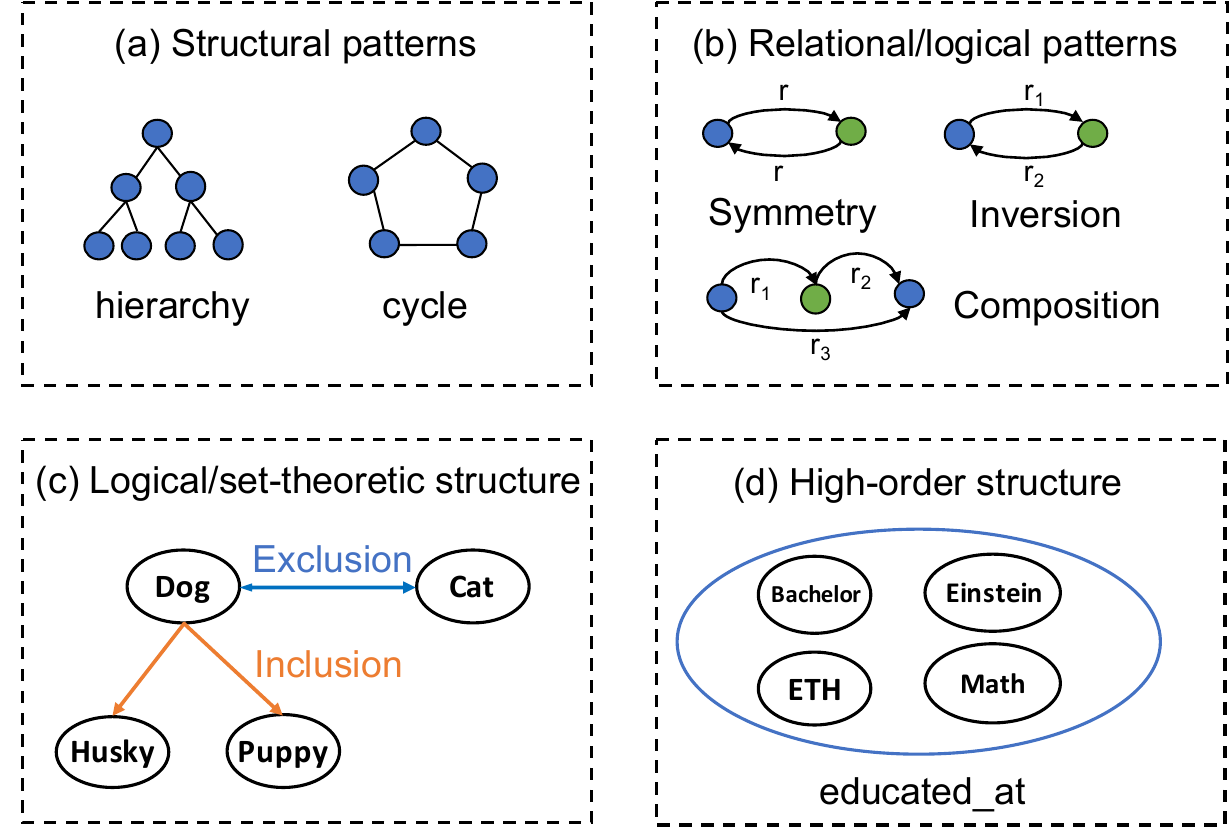}
    \caption{A schematic illustration of the discrete properties in relational data. (a) Structural patterns include hierarchical and cyclic structures in graphs; (b) Relational patterns are logical rules/implication over relations; (c) A logical/set-theoretical structure described by logical or set operators (inclusion and exclusion). (d) A high-order structure described by multi-fold relations among multiple entities.  }
    \label{fig:properties}
\end{figure}

\textbf{Challenges.} These learned representations enable effective inference of missing knowledge directly from the embedding space. However, most relational representation learning approaches consider plain Euclidean vectors as the embeddings, which is similar to the embeddings of image and text data, but they may fall short in capturing crucial properties of relational data that are not easily modeled in a plain, low-dimensional vector space. These properties are typically structural, discrete, and symbolic, while plain vector embeddings are designed to capture similarity only. 
Fig \ref{fig:properties} shows some examples of these discrete properties. 
We elaborate these properties as follow:

\noindent
\begin{itemize}[leftmargin=0.5cm]
    \item \textbf{Structural patterns.} 
    Relational data exhibit highly complex structural patterns such as hierarchies and cycles. Typical examples include WordNet \cite{DBLP:journals/cacm/Miller95} describing the word sense hierarchy and Gene ontology \cite{ashburner2000gene} describing the hierarchy of gene functions. The plain Euclidean-based vector embeddings are designed for flat data  but struggle with modeling data with complex structural patterns. 
    For example, the number of nodes in a hierarchy grows exponentially while the volume of Euclidean space only grows linearly w.r.t the radius. 
    
    \item \textbf{Relational/logical patterns.} Multi-relational data or knowledge graphs, describe relational facts between entities. These relations exhibit many relational or logical patterns such as symmetry (e.g., \emph{has\_friend}), inversion (\emph{is\_director\_of} and \emph{is\_directed\_by}) and implication (e.g., \emph{mother\_of} $\rightarrow$ \emph{parent\_of}). 
    Modeling these relational or logical patterns is of great importance to the embeddings as it not only imropves expressiveness of knowledge graph embedding models, but more importantly, facilitates guaranteed generalization capability, i.e., once the patterns are learned, facts that adhere these patterns can be inferred. 
    
    \item \textbf{Logical/set-theoretic structure.} Relational data may have been defined by applying set-theoretic operators such as set inclusion and set exclusion or logical operators such as logical intersection and negation. This is especially useful when describing conceptual/schematic knowledge over concepts. In this sense, conceptual knowledge is described as logical statements expressed in the form of Description Logic (DL) languages. 
    Embedding the logical structure expressed in the logical statements in DL is challenging as the logical structures in the statements are supposed to preserved in the embedding space, which is beyond the similarity that is preservation in plain vector embeddings.  Furthermore, the logical structure in relational data can also be used to describe relational constraints over the output of machine learning models. Embedding these constraints is important for many multi-label machine learning models as it guarantees predictive coherence of predictions. 
    
    \item \textbf{High-order structure.} Relational data may describe high-order relational facts, where each fact can describe a complex multi-fold relationship between multiple entities and/or relations. 
    For example, in hyper-relational or n-ary relational knowledge graphs, each triple fact is contextualized by a set of qualifiers with each qualifier being an relation-entity pair. Another example is the nested relational knowledge graphs, in which a fact can describe a relationship over other facts, a.k.a., facts over facts. Most knowledge graph embedding approaches are designed for triple-based knowledge graphs and fail to capture high-order knowledge. Modeling these high-order knowledge is challenging as logical properties may exist not only in the triple level but also in the high-order structure level. For example, in a query of hyper-relational facts, attaching qualifiers to a fact may only narrow down the answer set but never enlarge it, a.k.a., qualifier monotonicity. In nested relational knowledge graphs, nested relational facts may express some logical rules. 
    
\end{itemize}

        
Going beyond plain vector embeddings, \emph{geometric relational embeddings} replace the plain vector representations with more advanced geometric objects, such as convex regions \cite{DBLP:conf/ijcai/KulmanovLYH19,ren2019query2box}, probabilistic density functions \cite{wang2022dirie,ren2020beta}, geometric elements of non-Euclidean manifolds \cite{chami2020low}, and their combinations \cite{suzuki2019hyperbolic}. 
Different from plain vector embeddings, geometric relational embeddings provide a rich geometric inductive bias for modeling various discrete properties of relational data while being still able to capture similarity. 
For example, embedding ontological concepts as convex regions allows for modeling not only similarity of concepts but also set-based and logical operators over concepts, such as set inclusion, set intersection \cite{xiong2022faithful} and logical negation \cite{zhang2021cone}. 
This is very useful for concept embeddings in ontologies. 
On the other hand, representing data on non-Euclidean Riemannian manifolds allows for capturing complex structural patterns, such as representing hierarchies in hyperbolic space \cite{chami2020low} and cycles in spherical space. 

Geometric relational embeddings have been successfully applied in many relational reasoning tasks including but not limited to knowledge graph (KG) completion \cite{abboud2020boxe},  ontology/hierarchy reasoning \cite{vilnis2018probabilistic}, hierarchical multi-label classification \cite{patel2021modeling}, and logical query answering \cite{ren2020beta}.
However, different downstream applications require varying capabilities from underlying embeddings and, hence, appropriate choice requires a sufficiently precise understanding of embeddings' characteristics and task requirements.

\section{Research Questions}

We investigate the following research questions (\textbf{RQs}):

\textbf{RQ 1:} How to faithfully modeling complex graph structural patterns such as hierarchies and cycles in graph data and what are the suitable geometric inductive biases for these structural patterns? how to develop the corresponding graph neural network components that are suitable for modeling these structural patterns (Chapter \ref{sec:qgcn}).

\textbf{RQ 2:} For multi-relational or knowledge graphs that simultaneously exhibit both complex graph structural patterns (e.g., hierarchies and cycles) and complex relational patterns (e.g., symmetry, anti-symmetry, inversion, and composition), how to faithfully modeling both of these patterns in a single embedding space? (Chapter \ref{sec:ultrae}).

\textbf{RQ 3:} For ontological data where facts are expressed as logical statements/axioms with Description Logic languages, how to faithfully represent these logical statements/axioms while preserving the underlying logical structure with embeddings? what are the geometric inductive biase for representing concepts?  (Chapter \ref{chap_ontological}).

\textbf{RQ 4:} How to represent relational constraints for machine learning models such that the model can produce outputs that are logically coherent to the relational constraints? (Chapter \ref{chap_logical}).

\textbf{RQ 5:} How to embed relational data with high-order relational structure such as hyper-relational knowledge graphs and nested relational knowledge graphs, in a way that the underlying logical properties such as logical patterns inherent in the relational data can be still modeled? (Chapter \ref{chap_hyper}).

\section{Thesis Contributions and Outline}

To address these limitations, this dissertation goes beyond vector embeddings and proposes various geometric embeddings that faithfully model various discreate properties of different types of relational data. Geopmetric embeddings, instead of mapping relational objects as plain vectors in Euclidean space, encode relational objects as geometric elements (e.g., balls, boxes, and convex cones) or as elements in non-Euclidean manifolds (e.g., hyperbolic or spherical space). The primal advantage of geometric embeddings is the faithful encoding of discreate properties of relational data.  In this thesis, we addressed several encoding issues of vector embeddings over relational data with geometric embeddings (Cf. Fig. \ref{fig:bigmap-thesis}). Our contributions and the remaining content of the thesis are summarised as below:

\begin{description}[leftmargin=1.5cm]

\item In Chapter \ref{chap_foundation}, we introduce some necessary preliminaries, foundational concepts, and related works that are relevant to this dissertation.  

\item In Chapter \ref{chap_structral}, we introduce pseudo-Riemannian manifold embeddings. This chapter makes the following two contributions.
\begin{itemize}
    \item We present a principled framework, pseudo-Riemannian GCN, which generalizes GCNs into pseudo-Riemannian manifolds with indefinite metrics, providing more flexible inductive biases to accommodate complex graphs with mixed topologies. We also defined neural network operations in pseudo-Riemannian manifolds with novel geodesic tools, to stimulate the applications of pseudo-Riemannian geometry in geometric deep learning. Extensive evaluations on three standard tasks demonstrate that our model outperforms baselines that operate in Riemannian manifolds. 
    \item We proposes UltraE, an ultrahyperbolic KG embedding method in a pseudo-Riemannian manifold that interleaves hyperbolic and spherical geometries, allowing for simultaneously modeling multiple hierarchical and non-hierarchical structures in KGs. We derive a relational embedding by exploiting the pseudo-orthogonal transformation, which is decomposed into various geometric operators including circular rotations/reflections and hyperbolic rotations, allowing for inferring complex relational patterns in KGs. On three standard KG datasets, UltraE outperforms many previous Euclidean and non-Euclidean counterparts, especially in  low-dimensional setting.
\end{itemize}

\begin{figure}[t!]
\begin{center}
\centerline{\includegraphics[width=\columnwidth]{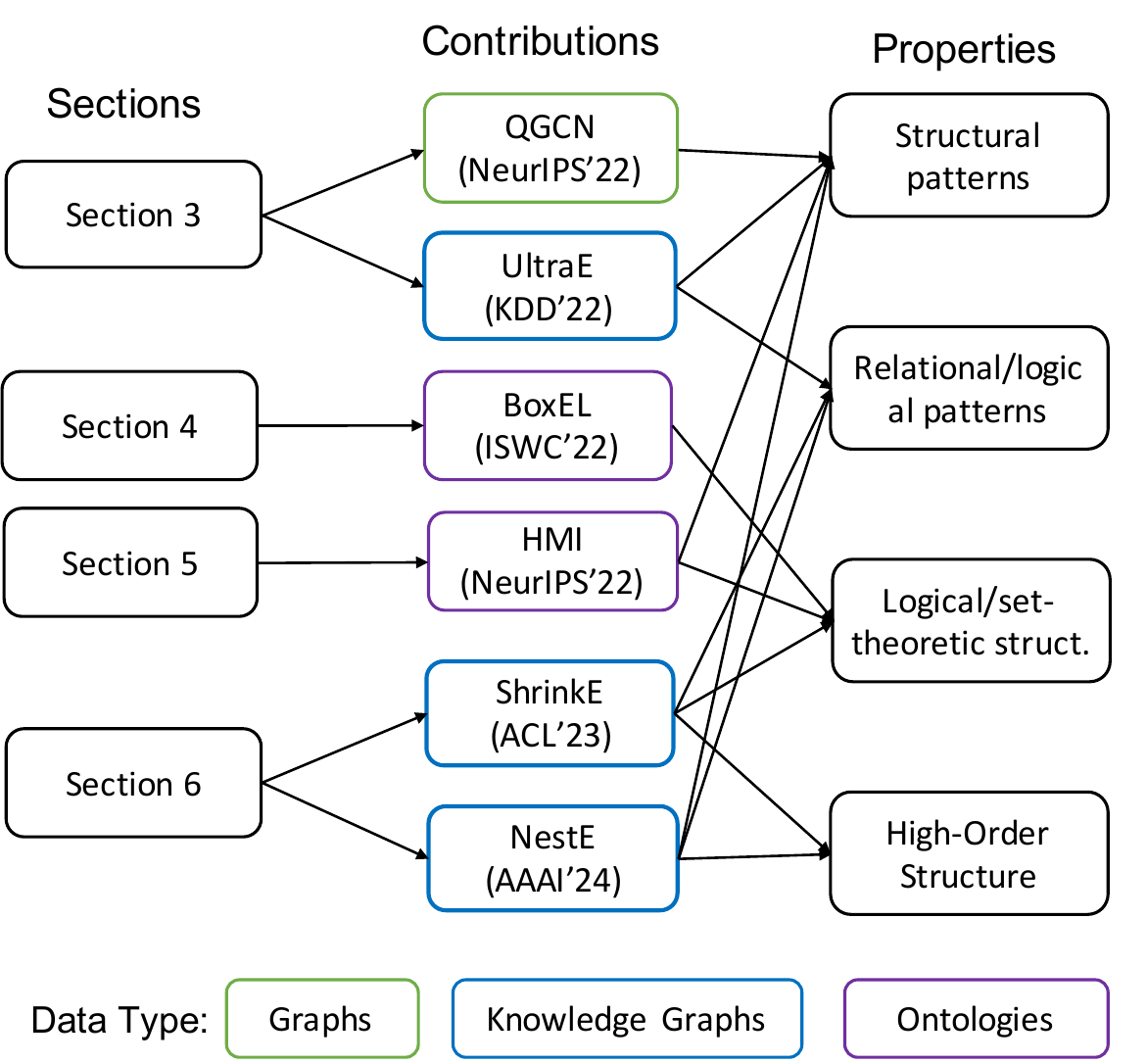}}
\caption{An overview of the proposed methodologies, the corresponding properties and the relational data types.}
\label{fig:bigmap-thesis}
\end{center}
\vskip -0.3in
\end{figure}

\item In Chapter \ref{chap_ontological}, we focus on ontology embeddings and make the following contribution. 
\begin{itemize}
    \item We propose BoxEL, a geometric knowledge base embedding method that explicitly models the logical structure expressed by the theories of $\mathcal{EL}^{++}$. Different from the standard KGEs that simply ignore the analytical guarantees, BoxEL provides \textit{soundness} guarantee for the underlying logical structure by incorporating background knowledge into machine learning tasks, offering a more reliable and logic-preserved fashion for knowledge base reasoning. 
    The empirical results further demonstrate that BoxEL outperforms previous KGEs and $\mathcal{EL}^{++}$ embedding approaches on subsumption reasoning over three ontologies and predicting protein-protein interactions in a real-world biomedical knowledge base.
\end{itemize}

\item In Chapter \ref{chap_logical}, we exploit ontology embeddings on improving machine learning models and make the following contribution. 
\begin{itemize}
    \item We focus on a structured multi-label prediction task whose output is supposed to respect the implication and exclusion constraints. We show that such a problem can be formulated in a hyperbolic Poincaré ball space whose linear decision boundaries (Poincaré hyperplanes) can be interpreted as convex regions. The implication and exclusion constraints are geometrically interpreted as insideness and disjointedness, respectively. Experiments on $12$ datasets show significant improvements in mean average precision and lower constraint violations, even with an order of magnitude fewer dimensions than baselines. 
\end{itemize}

\item In Chapter \ref{chap_hyper}, we introduce two geometric embeddings for high-order relational knowledge graphs.
\begin{itemize}
    \item We present \emph{ShrinkE}, a geometric hyper-relational KG embedding method aiming to explicitly model these patterns. ShrinkE models the primal triple as a spatial-functional transformation from the head into a relation-specific box. Each qualifier ``shrinks'' the box to narrow down the possible answer set and, thus, realizes qualifier monotonicity. The spatial relationships between the qualifier boxes allow for modeling core inference patterns of qualifiers such as implication and mutual exclusion. Experimental results demonstrate ShrinkE's superiority on three benchmarks of hyper-relational KGs. 
    
    \item We propose FactE, a family of hypercomplex embeddings capable of embedding both atomic and nested factual knowledge. This framework effectively captures essential logical patterns that emerge from nested facts. Empirical evaluation demonstrates the substantial performance enhancements achieved by FactE compared to existing baseline methods. Additionally, our generalized hypercomplex embedding framework unifies previous algebraic (e.g., quaternionic) and geometric (e.g., hyperbolic) embedding methods, offering versatility in embedding diverse relation types. 
\end{itemize}

\item In Chapter~\ref{chap_conclusion}, we conclude the whole dissertation, summarize the limitation, and foresee some future works.

\end{description}

\subsection{Publications}
This dissertation contains material in papers that were published already in conference proceedings as listed below. 
Unless explicitly indicated, these papers were significantly contributed by the author in terms of idea generation, experimentation, paper writing and so on. 

Published research papers that directly contribute to this dissertation:

\begin{itemize}

\item \textbf{B. Xiong,} M. Nayyeri, L. Luo, Z. Wang, S. Pan, S. Staab. \textit{NestE: Modeling Nested-Relational Structure for Knowledge Graph Reasoning.} The 38th Annual Conference on Artificial Intelligence. \textbf{AAAI 2024}. \cite{xiong_neste_aaai24}

\item \textbf{B. Xiong,} M. Nayyeri, S. Pan, S. Staab. \textit{Shrinking Embeddings for Hyper-Relational Knowledge Graphs.}. In Proceedings of the 60th Annual Meeting of the Association for Computational Linguistics. \textbf{ACL 2023}.\cite{xiong2023shrinking}

\item \textbf{B. Xiong,} M. Cochez, M. Nayyeri, S. Staab. \textit{Hyperbolic Embedding Inference for Structured Multilabel Prediction.} Advances in Neural Information Processing Systems. \textbf{NeurIPS 2022}. \cite{DBLP:conf/nips/XiongCNS22}

\item \textbf{B. Xiong$^\star$,} S. Zhu$^\star$, N. Potyka, S. Pan, C. Zhou, S. Staab. \textit{Pseudo-Riemannian Graph Convolutional Networks}. Advances in Neural Information Processing Systems. \textbf{NeurIPS 2022}.\cite{DBLP:conf/nips/XiongZPP0S22}

\item \textbf{B. Xiong,} N. Potyka, T. Tran, M. Nayyeri, S. Staab. \textit{Faithful Embeddings for EL++ Knowledge Bases}. International Semantic Web Conference. \textbf{ISWC 2022}. \cite{xiong2022faithful}

\item \textbf{B. Xiong}, S. Zhu, M. Nayyeri, C. Xu, S. Pan, C. Zhou, S. Staab. \textit{Ultrahyperbolic Knowledge Graph Embeddings}. In Proceedings of The 28th ACM SIGKDD Conference on Knowledge Discovery and Data Mining. \textbf{SIGKDD 2022}. \cite{DBLP:conf/kdd/XiongZNXP0S22}

\end{itemize}

Unpublished survey/position paper that contributes to the foundation chapter of the dissertation:

\begin{itemize}
    \item \textbf{B. Xiong,} M. Nayyeri, M. Jin, Y. He, M. Cochez, S. Pan, S. Staab. \textit{Geometric Relational Embeddings: A Survey}. Technique Report, 2023. \cite{DBLP:journals/corr/abs-2304-11949}
    
\end{itemize}

Published tutorial proposals related to this dissertation:

\begin{itemize}
    \item  \textbf{Bo Xiong}, Mojtaba Nayyeri, Daniel Daza, Michael Cochez. \textit{Reasoning beyond Triples: Recent Advances in Knowledge Graph Embeddings.} Tutorial at The 32nd ACM International Conference on Information and Knowledge Management. \textbf{CIKM 2023}. \cite{DBLP:conf/cikm/XiongNDC23}
    
    \item Min Zhou, Menglin Yang, \textbf{Bo Xiong}, Hui Xiong, Irwin King. \textit{Hyperbolic Graph Neural Networks: A Tutorial on Methods and Applications.} Tutorial at The 28th ACM SIGKDD Conference on Knowledge Discovery and Data Mining. \textbf{SIGKDD 2023}. \cite{DBLP:conf/kdd/00060XXK23}
\end{itemize}

Published papers that were co-authored by the author but not included in this dissertation: 

\begin{itemize}

\item E. Lee, \textbf{B. Xiong}, C. Yang and J. Ho. \textit{HypMix: Hyperbolic Representation Learning for Graphs with Mixed Hierarchical and Non-hierarchical Structures.} The 33rd ACM International
Conference on Information and Knowledge Management. \textbf{CIKM 2024}. \cite{HypMixEric}

\item Y. He, D. Hernandaza, M. Nayyeri, \textbf{B. Xiong}, Y. Zhu, E. Kharlamov, S. Staab. Generating SROI Ontologies via Knowledge Graph Query Embedding Learning. Proceedings of the 27th European Conference on Artificial Intelligence. \textbf{ECAI 2024}. \cite{PConE2024Yunjie}

\item H. Lv, Z. Chen, Y. Yang, S. Pan, \textbf{B. Xiong}, Y. Tan, C. Yang. \textit{Enhancing Semantic and Structure Modeling of Diseases for Diagnosis Prediction.} Proceedings of the
American Medical Informatics Association Informatics Summit. \textbf{AMIA 2024.} \cite{Hang2024AMIA}

\item M. Yang, A. Feng, \textbf{B. Xiong}, J. Liu, I. King, R. Ying. \textit{Enhancing LLM Complex Reasoning Capability through Hyperbolic Geometry.} Workshop on LLMs and Cogni-
tion. International Conference on Machine Learning. \textbf{LLM \& Cognition@ICML, 2024.} \cite{ICMLWorkshop2024Menglin}

\item Z. Ding, H. Cai, J. Wu, Y. Ma, R. Liao, \textbf{B. Xiong}, V. Tres. \textit{Zero-Shot Relational Learning on Temporal Knowledge Graphs with Large Language Models.} Proceedings of the 2024 Conference of the North American Chapter of the Association for Computational Linguistics. \textbf{NAACL 2024}. \cite{DBLP:conf/naacl/DingCWMLXT24}

\item D. Zhou, H. Yang, \textbf{B. Xiong}, Y. Ma, Evgeny Kharlamov. \textit{Alleviating Over-Smoothing via Aggregation over Compact Manifolds. The Pacific-Asia Conference on Knowledge Discovery and Data Mining.} \textbf{PAKDD 2024.} \cite{DBLP:conf/pakdd/ZhouYXMK24}

\item Y. Jia, Y. Song, \textbf{B. Xiong}, J. Cheng, W. Zhang, S. X. Yang, S. Kwong. Hierarchical Perception-Improving for Decentralized Multi-Robot Motion Planning in Complex Scenarios. IEEE Transactions on Intelligent Transportation Systems (TITS), 2024 \cite{DBLP:journals/tits/JiaSXCZYK24}

\item Y. Tan, H. Lv, Z. Zhou, W. Guo, \textbf{B. Xiong}, W. Liu, C. Chen, S. Wang, and C. Yang. \textit{Logical Relation Modeling and Mining in Hyperbolic Space for Recommendation.} In Proceedings of The 40th IEEE International Conference on Data Engineering. \textbf{ICDE 2024}.\cite{Yanchao2024LogicR}
 
\item M. Nayyeri, \textbf{B. Xiong,} M.Mohammadi, M. Mahfuja Akter, M. Mohtashim Alam, J. Lehmann, S. Staab. \textit{Knowledge Graph Embeddings using Neural Ito Process: From Multiple Walks to Stochastic Trajectories.}. In Findings of the 60th Annual Meeting of the Association for Computational Linguistics. \textbf{Findings of ACL 2023}.\cite{DBLP:conf/acl/NayyeriXMAA0S23}

\item J. Lu, J. Shen,  \textbf{B. Xiong}, W. Ma, S. Staab, and C. Yang. \textit{HiPrompt: Few-Shot Biomedical Knowledge Fusion via Hierarchy-Oriented Prompting}. In Proceedings of The 46th International ACM SIGIR Conference on Research and Development in Information Retrieval. \textbf{SIGIR 2023}. \cite{DBLP:conf/sigir/jiayinglu}

\item Y. Zhu, N. Potyka, \textbf{B. Xiong}, K. Tran, M. Nayyeri, S. Staab and E. Kharlamov. International Semantic Web Conference. Poster \& Demo track of the International Semantic Web Conference. \textbf{ISWC 2023}. \cite{DBLP:conf/semweb/ZhuPXTNSK23}

\item Y. He, M. Nayyeri, \textbf{B. Xiong}, Y. Zhu, E. Kharlamov and S. Staab. Can Pattern Learning Enhance Complex Logical Query Answering? Poster \& Demo track opf the International Semantic Web Conference. \textbf{ISWC 2023}.  \cite{DBLP:conf/semweb/HeNXZKS23}

\item C. Xu, F. Su, \textbf{B. Xiong}, J. Lehmann. \textit{Time-aware Entity Alignment using Temporal Relational Attention}. In Proceedings of the ACM Web Conference. \textbf{WWW 2022}.\cite{DBLP:conf/www/XuSX022}

\item B. Xiong, P. Bao, Y. Wu. Learning semantic and relationship joint embedding for author name disambiguation. Neural Computing and Applications 33 (6), 1987-1998. 2021. \cite{DBLP:journals/nca/XiongBW21}

\end{itemize}

Preprint papers that were co-authored by the author but not included in this dissertation.

\begin{itemize}

    \item R. Chen, W. Jiang, C. Qin, I.S. Rawal, C. Tan, D. Choi, \textbf{B. Xiong}, B. Ai. LLM-Based Multi-Hop Question Answering with Knowledge Graph Integration in Evolving Environments. arXiv preprint arXiv:2408.15903. \cite{chen2024llm}

    \item Y. Zhu, N. Potyka, M. Nayyeri, \textbf{B. Xiong}, Y. He, E. Kharlamov, S. Staab. Predictive Multiplicity of Knowledge Graph Embeddings in Link Prediction. arXiv preprint arXiv:2408.08226. \cite{zhu2024predictive}

    \item Y. Zhu, N. Potyka, J. Pan, \textbf{B. Xiong}, Y. He, E. Kharlamov, S. Staab. Conformalized Answer Set Prediction for Knowledge Graph Embedding. arXiv preprint arXiv:2408.08248. \cite{zhu2024conformalized} 

    \item Y. Zhu, N. Potyka, \textbf{B. Xiong}, T.K. Tran, M. Nayyeri, E. Kharlamov, S. Staab. Approximating Probabilistic Inference in Statistical EL with Knowledge Graph Embeddings. arXiv preprint arXiv:2407.11821. \cite{zhu2024approximating}

    \item L. Luo, J. Ju, \textbf{B. Xiong}, Y, Li, G. Haffari, and S. Pan. \textit{ChatRule: Mining Logical Rules with Large Language Models for Knowledge Graph Reasoning.} arXiv preprint arXiv:2309.01538, 2023. \cite{DBLP:journals/corr/abs-2309-01538}
        
    \item M. Jin, G. Shi, Y. Li, Q. Wen, \textbf{B. Xiong}, T. Zhou, S. Pan. How Expressive are Spectral-Temporal Graph Neural Networks for Time Series Forecasting? arXiv preprint arXiv:2305.06587. \cite{jin2023expressive}

    \item J. Wang, B. Wang, M. Qiu, S. Pan, \textbf{B. Xiong}, H. Liu, L. Luo, T. Liu, Y. Hu, B. Yin. A survey on temporal knowledge graph completion: Taxonomy, progress, and prospects. arXiv preprint arXiv:2308.02457. \cite{wang2023survey}

\end{itemize}

\cleardoublepage

\chapter{Foundations}
\label{chap_foundation}

In this chapter, we introduce necessary preliminaries, foundational concepts, and related works that are relevant to this dissertation.

\section{Relational Data}

In this dissertation, our primary focus is \emph{relational data} that describes diverse relationships between entities and/or concepts in a graph-structured format. We choose to model relational data as a graph because it offers greater flexibility for integrating new sources of data \cite{hogan2021knowledge}. This is in contrast to the standard relational data model where a schema must be pre-defined and adhered to at each step.
 Graph-structured data models have found extensive use in organizing various real-world relational data, such as information networks, knowledge graphs, and biomedical ontologies.

\subsection{Graph-Structured Data Models}

There are several graph-structured data models, such as directed edge-labeled graphs, heterogeneous graphs, and property graphs, which we introduce as follows.

\textbf{Directed edge-labeled graphs}, also called multi-relational graphs \cite{DBLP:conf/icml/NickelTK11}, are one of the graph-structured data models. A directed edge-labeled graph is defined by a set of nodes representing entities or concepts, like \emph{New York} and \emph{USA}, and a set of directed labeled edges connecting these nodes, with each labeled edge representing a relationship between these connected nodes, such as (\emph{New York}, \emph{CityOf}, \emph{USA}). Formally, a directed edge-labeled graph is defined as:

\begin{definition}[Directed edge-labeled graph \cite{hogan2021knowledge}]
A directed edge-labeled graph is a tuple $G=(V, E, L)$, where $V \subseteq \Con$ is a set of nodes, $L \subseteq \Con$ is a set of edge labels, and $E \subseteq V \times L \times V$ is a set of edges, where $\Con$ denotes a countably infinite set of constants.
\end{definition}

Note that this definition is very flexible, as we do not assume that $V$ and $L$ are disjoint. In principle, a node can also serve as an edge label, and nodes and edge labels can be present without any associated edge. Moreover, although the edge is directional, bidirectional edges can be simply represented with two edges with inverse directions.

One limitation of this definition is that it does not distinguish between nodes and the type of nodes but rather expresses the type as a relation, e.g., (\emph{New York}, \emph{Type}, \emph{City}).

\textbf{Heterogeneous graphs} or heterogeneous information networks \cite{DBLP:conf/cikm/HusseinYC18} represent relational data as a set of nodes and a set of edges, with each node and edge associated with a type or label. A heterogeneous graph is formally defined as follows.

\begin{definition}[Heterogeneous graph \cite{hogan2021knowledge}]
A heterogeneous graph is a tuple $G=(V, E, L, l)$, where $V \subseteq \Con$ is a set of nodes, $L \subseteq \Con$ is a set of edge/node labels, $E \subseteq V \times L \times V$ is a set of edges, and $l: V \rightarrow L$ maps each node to a label, where $\Con$ denotes a countably infinite set of constants.
\end{definition}

In contrast to a directed edge-labeled graph, a heterogeneous graph encodes the type of a node as part of the node itself, rather than modeling types of nodes with a \emph{Type} relation. Hence, one of the main advantages of a heterogeneous graph is that it allows for explicit distinction between nodes and the types of nodes. This is particularly useful when the type of each node is unique. However, a heterogeneous graph cannot express multiple types for a single node (i.e., many-to-one mapping).

\textbf{Property graphs} constitute a graph-structured data model that provides additional flexibility when modeling complex relations. Notably, a property graph allows for annotating more intricate details to each edge, such as the degree and major obtained by a person from a university. This is represented by a set of property–value pairs associated with edges. Unlike both directed edge-labeled graphs and heterogeneous graphs, annotating edges with additional properties is not straightforward. A property graph is defined as follows.

\begin{definition}[Property graph \cite{hogan2021knowledge}]
A property graph is a tuple $G=(V, E, L, P, U, e, l, p)$, where $V \subseteq \Con$ is a set of node ids, $E \subseteq \Con$ is a set of edge ids, $L \subseteq \Con$ is a set of labels, $P \subseteq \Con$ is a set of properties, $U \subseteq \Con$ is a set of values, $e: E \rightarrow V \times V$ maps an edge id to a pair of node ids, $l: V \cup E \rightarrow 2^L$ maps a node or edge id to a set of labels, and $p: V \cup E \rightarrow 2^{P \times U}$ maps a node or edge id to a set of property-value pairs.
\end{definition}

In contrast to directed edge-labeled graphs and heterogeneous graphs, a property graph allows a node or edge to have several values for a given property. Property graphs can be converted to/from directed edge-labeled graphs, and this process is called reification.

In summary, each of these three models has its advantages and disadvantages. Directed edge-labeled graphs offer a simple model, while property graphs provide a more flexible choice. The selection of a model typically depends on practical factors such as available implementations for different models, etc. For a detailed discussion, we recommend readers refer to \cite{hogan2021knowledge}.

\subsection{Graphs, Knowledge Graphs, and Ontologies}

We now introduce some popular instances of relational data that have been considered in the machine learning community, including (homogeneous) graphs, knowledge graphs, and ontologies.

\textbf{Homogeneous graphs} can be viewed as a special case of directed edge-labeled graphs or heterogeneous graphs in which there is only one type of edges and only one type of nodes. Specifically, a homogeneous graph is defined as follows.

\begin{definition}[Homogeneous graph]
A homogeneous graph is a tuple $G=(V,E)$, where $V$ is a set of nodes, and $E \subseteq V \times V$ is a set of edges, with each edge connecting two nodes.
\end{definition}

A homogeneous graph is an \emph{undirected} graph if all edges are bidirectional, meaning that if $(V_i, V_j) \in E$ holds, then $(V_j, V_i) \in E$ also holds. For example, in co-author networks, the co-authorship relationship is an \emph{undirected} edge. Otherwise, the graph is called a \emph{directed} graph. For example, in citation networks, the citation relationship is a directed edge.

Note that a homogeneous graph is the minimal model of graph-structured data. To allow for multiple node types, homogeneous graphs can be extended to single-relational graphs, defined as,

\begin{definition}[Single-relational graph]
    A single-relational graph is a tuple $G=(V, E, L, l)$, where $V \subseteq \Con$ is a set of nodes, $L \subseteq \Con$ is a set of node labels, $E \subseteq V \times V$ is a set of edges, and $l: V \rightarrow L$ maps each node to a label, where $\Con$ denotes a countably infinite set of constants.
\end{definition}

\textbf{Node classification.} Single-relational graphs are useful for applications that involve only one type of edges but multiple types of nodes, such as node classification. Given a single-relational graph $G=(V, E, L, l)$ and the known labeling of nodes, \emph{node classification} aims to classify the missing labeling of nodes based on the node features and the graph's structure.

\begin{equation}
f_{nc}: V \rightarrow L
\end{equation}

For example, in a social network, node classification could involve predicting the communities of users based on their connections and shared content.

\textbf{Link prediction} is a task where the objective is to predict the likelihood of the existence of an edge (link) between two nodes in a graph or the specific edge label. Namely,

\begin{equation}
f_{lp}: V \times {0,1} \quad \text{or} \quad f_{lp}: V \times V \rightarrow L
\end{equation}

Link prediction is often used to predict missing or future connections in a graph. For instance, in a citation network, link prediction could involve predicting potential future collaborations between authors based on their previous co-authorships.

\textbf{Graph classification} involves assigning a label or category to an entire graph. The goal is to learn a model that can distinguish between different types or classes of graphs. Given a graph and a set of graph labels $L_G$, graph classification is defined as

\begin{equation}
f_{gc}: G \rightarrow L_G
\end{equation}

For example, in chemical informatics, graph classification might be used to predict whether a molecular graph represents a toxic or non-toxic compound based on its structure.

\textbf{Graph reconstruction} aims to reconstruct a graph by preserving the pair-wise distance of nodes. One straightforward method is to minimize the graph distortion  \cite{bachmann2020constant, Ganea2018} given by,
\begin{equation}
    \frac{1}{|V|^{2}} \sum_{u,v}\left(\left(\frac{d\left(\mathbf{u},\mathbf{v}\right)}{d_{G}(u,v)}\right)^{2}-1\right)^{2},
\end{equation}
where $d\left(\mathbf{u},\mathbf{v}\right)$ is the distance function in the embedding space and $d_{G}(u,v)$ is the graph distance (the length of shortest path) between node $u$ and $v$, $|V|$ is the number of nodes in graph.
The objective function is to preserve all pairwise graph distances. 
Motivated by the fact that most of graphs are partially observable, we can also minimize an alternative loss function \cite{law2020ultrahyperbolic, nickel2017poincare} that preserves local graph distance, given by,
\begin{equation}
    \mathcal{L}(\Theta)=\sum_{(u, v) \in \mathcal{D}} \log \frac{e^{-d(\mathbf{u},\mathbf{v})}}{\sum_{\mathbf{v}^{\prime} \in \mathcal{E}(u)} \exp^{-d\left(\mathbf{u},\mathbf{v}^{\prime}\right)}},
\end{equation}
where $\mathcal{D}$ is the connected relations in the graph, $\mathcal{E}(u)=\{v|(u,v) \notin D \cup {u}\}$ is the set of negative examples for node $u$, $d(\mathbf{u},\mathbf{v})$ is the distance function in the embedding space. 
For evaluation, we apply the mean average precision (mAP) to evaluate the graph reconstruction task. mAP is a local metric that measures the average proportion of the nearest points of a node which are actually its neighbors in the original graph. The mAP is defined as Eq.~(\ref{eq:map}).
\begin{equation}\label{eq:map}
    \operatorname{mAP}(f)=\frac{1}{|V|} \sum_{u \in V} \frac{1}{\operatorname{deg}(u)} \sum_{i=1}^{\left|\mathcal{N}_{u}\right|} \frac{\left|\mathcal{N}_{u} \cap R_{u, v_{i}}\right|}{\left|R_{u, v_{i}}\right|},
\end{equation}
where $f$ is the embedding function, $|V|$ is the number of nodes in graph, $\operatorname{deg}(u)$ is the degree of node $u$, $\mathcal{N}_u$ is the one-hop neighborhoods in the graph, $\mathbb{R}_{u,v}$ denotes whether two nodes $u$ and $v$ are connected.

\textbf{Knowledge graphs} can be viewed as an instance of directed edge-labeled graphs or  multi-relational graphs in which nodes represent entities while edges represent relations between those entities. Like directed edge-labeled graphs, in the classical form of knowledge graphs, there is also no explicit distinction between entities and concepts. 

\begin{definition}[Knowledge graph]
    A knowledge graph is a tuple $G=(V, R, E)$, where $V$ is a set of entities, $R$ is a set of relations, and $E \subseteq V \times R \times V$ is a set of triples each describing a relationship between two entities. 
\end{definition}

A knowledge graph serves as a structured representation of knowledge in a graph-based format. In this formalization, the components of the knowledge graph are:

\begin{itemize}
    \item \textbf{Entities ($V$):} Entities represent the fundamental building blocks or objects within the knowledge graph. These could be people, places, events, or any distinguishable item of interest.

    \item \textbf{Relations ($R$):} Relations define the connections or associations between entities in the knowledge graph. These relationships capture the nuanced connections that exist in the real world, expressing how entities are related to one another.

    \item \textbf{Triples ($E$):} The set of triples $E \subseteq V \times R \times V$ describe relationships between entities. Each triple consists of a subject entity, a relation, and an object entity, collectively forming a statement about the knowledge contained in the graph.

\end{itemize}

Knowledge graphs provide a structured and semantically rich representation of information, allowing for the modeling of complex relationships and dependencies. It serves as a powerful tool for organizing, linking, and querying data in a way that reflects the inherent connections present in the real world.

\textbf{Hyper-relational knowledge graph. } 
In hyper-relational knowledge graph, each fact is a hyper-relational fact represented in the form of a primal triple coupled with a set of qualifiers. Namely,

\begin{definition}[Hyper-relational knowledge graph]
    A hyper-relational knowledge graph is a tuple $G=(V, R, E)$, where $V$ is a set of entities, $R$ is a set of relations, and $E$ is a set of  hyper-relational fact. 
\end{definition}

\begin{definition}[Hyper-relational fact]
Let $\mathcal{E}$ and $\mathcal{R}$ denote the sets of entities and relations, respectively. A hyper-relational fact $\mathcal{F}$ is a tuple $(\mathcal{T}, \mathcal{Q})$, where $\mathcal{T}=(h, r, t),\ h, t \in \mathcal{E}, r \in \mathcal{R} $ is a primal triple and $\mathcal{Q}=\left\{\left(k_{i}: v_{i}\right)\right\}_{i=1}^{m} \ k_{i} \in \mathcal{R}, v_{i} \in \mathcal{E}$ is a set of qualifiers.
We call the number of involved entities in $\mathcal{F}$, i.e., $(m+2)$, the arity of the fact. 
\end{definition}
A hyper-relational fact reduces to a triple/binary fact when $m=0$. When $m>0$, each qualifier can be viewed as an auxiliary description that contextualizes or specializes the semantics of the primal triple. 
In typical open-world settings, facts with the same primal triple might have different numbers of qualifiers. 
To characterize this property, we introduce the concepts of partial fact and qualifier monotonicity in hyper-relational knowledge graphs.

\begin{definition}[Partial fact \cite{DBLP:conf/acl/GuanJGWC20}]
Given two facts $\mathcal{F}_1=\left(\mathcal{T},\mathcal{Q}_1\right)$ and $\mathcal{F}_2=\left(\mathcal{T},\mathcal{Q}_2\right)$ that share the same primal triple. We call $\mathcal{F}_1$ a partial fact of $\mathcal{F}_2$ iff $\mathcal{Q}_1 \subseteq \mathcal{Q}_2$. 
\end{definition}

In this work, we follow the monotonicity assumption by restricting the model to respect the monotonicity property.\footnote{
Some kinds of qualifiers may represent semantically opaque contexts. For instance, ((\emph{Crimea}, \emph{belongs\_to}, \emph{Russia}), \{(\emph{said\_by}, \emph{Putin})\}) does not imply the primary triple and should therefore be excluded. } For this purpose, we consider the monotonicity of query and inference.

\begin{definition}[Qualifier monotonicity]
Let $\QA(\cdot)$ denote a query answering model taking a query and a knowledge graph as input and outputting the set of answer entities. Given any pair of queries $q_1=\left(\left(h,r,x?\right), \mathcal{Q}_1\right)$ and $q_2=\left(\left(h,r,x?\right), \mathcal{Q}_2\right)$ that share the same primal triple and $\mathcal{Q}_1 \subseteq \mathcal{Q}_2$, qualifier monotonicity is given iff,
\begin{equation}
\QA(q_2; \operatorname{KG}) \subseteq \QA(q_1;\operatorname{KG}).
\end{equation}
\end{definition}
Qualifier monotonicity implies that attaching any qualifiers to a query does not enlarge the answer set of the possible tail entities, and inversely, removing the qualifiers from a query can only return more possible tail entities. This implies that if a fact is true, then all its partial facts must also be true (a.k.a. weakening of inference rule), i.e.,
\begin{equation}
    \left(\mathcal{T},\mathcal{Q}_1\right)  \wedge (\mathcal{Q}_2 \subseteq \mathcal{Q}_1) \rightarrow \left(\mathcal{T},\mathcal{Q}_2\right).
    \label{eq:momonotonicity}
\end{equation}

\textbf{Nested relational knowledge graphs.}
Given a knowledge graph $G$, which can be viewed an atomic factual knowledge graph, and each of the triple $(h,r,t)\in\sT$ is referred to as an atomic triple. 
The nested triple and nested relational knowledge graph are defined as follows.

\begin{definition}[Nested triple]
Given an atomic factual knowledge graph $G=(\sV,\sR,\sT)$, a set of nested triples is defined by $\sTT=\{\langle T_i, \rhat, T_j\rangle:T_i,T_j\in\sT, \rhat\in\Rhat \}$, where $\sT$ is the set of atomic triples and $\Rhat$ is the set of nested relation names.
\end{definition}


\begin{definition}[Nested relational knowledge graph]
Given a knowledge graph $G=(\sV,\sR,\sT)$, a set of nested relation names $\Rhat$, and a set of nested triples $\sTT$ defined on $G$ and $\Rhat$,
a nested relational knowledge graph is defined as $\Ghat=(\sV, \sR, \sT, \Rhat, \sTT)$.
\end{definition}

We can now define triple prediction and conditional link prediction \cite{DBLP:conf/aaai/Chanyoung} as follows.

\begin{definition}[Triple prediction]
Given a nested relational knowledge graph $\Ghat=(\sV, \sR, $ $ \sT, \Rhat, \sTT)$, the triple prediction problem involves answering a query $\langle T_i, \rhat, ?t\rangle$ or $\langle h? , \rhat, T_j\rangle$ with $T_i, T_j \in \sT$ and $\rhat \in \Rhat$, where the variable $?h$ or $?t$ needs to be bounded to an atomic triple within $\Ghat$.
\end{definition}

\begin{definition}[Conditional link prediction]
Given a nested relational knowledge graph $\Ghat=(\sV, \sR, \sT, \Rhat, \sTT)$, let $T_i = (h_i,r_i,t_i)$ and $T_j = (h_j,r_j,t_j)$. The conditional link prediction problem involves queries $\langle T_i, \rhat, (h_j,$ $r_j,?) \rangle$, $\langle T_i, \rhat, (?,r_j,t_j) \rangle$, $\langle (h_i,r_i,?), \rhat, T_j \rangle$, or $\langle (?,r_i,t_i), \rhat, T_j \rangle$, where the variables need to be bound to entities within $\Ghat$.
\end{definition}

\textbf{Description Logics.}
Description Logics (DLs) are a family of formal knowledge representation languages used to represent and reason about the semantics of information in a structured and logical manner. They are a subset of first-order logic and are designed to express and reason about concepts, relationships, and individuals in a domain. DLs provide a formal foundation for knowledge representation and are widely used in the field of artificial intelligence, particularly in areas such as ontology engineering and knowledge-based systems.

\begin{enumerate}
  \item \textbf{Concepts:} Represent abstract classes or sets of individuals. Concepts are used to classify and categorize entities in a domain.
  
  \item \textbf{Roles (Properties):} Describe relationships between individuals or between individuals and concepts. Roles can be used to represent attributes, roles, or associations between entities.
  
  \item \textbf{Individuals:} Represent specific instances or objects in the domain. Individuals are members of concepts and can be related to each other through roles.
  
  \item \textbf{Axioms:} Specify constraints and relationships between concepts and roles, providing the logical foundation for reasoning in the knowledge base.
\end{enumerate}

\begin{definition}[DL knowledge base]
    DL knowledge base $\mathrm{K}$ is defined as a tuple $(\mathrm{A}, \mathrm{T}, \mathrm{R})$, where $\mathrm{A}$ is the A-Box: a set of assertional axioms; $\mathrm{T}$ is the T-Box: a set of class (aka concept/terminological) axioms; and $\mathrm{R}$ is the R-Box: $a$ set of relation (aka property/role) axioms.
\end{definition}

A Description Logic knowledge base is a structured collection of information using Description Logics. It typically consists of:

\begin{enumerate}
  \item \textbf{TBox (Terminological Box):} Contains the terminology or concept hierarchy of the domain. It defines the relationships between concepts and includes axioms that express constraints on the concepts.
  
  \item \textbf{ABox (Assertional Box):} Contains assertions about individuals in the domain. It specifies which individuals belong to which concepts and how they are related through roles.
  
  \item \textbf{RBox (Role Box):} Contains information about roles and their properties, including relationships between roles.
\end{enumerate}

The combination of the TBox, ABox, and RBox provides a comprehensive representation of the knowledge about a domain in a formal and logical manner. Description Logics and their associated knowledge bases are commonly used in various applications, including semantic web technologies, ontology development, knowledge-based systems, and the representation of domain-specific knowledge in AI systems.

\begin{table}
\vspace{-0.4cm}
\centering
\caption{Syntax and semantic of $\mathcal{EL}^{++}$ (role inclusions and concrete domains are omitted).}
\label{tab:el_syntax_seman}
\resizebox{\linewidth}{!}{
\begin{tabular}{ccccccc}
\hline\noalign{\smallskip} 
 & Name & Syntax & Semantics \\
\hline\noalign{\smallskip}
\multirow{5}{*}{Constructors} & Top concept &  $\top$ & $\Delta^{\mathcal{I}}$ \\
& Bottom concept & $\bot$ & $\emptyset$ \\
& Nominal & $\{a\}$ & $\{a^{\mathcal{I}}\}$ \\
& Conjunction & $C\sqcap D$ & $C^{\mathcal{I}}\cap D^{\mathcal{I}}$ \\
& Existential restriction & $\exists r . C$ & \makecell{$\left\{x \in \Delta^{\mathcal{I}} \mid \exists y \in \Delta^{\mathcal{I}}\right.$ \\
$\left.(x, y) \in r^{\mathcal{I}} \wedge y \in C^{\mathcal{I}}\right\}$ } \\
\hline\noalign{\smallskip}
\multirow{2}{*}{ABox}
& Concept assertion & $C(a)$ & $a^{\mathcal{I}} \in C^{\mathcal{I}} $ \\
& Role assertion & $r(a,b)$ & $(a^{\mathcal{I}}, b^{\mathcal{I}}) \in r^{\mathcal{I}}$ \\
\hline\noalign{\smallskip}
\multirow{1}{*}{TBox}
& Concept inclusion & $C \sqsubseteq D$ & $C^{\mathcal{I}} \subseteq D^{\mathcal{I}}$  \\
\noalign{\smallskip}\hline
\end{tabular}
}
\end{table}

\textbf{Description logic $\mathcal{EL}^{++}$}. The DL $\mathcal{EL}^{++}$ underlies multiple biomedical KBs like GALEN \cite{rector1996galen} and  the Gene Ontology \cite{gene2015gene}. Formally, the syntax of $\mathcal{EL}^{++}$ is built up from a set $N_I$ of \emph{individual names}, $N_C$ of \emph{concept names} and $N_R$ of \emph{role names} (also called \emph{relations}) using the constructors shown in Table \ref{tab:el_syntax_seman}, where $N_I$, $N_C$ and $N_R$ are pairwise disjoint. Strictly speaking, $\mathcal{EL}^{++}$ also allows for concrete domains, but we do not make use of them here. 

The semantics of $\mathcal{EL}^{++}$ is defined by \emph{interpretations} $\interp = (\interpDom, \interpMap)$, where the domain $\interpDom$ is a non-empty set
and $\interpMap$ is a mapping that associates every individual with an element in $\interpDom$, every concept name with a subset of $\interpDom$, and every relation name with a relation over $\interpDom \times \interpDom$. An \emph{interpretation} is satisfied if it satisfies the corresponding semantic conditions. The syntax and the corresponding semantics (i.e., interpretation of concept expressions) of $\mathcal{EL}^{++}$ are summarized in Table \ref{tab:el_syntax_seman}. 

An $\mathcal{EL}^{++}$ KB $(\abox, \tbox)$ consists of an ABox $\abox$ and a TBox $\tbox$. The \emph{ABox} is a set of \emph{concept assertions} ($C(a)$)  and \emph{role assertions}  ($r(a,b)$), where $C$ is a concept, $r$ is a relation, and $a,b$ are individuals. 
The TBox is a set of \emph{concept inclusions} of the form $C \sqsubseteq D$.
Intuitively, the ABox contains instance-level information (e.g. $\concept{Person}(\insta{John})$), $\rel{isFatherOf}(\insta{John}, \insta{Peter})$), while the TBox contains information about concepts (e.g. $\concept{Parent} \sqsubseteq \concept{Person}$ ). 
Every $\mathcal{EL}^{++}$ KB can be transformed
such that every TBox statement has the form $C_1 \sqsubseteq D$, $C_1 \sqcap C_2 \sqsubseteq D$, $C_1 \sqsubseteq \exists r. C_2$, 
$\exists r. C_1 \sqsubseteq D$, where $C_1, C_2, D$ can be the top 
concept, concept names or nominals and $D$ can also be the bottom concept \cite{baader2005pushing}. 
The normalized KB can be computed in linear time by introducing new concept names for complex concept expressions and is a conservative extension of the original KB, i.e., every model of the normalized KB is a model of the original KB and every model of the original KB can be extended to be a model
of the normalized KB \cite{baader2005pushing}.

\section{Relational Representation Learning}

\subsection{Graph Representation Learning}
Given a graph $G=(V,E)$, the goal of graph representation learning is to learn a node mapping function $f:V\to \mathbb{R}^d$, which project each node into low dimensional vectors in space $\mathbb{R}^d$, where $d\ll \left|V\right|$, while preserving the graph structure and node attributes. The learned representation $\mathbf{H}$ can be applied to downstream tasks such as node classification, link prediction, and graph reconstruction.

\subsection{Graph Neural Networks.}
Graph neural networks (GNNs) are a class of neural networks designed to operate on graph-structured data. 
GNNs are particularly effective for tasks where relationships or dependencies between data points can be naturally represented as a graph. A more specific type of GNN is known as Graph convolutional networks (GCNs). 
Given a graph \( G \), a GNN processes node features \( X \) and adjacency information \( A \) to learn a mapping \( f: \mathbb{R}^{|V| \times d} \times \mathbb{R}^{|V| \times |V|} \rightarrow \mathbb{R}^{|V| \times o} \), where \( d \) is the input feature dimension, \( o \) is the output dimension, and \( |V| \) is the number of nodes.

The GNN processes information in an iterative manner through layers. At each layer, the hidden representations \( H^{(l+1)} \) are computed as a function of the previous layer's representations \( H^{(l)} \):

\[
H^{(l+1)} = \sigma \left( A \cdot H^{(l)} \cdot W^{(l)} + b^{(l)} \right)
\]

where \( W^{(l)} \) and \( b^{(l)} \) are learnable parameters for layer \( l \), \( \cdot \) denotes matrix multiplication, and \( \sigma \) is a non-linear activation function.

GNNs can be applied to a variety of tasks, such as node classification, link prediction, and graph classification, making them versatile for learning from graph-structured data.

\subsection{Knowledge Graph Embeddings.}
A knowledge graph embedding is a function \(f: V \cup R \rightarrow \mathbb{R}^d\) that projects entities and relations into a continuous vector space of dimension \(d\). Formally:

\[
f(e_i) \in \mathbb{R}^d, \quad \forall e_i \in V
\]
\[
f(r_j) \in \mathbb{R}^d, \quad \forall r_j \in R
\]

The scoring function \(s: V \times R \times V \rightarrow \mathbb{R}\) evaluates the compatibility of a triple \((h, r, t)\) in the knowledge graph. 
A common scoring function, exemplified by TransE, is defined as:

\[
s(h, r, t) = \text{dist}(\mathbf{f}(h) + \mathbf{f}(r), \mathbf{f}(t))
\]

where \(\mathbf{f}(h)\), \(\mathbf{f}(r)\), and \(\mathbf{f}(t)\) are the embeddings of the head entity, relation, and tail entity, respectively. The function \(\text{dist}(\cdot, \cdot)\) measures the dissimilarity between two vectors, often represented by the Euclidean distance:

\[
\text{dist}(\mathbf{u}, \mathbf{v}) = \|\mathbf{u} - \mathbf{v}\|_2
\]

The objective of knowledge graph embeddings is to learn vectors that minimize the scores for true triples \((h, r, t) \in E\) and simultaneously maximize the scores for negative samples, enhancing the model's ability to capture and predict complex relationships within the knowledge graph.

Two common evaluation metrics for knowledge graph embeddings are Mean Reciprocal Rank (MRR) and Hit@k. Mean Reciprocal Rank (MRR is a metric used to evaluate the ability of a model to rank the correct entity (or relationship) higher in the list of candidates. It is particularly relevant for tasks like link prediction. MRR is calculated as the average of the reciprocal ranks of the correct entities:

\[
\text{MRR} = \frac{1}{N} \sum_{i=1}^{N} \frac{1}{\text{rank}_i}
\]

Where: $N$ is the total number of test instances. \(\text{rank}_i\) is the rank of the correct entity for the \( i\text{-th} \) test instance. The higher the MRR, the better the model is at correctly ranking the true entities higher in the list.

\textbf{Hit@k} is another metric that measures the proportion of test instances where the correct entity (or relationship) appears in the top \( k \) ranked candidates. This metric is also used for tasks like link prediction. The formula for Hit@k is given by:
\[
\text{Hit@k} = \frac{1}{N} \sum_{i=1}^{N} \text{hit}_i
\]
Where $N$ is the total number of test instances. $\text{hit}_i$ is a binary indicator function, which equals 1 if the correct entity is among the top $k$ ranked candidates for the $i_{th}$ test instance, and 0 otherwise. Hit@k provides insights into the model's ability to retrieve the correct entities within the top $k$ positions.

\section{Geometric Embeddings.} 

Vector embeddings map objects to a low-dimensional vector space. Vector embeddings
have been developed to learn representations of objects that allow for distinguishing relevant differences and ignoring irrelevant variations between objects. 
When vector embeddings are used to embed structural/relational data, they fail to represent key properties of relational data
 that cannot be easily modeled in a plain, low-dimensional vector space.  For example, relational data may have been defined by applying set operators such as set inclusion and exclusion \cite{xiong2022faithful}, logical operations such as negation \cite{ren2020beta}, or they may exhibit relational patterns like the symmetry of relations \cite{abboud2020boxe} and structural patterns (e.g., trees and cycles) \cite{chami2020low,DBLP:conf/kdd/XiongZNXP0S22}. 

Going beyond plain vector embeddings, \emph{geometric relational embeddings} replace the vector representations with more advanced geometric objects, such as convex regions \cite{DBLP:conf/ijcai/KulmanovLYH19,ren2019query2box,xiong2022faithful}, density functions \cite{wang2022dirie,ren2020beta}, elements of hyperbolic manifolds \cite{chami2020low}, and their combinations \cite{suzuki2019hyperbolic}. 
Geometric relational embeddings provide a rich geometric inductive bias for modeling relational/structured data. 
For example, embedding objects as convex regions allows for modeling not only similarity but also set-based and logical operators such as set inclusion, set intersection \cite{xiong2022faithful} and logical negation \cite{zhang2021cone} while representing data on non-Euclidean manifolds allows for capturing complex structural patterns, such as representing hierarchies in hyperbolic space \cite{chami2020low}. We group them into three lines of works. 

\subsection{Distribution-based Embeddings}

Probability distributions provide a rich geometry of the latent space. Their density can be interpreted as soft regions and it allows us to model uncertainty, asymmetry, set inclusion/exclusion, entailment, and so on. 


\textbf{Gaussian embeddings.} Word2Gauss \cite{vilnis2015word} maps words to multi-dimensional Gaussian distributions over a latent embedding space such that the linguistic properties of the words are captured by the relationships between the distributions.
A Gaussian $\mathcal{N}(\boldsymbol{\mu}, \boldsymbol{\Sigma})$ is parameterized by a mean vector $\boldsymbol{\mu}$ and a covariance matrix $\boldsymbol{\Sigma}$ (usually a diagonal matrix for the sake of computing efficiency). 
The model can be optimized by an energy function $-E\left(\mathcal{N}_i, \mathcal{N}_j\right)$ that is equivalent to the KL-divergence $D_{\operatorname{KL}}\left(\mathcal{N}_j \| \mathcal{N}_i\right)$ defined as
\begin{equation}
D_{\operatorname{KL}}\left(\mathcal{N}_j \| \mathcal{N}_i\right) =\int_{x \in \mathbb{R}^n} \mathcal{N}\left(x ; \mu_i, \Sigma_i\right) \log \frac{\mathcal{N}\left(x ; \mu_j, \Sigma_j\right)}{\mathcal{N}\left(x ; \mu_i, \Sigma_i\right)} d x.  
\end{equation}

KG2E \cite{he2015learning} extends this idea to knowledge graph embedding by mapping entities and relations as Gaussians. 
Given a fact $(h,r,t)$, the scoring function is defined as $f(h, r, t)=\frac{1}{2}\left(\mathcal{D}_{\operatorname{KL}}\left(\mathcal{N}_h, \mathcal{N}_r\right)+\mathcal{D}_{\operatorname{KL}}\left(\mathcal{N}_r, \mathcal{N}_t\right)\right)$.
The covariances of entity and relation embeddings allow us to model uncertainties in knowledge graphs.
While modeling the scores of triples as KL-divergence allows us to capture asymmetry. 
TransG \cite{xiao2016transg} generalizes KG2E to a Gaussian mixture distribution to deal with multiple relation semantics revealed by the entity pairs.
For example, the relation HasPart has at least two latent semantics: composition-related as  (\emph{Table}, \emph{HasPart}, \emph{Leg}) and location-related as (\emph{Atlantics}, \emph{HasPart}, \emph{NewYorkBay}).

\subsection{Region-based Embeddings}
Mapping data as convex regions is inspired by the Venn diagram. Region-based embeddings nicely model the set theory that can be used to capture uncertainty \cite{DBLP:conf/naacl/ChenBCDLM21}, logical rules \cite{abboud2020boxe}, transitive closure~\cite{vilnis2018probabilistic}, logical operations~\cite{ren2019query2box}, etc. Several convex regions have been explored including balls, boxes, and cones. 

\textbf{Ball embeddings} associate each object $w$ with an $n$-dimensional ball $\mathbb{B}_w\left(\mathbf{c}_w, r_w\right)$, where $\mathbf{c}_w$ and $r_w$ are the central point and its radius of the ball, respectively.
A ball is defined as the set of vectors whose Euclidean distance to $\mathbf{c}_w$ is less than $r_w$: $\mathbb{B}\left(\mathbf{c}_w, r_w\right) = \left\{\mathbf{p}| \left\|\mathbf{c}_w-\mathbf{p}\right\|<r_w\right\}$. 
In ElEm \cite{DBLP:conf/ijcai/KulmanovLYH19}, each concept in $\mathcal{E}\mathcal{L}^{++}$ ontologies is represented as an open $n$-ball, and subsumption relations between concepts are modeled as ball containment. This explicit modeling of subsumption structure leads to significant improvements in predicting human protein-protein interactions. \cite{DBLP:conf/iclr/DongBJLCSCZ19} represents categories with $n$-balls while considering tree-structured category information. By embedding subordinate relations between categories as ball containment, it shows promising results on NLP tasks compared to conventional word embeddings.

\textbf{Box embeddings} represent objects with $d$ dimensional rectangles, i.e., a Cartesian product of $d$ closed intervals denoted by $\prod_{i=1}^d\left[x_i^{\mathrm{m}}, x_i^{\mathrm{M}}\right], \quad$ where $x_i^{\mathrm{m}}<x_i^{\mathrm{M}}$ and $x_i^{\mathrm{m}}$ and $x_i^{\mathrm{M}}$ are lower-left and top-right coordinates of boxes \cite{vilnis2018probabilistic}.
Box embeddings capture anticorrelation and disjoint concepts better than order embeddings. 
Joint Box \cite{patel2020representing} improves box embeddings' ability to express multiple hierarchical relations (e.g., hypernymy and meronymy) and proposes joint embedding of these relations in the same subspace to enhance entity characterization, resulting in improved performance.
BoxE \cite{abboud2020boxe} proposes embedding entities and relations as points and boxes in knowledge bases, improving expressivity by modeling rich logic hierarchies in higher-arity relations. 
BEUrRE \cite{DBLP:conf/naacl/ChenBCDLM21} is similar to BoxE but it differs in embedding entities and relations as boxes and affine transformations, respectively, which enables better modeling marginal and joint entity probabilities. 
Query2Box \cite{ren2019query2box} shares BoxE's idea, embedding queries as boxes, with answer entities inside, to support a range of querying operations, including disjunctions, in large-scale knowledge graphs. 
For multi-label classification, MBM \cite{patel2021modeling} proposes a multi-label box model that marries neural networks with \cite{vilnis2018probabilistic}, where labels are embedded as boxes so that label-label relations can be easily modeled. 

\textbf{Cone embedding} was first proposed in Order embedding (OE) \cite{vendrov2015order} to represent a partial ordered set (poset). However, this method is restricted to axis-parallel cones and it only captures positive correlations as any two axis-parallel cones intersect at infinite.
Recent cone embeddings formulate cones with additional angle parameters. 
As one of the main advantages, angular cone embedding has been used to model the negation operator since the angular cone is closed under negation. 
For example, negation can be modeled as the polarity of a cone as used in ALC ontology \cite{o2021cone} or the complement of a cone as used in ConE for logical queries \cite{zhang2021cone}.

\subsection{Manifold/Differential Geometry.}
A smooth manifold $(\mathcal{M},g)$ is defined as a smooth manifold equipped with a metric $g$, which induces a scalar product for each point $\mathbf{x}\in\mathcal{M}$ on the \emph{tangent} space $\mathcal{T}_{\mathbf{x}}\mathcal{M}$.
For each point $\mathbf{x}$, the metric tensor is defined as $g_{\mathbf{x}}:\mathcal{T}_{\mathbf{x}}\mathcal{M}\times\mathcal{T}_{\mathbf{x}}\mathcal{M}\rightarrow\mathbb{R}$, which varies smoothly with the point $\mathbf{x}$ and induces geometric notions such as geodesic distances, angles and curvatures.

\par
\textbf{Riemannian manifold.} A Riemannian Manifold $(\mathcal{M},g)$ is a manifold $\mathcal{M}$ equiped with a Riemannian metric $g$. The Riemannian metric is positive definite that means $\forall \bm{\xi}\in\mathcal{T}_{\mathbf{x}}\mathcal{M},g_{\mathbf{x}}(\bm{\xi},\bm{\xi})>0 \text{ iff } \bm{\xi}\neq \mathbf{0}$.
We denote the \emph{curvature} of a manifold as $K$ and different notions of \emph{curvature} quantify the ways in which a surface is locally curved around a point. There are three different types of constant curvature Riemannian manifold $\mathcal{M}$ with respect to the sign of curvature: hyperboloid $\mathbb{H}_K$ (negative curvature), hypersphere $\mathbb{S}_K$ (positive curvature) and Euclidean space $\mathbb{R}$ (zero curvature).
\begin{equation}\label{eq:manifolds}
\mathcal{M}
=
\left\{\begin{array}{ll}
\mathbb{S}_K^n = \{\mathbf{x}\in\mathbb{R}^{n+1}: \langle \mathbf{x},\mathbf{x}\rangle_2=1/K\}, & \text { if } K>0 \\
\mathbb{E}^n = \mathbb{R}^{n}, & \text { if } K=0 \\
\mathbb{H}_K^n = \{\mathbf{x}\in\mathbb{R}^{n+1}: \langle \mathbf{x},\mathbf{x}\rangle_{\mathcal{L}}=1/K\}, & \text { if } K<0
\end{array}\right.
\end{equation}
where $\langle\cdot,\cdot\rangle_2$ is the standard Euclidean inner product, and $\langle\cdot,\cdot\rangle_{\mathcal{L}}$ is the Lorentz inner product. For $\forall \mathbf{x},\mathbf{y}\in\mathbb{R}^{n+1}$, the Lorentz inner product $\langle\mathbf{x},\mathbf{y}\rangle_{\mathcal{L}}$ is defined as follows.
\begin{equation}\label{eq:lorentz}
\langle\mathbf{x},\mathbf{y}\rangle_{\mathcal{L}} = -x_1y_1+\sum_{i=2}^{n+1}x_iy_i,
\end{equation}
\par
\textbf{Product manifolds.} Given a sequence of smooth manifolds $M_1,M_2,..., $ $M_k$, the product manifold is defined as the Cartesian product $M=M_1\times M_2\times ...\times M_k$, with the metric tensor $g(\mathbf{u},\mathbf{v})=\sum_{i=1}^{k}g_i(u_i,v_i)$. Correspondingly, the points $\mathbf{x}\in\mathcal{M}$ can be represented as their coordinates $\mathbf{x}=(x_1,...,x_k), x_i\in\mathcal{M}_i$.
\par
\textbf{Geodesics}. A smooth curve
$\gamma$ of minimal length between two points $\mathbf{x}$ and $\mathbf{y}$ is called a geodesic, which can be seen as the generalization of a straight-line in Euclidean space. Formally, the geodesic is defined as $\gamma_{\mathbf{x}\rightarrow \bm{\xi}}(\tau): I\rightarrow\mathcal{M}$ from an interval $I=[0,1]$ of the reals to the metric space $\mathcal{M}$, which maps a real value $\tau\in I$ to a point on the manifold $\mathcal{M}$ with initial velocity $\bm{\xi}\in\mathcal{T}_{\mathbf{x}}\mathcal{M}$.
By the means of geodesic, the exponential map can be defined as $\exp_{\mathbf{x}}(\bm{\xi})=\gamma_{\mathbf{x}\rightarrow\bm{\xi}}(1)$.
\par
\textbf{Exponential and logarithmic map in Riemannian manifold.}
The connections between manifolds and \textit{tangent} space are established by the differentiable exponential map and logarithmic map.
The exponential map $\exp_{\mathbf{x}}$ at $\mathbf{x}$ gives a way to project back a vector $\mathbf{v}\in \mathcal{T}_{\mathbf{x}}\mathcal{M}$ to a point $\exp_{\mathbf{x}}(\mathbf{v})\in\mathcal{M}$ on the manifold. And the logarithmic map $\log_{\mathbf{x}}:\mathcal{M}\rightarrow\mathcal{T}_{\mathbf{x}}\mathcal{M}$ is defined as the inverse of the exponential map. Common realizations of Riemannian manifolds are hypersphere $\mathbb{S}_K$, the Euclidean space $\mathbb{R}$, and the hyperboloid $\mathbb{H}_K$. And we summarize \textit{expmap} and \textit{logmap} operations in these three spaces compactly in Table \ref{tab:explog_riemannian}. 

\begin{table}[htbp]
\caption{Summary of \textit{expmap} and \textit{logmap} in $\mathbb{R}$, $\mathbb{S}_K$ and $\mathbb{H}_K$.}
\centering
\resizebox{\linewidth}{!}{
\begin{tabular}{ll}
\toprule
 Operations &  \\
\midrule
  \textit{expmap} in $\mathbb{R}$ & $\exp_{\mathbf{x}}(\mathbf{v})=\mathbf{x}+\mathbf{v}$ \\
  \textit{logmap} in $\mathbb{R}$ & $\log_{\mathbf{x}}(\mathbf{y})=\mathbf{y}-\mathbf{x}$ \\
  \textit{expmap} in $\mathbb{S}_K$ & $\exp_{\mathbf{x}}^{K}(\mathbf{v})=\cos\left(\sqrt{|K|}\|\mathbf{v}\|_2\right)+\sin\left(\sqrt{|K|}\|\mathbf{v}\|_2\right)\frac{\mathbf{v}}{\sqrt{|K|}\|\mathbf{v}\|_2} $ \\
  \textit{logmap} in $\mathbb{S}_K$ & $\log_{\mathbf{x}}^{K}(\mathbf{y})=\frac{\cos^{-1}(K\langle \mathbf{x}, \mathbf{y} \rangle_{2})}{\sqrt{1-K^2\langle \mathbf{x}, \mathbf{y} \rangle_{2}^2}}\left(\mathbf{y}-K\langle \mathbf{x}, \mathbf{y} \rangle_{2} \mathbf{x}\right) $ \\
  \textit{expmap} in $\mathbb{H}_K$ &
  $\exp_{\mathbf{x}}^{K}(\mathbf{v})=\cosh\left( \sqrt{|K|}\|\mathbf{v}\|_{\mathcal{L}} \right)\mathbf{x} + \sinh\left( \sqrt{|K|}\|\mathbf{v}\|_{\mathcal{L}} \right)\frac{\mathbf{v}}{\sqrt{|K|}\|\mathbf{v}\|_{\mathcal{L}}}$ \\
  \textit{logmap} in $\mathbb{H}_K$ & $\log_{\mathbf{x}}^{K}(\mathbf{y})=\frac{\cosh^{-1}(K\langle \mathbf{x}, \mathbf{y} \rangle_{\mathcal{L}})}{\sqrt{K^2\langle \mathbf{x}, \mathbf{y} \rangle_{\mathcal{L}}^2-1}}\left(\mathbf{y}-K\langle \mathbf{x}, \mathbf{y} \rangle_{\mathcal{L}} \mathbf{x}\right) $  \\
\bottomrule
\end{tabular}}
\label{tab:explog_riemannian}
\end{table}

\subsection{Riemannian manifold embeddings}

\textbf{Hyperbolic Embeddings}
The Poincaré ball $\left(\mathbb{D}^{n}, g^{\mathbb{D}}\right)$ is one of the models of hyperbolic geometry that is very suitable for representing hierarchies due to its exponentially growing volume~\cite{andrews2010ricci}. 
The Poincaré ball is defined as an open $n$-ball $\mathbb{D}^{n}=\left\{\mathbf{x} \in \mathbb{R}^{n}:\|\mathbf{x}\|<1\right\}$ equipped with a Riemannian metric $g_{\mathbf{x}}^{\mathbb{D}}=\lambda_{\mathbf{x}}^{2} g^{E}$, where $\lambda_{\mathbf{x}}=\frac{2}{1-\|\mathbf{x}\|^{2}}$, 
$g^{E}=\mathbf{I}_{n}$ is the Euclidean metric tensor, $\lambda_{\mathbf{x}}$ is the \emph{conformal factor}, and $\|\cdot\|^2$ denotes the $L^{2}$ norm in Euclidean space. 
The distance between two points $\mathbf{x}, \mathbf{y} \in \mathbb{D}^{n}$ can be defined by 
\begin{equation}
    d_{\mathbb{D}}(\mathbf{x}, \mathbf{y})=\cosh ^{-1}\left(1+2 \frac{\|\mathbf{x}-\mathbf{y}\|^{2}}{\left(1-\|\mathbf{x}\|^{2}\right)\left(1-\|\mathbf{y}\|^{2}\right)}\right)
\end{equation}

The Poincaré ball has been used for modeling hierarchical structures \cite{weber2018curvature}.
MuRP \cite{balazevic2019multi} models multi-relational data by transforming entity embeddings by Möbius matrix-vector multiplication and addition.
To capture both hierarchy and logical patterns e.g., symmetry and composition,
AttH \cite{chami2020low} models relation transformation by rotation/reflection and also embeds entities on the Poincaré ball.
FieldH and FieldP \cite{nayyeri2021knowledgefieldemi} embed entities on trajectories that lie on hyperbolic manifolds namely Hyperboloid and Poincaré ball to capture heterogeneous patterns formed by a single relation (e.g., a combination of loop and path) besides hierarchy and logical patterns. 
\cite{pan2021hyperbolic} embeds nodes on the extended Poincaré  Ball and  polar coordinate system to address numerical issues when dealing with the embedding of neighboring nodes on the boundary of the ball on previous models. 
HyboNet \cite{chen2021fully} proposes a fully hyperbolic method that uses Lorentz transformations to overcome the incapability of previous hyperbolic methods of fully exploiting the advantages of hyperbolic space. 

\textbf{Spherical embeddings} 
A spherical manifold $\mathcal{M} = \{x\in \mathbb{R}^d |\,\,\, \|x\|=1\}$ has been used as embedding space in several works to model loops in graphs. 
MuRS \cite{wang2021mixed} models relations as linear transformations on the tangent space of spherical manifold, together with spherical distance in score formulation.
The spherical distance between the transformed head and tail is used to formulate the score function.
A more sophisticated approach is FiledP \cite{nayyeri2021knowledgefieldemi} in which entity spaces are trajectories on a sphere based on relations to model more complex patterns compared to MuRS.
5*E \cite{nayyeri20215} embeds entities on flows on the product space of a complex projective line which is called the Riemann sphere. 
The projective transformation then covers translating, rotation, homothety, reflection, and inversion which are powerful for modeling various structures and patterns such as combination of loop and loop, loop and path, and two connected path structures.

\textbf{Mixed manifold embeddings}
Relational data exhibit heterogeneous structures and patterns where each class of patterns and structures can be modeled efficiently by a particular manifold.
Therefore, combining several manifolds for embedding is beneficial and addressed in the literature.
MuRMP \cite{wang2021mixed} extends MuRP \cite{balazevic2019multi} to a product space of spherical, Euclidean, and hyperbolic manifolds. 
GIE \cite{cao2022geometry} improves the previous mixed models by computing the geometric interaction on tangent space of Euclidean, spherical and hyperbolic manifolds via an attention mechanism to emphasize on the most relevant geometry and then projects back the obtained vector to the hyperbolic manifold for score calculation. 
DGS \cite{iyer2022dual} targets the same problem of heterogeneity of structures and utilizes the hyperbolic space, spherical space, and intersecting space in a unified framework for learning embeddings of different portions of two-view knowledge graphs (ontology and instance levels) in different geometric space.
DyERNIE \cite{han2020dyernietemporalhyperbolic} is a temporal knowledge graph embedding on product manifolds that models the evolution of entities over time on tangent space of product manifolds. 
While the reviewed works provide separate spherical, Euclidean, and hyperbolic manifolds and act on Riemannian manifolds, 
In this model, relations are modeled by the pseudo orthogonal transformation which is decomposed into cosine-sine forms covering hyperbolic/circular rotation and reflection. This model can capture heterogeneous structures and logical patterns.

\subsection{Pseudo-Riemannian Manifolds} 

A pseudo-Riemannian manifold \cite{1983semirie} $(\mathcal{M},g)$ is a smooth manifold $\mathcal{M}$ equipped with a nondegenerate and indefinite metric tensor $g$. Nondegeneracy means that for a given $\bm{\xi}\in\mathcal{T}_{\mathbf{x}}\mathcal{M}$, for any $\bm{\zeta}\in\mathcal{T}_{\mathbf{x}}\mathcal{M}$ we have $g_{\mathbf{x}}(\bm{\xi},\bm{\zeta})=0$, then $\bm{\xi}=\mathbf{0}$.
The metric tensor $g$ induces a scalar product on the \emph{tangent} space $\mathcal{T}_{\mathbf{x}}\mathcal{M}$ for each point $\mathbf{x}\in\mathcal{M}$ such that $g_\mathbf{x}: \mathcal{T}_{\mathbf{x}}\mathcal{M} \times\mathcal{T}_{\mathbf{x}}\mathcal{M} \rightarrow \mathbb{R} $, where the \emph{tangent} space $\mathcal{T}_{\mathbf{x}}\mathcal{M}$ can be seen as the first order local approximation of $\mathcal{M}$ around point $\mathbf{x}$. 
The elements of $\mathcal{T}_{\mathbf{x}}\mathcal{M}$ are called tangent vectors. \textit{Indefinity} means that the metric tensor could be of arbitrary signs. 
A principal special case is the Riemannian geometry, where the metric tensor is positive definite (i.e. $\forall \bm{\xi}\in\mathcal{T}_{\mathbf{x}}\mathcal{M},g_{\mathbf{x}}(\bm{\xi},\bm{\xi})>0 \text{ iff } \bm{\xi}\neq \mathbf{0}$).

\subsubsection{Pseudo-hyperboloid}
By analogy with hyperboloid and sphere in Euclidean space. Pseudo-hyperboloids are defined as the submanifolds in the ambient pseudo-Euclidean space $\mathbb{R}^{s,t+1}$ with the dimensionality of $d=s+t+1$ that uses the scalar product as $\forall \mathbf{x},\mathbf{y} \in \mathbb{R}^{s,t+1}, \langle \mathbf{x}, \mathbf{y}\rangle_t = -\sum_{i=0}^{t}x_i y_i+\sum_{j=t+1}^{s+t}x_jy_j$.
The scalar product induces a norm $\|\mathbf{x}\|_{t}^{2}=\langle \mathbf{x}, \mathbf{x}\rangle_{t}$ that can be used to define a pseudo-hyperboloid
$\mathcal{Q}_{\beta}^{s, t}=\{\mathbf{x}=\left(x_{0}, x_{1}, \cdots, x_{s+t}\right)^{\top} $ $\in \mathbb{R}^{s,t+1}:\|\mathbf{x}\|_{t}^{2}=\beta\}$, where $\beta$ is a nonzero real number parameter of curvature. $\mathcal{Q}_{\beta}^{s,t}$ is called  \emph{pseudo-sphere} when $\beta>0$ and \emph{pseudo-hyperboloid} when $\beta<0$. Since $\mathcal{Q}_{\beta}^{s,t}$ is interchangeable with $\mathcal{Q}_{-\beta}^{t+1,s-1}$, we consider the pseudo-hyperbololid here. 
Following the terminology of special relativity, a point in $\mathcal{Q}_{\beta}^{s,t}$ can be interpreted as an event \cite{sun2015space}, where the first $t+1$ dimensions are time dimensions and the last $s$ dimensions are space dimensions. 
Hyperbolic $\mathbb{H}$ and spherical $\mathbb{S}$ manifolds can be defined as the special cases of pseudo-hyperboloids by setting all time dimensions except one to be zero and setting all space dimensions to be zero, respectively, i.e. $\mathbb{H}_{\beta} = \mathcal{Q}_{\beta}^{s,1}, \mathbb{S}_{-\beta} = \mathcal{Q}_{\beta}^{0,t}$.
\par

\subsubsection{Geodesic tools of pseudo-hyperboloid} 

\textbf{Geodesic.}
A generalization of a \textit{straight-line} in the Euclidean space to a manifold is called the \textit{geodesic} \cite{Ganea2018,willmore2013introduction}. Formally, a geodesic $\gamma$  is defined as a constant speed curve $\gamma: \tau \mapsto \gamma(\tau) \in \mathcal{M}, \tau \in[0,1]$ joining two points $\mathbf{x},\mathbf{y} \in \mathcal{M}$ that minimizes the length, where the length of a curve is given by $L(\gamma)=\int_{0}^{1} \sqrt{\left\|\frac{d}{d t} \gamma(\tau)\right\|_{\gamma(\tau)}} dt$. 
The geodesic holds that $\gamma^{*}=\arg \min _{\gamma} L(\gamma)$, such that $\gamma(0)=$ $\boldsymbol{x}, \gamma(1)=\boldsymbol{y}$, and $\left\|\frac{d}{d\tau} \gamma(\tau)\right\|_{\gamma(\tau)}=1$. 

By the means of the geodesic, the distance between $\mathbf{x}, \mathbf{y} \in \mathcal{Q}_{\beta}^{s,t}$ is defined as the arc length of geodesic $\gamma(\tau)$. 

\textbf{Exponential and logarithmic maps.}
The connections between manifolds and \textit{tangent} space are established by the differentiable exponential map and logarithmic map.
The exponential map at $\mathbf{x}$ is defined as $\exp_{\mathbf{x}}(\bm{\xi})=\gamma(1)$, which gives a way to project a vector $\bm{\xi} \in \mathcal{T}_{\mathbf{x}}\mathcal{M}$ to a point $\exp_{\mathbf{x}}(\bm{\xi})\in\mathcal{M}$ on the manifold. The logarithmic map $\log_{\mathbf{x}}:\mathcal{M}\rightarrow\mathcal{T}_{\mathbf{x}}\mathcal{M}$ is defined as the inverse of the exponential map (i.e. $\log_{\mathbf{x}}=\exp_{\mathbf{x}}^{-1}$). 
Note that since $\mathcal{Q}_{\beta}^{s, t}$ is a geodesically complete manifold, the domain of the exponential map $\mathcal{D}_x$ is hence defined  on the entire tangent space, i.e. $\mathcal{D}_x = \mathcal{T}_\mathbf{x}\mathcal{Q}_{\beta}^{s, t}$. However, as we will explain later, the logarithmic map $\log _{\mathbf{x}}(\mathbf{y})$ is only defined when there exists a a length-minimizing geodesic between $\mathbf{x}, \mathbf{y} \in \mathcal{Q}_{\beta}^{s, t}$. 

\textbf{Parallel transport.}\quad
Given the geodesic $\gamma(\tau)$ on $\mathcal{Q}_{\beta}^{s,t}$ passing through $\mathbf{x}\in\mathcal{Q}_{\beta}^{s,t}$ with the tangent direction $\bm{\xi}\in\mathcal{T}_{\mathbf{x}}\mathcal{Q}_{\beta}^{s,t}$, the parallel transport of $\bm{\zeta}\in\mathcal{T}_{\mathbf{x}}\mathcal{Q}_{\beta}^{s,t}$ is defined as Eq.~(\ref{eq:pt}), where $\bm{\xi}=\log_{\mathbf{x}}(\mathbf{y})$.
\begin{equation}\label{eq:pt}\scriptsize
P_{\mathbf{x}\rightarrow\mathbf{y}}^{\beta}(\bm{\zeta}) = \left\{\begin{array}{ll}
\frac{\langle\bm{\zeta},\bm{\xi}\rangle}{\|\bm{\xi}\|}\left[ \mathbf{x}\sinh(\tau\|\bm{\xi}\|)+\frac{\bm{\xi}}{\|\bm{\xi}\|}\cosh(\tau\|\bm{\xi}\|) \right]+\left( \bm{\zeta}-\frac{\langle\bm{\zeta},\bm{\xi}\rangle}{\|\bm{\xi}\|^2}\bm{\xi} \right), & \text { if } \langle \bm{\xi},\bm{\xi}\rangle_t>0 \\
\frac{\langle\bm{\zeta},\bm{\xi}\rangle}{\|\bm{\xi}\|}\left[ \mathbf{x}\sin(\tau\|\bm{\xi}\|)-\frac{\bm{\xi}}{\|\bm{\xi}\|}\cos(\tau\|\bm{\xi}\|) \right]+\left( \bm{\zeta}+\frac{\langle\bm{\zeta},\bm{\xi}\rangle}{\|\bm{\xi}\|^2}\bm{\xi} \right), & \text { if } \langle \bm{\xi},\bm{\xi}\rangle_t<0 \\
\langle\bm{\zeta},\bm{\xi}\rangle\left(\tau\mathbf{x}+\frac{1}{2}\tau^2\bm{\xi}\right)+\bm{\zeta}, & otherwise
\end{array}\right.
\end{equation}

\textbf{Distance.}\quad
By the means of the geodesic, the distance between $\mathbf{x}, \mathbf{y} \in \mathcal{Q}_{\beta}^{s,t}$ is defined as the arc length of geodesic $\gamma(\tau)$, which can be formulated by using \textit{logmap}, given by Eq.~(\ref{eq:dist}). 
\begin{equation}\label{eq:dist}
    \mathrm{D}_\gamma(\mathbf{x}, \mathbf{y})=\sqrt{\left|\left\|\log _{\mathbf{x}}(\mathbf{y})\right\|_{t}^{2}\right|}.
\end{equation}

\textbf{Geodesical connectedness.}
A pseudo-Riemannian manifold $\mathcal{M}$ is \textit{connected} iff any two points of $\mathcal{M}$ can be joined by a piecewise (broken) geodesic with each piece being a smooth geodesic. 
The manifold is \textit{geodesically connected} (or \textit{g-connected}) iff any two points can be smoothly connected by a geodesic, where the two points are called \textit{g-connected}, otherwise called \textit{g-disconnected}. 
Different from Riemannian manifolds in which the geodesical completeness implies the g-connectedness (Hopf–Rinow theorem \cite{spiegel2016hopf}), pseudo-hyperboloid is a geodesically complete but not \textit{g-connected} manifold where there exist points that cannot be smoothly connected by a geodesic \cite{giannoni2000geodesical}. Formally, in the pseudo-hyperboloid, two points $\mathbf{x}, \mathbf{y} \in \mathcal{Q}_{\beta}^{s,t}$ are \textit{g-connected} iff $\langle\mathbf{x}, \mathbf{y}\rangle_{t} < |\beta|$. 
The set of \textit{g-connected} points of $\mathbf{x} \in \mathcal{Q}_{\beta}^{s, t}$ is denoted as its \textit{normal neighborhood} $\mathcal{U}_{\mathbf{x}}=\left\{\mathbf{y} \in \mathcal{Q}_{\beta}^{s, t}:\langle\mathbf{x}, \mathbf{y}\rangle_{t} < |\beta|\right\}$.
For \textit{g-disconnected} points $\mathbf{x}, \mathbf{y} \in \mathcal{Q}_{\beta}^{s, t}$, there does not exist a tangent vector $\bm{\xi}$ such that $\mathbf{y}=\operatorname{exp}^{\beta}_{\mathbf{x}}(\bm{\xi})$, which implies that its inverse $\log^{\beta}_\mathbf{x}\left(\cdot\right)$ is only defined in the \textit{normal neighborhood} of $\mathbf{x}$. 
In a nutshell, the geodesic tools for the \textit{g-disconnected} cases are not well-defined, making it impossible to define corresponding vector operations.

\cleardoublepage
\chapter{Geometric Embedding of Structural and Relational Patterns }
\label{chap_structral}

Real-world relational data exhibit various graph structural patterns such as hierarchies and cycles. 
Recent works find that non-Euclidean Riemannian manifolds provide specific inductive biases for embedding hierarchical or spherical data. 
However, they cannot align well with data of mixed graph topologies. 
To address this, we consider a larger class of pseudo-Riemannian manifolds that generalize hyperboloid and sphere. We develop novel geodesic tools that allow for extending neural network operations into geodesically disconnected pseudo-Riemannian manifolds. 
As a consequence, we derive a pseudo-Riemannian graph convolutional networks (GCNs) that model data in pseudo-Riemannian manifolds of constant nonzero curvature. 
Our method provides a geometric inductive bias that is sufficiently flexible to model mixed topologies like hierarchical graphs with cycles. 
Based on these geodesic tools, we further generalize knowledge graph embeddings into such space by simultaneously considering mixed structural patterns and relational patterns.
To capture various relational patterns, we formulate each relation as a pseudo-orthogonal transformation that can be decomposed into a circular rotation and a hyperbolic rotation on the pseudo-Riemannian manifold.  

In this chapter, we first introduce some backgrounds on pseudo-Riemannian manifolds. Next, we generalize GCNs into a pseudo-Riemannian manifold, namely pseudo-hyperboloid. Finally, we extend knowledge graph embeddings into such space by further considering relational patterns.

\section{Pseudo-Riemannian Graph Convolutional Networks}
\label{sec:qgcn}

In this section, we first introduce the background and motivation. Next, we describe how to tackle the \textit{g-disconnectedness} in pseudo-Riemannian manifolds. Then we present the pseudo-Riemannian GCNs based on the proposed geodesic tools.

\subsection{ Background and Motivation }

Learning from graph-structured data is a pivotal task in machine learning, for which graph convolutional networks (GCNs) ~\cite{bruna2013spectral, kipf2016semi, velivckovic2017graph, wu2019simplifying} have emerged as powerful graph representation learning techniques.
GCNs exploit both features and structural properties in graphs, which makes them well-suited for a wide range of applications.
For this purpose, graphs are usually embedded in Riemannian manifolds equipped with a positive definite metric. Euclidean geometry is a special case of Riemannian manifolds of constant zero curvature that can be understood intuitively and has well-defined operations.
However, the representation power of Euclidean space is limited \cite{sun2015space}, especially when embedding complex graphs exhibiting hierarchical structures \cite{boguna2021network}.
Non-Euclidean Riemannian manifolds of constant curvatures provide an alternative to accommodate specific graph topologies. For example, hyperbolic manifold of constant negative curvature has exponentially growing volume and is well suited to represent hierarchical structures such as tree-like graphs \cite{gromov1987hyperbolic,krioukov2010hyperbolic,xiong2022hyper,yang2022hicf,yang2022hrcf}. Similarly, spherical manifold of constant positive curvature is suitable for embedding spherical data in various fields \cite{wilson2014spherical, meng2019spherical, defferrard2019deepsphere} including graphs with cycles.
Some recent works \cite{Ganea2018, liu2019hyperbolic, zhu2020graph, zhang2021lorentzian,dai2021hyperbolic} have extended GCNs to such non-Euclidean manifolds and have shown substantial improvements.
\par
The topologies in real-world graphs \cite{boguna2021network}, however, usually exhibit highly heterogeneous topological structures, which are best represented by different geometrical curvatures. 
A globally homogeneous geometry lacks the flexibility for modeling complex graphs \cite{DBLP:conf/iclr/GuSGR19}. Instead of using a single manifold, product manifolds \cite{DBLP:conf/iclr/GuSGR19, skopek2019mixed} combining multiple Riemannian manifolds have shown advantages when embedding graphs of mixed topologies. However, the curvature distribution of product manifolds is the same at each point, which limits the capability of embedding topologically heterogeneous graphs. Furthermore, Riemannian manifolds are equipped with a positive definite metric disallowing for the faithful representation of the negative eigen-spectrum of input similarities \cite{laub2004feature}. 

\begin{figure}[t!]
\begin{center}
\centerline{\includegraphics[width=0.9\columnwidth]{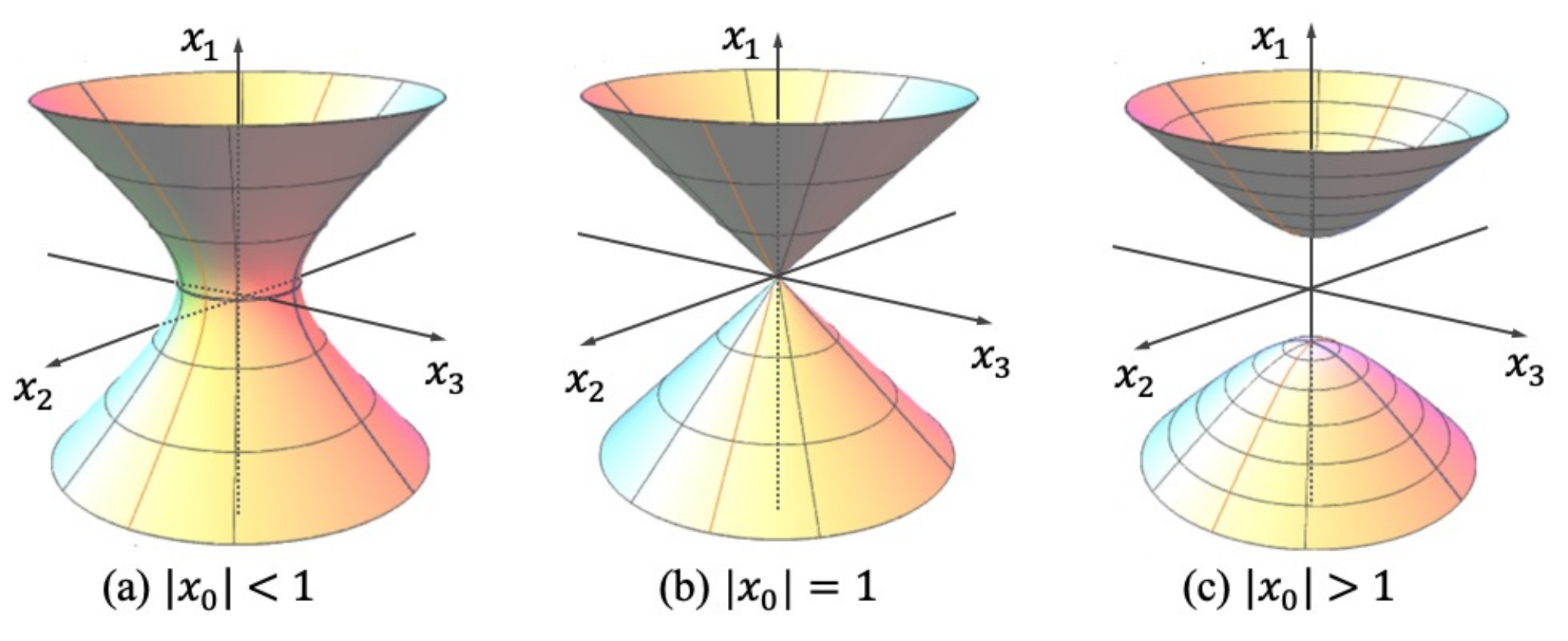}}
\caption{The different submanifolds of a four-dimensional pseudo-hyperboloid of curvature $-1$ with two time dimensions. By fixing one time dimension $x_0$, the induced submanifolds include (a) an one-sheet hyperboloid, (b) the double cone, and (c) a two-sheet hyperboloid.}
\label{fig:pseudo-hyperboloid}
\end{center}
\vskip -0.3in
\end{figure}

Going beyond Riemannian manifolds, pseudo-Riemannian manifolds equipped with indefinite metrics constitute a larger class of geometries, pseudo-Riemannian manifolds of constant nonzero curvature do not only generalize the hyperbolic and spherical manifolds, but also contain their submanifolds (Cf. Fig.~\ref{fig:pseudo-hyperboloid}), thus providing inductive biases specific to these geometries.
Pseudo-Riemannian geometry with constant zero curvature (i.e. Lorentzian spacetime) was applied to manifold learning for preserving local information of non-metric data \cite{sun2015space} and embedding directed acyclic graph \cite{clough2017embedding}.
To model complex graphs containing both hierarchies and cycles, pseudo-Riemannian manifolds with constant nonzero curvature have recently been applied into graph embeddings using non-parametric learning \cite{law2020ultrahyperbolic, sim2021directed}, but the representation power of these works is not on par with the Riemannian counterparts yet, mostly because of the absence of geodesic tools to extend neural network operations into pseudo-Riemannian geometry. 
\par
In this work, we take the first step to extend GCNs into pseudo-Riemannian manifolds foregoing the requirement to have a positive definite metric.
Exploiting pseudo-Riemannian geometry for GCNs is non-trivial because of the \textit{geodesical disconnectedness} in pseudo-Riemannian geometry. 
There exist broken points that cannot be smoothly connected by a geodesic, leaving necessary geodesic tools undefined.
To deal with it, we develop novel geodesic tools that empower manipulating representations in geodesically disconnected pseudo-Riemannian manifolds. 
We make it by finding diffeomorphic manifolds that provide alternative geodesic operations that smoothly avoid broken cases. 
Subsequently, we generalize GCNs to learn representations of complex graphs in pseudo-Riemannian geometry by defining corresponding operations such as \emph{linear transformation} and \emph{tangential aggregation}. 
Different from previous works, the initial features of GCN could be fully defined in the Euclidean space. Thanks to the diffeomorphic operation that is bijective and differentiable, the standard gradient descent algorithm can be exploited to perform optimization. 

To summarize, our main contributions are as follows: 1) We present neural network operations in pseudo-Riemannian manifolds with novel geodesic tools, to stimulate the applications of pseudo-Riemannian geometry in geometric deep learning. 2) We present a principled framework, pseudo-Riemannian GCN, which generalizes GCNs into pseudo-Riemannian manifolds with indefinite metrics, providing more flexible inductive biases to accommodate complex graphs with mixed topologies. 3) Extensive evaluations on three standard tasks demonstrate that our model outperforms baselines that operate in Riemannian manifolds.


\subsection{Methodology}

\subsubsection{Diffeomorphic geodesic tools}
One standard way to tackle the \textit{g-disconnectedness} in differential geometry is to introduce diffeomorphic manifolds in which the operations are well-defined. A diffeomorphic manifold can be derived from a diffeomorphism, defined as follows.
\begin{definition}[Diffeomorphism \cite{1983semirie}]\label{df:df}
Given two manifolds $\mathcal{M}$ and $\mathcal{M}^\prime$, a smooth map $\psi: \mathcal{M} \rightarrow \mathcal{M}^\prime$ is called a diffeomorphism if $\psi$ is bijective and its inverse $\psi^{-1}$ is smooth as well. If a diffeomorphism between $\mathcal{M}$ and $\mathcal{M}^\prime$ exists, we call them diffeomorphic and write $\mathcal{M}\simeq\mathcal{M}^\prime$. 
\end{definition} 
For pseudo-Riemannian manifolds, the following diffeomorphism \cite{law2020ultrahyperbolic} decomposes pseudo-hyperboloid into the product manifolds of an unit sphere and the Euclidean space.
\begin{theorem}[Theorem 4.1 in \cite{law2020ultrahyperbolic}] \label{lm:sr}
For any point $\mathbf{x} \in \mathcal{Q}_{\beta}^{s, t}$, there exists a diffeomorphism $\psi: \mathcal{Q}_{\beta}^{s, t} \rightarrow  \mathbb{S}_{1}^{t} \times \mathbb{R}^{s}$ that maps $\mathbf{x}$ into the product manifolds of an unit sphere and the Euclidean space.
\end{theorem}  
The diffeomorphism is given in the Appendix \ref{app:3.1}. 
In light of this, we introduce a new diffeomorphism that maps $\mathbf{x}$ to the product manifolds of sphere with curvature $-1/\beta$ and the Euclidean space. 
\begin{theorem}\label{lm:sbr}
For any point $\mathbf{x} \in \mathcal{Q}_{\beta}^{s, t}$, there exists a diffeomorphism $\psi: \mathcal{Q}_{\beta}^{s, t} \rightarrow  \mathbb{S}_{-\beta}^{t} \times \mathbb{R}^{s}$ that maps $\mathbf{x}$ into the product manifolds of a sphere and the Euclidean space  (proof in the Appendix \ref{app:lm:sbr}).
\end{theorem}
Compared with Theorem \ref{lm:sr}, this diffeomorphism preserves the curvatures in the diffeomorphic components, making it satisfy some geometric properties, e.g. the mapped point $\mathbf{x}_t \in \mathbb{S}_{-\beta}^{t}$ still lies on the surface of the pseudo-hyperboloid, making moving the tangent vectors from the pseudo-hyperboloid to the diffeomorphic manifold easy as we explained later. We call $\psi$ as the spherical projection.

\textbf{Exponential and logarithmic maps.}
Since pseudo-hyperboloid $\mathcal{Q}_{\beta}^{s, t}$ is \textit{g-disconnected}, we propose to transfer the \emph{logmap} and \emph{expmap} into the diffeomorphic manifold $\psi: \mathcal{Q}_{\beta}^{s, t} \rightarrow  \mathbb{S}_{-\beta}^{t} \times \mathbb{R}^{s}$, since the product manifold $\mathbb{S}_{-\beta}^{t} \times \mathbb{R}^{s}$ is \textit{g-connected}. To map tangent vectors between tangent space of $\mathcal{Q}_{\beta}^{s, t}$ and $\mathbb{S}_{-\beta}^{t} \times \mathbb{R}^{s}$, we exploit \textit{pushforward} that induce a linear approximation of smooth maps on tangent spaces.
\begin{definition}[Pushforward]\label{df:push}
Suppose that $\psi: \mathcal{M} \rightarrow \mathcal{M}^\prime$ is a smooth map, then the differential of $\psi$: $d\psi$ at point $\mathbf{x}$ is a linear map from the tangent space of $\mathcal{M}$ at $\mathbf{x}$ to the tangent space of $\mathcal{M}^\prime$ at $\psi(\mathbf{x})$. Namely, $d\psi:\mathcal{T}_x\mathcal{M} \rightarrow \mathcal{T}_{\psi(x)} \mathcal{M}^\prime$.
\end{definition}
Intuitively, \textit{pushforward} can be used to \textit{push} tangent vectors on $\mathcal{T}_x\mathcal{Q}_{\beta}^{s, t}$ \textit{forward} to tangent vectors on $\mathcal{T}_{\psi(x)}\mathbb{S}_{-\beta}^{t} \times \mathbb{R}^{s}$. Based on this, the new \textit{logmap} and its inverse \textit{expmap} can be defined by Eq.~(\ref{eq:log_exp_diff}).
\begin{equation}\label{eq:log_exp_diff}
    \widehat{\log}_{\mathcal{Q}_{\beta}^{s, t}}(\mathbf{x}) = \psi^{-1}(\log_{ \mathbb{S}_{-\beta}^{t} \times \mathbb{R}^{s}}(\psi(\mathbf{x}))), \widehat{\exp}_{\mathcal{Q}_{\beta}^{s, t}}(\bm{\xi}) = \psi^{-1}(\exp _{\mathbb{S}_{-\beta}^{t} \times \mathbb{R}^{s}}(\psi(\bm{\xi}))),
\end{equation}
where $\psi(\cdot)$ is the spherical projection and $\psi^{-1}(\cdot)$ is the inverse. 
The mapping of tangent vectors is achieved by \textit{pushforward} operations. 
The operations $\log_{\mathbb{S}_{-\beta}^{t} \times \mathbb{R}^{s}}(\cdot)$ and $\exp_{\mathbb{S}_{-\beta}^{t} \times \mathbb{R}^{s}}(\cdot)$ in the product manifolds can be defined as the concatenation of corresponding operations in different components. 
\begin{equation}\label{eq:log_exp}\small
\log_{\mathbb{S}_{-\beta}^{t} \times \mathbb{R}^{s}}(\mathbf{x}^\prime) =  \log_{\mathbb{S}_{-\beta}^{t}}(\mathbf{x}_t^\prime) \mathbin\Vert  \log_{\mathbb{R}^{s}}(\mathbf{x}_s^\prime), 
\exp_{\mathbb{S}_{-\beta}^{t} \times \mathbb{R}^{s}}(\bm{\xi}) = \exp_{\mathbb{S}_{-\beta}^{t}}(\bm{\xi}_{t}) \mathbin\Vert \exp_{\mathbb{R}^{s}}(\bm{\xi}_{s}),
\end{equation}
where $||$ denotes the concatenation, $\mathbf{x}^\prime=\psi_{\mathbb{S}}(\mathbf{x})$ consists of spherical features  $\mathbf{x}_t^\prime \in \mathbb{S}_{-\beta}^{t}$ and Euclidean features $\mathbf{x}_s^\prime \in \mathbb{R}^{s}$.
$\bm{\xi}$ is the tangent vector induced by $\mathbf{x}$ on $\mathcal{Q}_{\beta}^{s, t}$.

We choose points where space dimension $\mathbf{s}=\mathbf{0}$ as the reference points due to the following property. 
\begin{theorem}[]\label{theory:tangent_sharing}
For any reference point $\mathbf{x}=\left(\begin{array}{c}
\mathbf{t}\\
\mathbf{s}
\end{array}\right) \in \mathcal{Q}_{\beta}^{s, t}$ with space dimension $\mathbf{s}=\mathbf{0}$,
the induced tangent space of $\mathcal{Q}_{\beta}^{s, t}$ is equal to the tangent space of its diffeomorphic manifold $\mathbb{S}_{-\beta}^{t} \times \mathbb{R}^{s}$, namely, $\mathcal{T}_\mathbf{x}({\mathbb{S}_{-\beta}^{t} \times \mathbb{R}^{s}}) = \mathcal{T}_\mathbf{x} \mathcal{Q}_{\beta}^{s, t}$. (proof in the Appendix \ref{app:invariance_space}).
\end{theorem}
The intuition of proof is that if space dimension $\mathbf{s}=\mathbf{0}$, the pushforward (differential) function just influences time dimension, for which the mapping is just an identity function (see Appendix \ref{app:lm:sbr}). 
In this way, although we transfer \textit{logmap} and \textit{expmap} to the diffeomorphic manifold $\mathbb{S}_{-\beta}^{t} \times \mathbb{R}^{s}$, the diffeomorphic operations $\widehat{\log}_{\mathbf{x}}(\cdot)$ and $\widehat{\exp}_{\mathbf{x}}(\cdot)$ are still bijective functions from the pseudo-hyperboloid to the tangent space of the manifold itself. Hence, our final operations are actually still defined in the tangent space of the pseudo-hyperboloid. 
Note that such property only holds when our Theorem \ref{lm:sbr} is applied and the special reference points with space dimension $\mathbf{s}=\mathbf{0}$ are chosen. 

\par
By leveraging the new \textit{logmap} and \textit{expmap}, we further formulate the diffeomorphic version of tangential operations as follows.

\textbf{Tangential operations.}\quad
For function $f: \mathbb{R}^d \rightarrow \mathbb{R}^{d^\prime}$, 
the pseudo-hyperboloid version $f^\otimes: \mathcal{Q}_{\beta}^{s, t} \rightarrow \mathcal{Q}_{\beta}^{s^{\prime}, t^{\prime}}$ with $s+t=d$ and $s^{\prime}+t^{\prime}=d^{\prime}$ can be defined by the means of $\widehat{\log}^{\beta}_{\mathbf{x}}(\cdot)$ and $\widehat{\exp}^{\beta}_{\mathbf{x}}(\cdot)$ as Eq.~(\ref{eq:mobius_function}). 
\begin{equation}\label{eq:mobius_function}\small
    f^\otimes(\cdot) := \widehat{\exp} _{\mathbf{x}}^{\beta}\left( f\left( \widehat{\log} _{\mathbf{x}}^{\beta} \left(\cdot\right)\right)\right),
\end{equation}
where $\mathbf{x}$ is the reference point. Note that this function is a morphism (i.e. $(f\circ g)^{\otimes} = f^\otimes \circ g^\otimes$) and direction preserving (i.e. $f^\otimes(\cdot)/\|f^\otimes(\cdot)\| = f(\cdot)/\|f(\cdot)\|$) \cite{Ganea2018}, making it a natural way to define pseudo-hyperboloid version of vector operations such like scalar multiplication, matrix-vector multiplication, tangential aggregation and point-wise non-linearity and so on.

\textbf{Parallel transport.}\quad
Parallel transport is the generalization of Euclidean translation into manifolds. Formally, for any two points $\mathbf{x}$ and $\mathbf{y}$ connected by a geodesic, parallel transport $P^{\beta}_{\mathbf{x} \rightarrow \mathbf{y}}(\bm{\xi}):\mathcal{T}_\mathbf{x}\mathcal{M} \rightarrow \mathcal{T}_\mathbf{y}\mathcal{M}$ is an isomorphism between two tangent spaces by moving one tangent vector $\bm{\zeta} \in \mathcal{T}_x\mathcal{M}$ with tangent direction $\bm{\xi}\in\mathcal{T}_x\mathcal{M}$ to another tangent space $\mathcal{T}_\mathbf{y}\mathcal{M}$.
The parallel transport in pseudo-hyperboloid can be defined as the combination of Riemannian parallel transport \cite{gao2018semi}.
However, the parallel transport has not been defined when there does not exist a geodesic between $\mathbf{x}$ and $\mathbf{y}$. i.e., the tangent vector $\bm{\zeta}$ induced by $\mathbf{x}$ can not be transported to the tangent space of points outside of the normal neighborhood $\mathcal{U}_{\mathbf{x}}$.
Intuitively, the normal neighborhoods satisfy the following property. 
\begin{theorem}\label{lm:nn}
For any point $\mathbf{x} \in \mathcal{Q}_{\beta}^{s, t}$, the union of the normal neighborhood of $\mathbf{x}$ and the normal neighborhood of its antipodal point $\mathbf{-x}$ cover the entire manifold. Namely, $\mathcal{U}_{\mathbf{x}} \cup \mathcal{U}_{\mathbf{-x}} = \mathcal{Q}_{\beta}^{s, t}$ (proof in the Appendix \ref{app:theorem_nn}).
\end{theorem}
This theorem ensures that if a point $\mathbf{y} \notin \mathcal{U}_{\mathbf{x}}$, its antipodal point $\mathbf{-y} \in \mathcal{U}_{\mathbf{x}}$. Besides, $\mathcal{T}_{\mathbf{y}}\mathcal{M}$ is parallel to $\mathcal{T}_{\mathbf{-y}}\mathcal{M}$. Hence, $P^{\beta}_{\mathbf{x} \rightarrow \mathbf{y}}$ can be alternatively defined as $P^{\beta}_{\mathbf{x} \rightarrow \mathbf{-y}}$ for broken points. This result is crucial to define the pseudo-hyperbolic addition, such as bias translation, detailed in section \ref{sec:bias}.

\textbf{Broken geodesic distance.}\quad
By the means of geodesic, the induced distance between $\mathbf{x}$ and $\mathbf{y}$ in pseudo-hyperboloid is defined as the arc length of geodesic $\gamma(\tau)$, given by $\mathrm{d_\gamma}(\mathbf{x}, \mathbf{y})=\sqrt{\left\|\log _{\mathbf{x}}(\mathbf{y})\right\|_{t}^{2}}.$ 
For broken cases in which $\log _{\mathbf{x}}(\mathbf{y})$ is not defined, one approach is to use approximation like \cite{law2020ultrahyperbolic}. Different from that, we define following closed-form distance, given by Eq.~(\ref{eq:distance}). 

\begin{equation}\label{eq:distance}
\mathrm{D}_{\gamma}(\mathbf{x}, \mathbf{y})=\left\{ \begin{array}{ll}
\mathrm{d}_{\gamma}(\mathbf{x}, \mathbf{y}), & \text { if }\langle\mathbf{x}, \mathbf{y} \rangle_{t} < |\beta| \\
\pi \sqrt{|\beta|}+ \mathrm{d}_{\gamma}(\mathbf{x}, -\mathbf{y}), & \text { if }\langle\mathbf{x}, \mathbf{y} \rangle_{t} \geq |\beta| 
\end{array}\right.
\end{equation}
The intuition is that when $\mathbf{x}, \mathbf{y} \in \mathcal{Q}_{\beta}^{s, t}$ are \textit{g-disconnected}, we consider the distance as $\mathrm{d_\gamma}(\mathbf{x}, \mathbf{y}) = \mathrm{d_\gamma}(\mathbf{x},-\mathbf{x})+\mathrm{d_\gamma}(-\mathbf{x}, \mathbf{y})$ or
$\mathrm{d_\gamma}(\mathbf{x}, \mathbf{y}) =
\mathrm{d_\gamma}(\mathbf{x},-\mathbf{y})+\mathrm{d_\gamma}(-\mathbf{y}, \mathbf{y})$. Since $\mathrm{d_\gamma}(\mathbf{x},-\mathbf{x})=\mathrm{d_\gamma}(\mathbf{-y},\mathbf{y})=\pi \sqrt{|\beta|}$ is a constant and $\mathrm{d}_{\gamma}(-\mathbf{x}, \mathbf{y})=\mathrm{d_\gamma}(\mathbf{x}, -\mathbf{y})$, the distance between broken points can be calculated as 
$\mathrm{d_\gamma}(\mathbf{x}, \mathbf{y})=\pi \sqrt{|\beta|} + \mathrm{d_\gamma}(\mathbf{x}, -\mathbf{y})$.

To clarify theoretical contributions, our Theorem \ref{lm:sbr} is nessasary for our Theorem \ref{theory:tangent_sharing} while Theorem \ref{theory:tangent_sharing} is nessasary for transforming the GCN operations directly into the tangent space of the pseudo-hyperboloid.
Besides, we are the first to formulate the diffeomorphic \textit{expmap}, \textit{logmap} and tangential operations of pseudo-hyperboloid to avoid broken cases. The theoretical properties of parallel transport and geodesic distance are discussed in the literature \cite{law2020ultrahyperbolic,gao2018semi}. However, we re-formulate parallel transport with Theorem \ref{lm:nn} to avoid broken issues and propose a new distance measure using the broken geodesic (Eq.\ref{eq:distance} ), which is different from the approximated distance in \cite{law2020ultrahyperbolic}.

\subsubsection{Model architecture}
GCNs can be interpreted as performing neighborhood aggregation after a linear transformation on node features of each layer.
We present pseudo-Riemannian GCNs ($\mathcal{Q}$-GCN) by deriving corresponding operations with the developed geodesic tools in the $\mathcal{Q}_{\beta}^{s,t}$. 

\textbf{Feature initialization.}\quad
We first map the features from Euclidean space to pseudo-hyperboloid, considering that the input features of nodes usually live in Euclidean space. 
Following the feature transformation from Euclidean space to pseudo-hyperboloid in \cite{law2020ultrahyperbolic}, we initialize the node features by performing a differentiable mapping $\varphi: \mathbb{R}_*^{t+1} \times \mathbb{R}^s \rightarrow \mathcal{Q}_{\beta}^{s,t}$
that can be implemented by a double projection \cite{law2020ultrahyperbolic} based on Theorem \ref{lm:sbr}, i.e. $\varphi=\psi^{-1} \circ \psi$.
The intuition is that we first map the Euclidean features into diffeomorphic manifolds $\mathbb{S}_{-\beta}^t \times \mathbb{R}^{s}$ via $\psi(\cdot)$, and then map them into the pseudo-hyperboloid $\mathcal{Q}_{\beta}^{s,t}$ via $\psi^{-1}(\cdot)$, where the mapping functions are given by Eq.~(\ref{eq:map_to_sr}).

\begin{equation}\label{eq:map_to_sr}\small
\psi(\mathbf{x})=\left(\begin{array}{c}
\sqrt{|\beta|} \frac{\mathbf{t}}{\|\mathbf{t}\|} \\
 \mathbf{s}
\end{array}\right) \quad \text {,} \quad \psi^{-1}(\mathbf{z})=\left(\begin{array}{c}
\frac{\sqrt{|\beta|+\|\mathbf{v}\|^{2}}}{\sqrt{|\beta|}} \mathbf{u} \\
\mathbf{v}
\end{array}\right),
\end{equation}
where $\mathbf{x}=\left(\mathbf{t},\mathbf{s}\right)^{\top}\in \mathcal{Q}_{\beta}^{s, t}$ with  $\mathbf{t} \in \mathbb{R}_*^{t}$ and $\mathbf{s} \in \mathbb{R}^s$. $\mathbf{z}=\left(\mathbf{u},\mathbf{v}\right)^{\top} \in \mathbb{S}_{-\beta}^{t} \times \mathbb{R}^{s}$ with $\mathbf{u} \in \mathbb{S}_{-\beta}^{t}$ and $\mathbf{v} \in \mathbb{R}^{s}$.

\textbf{Tangential aggregation.}\quad
The linear combination of neighborhood features is lifted to the tangent space, which is an intrinsic operation in differential manifolds \cite{DBLP:conf/nips/ChamiYRL19,liu2019hyperbolic}. 
Specifically, $\mathcal{Q}$-GCN aggregates neighbours' embeddings in the tangent space of the reference point $\mathbf{o}$ before passing through a tangential activation function, and then projects the updated representation back to the manifold.
Formally, at each layer $\ell$, the updated features of each node $i$ are defined as Eq.~(\ref{eq:agg}).
\begin{equation}\label{eq:agg}\small
    \mathbf{h}_i^{\ell+1}
    =\widehat{\exp} _{\mathbf{o}}^{\beta_{\ell+1}}\left(
    \sigma\left(\sum_{j \in \mathcal{N}(i)\cup \{i\}} \widehat{\log}_{\mathbf{o}}^{\beta_{\ell}}\left(
    W^{\ell}\otimes^{\beta_{\ell}}\mathbf{h}_{j}^{\ell}\oplus^{\beta_{\ell}}\mathbf{b}^{\ell} \right)\right)\right),
\end{equation}
where $\sigma(\cdot)$ is the activation function, $\beta_{\ell}$ and $\beta_{\ell+1}$ are two layer-wise curvatures, $\mathcal{N}(i)$ denotes the one-hop neighborhoods of node $i$, and the $\otimes,\oplus$ denote two basic operations, i.e. tangential transformation and bias translation, respectively.

\textbf{Tangential transformation.}\label{sec:bias}\quad
We perform Euclidean transformations on the tangent space by leveraging the \textit{expmap} and \textit{logmap} in Eq.~(\ref{eq:log_exp_diff}). Specifically, we first project the hidden feature into the tangent space of \emph{south pole} $\mathbf{o} = [|\beta|,0,,...,0]$ using \textit{logmap} and then perform Euclidean matrix multiplication. Afterwards, the transformed features are mapped back to the manifold using \textit{expmap}. 
Formally, at each layer $\ell$, the tangential transformation is given by $W^{\ell} \otimes^{\beta} \mathbf{h}^{\ell} :=
\widehat{\exp} _{\mathbf{o}}^{\beta}(W^{\ell} \widehat{\log} _{\mathbf{o}}^{\beta} $ $ (\mathbf{h}^{\ell}))$, where $\otimes^{\beta}$ denotes the pseudo-hyperboloid tangential multiplication, and $W^{\ell} \in \mathbb{R}^{d^{\prime} \times d}$ denotes the layer-wise learnable matrix in Euclidean space. 

\textbf{Bias translation.}\quad 
It is noteworthy that simply stacking multiple layers of the tangential transformation would collapse the composition \cite{DBLP:conf/nips/ChamiYRL19, Ganea2018}, i.e. $\exp _{\mathbf{o}}^{\beta} ... (W^{1}\log _{\mathbf{o}}^{\beta}(\exp _{\mathbf{o}}^{\beta} ( W^{0}\log _{\mathbf{o}}^{\beta}(x) ) )) = \exp _{\mathbf{o}}^{\beta} (W^{0} \times W^{1} \times ...\times \log _{\mathbf{o}}^{\beta}(x) )$, which means that these multiplications can simply be performed in Euclidean space except the first \emph{logmap} and last \emph{expmap}. 
To avoid model collapsing , we perform bias translation after the tangential transformation. By the means of pseudo-hyperboloid \emph{parallel transport}, the bias translation can be performed by parallel transporting a tangent vector $\mathbf{b}^{\ell}\in \mathcal{T}_{\mathbf{o}}\mathcal{Q}_{\beta}^{s,t}$
to the tangent space of the point of interest. 
Finally, the transported tangent vector is mapped back to the manifold with $expmap$. 
Considering that $\widehat{\exp}(\cdot)$ is only defined at point $\mathbf{x} \in \mathcal{Q}_{\beta}^{s,t}$ with the space dimension $\mathbf{x}_s=\mathbf{0}$, we perform the original $\exp_{\mathcal{Q}_{\beta}^{s,t}}(\cdot)$ at the point of interest. The bias translation is formally given by:
\begin{equation}\label{eq:bias_add}\small
\widetilde{\mathbf{h}}^{\ell} \oplus^{\beta} \mathbf{b}^{\ell}:=\left\{ \begin{array}{ll}
\exp^{\beta}_{\widetilde{\mathbf{h}}^{\ell}}\left(P_{\mathbf{o} \rightarrow \widetilde{\mathbf{h}}^{\ell}}^{\beta} \left(\mathbf{b}^{\ell} \right)\right), & \text{if } \langle\mathbf{o}, \widetilde{\mathbf{h}}^{\ell}\rangle_{t} < |\beta| \\
-\exp^{\beta}_{-\widetilde{\mathbf{h}}^{\ell}}\left(P_{\mathbf{o} \rightarrow -\widetilde{\mathbf{h}}^{\ell}}^{\beta}\left(\mathbf{b}^{\ell}\right)\right), & \text{if } \langle\mathbf{o}, \widetilde{\mathbf{h}}^{\ell}\rangle_{t} \geq |\beta|
\end{array}\right.
\end{equation}where $\widetilde{\mathbf{h}}^{\ell}=W^{\ell} \otimes^{\beta} \mathbf{h}^{\ell}$, $\oplus^{\beta}$ denotes the pseudo-hyperboloid addition.
For the broken cases where $\langle\mathbf{o}, \widetilde{\mathbf{h}}^{\ell}\rangle_{t} \geq |\beta|$, the parallel transport $P_{\mathbf{o} \rightarrow \widetilde{\mathbf{h}}^{\ell}}^{\beta}$ is not defined. In this case, we parallel transport $\mathbf{b}^{\ell}$ to the tangent space of the antipodal point $-\widetilde{\mathbf{h}}^{\ell}$, and then perform $\exp^{\beta}_{-\widetilde{\mathbf{h}}^{\ell}}$ to map it back to the manifold. 
Note that the case $\langle\mathbf{o}, \widetilde{\mathbf{h}}^{\ell}\rangle_{t} = |\beta|$ occurs if and only if $\mathbf{h}=\mathbf{-o}$, in which case $P_{\mathbf{o} \rightarrow -\widetilde{\mathbf{h}}^{\ell}}^{\beta}\left(\mathbf{b}^{\ell}\right) = P_{\mathbf{o} \rightarrow \mathbf{-o} }^{\beta}\left(\mathbf{b}^{\ell}\right) = \mathbf{-b}$. 

\subsection{Experiments}
We evaluate the effectiveness of $\mathcal{Q}$-GCN on graph reconstruction, node classification and link prediction. 
Firstly, we study the geometric properties of the used datasets including the graph sectional curvature \cite{DBLP:conf/iclr/GuSGR19} and the $\delta$-hyperbolicity \cite{gromov1987hyperbolic}.
Fig.~\ref{fig:secs} shows the histograms of sectional curvature and the mean sectional curvature for all datasets. 
It can be seen that all datasets have both positive and negative sectional curvatures, showcasing that all graphs contain mixed graph topologies. 
To further analyze the degree of the hierarchy, we apply $\delta$-hyperbolicity to identify the tree-likeness, as shown in Table~\ref{tab:hyp}. 
We conjecture that the datasets with positive graph sectional curvature or larger $\delta$-hyperbolicity should be suitable for pseudo-hyperboloid with a smaller time dimension, while datasets with negative graph sectional curvature or smaller $\delta$-hyperbolicity should be aligned well with pseudo-hyperboloid with a larger time dimension.

\begin{figure}[t!]
    \centering  
    \includegraphics[width=0.9\columnwidth]{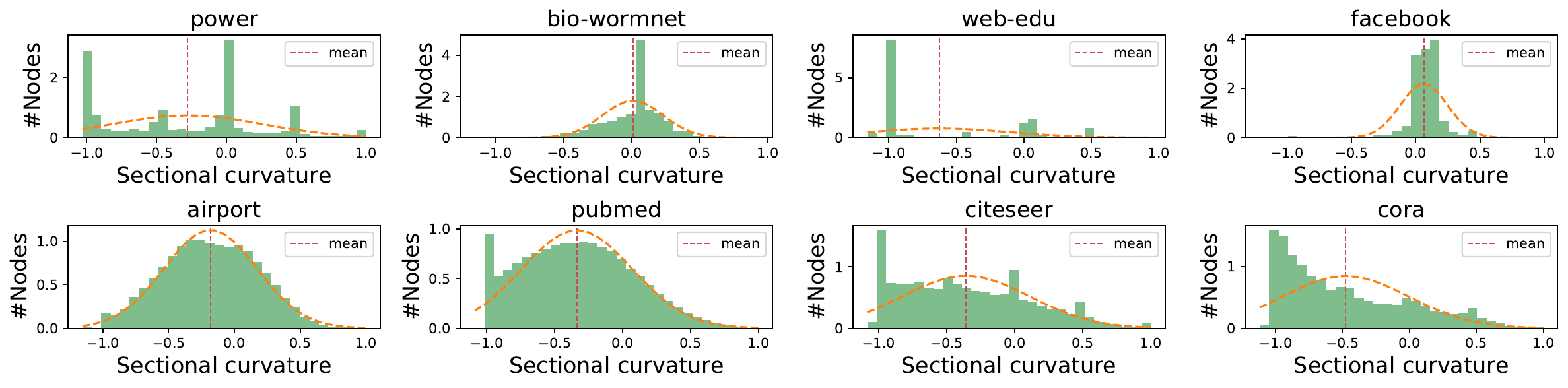}%
    \caption{The histograms of sectional curvature for all used datasets.}%
    \label{fig:secs}%
\end{figure}

\begin{table}[t!]\small
\centering
\caption{The $\delta$-hyperbolicity distribution of all used datasets.}\label{tab:hyp}
\begin{tabular}{ccccccc}
\hline
\noalign{\smallskip} 
Datasets &  0 & 0.5 & 1.0 & 1.5 & 2.0 & 2.5 \\
\hline\noalign{\smallskip}
Power &  0.4025 & 0.1722 & 0.1436 & 0.0773 & 0.0639 & 0.0439\\
Bio-Worm & 0.5635 & 0.3949 & 0.0410 & 0 & 0 & 0 \\
Web-Edu & 0.9532 & 0.0468 & 0 & 0 & 0 & 0 \\
Facebook & 0.8209 & 0.1569 & 0.0221 & 0 & 0 & 0 \\
Airport & 0.6376 & 0.3563 & 0.0061 & 0 & 0 & 0 \\
Pubmed & 0.4239 & 0.4549 & 0.1094 & 0.0112 & 0.0006 & 0 \\
CiteSeer & 0.3659 & 0.3538 & 0.1699 & 0.0678 & 0.0288 & 0 \\
Cora & 0.4474 & 0.4073 & 0.1248 & 0.0189 & 0.0016 & 0.0102 \\
\noalign{\smallskip}\hline
\end{tabular}
\end{table}

\subsubsection{Graph reconstruction}

\textbf{Datasets and baselines.}\quad
We benchmark graph reconstruction on four real-world graphs including 1) Web-Edu \cite{gleich2004fast}: a web network consisting of the $.edu$ domain; 2) Power \cite{watts1998collective}: a power grid distribution network with backbone structure; 3) Bio-Worm \cite{cho2014wormnet}: a worms gene network; 4) Facebook \cite{mcauley2012learning}: a dense social network from Facebook.  
We compare our method with Euclidean GCN \cite{kipf2016semi}, hyperbolic GCN (HGCN) \cite{DBLP:conf/nips/ChamiYRL19}, spherical GCN, and product manifold GCNs ($\kappa$-GCN) \cite{bachmann2020constant} with three signatures (i.e. $\mathbb{H}^{5}\times \mathbb{H}^{5}$, $\mathbb{H}^{5}\times \mathbb{S}^{5}$ and $\mathbb{S}^{5}\times \mathbb{S}^{5}$). Besides, five variants of our model are implemented with different time dimension in $[1,3,5,7,10]$ for comparison.

\begin{wraptable}{r}{0.59\linewidth}
\vspace{-0.6cm}
\centering
 \caption{The graph reconstruction results in mAP (\%), top three results are highlighted. Standard deviations are relatively small (in range $[0, 1.2 \times 10^{-2}]$) and are omitted.}
    \resizebox{\linewidth}{!}{
    \begin{tabular}{ccccc}
\hline\noalign{\smallskip} 
Model & Web-Edu & Power & Bio-Worm & Facebook \\
\hline\noalign{\smallskip} 
Curvature & -0.6 & -0.3 & 0.0 & 0.1 \\
\hline\noalign{\smallskip}
GCN ($\mathbb{E}^{10}$)  & 83.66 & 86.61 & 90.19  & 81.73\\
HGCN ($\mathbb{H}^{10}$) & 88.33 & 93.80  & 93.12 & 83.40\\
GCN ($\mathbb{S}^{10}$)  & 82.72 & 92.73   & 88.98  & 81.04\\
\hline\noalign{\smallskip}
$\kappa$-GCN ($\mathbb{H}^{5}\times \mathbb{H}^{5}$)     & 89.21 & 94.40 & 94.00  & 84.94\\
$\kappa$-GCN ($\mathbb{S}^{5}\times \mathbb{S}^{5}$)     & 86.70 & 94.58 & 90.36  & 84.56\\
$\kappa$-GCN ($\mathbb{H}^{5}\times \mathbb{S}^{5}$)    & 87.96 & \textcolor{cyan}{\textbf{95.82}}  & 94.74  & 87.73\\
\hline
$\mathcal{Q}$-GCN ($\mathcal{Q}^{9,1}$)                 & 87.03  & 94.35 & 92.83  & 81.60              \\
$\mathcal{Q}$-GCN ($\mathcal{Q}^{7,3}$) & \textcolor{red}{\textbf{99.67}}  & \textcolor{red}{\textbf{100.00}}     & \textcolor{red}{\textbf{97.23}} & \textcolor{blue}{\textbf{87.74}}      \\
$\mathcal{Q}$-GCN ($\mathcal{Q}^{5,5}$)                 & \textcolor{blue}{\textbf{98.49}}   & \textcolor{red}{\textbf{100.00}}    & \textcolor{blue}{\textbf{95.75}}  & \textcolor{cyan}{\textbf{87.03}}   \\
$\mathcal{Q}$-GCN ($\mathcal{Q}^{3,7}$)                 & \textcolor{cyan}{\textbf{97.31}}  & 95.08  & \textcolor{cyan}{\textbf{90.14}} & \textcolor{red}{\textbf{91.75}}        \\
$\mathcal{Q}$-GCN ($\mathcal{Q}^{0,10}$)                & 82.57  & 94.20  & 88.67 & 83.81         \\
\noalign{\smallskip}\hline
\end{tabular}
    }
\label{tab:map}
\end{wraptable}

\textbf{Experimental settings.}
We use one-hot embeddings as initial node features following \cite{DBLP:conf/iclr/GuSGR19, law2020ultrahyperbolic, bachmann2020constant}. To avoid the time dimensions being $\mathbf{0}$, we uniformly perturb each dimension with a small random value in the interval $[-\epsilon,\epsilon]$, where $\epsilon=0.02$ in practice.
In addition, the same $10$-dimensional embedding and $2$ hidden layers are used for all baselines to ensure a fair comparison. 
The learning rate is set to $0.01$, the learning rate of curvature is set to $0.0001$. 
$\mathcal{Q}$-GCN is implemented with the Adam optimizer. We repeat the experiments $10$ times via different random seeds influencing weight initialization and data batching.

\textbf{Results.}\quad 
Table \ref{tab:map} shows the mean average precision (mAP) \cite{DBLP:conf/iclr/GuSGR19} results of graph reconstruction on four datasets.
It shows that $\mathcal{Q}$-GCN achieves the best performance across all benchmarks compared with both Riemannian space and product manifolds. 
We observe that by setting proper signatures, the product spaces perform better than a single geometry. It is consistent with our statement that the expression power of a single view geometry is limited. 
Specifically, all the top three results are achieved by $\mathcal{Q}$-GCN, with one exception on Power where $\mathbb{H}^5 \times \mathbb{S}^5$ achieved the third-best performance.
More precisely, for datasets that have smaller graph sectional curvature like Web-Edu, Power and Bio-Worm, $\mathcal{Q}^{7,3}$ perform the best, while $\mathcal{Q}^{3,7}$ perform the best on Facebook with positive sectional curvature. We conjecture that the number of time dimensions controls the geometry of the pseudo-hyperboloid. 
We find that the graphs with more hierarchical structures are inclined to be embedded with fewer time dimensions. By analyzing the sectional curvature in Fig.~\ref{fig:secs}, we find that this makes sense as the mean sectional curvature of Power, Bio-Wormnet and Web-Edu are negative while it is negative for Facebook. Such results give us an intuition to determine the best time dimension based on the geometric properties of graphs.

\subsubsection{Node classification and link prediction}
\textbf{Datasets and baselines.}\quad
We consider four benchmark datasets: Airport, Pubmed, Citeseer and Cora, where Airport is airline networks, Pubmed, Citeseer and Cora are three citation networks. We observe that the graph sectional curvatures of the four datasets are consistently negative without significant differences in Fig.~\ref{fig:secs}, hence we report the additional $\delta$-hyperbolicity for comparison in Table~\ref{tab:lp-nc}.  
GCN \cite{kipf2016semi}, GAT \cite{velivckovic2017graph}, SAGE \cite{hamilton2017inductive} and SGC \cite{wu2019simplifying} are used as Euclidean GCN counterparts. For non-Euclidean GCN baselines, we compare HGCN \cite{DBLP:conf/nips/ChamiYRL19} and $\kappa$-GCN \cite{bachmann2020constant} with its three variants as explained before. For $\mathcal{Q}$-GCN, we empirically set the time dimension as $[1,2,3,14,15,16]$ as six variants since these settings best reflect the geometric properties of hyperbolic and spherical space, respectively. 

\textbf{Experimental settings.}\quad
For node classification, we use the same dataset split as \cite{yang2016revisiting} for citation datasets, where $20$ nodes per class are used for training, and $500$ nodes are used for validation and $1000$ nodes are used for testing. For Airport, we split the dataset into $70/15/15$. 
For link prediction, the edges are split into $85/5/10$ percent for training, validation and testing for all datasets. 
To ensure a fair comparison, we set the same $16$-dimension hidden embedding, $0.01$ initial learning rate and $0.0001$ learning rate for curvature. The optimal regularization with weight decay, dropout rate, the number of layers and activation functions are obtained by grid search for each method. 
We report the mean accuracy over $10$ random seeds influencing weight initialization and batching sequence. 

\textbf{Results.}\quad
Table \ref{tab:lp-nc} shows the averaged ROC AUC for link prediction, and F1 score for node classification. As we can see from the $\delta$-hyperbolicity, Airport and Pubmed are more hierarchical than CiteSeer and Cora.
For Airport and Pubmed with dominating hierarchical properties (lower $\delta$), $\mathcal{Q}$-GCNs with fewer time dimensions achieve the results on par with hyperbolic space based methods such as HGCN \cite{DBLP:conf/nips/ChamiYRL19}, $\kappa$-GCN ($\mathbb{H}^{16}$). 
While for CiteSeer and Cora with less tree-like properties (higher $\delta$), $\mathcal{Q}$-GCNs achieve the state-of-the-art results, showcasing the flexibility of our model to embed complex graphs with different curvatures. 
More specifically, $\mathcal{Q}$-GCN with more time dimensions consistently performs best on Cora. While for CiteSeer, albeit $\mathcal{Q}$-GCN achieves the best results, the corresponding best variants are not consistent on both tasks.

\begin{table}[t!]
\centering
\caption{ROC AUC (\%) for Link Prediction (LP) and F1 score for Node Classification (NC).}\label{tab:lp-nc} 
\resizebox{\columnwidth}{!}{\begin{tabular}{ccccccccc}
\hline\noalign{\smallskip}
Dataset & \multicolumn{2}{c}{Airport} & \multicolumn{2}{c}{Pubmed} & \multicolumn{2}{c}{CiteSeer} & \multicolumn{2}{c}{Cora}\\
$\delta$-hyperbolicity & \multicolumn{2}{c}{1.0} & \multicolumn{2}{c}{3.5} & \multicolumn{2}{c}{4.5} & \multicolumn{2}{c}{11.0}\\
\hline\noalign{\smallskip}
Method & LP & NC & LP & NC & LP & NC & LP & NC \\
\hline\noalign{\smallskip}
GCN & 89.24±0.21 & 81.54±0.60 & 91.31±1.68 & 79.30±0.60 & 85.48±1.75 & 72.27±0.64 & 88.52±0.85 & 81.90±0.41 \\
GAT  & 90.35±0.30 & 81.55±0.53 & 87.45±0.00 & 78.30±0.00 & 87.24±0.00 & 71.10±0.00 & 85.73±0.01 & \textcolor{cyan}{83.05±0.08} \\
SAGE & 89.86±0.52 & 82.79±0.17 & 90.70±0.07 & 77.30±0.09 & 90.71±0.20 & 69.20±0.10 & 87.52±0.22 & 74.90±0.07 \\
SGC	& 89.80±0.34 & 80.69±0.23 & 90.54±0.07 & 78.60±0.30 & 89.61±0.23 & 71.60±0.03 & 89.42±0.11 & 81.60±0.43 \\
\hline
HGCN ($\mathbb{H}^{16}$) & \textcolor{cyan}{\textbf{96.03±0.26}} & \textcolor{red}{\textbf{90.57±0.36}} & \textcolor{cyan}{\textbf{96.08±0.21}} & \textcolor{cyan}{\textbf{80.50±1.23}} & \textcolor{blue}{\textbf{96.31±0.41}} & 68.90±0.63 & 91.62±0.33 & 79.90±0.18\\
$\kappa$-GCN ($\mathbb{H}^{16}$) & \textcolor{red}{\textbf{96.35±0.62}} & \textcolor{cyan}{\textbf{87.92±1.33}} & 96.60±0.32 & 77.96±0.36 & 95.34±0.16 & \textcolor{cyan}{\textbf{73.25±0.51}} & 94.04±0.34 & 79.80±0.50 \\
$\kappa$-GCN ($\mathbb{S}^{16}$) & 90.38±0.32 & 81.94±0.58 & 94.84±0.13 & 78.80±0.49 & 95.79±0.24 & 72.13±0.51 & 93.20±0.48 & 81.08±1.45 \\
$\kappa$-GCN ($\mathbb{H}^{8}\times \mathbb{S}^{8}$) & 93.10±0.49 & 81.93±0.45 & 94.89±0.19 & 79.20±0.65 & 93.44±0.31 & 73.05±0.59 & 92.22±0.48 & 79.30±0.81 \\
\noalign{\smallskip}\hline
$\mathcal{Q}$-GCN ($\mathcal{Q}^{15,1}$) & \textcolor{blue}{\textbf{96.30±0.22}} & \textcolor{blue}{\textbf{89.72±0.52}} & 95.42±0.22 & \textcolor{cyan}{\textbf{80.50±0.26}} & 94.76±1.49 & 72.67±0.76 & 93.14±0.30 & 80.57±0.20 \\
$\mathcal{Q}$-GCN ($\mathcal{Q}^{14,2}$)  & 94.37±0.44 & 84.40±0.35 & \textcolor{red}{\textbf{96.86±0.37}} & \textcolor{red}{\textbf{81.34±1.54}} & 94.78±0.17 & \textcolor{blue}{\textbf{73.43±0.58}}  & 93.41±0.57 & 81.62±0.21 \\
$\mathcal{Q}$-GCN ($\mathcal{Q}^{13,3}$)  & 92.53±0.17 & 82.38±1.53 & \textcolor{blue}{\textbf{96.20±0.34}} & \textcolor{blue}{\textbf{80.94±0.45}} & 94.54±0.16 & \textcolor{red}{\textbf{74.13±1.41}} & 93.56±0.18 & 79.91±0.42 \\
\hline
$\mathcal{Q}$-GCN ($\mathcal{Q}^{2,14}$)  & 90.03±0.12 & 81.14±1.32 & 94.30±1.09 & 78.40±0.39 & 94.80±0.08 & 72.72±0.47 & \textcolor{cyan}{\textbf{94.17±0.38}} & \textcolor{blue}{\textbf{83.10±0.35} }\\
$\mathcal{Q}$-GCN ($\mathcal{Q}^{1,15}$)  & 89.07±0.58 & 81.24±0.34 & 94.66±0.18 & 78.11±1.38 & \textcolor{red}{\textbf{97.01±0.3}0} & \textcolor{cyan}{\textbf{73.19±1.58}} & \textcolor{blue}{\textbf{94.81±0.27}} &  \textcolor{red}{\textbf{83.72±0.43}}\\
$\mathcal{Q}$-GCN ($\mathcal{Q}^{0,16}$)  & 89.01±0.61 & 80.91±0.65 & 94.49±0.28 & 77.90±0.80 & \textcolor{blue}{\textbf{96.21±0.38}} & 72.54±0.27 & \textcolor{red}{\textbf{95.16±1.25}} & 82.51±0.32 \\
\hline
\end{tabular}}
\end{table}

\begin{figure}
\centering
\includegraphics[width=0.46\textwidth]{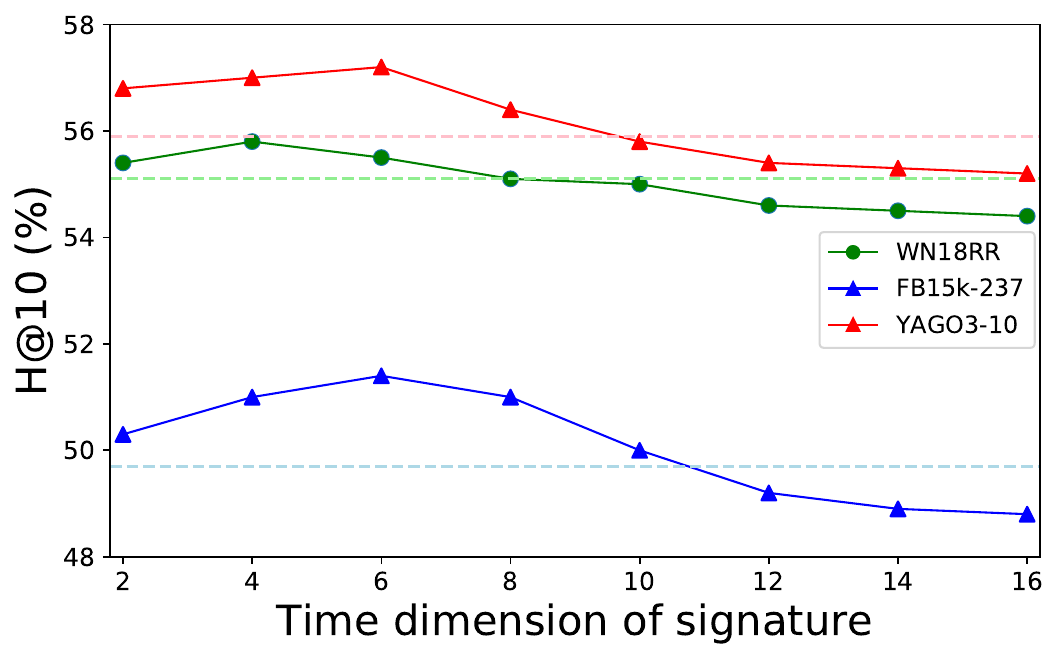}
\caption{The mAP of graph reconstruction with varying number of time dimensions.}
\label{fig:time_dimension}
\end{figure}

\subsubsection{Parameter sensitivity and analysis}\label{sec:analysis}
\textbf{Time dimension.}\quad
We study the influence of time dimension for graph reconstruction by setting varying number of time dimensions under the condition of $s+t=10$.
Fig.~\ref{fig:time_dimension} shows that the \textit{time dimension} $t$ acts as a knob for controlling the geometric properties of $\mathcal{Q}_{\beta}^{s,t}$. The best performance are achieved by neither hyperboloid ($t=1$) nor sphere cases ($t=10$), showcasing the advantages of $\mathcal{Q}_{\beta}^{s,t}$ on representing graphs of mixed topologies. 
It shows that on Web-Edu, Power and Bio-Worm with smaller mean sectional curvature, after the optimal value is reached at a lower $t$, the performance decrease as $t$ increases. 
While on Facebook with larger (positive) mean sectional curvature, as $t$ rises, the effect gradually increases until it reaches a peak at a higher $t$. 
It is consistent with our hypothesis that graphs with more hierarchical structure are inclined to be embedded in $\mathcal{Q}_{\beta}^{s,t}$ with smaller $t$, while cyclical data is aligned well with larger $t$. 
The results give us an intuition to determine the best time dimension based on the geometric properties of graphs. 

\begin{wraptable}{r}{0.59\linewidth}
\vspace{-0.7cm}
\centering
 \caption{The graph reconstruction results in mAP (\%), top three results are highlighted. Standard deviations are relatively small (in range $[0, 1.2 \times 10^{-2}]$) and are omitted.}
    \resizebox{\linewidth}{!}{
    \begin{tabular}{cccc}\\\toprule  \label{tab:ablation_QNN_QGCN}
    Method  & Pubmed & CiteSeer & Cora \\\midrule
    MLP & 72.30±0.30 & 60.22±0.42 & 55.80±0.08 \\
    HNN & 74.60±0.40 & 59.92±0.87 & 59.60±0.09 \\
    $\mathcal{Q}$-NN ($\mathcal{Q}^{15,1}$) & 74.31±0.33 & 59.33±0.35 & 60.38±0.56 \\  
    $\mathcal{Q}$-NN ($\mathcal{Q}^{14,2}$) & \textbf{76.26±0.31} & \textbf{64.33±0.35} & 62.77±0.30 \\  
    $\mathcal{Q}$-NN ($\mathcal{Q}^{13,3}$) & 75.85±0.79 & 63.65±0.57 & 59.04±0.45\\  
    $\mathcal{Q}$-NN ($\mathcal{Q}^{2,14}$) & 74.44±0.68 & 60.48±0.29 & 63.85±0.22 \\  
    $\mathcal{Q}$-NN ($\mathcal{Q}^{1,15}$) & 73.44±0.28 & 60.33±0.40 & \textbf{64.85±0.24} \\  
    $\mathcal{Q}$-NN ($\mathcal{Q}^{0,16}$) & 73.31±0.17 & 61.05±0.22 & 63.96±0.41 \\  
    \bottomrule
    \end{tabular}
    }
\label{tab:qnn}
\end{wraptable}

\textbf{$\mathcal{Q}$-NN VS $\mathcal{Q}$-GCN.}\quad 
We also introduce $\mathcal{Q}$-NN, a generalization of MLP into pseudo-Riemannian manifold, defined as multiple layers of $f(\mathbf{x})=\sigma^{\otimes}(W \otimes^{\beta} \mathbf{x} \oplus^{\beta} \mathbf{b})$, where $\sigma^{\otimes}$ is the tangential activation. 
Table \ref{tab:qnn} shows that $\mathcal{Q}$-NN with appropriate time dimension outperforms MLP and HNN on node classification, showcasing the expression power of pseudo-hyperboloid. Furthermore, compared with the results of $\mathcal{Q}$-GCN in Table \ref{tab:lp-nc}, $\mathcal{Q}$-GCN performs better than $\mathcal{Q}$-NN, suggesting that the benefits of the neighborhood aggregation equipped with the proposed GCN operations.

\begin{wraptable}{r}{0.47\linewidth}
\vspace{-0.6cm}
\centering
 \caption{The running time (seconds) of graph reconstruction on Web-Edu and Facebook.}
    \resizebox{\linewidth}{!}{
    \begin{tabular}{ccc}\\ \toprule  
        Manifolds & Web-Edu & Facebook \\\midrule
        GCN ($\mathbb{E}^{10}$) & 2284 & 5456 \\  
        Prod-GCN ($\mathbb{H}^{5} \times \mathbb{S}^{5}$) & 4336 & 10338 \\  
        $\mathcal{Q}$-GCN ($\mathcal{Q}^{9,1}$) & 2769 & 6981 \\  
        $\mathcal{Q}$-GCN ($\mathcal{Q}^{7,3}$) & 3363 & 6303\\  
        $\mathcal{Q}$-GCN ($\mathcal{Q}^{5,5}$) & 3620 & 7142 \\  
        $\mathcal{Q}$-GCN ($\mathcal{Q}^{3,7}$) & 3685 & 7512 \\  
        $\mathcal{Q}$-GCN ($\mathcal{Q}^{1,9}$) & 3532 & 7980 \\  
        $\mathcal{Q}$-GCN ($\mathcal{Q}^{0,10}$) & 2778 & 7037 \\  
        \noalign{\smallskip}\hline
        \end{tabular}
        }\label{tab:efficiency}
    \vspace{-0.3cm}
\end{wraptable}

\textbf{Computation efficiency.}\label{sec:computation} 
We compare the running time of $\mathcal{Q}$-GCN, GCN and Prod-GCN per epoch. Table~\ref{tab:efficiency} shows that $\mathcal{Q}$-GCN achieves higher efficiency than Prod-GCN ($\mathbb{H}^5 \times \mathbb{S}^5 $). This is mainly owing to that the component $\mathbb{R}$ in our diffeomorphic manifold ($\mathbb{S} \times \mathbb{R}$) runs faster than non-Euclidean components in $\mathbb{H}^5 \times \mathbb{S}^5$.
The additional running time mainly comes from the mapping operations and the projection from the time dimensions to $\mathbb{S}$. 
The running time grows when increasing the number of time dimensions.
Overall, albeit slower than Euclidean GCN, the running time of all variants of $\mathcal{Q}$-GCN is smaller than the twice of time in Euclidean GCN, which is within the acceptable limits.

\begin{figure}[t!]
\begin{minipage}[t]{0.3\linewidth}
\centering
\footnotesize
\begin{minipage}[t]{0.5\linewidth}
    \includegraphics[width=\linewidth]{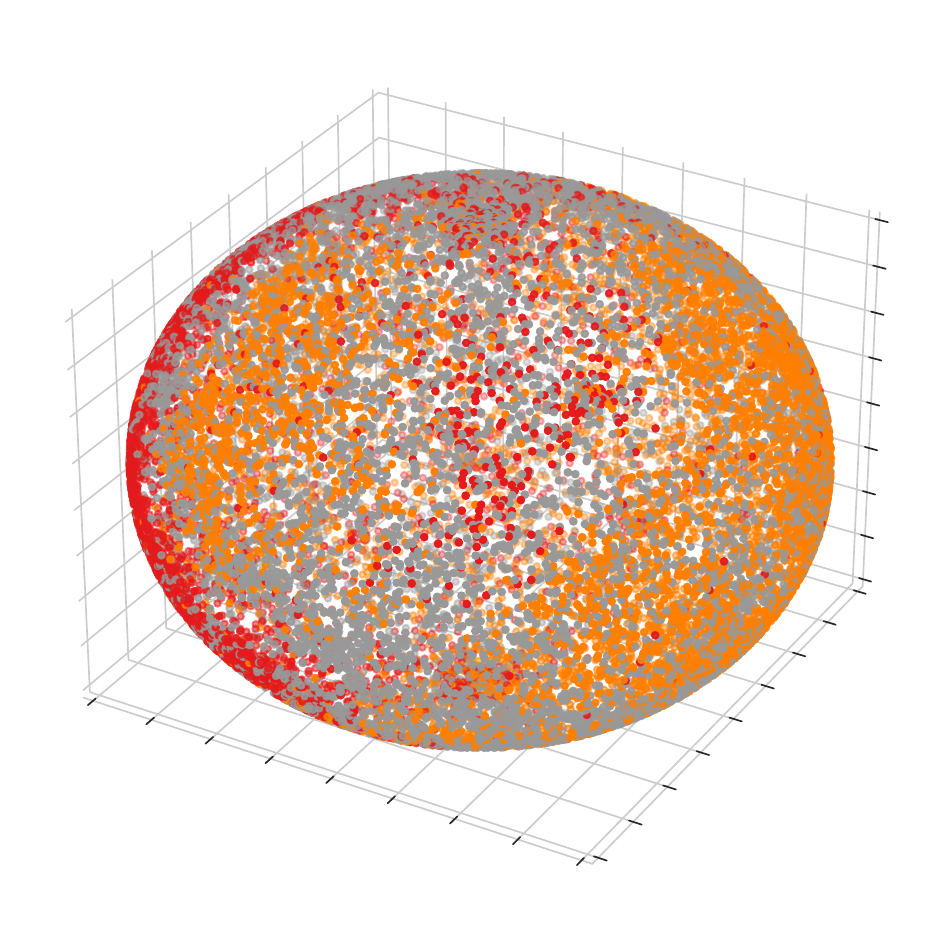}\\
\end{minipage}%
    \hfill%
\begin{minipage}[t]{0.5\linewidth}
    \includegraphics[width=\linewidth]{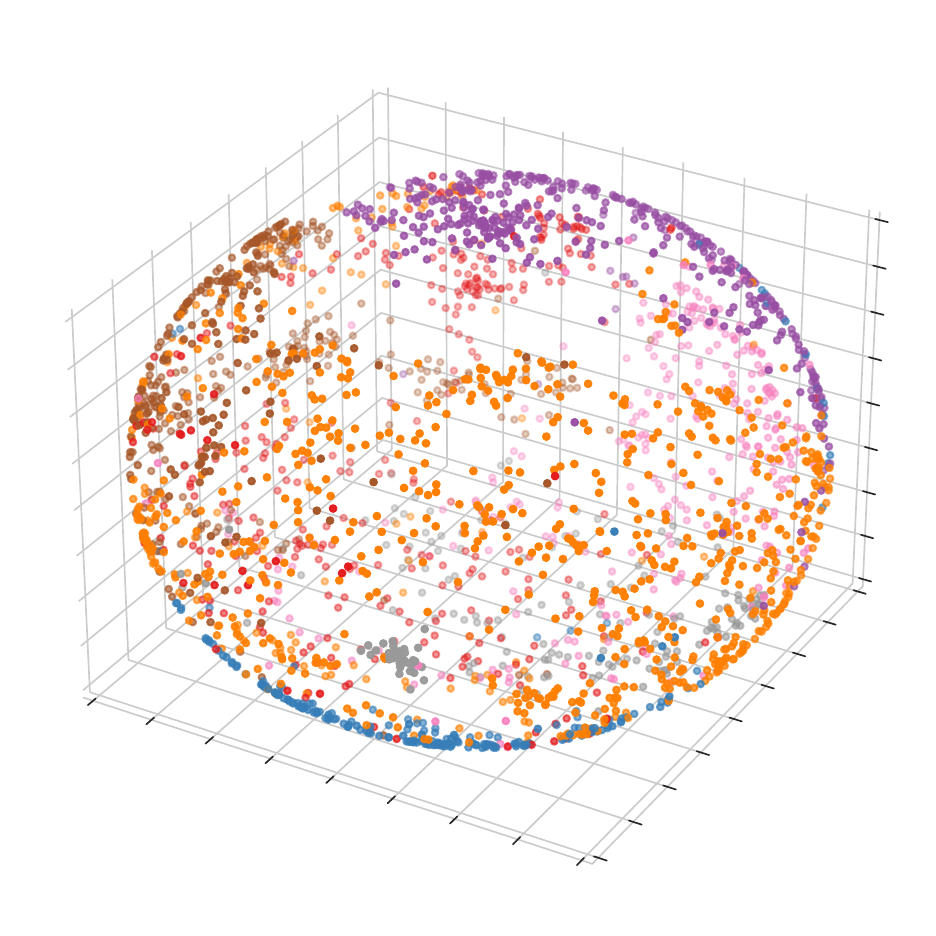}\\
\end{minipage}\\
(a) Spherical projection
\end{minipage}
    \hfill%
\begin{minipage}[t]{0.3\linewidth}
\centering
\footnotesize
\begin{minipage}[t]{0.5\linewidth}
    \includegraphics[width=\linewidth]{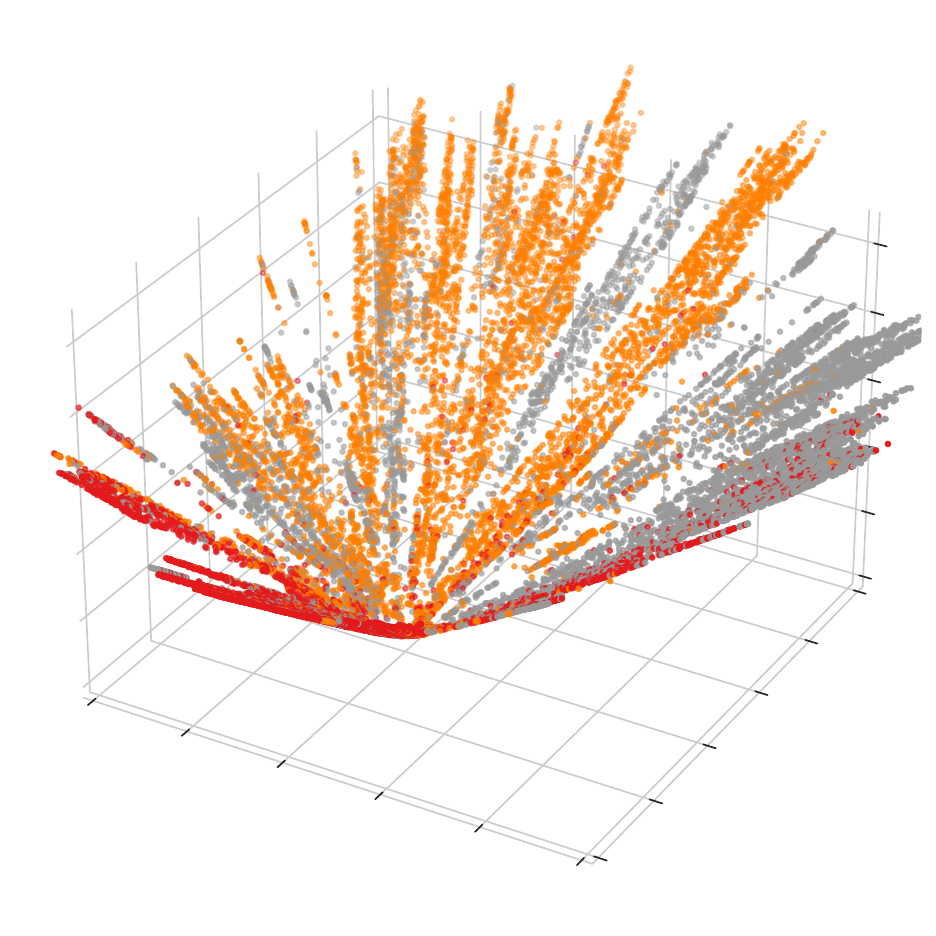}\\
\end{minipage}%
    \hfill%
\begin{minipage}[t]{0.5\linewidth}
    \includegraphics[width=\linewidth]{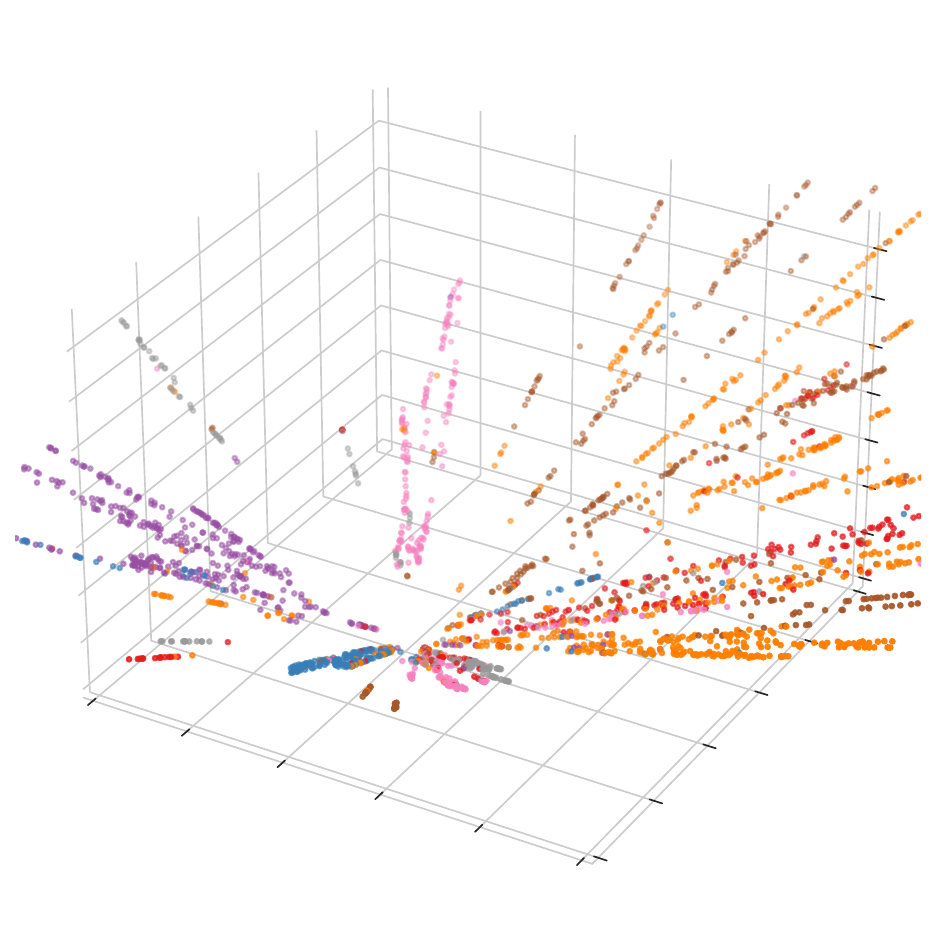}\\
\end{minipage}\\
(b) Hyperbolic projection (3D)
\end{minipage}
    \hfill%
\begin{minipage}[t]{0.3\linewidth}
\centering
\footnotesize
\begin{minipage}[t]{0.5\linewidth}
    \includegraphics[width=\linewidth]{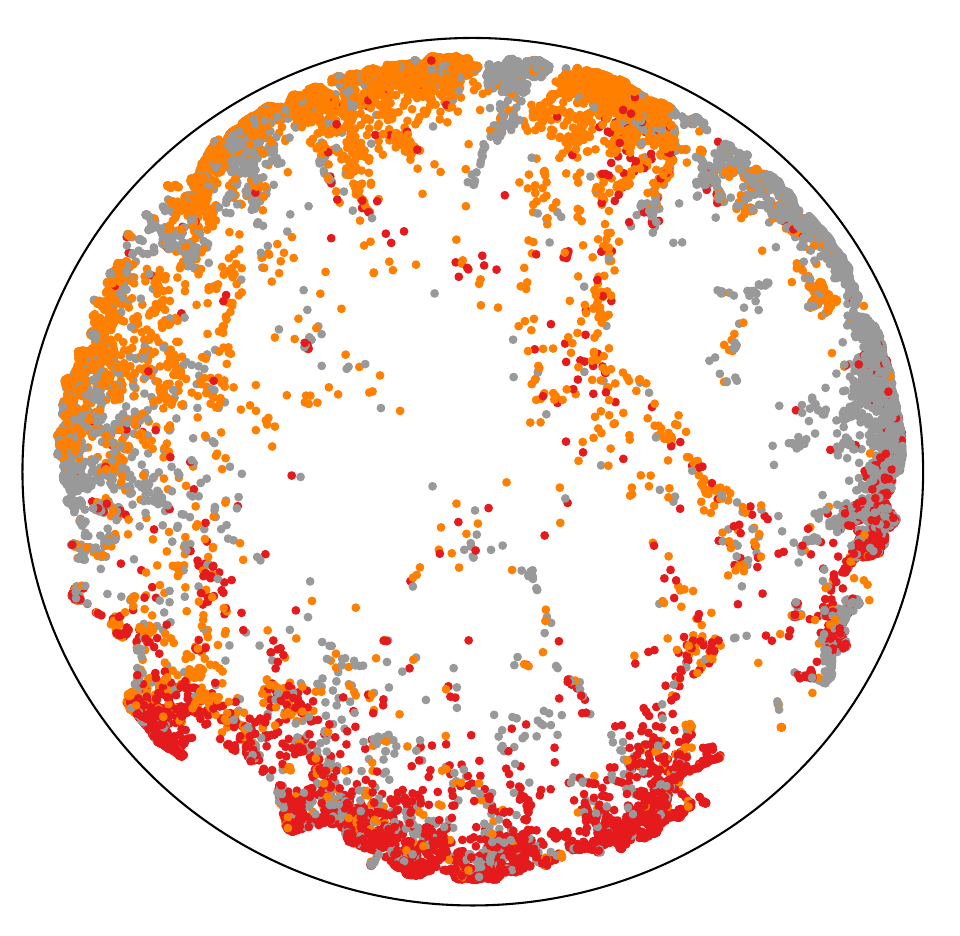}\\
\end{minipage}%
    \hfill%
\begin{minipage}[t]{0.5\linewidth}
    \includegraphics[width=\linewidth]{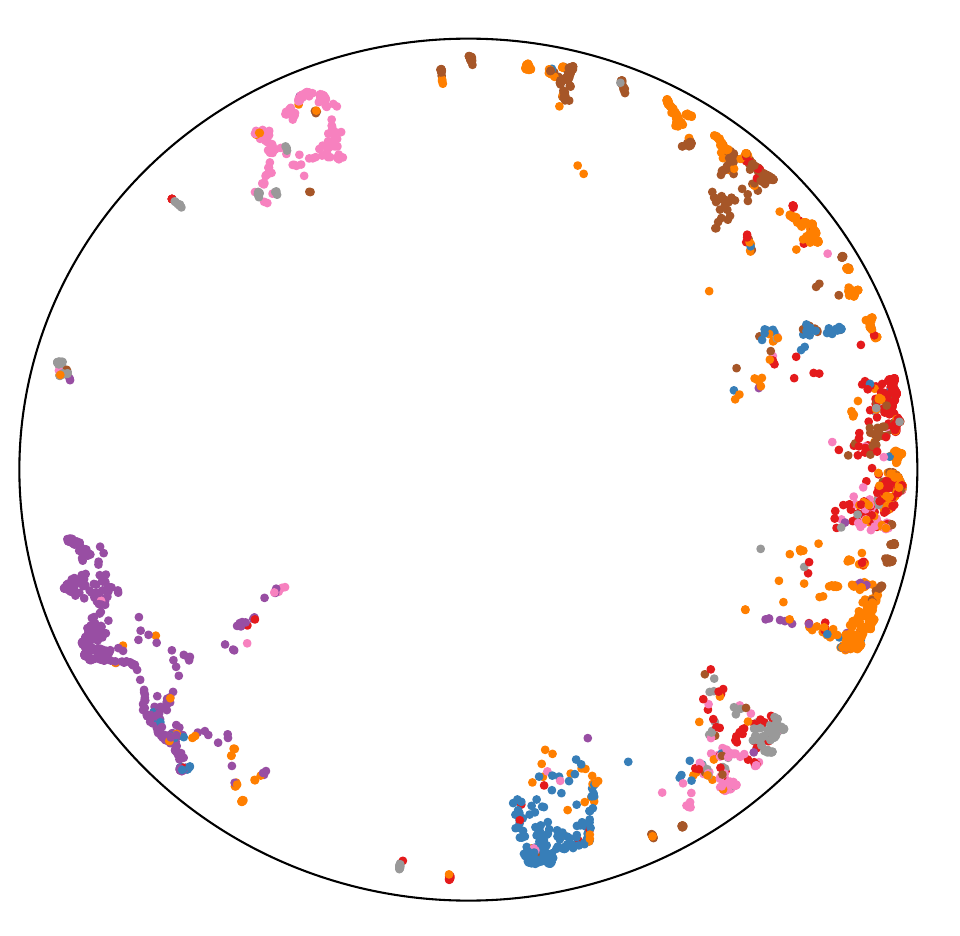}\\
\end{minipage}\\
(c) Hyperbolic projection (disk)
\end{minipage}
 \caption{Visualization of the learned embeddings for link prediction on Pubmed (left) and Cora (right), where the colors denote the class of nodes. We apply (a) spherical projection, (b) hyperbolic projection (3D) and (c) hyperbolic projection (Poincar$\acute{\text{e}}$ disk) on the learned embeddings of $\mathcal{Q}$-GCN to visualize various views of the learned embeddings.}
 \label{fig:vis}
\end{figure}

\textbf{Visualization.} To visualize the embeddings learned by $\mathcal{Q}$-GCN, we use UMAP tool \footnote{https://umap-learn.readthedocs.io/en/latest/index.html} to project the learned embeddings for Link Prediction on Pubmed and Cora into low-dimensional spherical and hyperbolic spaces.  The projections include spherical projection into 3D sphere, hyperbolic projection into 3D plane, and 2D Poincar$\acute{\text{e}}$ disk.  
As shown in Fig.~\ref{fig:vis} (a,b), for Pubmed with more hyperbolic structures, the class separability is more significant in hyperbolic projection than that is in spherical projection. 
While the corresponding result is opposite for less tree-like Cora. 
Furthermore, Fig.~\ref{fig:vis} (c) provides a more clear insight of the hierarchy. It shows that there are more hub nodes near the origin of Poincar$\acute{\text{e}}$ disk in Pubmed than in Cora, showcasing the dominating tree-likeness of Pubmed.

\subsection{Conclusion}
In this paper, we generalize GCNs to pseudo-Riemannian manifolds of constant nonzero curvature with elegant theories of diffeomorphic geometry tools. The proposed $\mathcal{Q}$-GCN have the flexibility to fit complex graphs with mixed curvatures and have shown promising results on graph reconstruction, node classification and link prediction. 
One limitation might be the choice of time dimension, we provide some insights to decide the best time dimension but this could still be improved, which we left for our future work. 
The developed geodesic tools are application-agnostic and could be extended to more deep learning methods. We foresee our work would shed light on the direction of non-Euclidean geometric deep learning. 

\section{Pseudo-Riemannian Knowledge Graph Embeddings}
\label{sec:ultrae}

Recent knowledge graph (KG) embeddings have been advanced by hyperbolic geometry due to its superior capability for representing hierarchies. The topological structures of real-world KGs, however, are rather heterogeneous, i.e., a KG is composed of multiple distinct hierarchies and non-hierarchical graph structures. Therefore, a homogeneous (either Euclidean or hyperbolic) geometry is not sufficient for fairly representing such heterogeneous structures. To capture the topological heterogeneity of KGs, we present an ultrahyperbolic KG embedding (UltraE) in an ultrahyperbolic (or pseudo-Riemannian) manifold that seamlessly interleaves hyperbolic and spherical manifolds. In particular, we model each relation as a pseudo-orthogonal transformation that preserves the pseudo-Riemannian bilinear form. The pseudo-orthogonal transformation is decomposed into various operators (i.e., circular rotations, reflections and hyperbolic rotations), allowing for simultaneously modeling heterogeneous structures as well as complex relational patterns. Experimental results on three standard KGs show that UltraE outperforms previous Euclidean, hyperbolic, and mixed-curvature KG embedding approaches.

\subsection{Background and Motivation}

Knowledge graph (KG) embeddings, which map entities and relations into a low-dimensional space, have emerged as an effective way for a wide range of KG-based applications \cite{DBLP:journals/bmcbi/CelebiUYGDD19, DBLP:conf/www/0003W0HC19,DBLP:conf/wsdm/HuangZLL19}. In the last decade, various KG embedding methods have been proposed. Prominent examples include the \emph{additive} (or \emph{translational}) family \cite{DBLP:conf/nips/BordesUGWY13,DBLP:conf/aaai/WangZFC14,DBLP:conf/aaai/LinLSLZ15} and the \emph{multiplicative} (or \emph{bilinear}) family \cite{DBLP:conf/icml/NickelTK11,DBLP:conf/nips/YangYC17,DBLP:conf/icml/LiuWY17}. 
Most of these approaches, however, are built on the Euclidean geometry that suffers from inherent limitations when dealing with hierarchical KGs such as WordNet \cite{DBLP:journals/cacm/Miller95}.
Recent studies \cite{DBLP:conf/nips/ChamiYRL19, DBLP:conf/nips/NickelK17} show that hyperbolic geometries (e.g., the Poincaré ball or Lorentz model) are more suitable for embedding hierarchical data because of their exponentially growing volumes.
Such \emph{tree-like} geometric space has been exploited in developing various hyperbolic KG embedding models such as MuRP \cite{DBLP:conf/nips/BalazevicAH19}, RotH \cite{chami2020low} and HyboNet \cite{DBLP:journals/corr/abs-2105-14686}, boosting the performance of link prediction on KGs with rich hierarchical structures and remarkably reducing the dimensionality.

Although hierarchies are the most dominant structures, the real-world KGs usually exhibit heterogeneous topological structures, e.g., a KG consists of multiple hierarchical and non-hierarchical relations.
Typically, different hierarchical relations (e.g., \textit{subClassOf} and \textit{partOf}) form distinct hierarchies, while various non-hierarchical relations (e.g., \textit{similarTo} and \textit{sisterTerm}) capture the corresponding interactions between the entities at the same hierarchy level \cite{bai2021modeling}. 
Fig.\ref{fig:example}(a) shows an example of KG consisting of a heterogeneous graph structure. 
However, current hyperbolic KG embedding methods such as MuRP \cite{DBLP:conf/nips/BalazevicAH19} and HyboNet \cite{DBLP:journals/corr/abs-2105-14686} can only model a globally homogeneous hierarchy. 
RotH \cite{chami2020low} implicitly considers the topological "heterogeneity" of KGs and alleviates this issue by learning relation-specific curvatures that distinguish the topological characteristics of different relations. 
However, this does not entirely solve the problem, because hyperbolic geometry inherently \emph{mismatches} non-hierarchical data (e.g., data with cyclic structure) \cite{DBLP:conf/iclr/GuSGR19}.

To deal with data with heterogeneous topologies, a recent work \cite{DBLP:conf/www/WangWSWNAXYC21} learns KG embeddings in a product manifold and shows some improvements on KG completion. However, such product manifold is still a homogeneous space in which all data points have the same degree of heterogeneity (i.e., hierarchy and cyclicity), while KGs require relation-specific geometric mappings, e.g., relation \textit{partOf} should be more "hierarchical" than relation \textit{similarTo}.
Different from previous works, we consider an ultrahyperbolic manifold that seamlessly interleaves the hyperbolic and spherical manifolds. Fig.\ref{fig:example} (b) shows an example of ultrahyperbolic manifold that contains multiple distinct geometries. Ultrahyperbolic manifold has demonstrated impressive capability on embedding graphs with heterogeneous topologies such as hierarchical graphs with cycles \cite{law2020ultrahyperbolic,sim2021directed}.
However, such powerful representation space has not yet been exploited for embedding KGs with heterogeneous topologies.

\begin{figure}[t!]
    \centering
    \subfloat[\centering ]{{\includegraphics[width=.62\columnwidth]{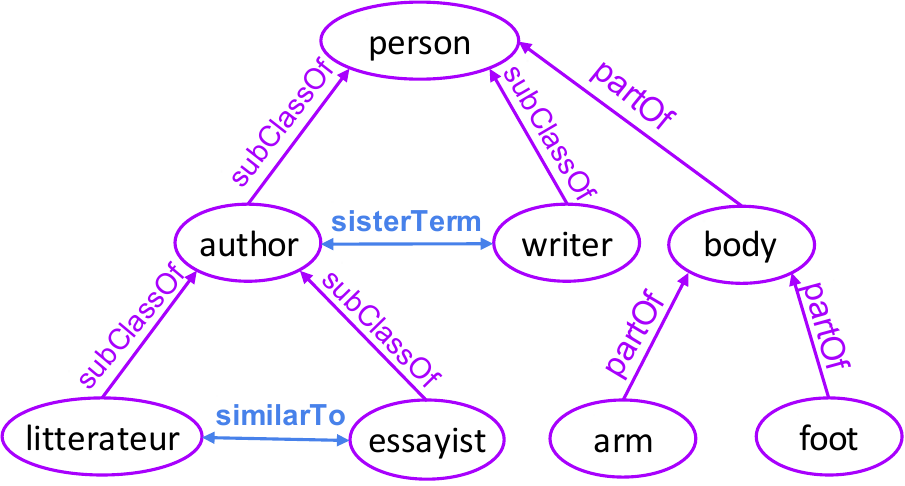}}}
    \subfloat[\centering ]{{\includegraphics[width=.34\columnwidth]{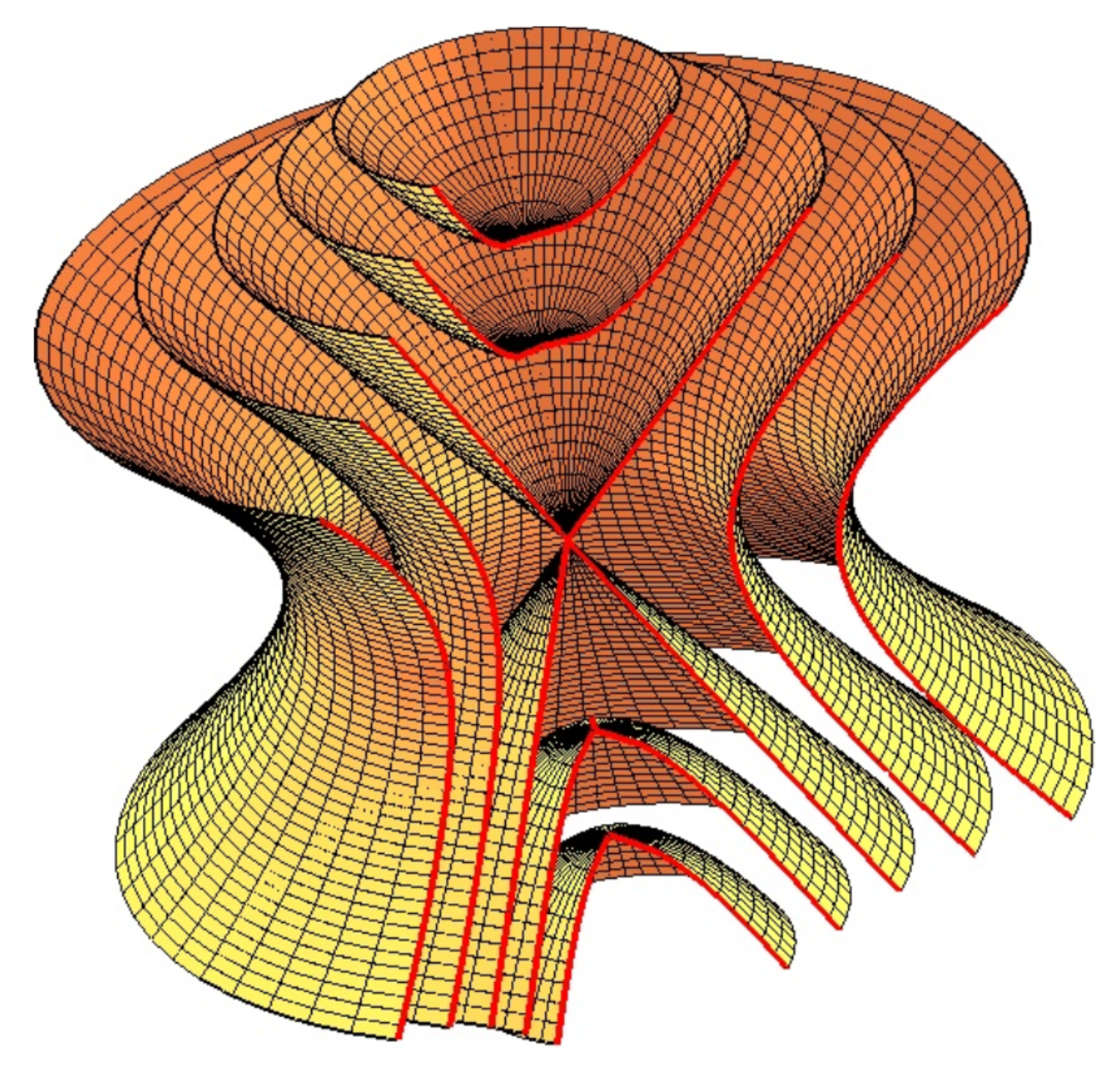} }}
    \caption{(a) A KG contains multiple distinct hierarchies (e.g., \textit{subClassOf} and \textit{partOf}) and non-hierarchies  (e.g., \textit{similarTo} and \textit{sisterTerm}). 
    (b) An ultrahyperbolic manifold generalizing hyperbolic and spherical manifolds (figure from \cite{law2020ultrahyperbolic}).} 
    \label{fig:example}
\end{figure}

In this paper, we propose ultrahyperbolic KG embeddings (UltraE), the first KG embedding method that simultaneously embeds multiple distinct hierarchical relations and non-hierarchical relations in a single but heterogeneous geometric space. 
The intuition behind the idea is that there exist multiple kinds of local geometries that could describe their corresponding  relations. For example, as shown in Fig.~\ref{fig:distance}(a), two points in the same \emph{circular conic section} are described by spherical geometry, while two points in the same half of a \emph{hyperbolic conic section} can be described by hyperbolic geometry.
In particular, we model entities as points in the ultrahyperbolic manifold and model relations as pseudo-orthogonal transformations, i.e., isometries in the ultrahyperbolic manifold. 
We exploit the theorem of hyperbolic Cosine-Sine decomposition \cite{DBLP:journals/siammax/StewartD05} to decompose the pseudo-orthogonal matrices into various geometric operations including circular rotations/reflections and hyperbolic rotations. Circular rotations/reflections allow for modeling relational patterns (e.g., composition), while hyperbolic rotations allow for modeling hierarchical graph structures. As Fig.~\ref{fig:distance}(b) shows, a combination of circular rotations/reflections and hyperbolic rotations induces various geometries including circular, elliptic, parabolic and hyperbolic geometries. 
These geometric operations are parameterized by Givens rotation/reflection \cite{chami2020low} and trigonometric functions, such that the number of relation parameters grows linearly w.r.t embedding dimensionality.
The entity embeddings are parametrized in Euclidean space and projected to the ultrahyperbolic manifold with differentiable and bijective mappings, allowing for stable optimization via standard Euclidean based gradient descent algorithms. 

\begin{figure}[t!]
    \centering
    \subfloat[\centering ]{{\includegraphics[width=.42\columnwidth]{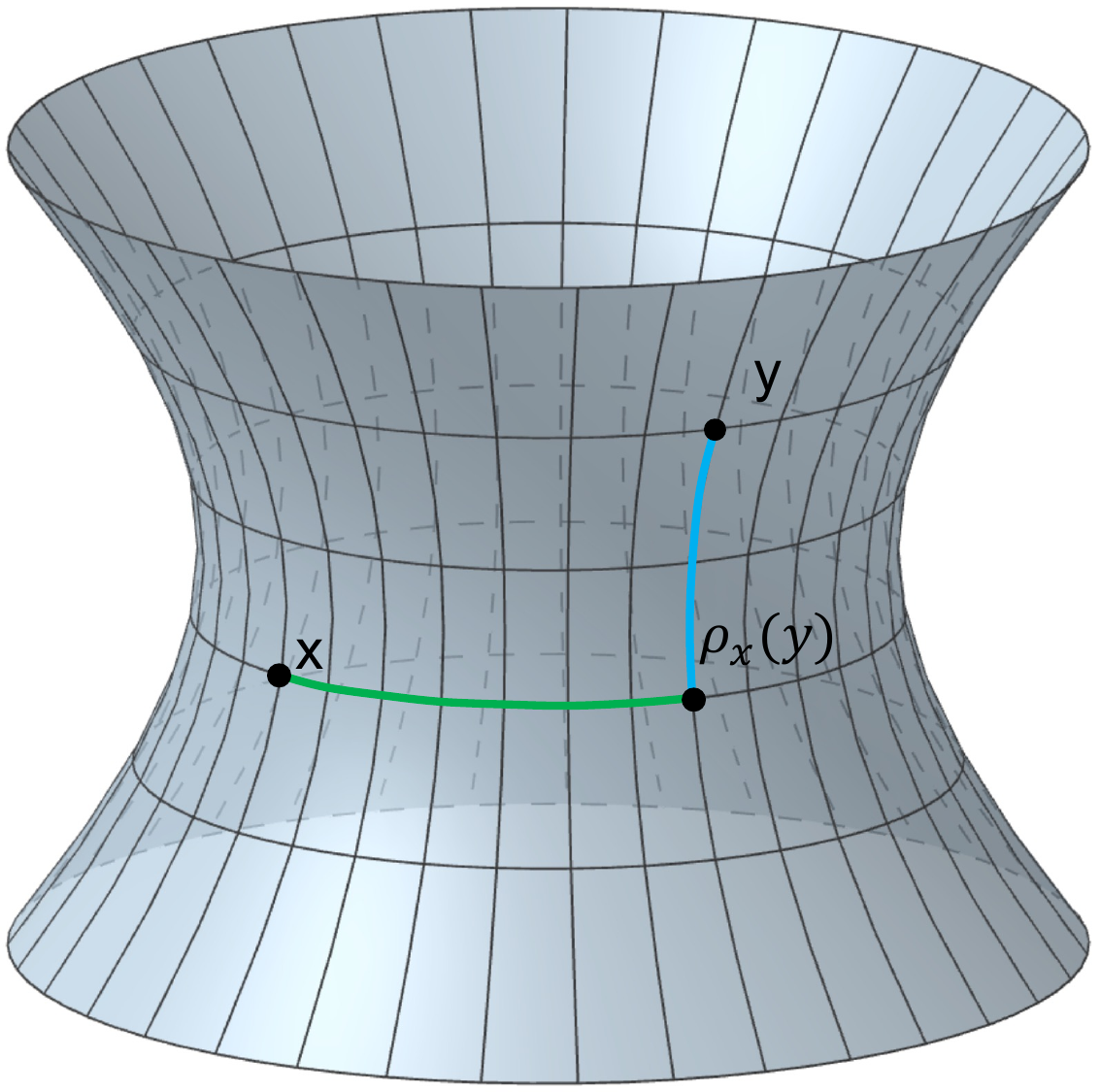} }}%
    \qquad
     \subfloat[\centering ]{{\includegraphics[width=0.48\columnwidth]{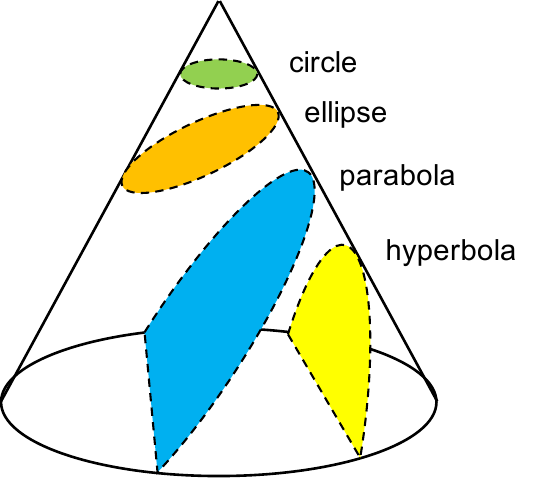}}}%
     \qquad
    \caption{(a) An illustration of a spherical (\textit{green}) geometry in the circular conic section and a hyperbolic (\textit{blue}) geometry in the hyperbolic conic section. The Manhattan-like distance of two points is defined by summing up the \emph{energy} moving from one point to another point with a circular rotation and a hyperbolic rotation. $\rho_{\mathbf{x}}(\mathbf{y})$ is a projection of $\mathbf{y}$ on an circular conic section crossing $\mathbf{x}$, such that $\rho_{\mathbf{x}}(\mathbf{y})$ and $\mathbf{x}$ are connected by a   circular rotation while $\rho_{\mathbf{x}}(\mathbf{y})$ and $\mathbf{y}$ are connected by a hyperbolic rotation. (b) An illustration of geometries covered by the circular and hyperbolic rotation, including circular, elliptic, parabolic, and hyperbolic geometries.}
    \label{fig:distance}
\end{figure}

\textbf{Contributions.} 
Our key contributions include: 1) We propose a novel KG embedding method, dubbed UltraE, that models entities in an ultrahyperbolic manifold seamlessly covering various geometries including hyperbolic, spherical and their combinations. UltraE enables modeling multiple hierarchical and non-hierarchical structures in a single but heterogeneous space; 2) We propose to decompose the relational transformation into various operators and parameterize them via Givens rotations/reflections such that the number of parameters is linear to the dimensionality. The decomposed operators allow for modeling multiple relational patterns including inversion, composition, symmetry, and anti-symmetry; 3) We propose a novel Manhattan-like distance in the ultrahyperbolic manifold, to retain the identity of indiscernibles while without suffering from the broken geodesic issues. 4) We show the theoretical connection of UltraE with some existing approaches. Particularly, by exploiting the theorem of Lorentz transformation, we identify the connections between multiple hyperbolic KG embedding methods, including MuRP, RotH/RefH and HyboNet. 5) We conduct extensive experiments on three standard benchmarks, and the experimental results show that UltraE outperforms previous Euclidean, hyperbolic and mixed-curvature (product manifold) baselines on  KG completion tasks.

\subsection{Methodology}

Let $\mathcal{E}$ and $\mathcal{R}$ denote the set of entities and relations.
A KG $\mathcal{K}$ consists of a set of triples $(h, r, t) \in \mathcal{K}$ where $h, t \in \mathcal{E}, r \in \mathcal{R}$ denote the head, the tail and their relation, respectively. 
The objective is to associate each entity with an embedding $\mathbf{e} \in \mathbb{U}^{p,q}$ in the ultrahyperbolic manifold, as well as a relation-specific transformation $f_r:\mathbb{U}^{p,q} \rightarrow \mathbb{U}^{p,q}$ that transforms one entity to another one in the ultrahyperbolic manifold.

\subsubsection{Relation as Pseudo-Orthogonal Matrix}
We propose to model relations as pseudo-orthogonal (or $J$-orthogonal) transformations \cite{DBLP:journals/siamrev/Higham03}, a generalization of \emph{orthogonal transformation} in pseudo-Riemannian geometry. 
Formally, a real, square matrix $\mathbf{Q} \in \mathbb{R}^{d \times d}$ is called $J$-orthogonal if
\begin{equation}
    \mathbf{Q}^{T} \mathbf{J} \mathbf{Q}=\mathbf{J},
\end{equation}
where
$\mathbf{J}=\left[\begin{array}{cc} 
\mathbf{I}_{p} & \mathbf{0} \\
\mathbf{0} & -\mathbf{I}_{q}
\end{array}\right], p+q=d
$ and $\mathbf{I}_p$, $\mathbf{I}_q$ are identity matrices. 
$\mathbf{J}$ is called a signature matrix of signature $(p,q)$. Such $J$-orthogonal matrices form a multiplicative group called pseudo-orthogonal group $O(p,q)$. Conceptually, a matrix $\mathbf{Q} \in O(p,q)$ is an \emph{isometry} (distance-preserving transformation) in the ultrahyperbolic manifold that preserves the bilinear form (i.e., $\forall \mathbf{x} \in \mathbb{U}^{p,q}, \mathbf{Q}\mathbf{x} \in \mathbb{U}^{p,q} $). Therefore, the matrix acts as a linear transformation in the ultrahyperbolic manifold.

There are two challenges to model relations as $J$-orthogonal transformations: 1) $J$-orthogonal matrix requires $\mathcal{O}(d^2)$ parameters. 2) Directly optimizing the $J$-orthogonal matrices results in constrained optimization, which is practically challenging within the standard gradient based framework. 

\textbf{Hyperbolic Cosine-Sine Decomposition.}
To solve these issues, we seek to decompose the $J$-orthogonal matrix by exploiting the Hyperbolic Cosine-Sine (CS) Decomposition. 
\begin{proposition}[Hyperbolic CS Decomposition \cite{DBLP:journals/siammax/StewartD05}]\label{prop:cs_decomposition}
Let $\mathbf{Q}$ be $J$-orthogonal and assume that $q \leq p.$ Then there are orthogonal matrices $\mathbf{U}_{1}, \mathbf{V}_{1} \in \mathbb{R}^{p \times p}$ and $\mathbf{U}_{2}, \mathbf{V}_{2} \in \mathbb{R}^{q \times q}$ s.t.
\begin{equation}\label{eq:cs_decomposition}
\mathbf{Q}=\left[\begin{array}{cc}
\mathbf{U}_{1} & \mathbf{0} \\
\mathbf{0} & \mathbf{U}_{2}
\end{array}\right]\left[\begin{array}{ccc}
\mathbf{C} & \mathbf{0} & \mathbf{S} \\
\mathbf{0} & I_{p-q} & \mathbf{0} \\
\mathbf{S} & \mathbf{0} & \mathbf{C} 
\end{array}\right]\left[\begin{array}{cc}
\mathbf{V}_{1}^{T} & \mathbf{0} \\
\mathbf{0} & \mathbf{V}_{2}^{T}
\end{array}\right],
\end{equation}
where $\mathbf{C}=\operatorname{diag}\left(c_1, \ldots, c_q\right), \mathbf{S}=\operatorname{diag}\left(s_1, \ldots, s_q\right)$ and $\mathbf{C}^{2}-\mathbf{S}^{2}=\mathbf{I}_q$. For cases where $q > p$, the decomposition can be defined analogously. For simplicity, we only consider $q \leq p$.
\end{proposition}

Geometrically, the $J$-orthogonal matrix is decomposed into various geometric operators. The orthogonal matrices $\mathbf{U}_{1}, \mathbf{V}_{1}$ represent circular rotation or reflection (depending on their determinant) \footnote{Depending on the determinant, a orthogonal matrix $\mathbf{U}$ denotes a rotation iff $\operatorname{det}(\mathbf{U})=1$ or a reflection iff $\operatorname{det}(\mathbf{U})=-1$} in the space dimension, while $\mathbf{U}_{2}, \mathbf{V}_{2}$ represent circular rotation or reflection in the time dimension. The intermediate matrix that is uniquely determined by $\mathbf{C},\mathbf{S}$, denotes a hyperbolic rotation (analogous to the "circular rotation") across the space and time dimensions. Fig. \ref{fig:rotation} shows a $2$-dimensional example of circular rotation and hyperbolic rotation. 

It is worth noting that both circular rotation/reflection and hyperbolic rotation are important operations for KG embeddings. 
On the one hand, circular rotations/reflections are able to model complex relational patterns including inversion, composition, symmetry, or anti-symmetry. Besides, these relational patterns usually form
some non-hierarchies (e.g., cycles). Hence, circular rotations/reflections inherently encode non-hierarchical graph structures. 
Hyperbolic rotation, on the other hand, is able to model hierarchies, i.e., by connecting entities at different levels of hierarchies. 
Therefore, this decomposition shows that $J$-orthogonal transformation is powerful for representing both relational patterns and graph structures. 

\begin{figure}[t!]
    \centering
    \includegraphics[width=0.8\columnwidth]{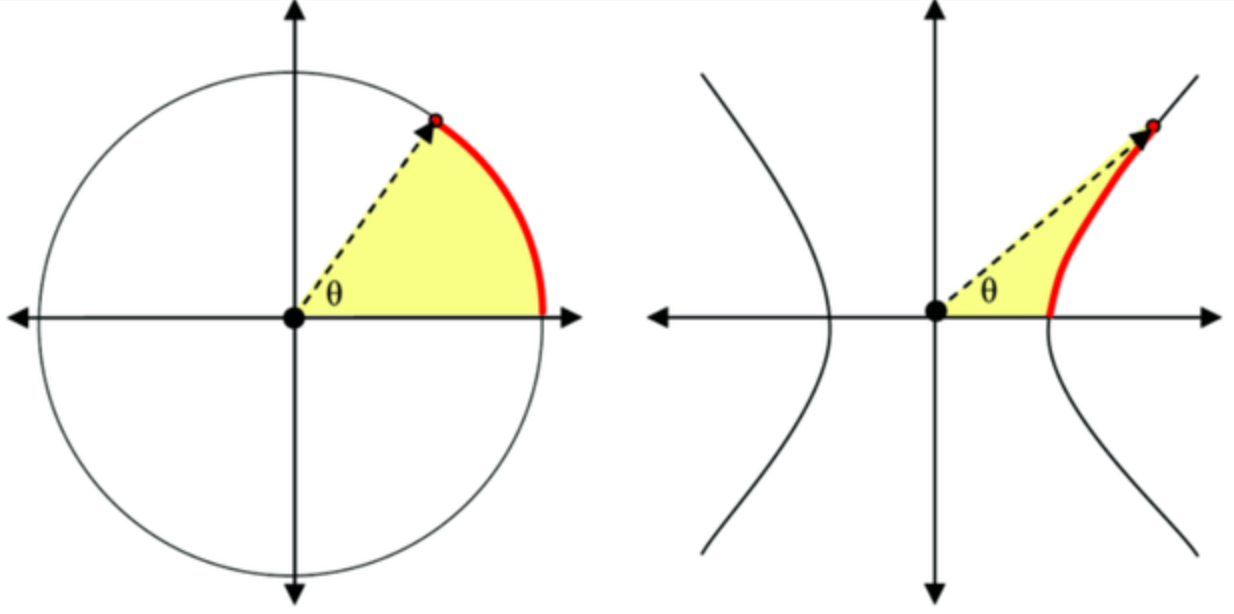}
    \caption[A two-dimensional example of circular rotation (\textit{left}) and hyperbolic rotation (\textit{right}), with $\theta$ being a angle.]{A two-dimensional example of circular rotation (\textit{left}) and hyperbolic rotation (\textit{right}), with $\theta$ being a angle.\footnotemark}
    \label{fig:rotation}
\end{figure}
\footnotetext{Source: \url{https://en.m.wikipedia.org/wiki/File:Planar_rotations.png}}



\subsubsection{Relation Parameterization}

\textbf{Circular Rotation/Reflection.} Parameterizing circular rotation or reflection via orthogonal matrices is non-trivial and there are some \emph{trivialization} approaches such as using Cayley Transform \cite{shepard2015representation}. However, such parameterization requires  $\mathcal{O}(d^2)$ parameter complexity. To simplify the complexity, we consider Given transformations denoted by $2 \times 2$ matrices. Suppose the number of dimension $p,q$ are even, circular rotation and reflection can be denoted by block-diagonal matrices of the form, given as
\begin{equation}\label{eq:rot_ref}
    \begin{array}{c}
    \operatorname{\mathbf{Rot}}\left(\Theta_{r}\right)=\operatorname{diag}\left(\mathbf{G}^{+}\left(\theta_{r, 1}\right), \ldots, \mathbf{G}^{+}\left(\theta_{r, \frac{p+q}{2}}\right)\right) \\
    \operatorname{\mathbf{Ref}}\left(\Phi_{r}\right)=\operatorname{diag}\left(\mathbf{G}^{-}\left(\phi_{r, 1}\right), \ldots, \mathbf{G}^{-}\left(\phi_{r, \frac{p+q}{2}}\right)\right) \\
    \text {where} \quad \mathbf{G}^{\pm}(\theta):=\left[\begin{array}{cc}
    \cos (\theta) & \mp \sin (\theta) \\
    \sin (\theta) & \pm \cos (\theta)
    \end{array}\right]
    \end{array},
\end{equation}
where $\Theta_{r}:=\left(\theta_{r, i}\right)_{i \in\left\{1, \ldots \frac{p+q}{2}\right\}}$ and $\Phi_{r}:=\left(\phi_{r, i}\right)_{i \in\left\{1, \ldots \frac{p+q}{2}\right\}}$ are relation-specific parameters. 

Although circular rotation is theoretically able to infer symmetric patterns \cite{DBLP:conf/emnlp/WangLLS21} (i.e., by setting rotation angle $\theta=\pi$ or $\theta=0$), circular reflection can more effectively represent symmetric relations since their second power is the identity. AttH \cite{DBLP:conf/emnlp/WangLLS21} combines circular rotations and circular reflections by using an attention mechanism learned in the tangent space, which requires additional parameters.
We also combine circular rotation and circular reflection operators but in a different way. Since the $J$-orthogonal matrix is decomposed into two rotation/reflection matrices, we set the first matrix in Eq.~(\ref{eq:cs_decomposition}) to be circular rotation while the third part to be circular reflection matrices, given by
\begin{equation}\small\label{eq:uv}
\begin{split}
\mathbf{U_{\Theta_r}}=\left[\begin{array}{cc}
\operatorname{\operatorname{\mathbf{Rot}}}\left(\Theta_{r_{p}}\right) & \mathbf{0} \\
\mathbf{0} & \operatorname{\operatorname{\mathbf{Rot}}}\left(\Theta_{r_{q}}\right)
\end{array}\right],
\end{split}
\quad
\begin{split}
\mathbf{V_{\Phi_r}}=\left[\begin{array}{cc}
\operatorname{\operatorname{\mathbf{Ref}}}\left(\Phi_{r_{p}}\right) & \mathbf{0} \\
\mathbf{0} & \operatorname{\operatorname{\mathbf{Ref}}}\left(\Phi_{r_{q}}\right)
\end{array}\right]
\end{split},
\end{equation}

Clearly, the parameterization of circular rotation and reflection in Eq.~ (\ref{eq:rot_ref}), as well as the combination of them in Eq.(\ref{eq:uv}), lose a certain degree of freedom of $J$-orthogonal transformation. However, it 1) results in a linear ($\mathcal{O}(d)$) memory complexity of relational embeddings; 2) significantly reduces the risk of overfitting; and 3) is sufficiently expressive to model complex relational patterns as well as graph structures. This is similar to many other Euclidean models, such as SimplE \cite{DBLP:conf/nips/Kazemi018}, that sacrifice some degree of freedoms of the multiplicative model (i.e., RESCAL \cite{DBLP:conf/icml/NickelTK11}) parameterized by quadratic matrices while pursuing a linearly complex, less overfitting and highly expressive relational model. 

\textbf{Hyperbolic Rotation.} 
The hyperbolic rotation matrix is parameterized by two diagonal matrices $\mathbf{C},\mathbf{S}$ that satisfy the condition $\mathbf{C}^{2}-\mathbf{S}^{2}=\mathbf{I}$. The hyperbolic rotation matrix can be seen as a generalization of the $2\times2$ hyperbolic rotation given by $\left[\begin{array}{ll}\cosh (\mu) & \sinh (\mu) \\ \sinh (\mu) & \cosh (\mu)\end{array}\right]$, where the trigonometric functions $\sinh$ and $\cosh$ are hyperbolic versions of the $\sin$ and $\cos$ functions. Clearly, it satisfies the condition $\cosh (\mu)^2-\sinh(\mu)^2=1$. Analogously, to satisfy the condition $\mathbf{C}^{2}-\mathbf{S}^{2}=\mathbf{I}$, we parameterize $\mathbf{C},\mathbf{S}$ by diagonal matrices
\begin{align}
    & \mathbf{C}(\mu)=\operatorname{diag}\left(\cosh (\mu_1), \ldots, \cosh (\mu_q)\right), \\
    & \mathbf{S}(\mu)=\operatorname{diag}\left(\sinh (\mu_1), \ldots, \sinh (\mu_q)\right),
\end{align}
where $\mu=(\mu_1, \cdots, \mu_q)$ is the parameter of hyperbolic rotation to learn. 
Therefore, the hyperbolic rotation matrix can be denoted by
\begin{equation}\scriptsize\label{eq:translation}
\mathbf{H}_{\mu_r}=\left[\begin{array}{ccc}
\operatorname{diag}\left(\cosh (\mu_{r,1}), \ldots, \cosh (\mu_{r,q})\right)  & \mathbf{0} & \operatorname{diag}\left(\sinh (\mu_{r,1}), \ldots, \sinh (\mu_{r,q})\right) \\
\mathbf{0} & I_{p-q} & \mathbf{0}, \\
\operatorname{diag}\left(\sinh (\mu_{r,1}), \ldots, \sinh (\mu_{r,q})\right) & \mathbf{0} & \operatorname{diag}\left(\cosh (\mu_{r,1}), \ldots, \cosh (\mu_{r,q})\right) 
\end{array}\right],
\end{equation}

Given the parameterization of each component, the final transformation function of relation $r$ is given by
\begin{equation}\label{eq:f_r}
    f_r=\mathbf{U}_{\theta_r} \mathbf{H}_{\mu_r} \mathbf{V}_{\Phi_r}.
\end{equation}
Notably, the combination of circular rotation/reflection and hyperbolic rotation covers various kinds of geometric transformations, including circular, elliptic, parabolic, and hyperbolic transformations (See Fig.~\ref{fig:distance}(a)). Hence, our relational embedding is able to work with all corresponding geometrical spaces.


\subsubsection{Objective and Manhattan-like Distance}

\textbf{Objective Function.} 
Given $f_r$ and entity embeddings $e$, we design a score function for each triplet $(h,r,t)$ as
\begin{equation}
s(h, r, t)=-d_{\mathbf{U}}^{2}\left(f_{r}\left(\mathbf{e}_{h}\right), \mathbf{e}_{t}\right)+b_{h}+b_{t}+\delta,
\end{equation}
where $f_{r}\left(\mathbf{e}_{h}\right) = \mathbf{U}_{\theta_r} \mathbf{H}_{\mathbf{\mu}_r} \mathbf{V}_{\Phi_r} \mathbf{e_h}$, and $\mathbf{e}_{h}, \mathbf{e}_{t} \in \mathbb{U}^{p,q}$ are the embeddings of head entity $h$ and tail entity $t$, $b_h,b_t \in \mathbf{R}^d$ are entity-specific biases, and each bias defines an entity-specific \emph{sphere of influence} \cite{DBLP:conf/nips/BalazevicAH19} surrounding the center point. $\delta$ is a global margin hyper-parameter. $d_{\mathbb{U}}(\cdot)$ is a function that quantifies the nearness/distance between two points in the ultrahyperbolic manifold. 

\textbf{Manhattan-like Distance.} 
Defining a proper distance $d_{\mathbb{U}}(\cdot)$ in the ultrahyperbolic manifold is non-trivial. Different from Riemannian manifolds that are geodesically connected, ultrahyperbolic manifolds are not geodesically connected, and there exist \emph{broken cases} in which the geodesic distance is not defined \cite{law2020ultrahyperbolic}. 
Some approximation approaches \cite{law2020ultrahyperbolic,DBLP:conf/nips/XiongZPP0S22} are proposed and satisfy some of the axioms of a classic metric (e.g., symmetric premetric). However, these distances suffer from the lack of the identity of indiscernibles, that is, one may have $\mathbf{d}_{\mathbb{U}}(\mathbf{x}, \mathbf{y})=0$ for some distinct points $\mathbf{x}\neq\mathbf{y}$. This is not a problem for metric learning that learns to preserve the pair-wise distances
\cite{law2020ultrahyperbolic,DBLP:conf/nips/XiongZPP0S22}. However, our preliminary experiments find that the geodesic distance lacks the identity of indiscernibles results in unstable and non-convergent training. We conjecture this is due to the fact that the target of KG embedding is different from the graph embedding aiming at preserving pair-wise distance. KG embedding aims at satisfying $f_r(\mathbf{e}_h) \approx \mathbf{e}_t$ for each positive triple $(h,r,t)$ while not for negative triples. Hence, we need to retain the \emph{identity of indiscernibles}, that is, $\mathbf{d}_{\mathbb{U}}(\mathbf{x}, \mathbf{y})=0 \Leftrightarrow \mathbf{x}=\mathbf{y}$. 

To address this issue, we propose a novel Manhattan-like distance function, which is defined by a composition of a spherical distance and a hyperbolic distance. 
Fig.~\ref{fig:distance} shows the Manhattan-like distance. 
Formally, given two points $\mathbf{x},\mathbf{y} \in \mathbb{U}^{p,q}$, we first define a projection $\rho_{\mathbf{x}}(\mathbf{y})$ of $\mathbf{y}$ on the \emph{circular conic section} crossing $\mathbf{x}$, such that $\rho_{\mathbf{x}}(\mathbf{y})$ and $\mathbf{x}$ share the same space dimension while $\rho_{\mathbf{x}}(\mathbf{y})$ and $y$ lie on a hyperbolic subspace. 
\begin{equation}\label{eq:dis_proj}
    \rho_{\mathbf{x}}(\mathbf{y}) = \left(\begin{array}{c}
    \mathbf{ \mathbf{x}_p } \\
    \alpha \mathbf{y}_q \frac{\|x_p\|}{\|y_p\|} 
    \end{array}\right),
\end{equation}
This projection makes that $\mathbf{x}$ and $\rho_\mathbf{x}(\mathbf{y})$ are connected by a purely space-like geodesic while $\rho_\mathbf{x}(\mathbf{y}),\mathbf{y}$ are connected by a purely time-like geodesic. 
The distance function of $\mathbb{U}$ hence can be defined as a Manhattan-like distance, i.e., the sum of the two distances, given as
\begin{align}\label{eq:distance}
     d_{\mathbb{U}}(\mathbf{x},\mathbf{y}) = \min\{& d_\mathbb{S}(\mathbf{y},\rho_\mathbf{y}(\mathbf{x})) + d_\mathbb{H}(\rho_\mathbf{y}(\mathbf{x}),\mathbf{x}), \\ 
     & d_\mathbb{S}(\mathbf{x},\rho_\mathbf{x}(\mathbf{y})) + d_\mathbb{H}(\rho_\mathbf{x}(\mathbf{y}),\mathbf{y})\},
\end{align}
where $d_\mathbb{S}$ and $d_\mathbb{H}$ are spherical and hyperbolic distances, respectively, which are well-defined and maintain the identity of indiscernibles.

\subsubsection{Optimization}
For each triplet $(h,r,t)$, we create $k$ negative samples by randomly corrupting its head or tail entity. The probability of a triple is calculated as $p=\sigma(s(h, r, t))$ where $\sigma(.)$ is a sigmoid function. We minimize the binary cross entropy loss, given as
\begin{equation}
    \mathcal{L}=-\frac{1}{N} \sum_{i=1}^{N}\left(\log p^{(i)}+\sum_{j=1}^{k} \log \left(1-\tilde{p}^{(i, j)}\right)\right),
\end{equation}
where $p^{(i)}$ and $\tilde{p}^{(i, j)}$ are the probabilities for positive and negative triplets respectively, and $N$ is the number of samples.

Notably, directly optimizing the embeddings in ultrahyperbolic manifold is challenging. The issue is caused by the fact that there exist some points that cannot be connected by a geodesic in the manifold (hence no tangent direction for gradient descent). One way to sidestep the problem is to define entity embeddings in the Euclidean space and use a \emph{diffeomorphism} to map the points into the manifold. 
In particular, we consider the following diffeomorphism. 
\begin{theorem}\label{lm:sbr}[Diffeomorphism \cite{DBLP:conf/nips/XiongZPP0S22}]
For any point $\mathbf{x} \in \mathbb{U}_{\alpha}^{p, q}$, there exists a diffeomorphism $\psi: \mathbb{U}_{\alpha}^{p, q} \rightarrow \mathbb{R}^{p} \times \mathbb{S}_{\alpha}^{q}$ that maps $\mathbf{x}$ into the product manifolds of a sphere and the Euclidean space. The mapping and its inverse are given by
\end{theorem}
\begin{equation}
    \psi(\mathbf{x})=\left(\begin{array}{c}
    \mathbf{s} \\
    \alpha \frac{\mathbf{t}}{\|\mathbf{t}\|}
    \end{array}\right), \quad \psi^{-1}(\mathbf{z})=\left(\begin{array}{c}
    \mathbf{v}\\
    \frac{\sqrt{\alpha^2+\|\mathbf{v}\|^{2}}}{\alpha} \mathbf{u}
    \end{array}\right) \text {, }
\end{equation}
where $\mathbf{x}=\left(\begin{array}{c}\mathbf{s} \\ \mathbf{t}\end{array}\right) \in \mathbb{U}_{\alpha}^{p,q}$ with $\mathbf{s} \in \mathbb{R}^{p}$ and $\mathbf{t} \in \mathbb{R}_{*}^{q}$. $\mathbf{z}=\left(\begin{array}{c}\mathbf{v} \\ \mathbf{u} \end{array}\right) \in \mathbb{R}^{p} \times \mathbb{S}_{\alpha}^{q} $ with $\mathbf{v} \in \mathbb{R}^{p}$ and $\mathbf{u} \in \mathbb{S}_{\alpha}^{q}$. 

With these mappings, any vector $\mathbf{x} \in \mathbb{R}^{p} \times \mathbb{R}_{*}^{q}$ can be mapped to $\mathbb{U}_{\alpha}^{p, q}$ by a double projection $\varphi=\psi^{-1} \circ \psi$. Note that since the diffeomorphism is differential and bijective, the canonical chain rule can be exploited to perform standard gradient descent optimization. 

\subsection{Theoretical Analysis}
In this section, we provide some theoretical analyses of UltraE and the connections with related approaches.

\subsubsection{Complexity Analysis}
To make the model scalable to the size of the current KGs and keep up with their growth, a KG embedding model should have linear time and parameter complexity \cite{DBLP:journals/corr/abs-1304-7158,DBLP:conf/nips/Kazemi018}. 
In our case, the number of relation parameters of circular rotation, circular reflection, and hyperbolic rotation grows linearly with the dimensionality given by $p+q$. The total number of parameters is then $\mathcal{O}( (N_e + N_r)\times d)$, where $N_e$ and $N_r$ are the numbers of entities and relations and $d=p+q$ is the embedding dimensionality. Similar to TransE \cite{DBLP:conf/nips/BordesUGWY13} and RotH \cite{chami2020low}, UltraE has time complexity $\mathcal{O}(d)$. 
The additional cost is proportional to the number of relations, which is usually much smaller than the number of entities.

\subsubsection{Inference Patterns and Subsumption}

Following proposition shows that UltraE can infer various relational patterns
\begin{proposition}
UltraE can infer symmetry, anti-symmetry, inversion, and composition relations.
\end{proposition}
Detailed proofs and parameter settings of these inference patterns are given in the Appendix.

UltraE has close connections with some existing hyperbolic KG embedding methods, including HyboNet \cite{DBLP:journals/corr/abs-2105-14686}, RotH/RefH \cite{chami2020low}, and MuRP \cite{DBLP:conf/nips/BalazevicAH19}. Essentially, we have the following two propositions (full derivations are in the Appendix). 
\begin{proposition}
UltraE, if parameterized by a full $J$-orthogonal matrix, generalizes HyboNet to support arbitrary signature. 
\end{proposition}
That is, HyboNet is the case of UltraE where $q=1$. 

\begin{proposition}
HyboNet subsumes MuRP and RotH/RefH. 
\end{proposition}

\subsection{Empirical Evaluation}
In this section, we evaluate the performance of UltraE on link prediction in three KGs that contain both hierarchical and non-hierarchical relations. 
We systematically study the major components of our framework and show that (1) UltraE outperforms Euclidean and non-Euclidean baselines on embedding KGs with heterogeneous topologies, especially in low-dimensional cases (Sec.~\ref{sec:low_dim}); (2) the signature of embedding space works as a knob for controlling the geometry, and hence influences the performance of UltraE (Sec.~\ref{sec:signature}); (3) UltraE is able to improve the embeddings of relations with heterogeneous topologies (Sec.~\ref{sec:topo}); and (4) the combination of rotation and reflection outperforms a single operator (Sec.~\ref{sec:operators}).

\subsubsection{Experiment Setup}

\textbf{Dataset.} We use three standard benchmarks: WN18RR \cite{DBLP:conf/nips/BordesUGWY13}, a subset of WordNet containing $11$ lexical relationships, FB15k-237 \cite{DBLP:conf/nips/BordesUGWY13}, a subset of Freebase containing general world knowledge, and YAGO3-10 \cite{DBLP:conf/cidr/MahdisoltaniBS15}, a subset of YAGO3 containing information of relationships between people. 
All three datasets contain hierarchical (e.g., \textit{partOf}) and non-hierarchical (e.g., \textit{similarTo}) relations, and some of which contain relational patterns like symmetry (e.g., \textit{isMarriedTo}). 
For each KG, we follow the standard data augmentation protocol \cite{DBLP:conf/icml/LacroixUO18} and the same train/valid/test splitting as used in \cite{DBLP:conf/icml/LacroixUO18} for fair comparision. 
Following the previous work \cite{chami2020low}, we use the global graph curvature \cite{DBLP:conf/iclr/GuSGR19} to measure the geometric properties of the datasets. The statistics of datasets are summarized in Table \ref{tab:ultrae-dataset}. As we can see, all datasets are globally hierarchical (i.e., the curvature is negative) but none of which is a pure tree structure. 
Comparatively, WN18RR is more hierarchical than FB15k-237 and YAGO3-10 since it has a smaller global graph curvature.

\begin{table}[]
    \centering
     \caption{The statistics of KGs, where $\xi_{G}$ measures the tree-likeness (the lower the $\xi_{G}$ is, the more tree-like the KG is).}
    \begin{tabular}{lcccc}
        \hline Dataset & \#entities & \#relations & \#triples & $\xi_{G}$ \\
        \hline WN18RR & $41 \mathbf{k}$ & 11 & $93 \mathbf{k}$ & $-2.54$ \\
        \hline FB15k-237 & $15 \mathbf{k}$ & 237 & $310 \mathbf{k}$ & $-0.65$ \\
        \hline YAGO3-10 & $123 \mathbf{k}$ & 37 & $1 \mathbf{M}$ & $-0.54$ \\
        \hline
        \end{tabular}
    \label{tab:ultrae-dataset}
\end{table}

\noindent
\textbf{Evaluation protocol.} 
Two popular ranking-based metrics are reported: 1) mean reciprocal rank (MRR), the mean of the inverse of the true entity ranking in the prediction; and 2) hit rate $H@K$ ($K \in \{1,3,10\}$), the percentage of the correct entities appearing in the top $K$ ranked entities. As a standard, we report the metrics in the filtered setting \cite{DBLP:conf/nips/BordesUGWY13}, i.e., when calculating the ranking during evaluation, we filter out all true triples in the training set, since predicting a low rank for these triples should not be penalized.


\noindent
\textbf{Hyperparameters.} 
For each KG, we explore batch size $\in \{500, $ $1000\}$, global margin $\in \{2,4,6,8\}$ and learning rate $\in \{3e-3,5e-3,7e-3\}$ in the validation set. The negative sampling size is fixed to $50$. The maximum number of epochs is set to $1000$. The radius of curvature $\alpha$ is fixed to $1$ since our model does not need relation-specific curvatures but is able to learn relation-specific mappings in the ultrahyperbolic manifold. The signature of the product manifold is set as the same as \cite{DBLP:conf/www/WangWSWNAXYC21}.

\noindent
\textbf{Baselines.} Our baselines are divided into two groups:
\begin{itemize}
    \item \textbf{Euclidean models}. 1) TransE \cite{DBLP:conf/nips/BordesUGWY13}, the first translational model; 2) RotatE \cite{DBLP:conf/iclr/SunDNT19}, a rotation model in a complex space; 3) DistMult \cite{yang2014embedding}, a multiplicative model with a diagonal relational matrix; 4) ComplEx \cite{DBLP:conf/icml/TrouillonWRGB16}, an extension of DisMult in a complex space; 5) QuatE \cite{DBLP:conf/aaai/CaoX0CH21}, a generalization of complex KG embedding in a hypercomplex space ; 6) 5$\star$E that models a relation as five transformation functions; 7) MuRE \cite{DBLP:conf/nips/BalazevicAH19}, a Euclidean model with a diagonal relational matrix.
    \item \textbf{Non-Euclidean models.} 1) MuRP \cite{DBLP:conf/nips/BalazevicAH19}, a hyperbolic model with a diagonal relational matrix; 2) MuRS, a spherical analogy of MuRP; 3) RotH/RefH \cite{chami2020low}, a hyperbolic embedding with rotation or reflection; 4) AttH \cite{chami2020low}, a combination of RotH and RefH by attention mechanism; 5) MuRMP \cite{DBLP:conf/www/WangWSWNAXYC21}, a generalization of MuRP in the product manifold.
\end{itemize}
We compare UltraE with varied signatures (time dimensions). 
Since for all KGs, the hierarchies are much more dominant than cyclicity and we assume that both space and time dimension are even numbers, we set the time dimension to be a relatively small value (i.e., $q=2,4,6,8$ and $q=20,40,80,160$ for low-dimension settings and high-dimension settings, respectively) for comparison. A full possible setting of time dimension with $q \leq p$ is studied in Sec.~\ref{sec:signature}. 

\subsubsection{Overall Results}

\begin{table}[]
    \centering
    \caption{Link prediction results (\%) on WN18RR, FB15k-237 and YAGO3-10 for low-dimensional embeddings ($d=32$) in the filtered setting. The first group of models are Euclidean models, the second groups are non-Euclidean models, and MuRMP is a mixed-curvature baseline. RotatE, MuRE, MuRP, RotH, RefH and AttH results are taken from \cite{chami2020low}. 
    RotatE results are reported without self-adversarial negative sampling. The best score and best baseline are in \textbf{bold} and underlined, respectively. }
    \resizebox{\columnwidth}{!}{
    \begin{tabular}[width=\textwidth]{lrrrrrrrrrrrrrrr}
    \hline & \multicolumn{4}{c}{WN18RR} & \multicolumn{4}{c}{FB15k-237} & \multicolumn{4}{c}{ YAGO3-10} \\
    Model & MRR & H@ 1 & H@3 & H@10 & MRR & H@1 & H@3 & H@10 & MRR & H@1 & H@3 & H@10 \\
    \hline 
    TransE  & 36.6 & 27.4 & 43.3 & 51.5 & 29.5 & 21.0 & 32.2 & 46.6 & - & - & - & - \\
    RotatE  & 38.7 & 33.0 & 41.7 & 49.1 & 29.0 & 20.8 & 31.6 & 45.8 & - & - & - & - \\
    ComplEx & 42.1 & 39.1 & 43.4 & 47.6 & 28.7 & 20.3 & 31.6 & 45.6 & 33.6 & 25.9 & 36.7 & 48.4 \\
    QuatE   & 42.1 & 39.6 & 43.0 & 46.7 & 29.3 & 21.2 & 32.0 & 46.0 & - & - & - & - \\
    5$\star$E     & 44.9 & 41.8 & 46.2 & 51.0 & 32.3 & 24.0 & 35.5 & 50.1 & - & - & - & - \\
    MuRE    & 45.8 & 42.1 & 47.1 & 52.5 & 31.3 & 22.6 & 34.0 & 48.9 & 28.3 & 18.7 & 31.7 & 47.8 \\
    \hline 
    MuRP    & 46.5 & 42.0 & 48.4 & 54.4 & 32.3 & 23.5 & 35.3 & 50.1 & 23.0 & 15.0 & 24.7 & 39.2  \\
    RotH & \underline{47.2} & \underline{42.8} & \underline{49.0} & \underline{55.3} & 31.4 & 22.3 & 34.6 & 49.7 & 39.3 & 30.7 & 43.5 & 55.9 \\
    RefH & 44.7 & 40.8 & 46.4 & 51.8 & 31.2 & 22.4 & 34.2 & 48.9 & 38.1 & 30.2 & 41.5 & 53.0\\
    AttH & 46.6 & 41.9 & 48.4 & 55.1 & \underline{32.4} & \underline{23.6} & \underline{35.4} & 50.1 & \underline{39.7} & \underline{31.0} & \underline{43.7} & \underline{56.6} \\
    MuRMP & 47.0 & 42.6 & 48.3 & 54.7 & 31.9 & 23.2 & 35.1 & \underline{50.2} & 39.5 & 30.8 & 42.9 & \underline{56.6} \\
    \hline 
    UltraE (q=2) & 48.1 & 43.4 & 50.0 & 55.4 & 33.1 & 24.1 & 35.5 & 50.3 & 39.5 & 31.2 & 43.9 & 56.8  \\ 
    UltraE (q=4) & \textbf{48.8} & \textbf{44.0} & \textbf{50.3} & \textbf{55.8} & 33.4 & 24.3 & 36.0 & 51.0 & 40.0 & 31.5 & 44.3 & 57.0 \\ 
    UltraE (q=6) & 48.3 & 42.5 & 49.1 & 55.5 & \textbf{33.8} & \textbf{24.7} & \textbf{36.3} & \textbf{51.4} & \textbf{40.5} & \textbf{31.8} & \textbf{44.7} & \textbf{57.2} \\ 
    UltraE (q=8) & 47.5 & 42.3 & 49.0 & 55.1 & 32.6 & 24.6 & 36.2 & 51.0 & 39.4 & 31.3 & 43.4 & 56.5 \\ 
    \hline
\end{tabular}}
\label{tab:low_dim}
\end{table}

\label{sec:low_dim}
\textbf{Low-dimensional Embeddings.} 
Following previous approaches \cite{chami2020low}, we first evaluate UltraE in the low dimensional setting ($d = 32$). Table \ref{tab:low_dim} shows the performance of UltraE and the baselines. Overall, it is clear that UltraE with varying time dimension  ($q=2,4,6,8$) improves the performance of all methods. UltraE, even with only $2$ time dimension, consistently outperforms all baselines, suggesting that the heterogeneous structure imposed by the pseudo-Riemannian geometry leads to better representations. 
In particular, the best performance of WN18RR is achieved by UltraE ($q=4$) while the best performances of FB15k-237 and YAGO3-10 are achieved by UltraE ($q=6$). We believe that this is because WN18RR is more hierarchical than FB15k-237 and YAGO3-10, validating our conjecture that the number of time dimensions controls the geometry of the embedding space. 
Besides, we observed that the mixed-curvature baseline MuRMP does not consistently improve the hyperbolic methods. We conjecture that this is because MuRMP cannot properly model relational patterns.

\begin{table}[]
    \centering
    \caption{Link prediction results (\%) on WN18RR, FB15k-237 and YAGO3-10 for high-dimensional embeddings (best for $d \in \{200,400,500\}$) in the filtered setting. RotatE, MuRE, MuRP, RotH, RefH and AttH results are taken from \cite{chami2020low}. RotatE results are reported without self-adversarial negative sampling. The best score and best baseline are in \textbf{bold} and underlined, respectively.}
    \resizebox{\columnwidth}{!}{
    \begin{tabular}[width=\columnwidth]{lrrrrrrrrrrrrrr}
   \hline & \multicolumn{4}{c}{WN18RR} & \multicolumn{4}{c}{FB15k-237} & \multicolumn{4}{c}{ YAGO3-10} \\
    Model & MRR & H@ 1 & H@3 & H@10 & MRR & H@1 & H@3 & H@10 & MRR & H@1 & H@3 & H@10 \\
    \hline 
    TransE & 48.1 & 43.3 & 48.9 & 57.0 & 34.2 & 24.0 & 37.8 & 52.7 & - & - & - & - \\
    DistMult & 43.0 & 39.0 & 44.0 & 49.0 & 24.1 & 15.5 & 26.3 & 41.9 & 34.0 & 24.0 & 38.0 & 54.0 \\
    RotatE & 47.6 & 42.8 & 49.2 & 57.1 & 33.8 & 24.1 & 37.5 & 53.3 & 49.5 & 40.2 & 55.0 & 67.0 \\
   ComplEx & 48.0 & 43.5 & 49.5 & 57.2 & 35.7 & 26.4 & 39.2 & 54.7 & 56.9 & 49.8 & 60.9 & 70.1\\
    QuatE  & 48.8 & 43.8 & 50.8 & 58.2 & 34.8 & 24.8 & 38.2 & 55.0 & - & - & - & - \\
    5$\star$E  & \underline{50.0} & \underline{\textbf{45.0}} & 51.0 & \underline{59.0} & \underline{\textbf{37.0}} & \underline{\textbf{28.0}} & \underline{\textbf{40.0}} & 56.0 & - & - & - & -  \\
    MuRE & 47.5 & 43.6 & 48.7 & 55.4 & 33.6 & 24.5 & 37.0 & 52.1 & 53.2 & 44.4 & 58.4 & 69.4 \\
    \hline 
    MuRP & 48.1 & 44.0 & 49.5 & 56.6 & 33.5 & 24.3 & 36.7 & 51.8 & 35.4 & 24.9 & 40.0 & 56.7 \\
    RotH & 49.6 & 44.9 & \underline{51.4} & 58.6 & 34.4 & 24.6 & 38.0 & 53.5 & 57.0 & 49.5 & 61.2 & 70.6 \\
    RefH & 46.1 & 40.4 & 48.5 & 56.8 & 34.6 & 25.2 & 38.3 & 53.6 & \underline{57.6} & \underline{50.2} & \underline{61.9} & \underline{\textbf{71.1}} \\
    AttH & 48.6 & 44.3 & 49.9 & 57.3 & 34.8 & 25.2 & 38.4 & 54.0 & 56.8 & 49.3 & 61.2 & 70.2 \\
    MuRMP & 48.1 & 44.1 & 49.6 & 56.9 & 35.8 & 27.3 & 39.4 & \underline{56.1} & 49.5 & 44.8 & 59.1 & 69.8 \\
    \hline 
    UltraE (q=20) & 48.5 & 44.2 & 50.0 & 57.3 & 34.9 & 25.1 & 38.5 & 54.1 & 56.9 & 49.5 & 61.0 & 70.3 \\ 
    UltraE (q=40) & \textbf{50.1} & \textbf{45.0} & \textbf{51.5} & \textbf{59.2} & 35.1 & 27.5 & \textbf{40.0} & 56.0 & 57.5 & 49.8 & 62.0 & 70.8 \\ 
   UltraE (q=80) & 49.7 & 44.8 & 51.2 & 58.5 & 36.8 & 27.6 & \textbf{40.0} & \textbf{56.3} & \textbf{58.0} & \textbf{50.6} & \textbf{62.3} & \textbf{71.1} \\ 
   UltraE (q=160) & 48.6 & 44.5 & 50.3 & 57.4 & 35.4 & 26.0 & 39.0 & 55.5 & 57.0 & 49.5 & 61.8 & 70.5 \\ 
    \hline
\end{tabular}
}
\label{tab:high_dim}
\end{table}

\textbf{High-dimensional Embeddings} 
Table \ref{tab:high_dim} shows the results of link prediction in high dimensions (best for $d \in \{200,400,500\}$). Overall, UltraE achieves either better or competitive results  against a variety of other models. 
In particular, we observed that there is no significant performance gain among hyperbolic methods and mixed-curvature methods against Euclidean-based methods. We conjecture that this is because when the dimension is sufficiently large, both Euclidean and hyperbolic geometries have sufficient ability to represent complex hierarchies in KGs. 
However, UltraE roughly outperforms all compared approaches, with the only exception of 5$\star$E achieving competitive results. Again, the performance gain is not as significant as in the low-dimension cases, which further validates the hypothesis that KG embeddings are not sensitive to the choice of embedding space with high dimensions. The additional performance gain might be obtained from the flexibility of inference of the relational patterns.

\subsubsection{Parameter Sensitivity}

\textbf{The Effect of Dimensionality.} 
To investigate the effect of dimensionality, we conduct experiments on WN18RR and compare UltraE ($q=4$) against various state-of-the-art counterparts with varying dimensionality. For a fair comparison with RotH that only considers rotation, we only use rotation for the implementation of UltraE, denoted by UltraE (Rot).
Fig. \ref{fig:total_dim} shows the results obtained by averaging over 10 runs. 
It clearly shows that the mixed-curvature method MuRMP outperforms its counterparts (MuRE, MuRP) with a single geometry, showcasing the limitation of a single homogeneous geometry on capturing the intrinsic heterogeneous structures. However, RotH performs slightly better than MuRMP, especially in high dimensionality, we conjecture that this is due to the capability of RotH on inferring relational patterns.  
UltraE achieves further improvements across a broad range of dimensions, suggesting the benefits of ultrahyperbolic manifold for modeling relation-specific geometries as well as inferring relational patterns.

\begin{figure}[t!]
    \centering
    \includegraphics[width=0.45\textwidth]{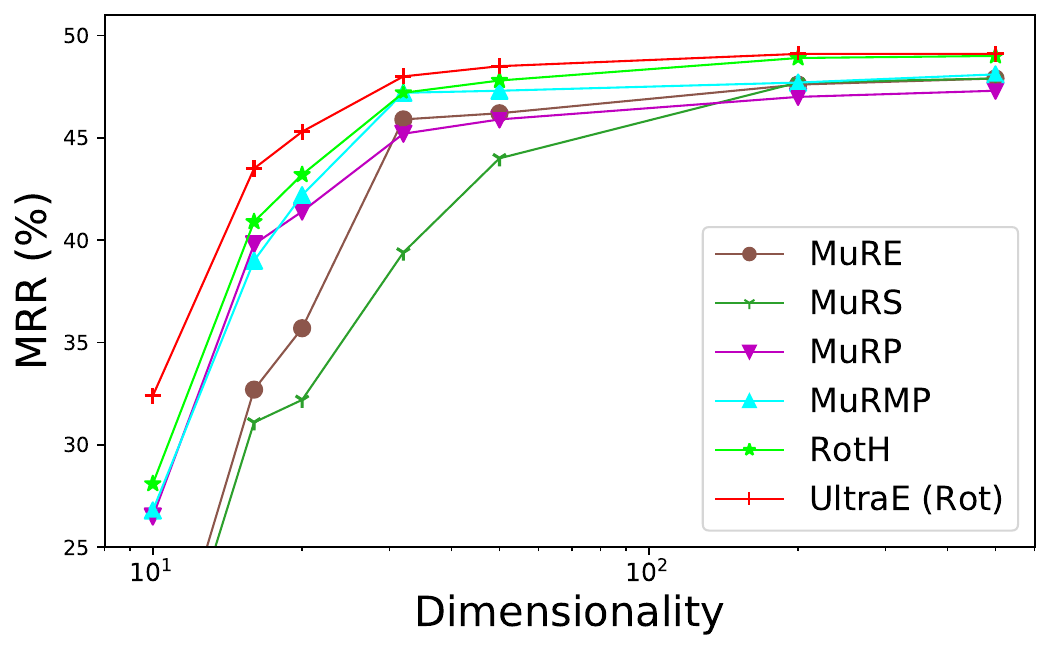}
    \caption[The MRR of various methods on WN18RR]{The MRR of various methods on WN18RR, with $d \in \{10,16,20,32,50,200,500\}$. UltraE is implemented with only rotation and $q=4$. 
    The results of MuRE, MuRS and MuRMP are taken from \cite{DBLP:conf/www/WangWSWNAXYC21} with $d \in \{10,15,20,40,100,200,500\}$. All results are averaged over 10 runs.}
    \label{fig:total_dim}
\end{figure}

\textbf{The effect of signature.}\label{sec:signature} We study the influence of the signature on WN18RR by setting a varying number of time dimensions under the condition of $d=p+q=32, p \geq q$. Fig. \ref{fig:time_dim} shows that in all three benchmarks, by increasing $q$, the performance grows first and starts to decline after reaching a peak, which is consistent with our hypothesis that the signature acts as a knob for controlling the geometric properties. One might also note that compared with hyperbolic baselines (the dashed horizontal lines), the performance gain for WN18RR is relatively smaller than those of FB15k-237 and YAGO3-10. We conjecture that this is because WN18RR is more hierarchical than FB15k-237 and YAGO3-10, and the hyperbolic embedding performs already well. This assumption is further validated by the fact that the best time dimension of WN18RR ($q=4$) is smaller than that of FB15k-237 and YAGO3-10 ($q=6$).

\begin{table}[t!]
    \centering
      \caption{Comparison of H@10 for WN18RR relations. Higher $\mathbf{Khs}_{G}$ and lower $\xi_{G}$ mean more hierarchical structure. UltraE is implemented by rotation and with best signature $(4,28)$. }
    \resizebox{\columnwidth}{!}{
    \begin{tabular}{lcccccc}
        \hline Relation & $\mathbf{Khs}_{G}$ & $\xi_{G}$ & RotE & RotH & UltraE (Rot)  \\
        \hline 
        \text {member meronym }              & 1.00 & -2.90 & 32.0 & 39.9 & 41.3   \\
        \text {hypernym }                    & 1.00 & -2.46 & 23.7 & 27.6 & 28.6   \\
        \text {has part }                    & 1.00 & -1.43 & 29.1 & 34.6 & 36.0  \\
        \text {instance hypernym }           & 1.00 & -0.82 & 48.8 & 52.0 & 53.2   \\
        \textbf {member of domain region }   & 1.00 & -0.78 & 38.5 & 36.5 & 43.3   \\
        \textbf { member of domain usage }    & 1.00 & -0.74 & 45.8 & 43.8 & 50.3  \\
        \text { synset domain topic of }      & 0.99 & -0.69 & 42.5 & 44.7 & 46.3   \\
        \text { also see }                    & 0.36 & -2.09 & 63.4 & 70.5 & 73.5   \\
        \hline
        \text { derivationally related form } & 0.07 & -3.84 & 96.0 & 96.8 & 97.1   \\
        \text { similar to }                  & 0.07 & -1.00 & 100.0 & 100.0 & 100.0  \\
        \text { verb group }                  & 0.07 & -0.50 & 97.4 & 97.4 & 98.0   \\
        \hline
    \end{tabular}}
    \label{tab:relation_type}
\end{table}

\textbf{The Effect of Relation Types.}\label{sec:topo} In this part, we investigate the per-relationship performance of UltraE on WN18RR.
Similar to RotE and RotH that only consider rotation, we consider UltraE (Rot) as before. 
Two metrics that describe the geometric properties of each relation are reported, including global graph curvature and Krackhardt hierarchy score \cite{chami2020low}, for which higher $\mathbf{Khs}_{G}$ and lower $\xi_{G}$ means more hierarchical. 
As shown in Table \ref{tab:relation_type}, although RotH outperforms RotE on most of the relation types, the performance is not on par with RotE on relations "member of domain region" and "member of domain usage". UltraE (Rot), however, consistently outperforms both RotE and RotH on all relations, with significant performance gains on relations "member of domain region " and "member of domain usage " that RotH fails on. 
The overall observation also verifies the flexibility and effectiveness of the proposed method in dealing with heterogeneous topologies of KGs.

\begin{figure}[t!]
    \centering
    \includegraphics[width=0.45\textwidth]{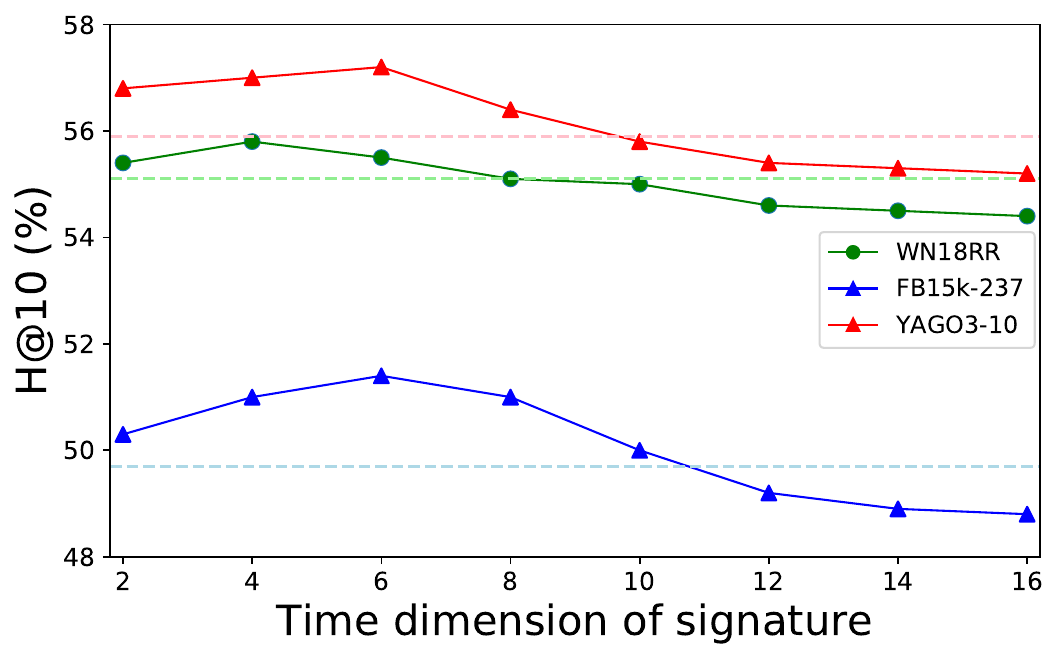}
    \caption{The performance (H@10) of UltraE with varied signature (time dimensions) under the condition of $d=p+q=32, q \leq p$ on WN18RR. The dashed horizontal lines denote the results of RotH. As $q$ increases, the performance first increases and starts to decrease after reaching a peak.}
    \label{fig:time_dim}
\end{figure}

\textbf{The Effect of Rotation and Reflection.}\label{sec:operators}
To investigate the role of rotation and reflection, we compare UltraE against its two variants: UltraE with only rotation (UltraE (Rot)) and UltraE with only reflection (UltraE (Ref)). 
Table \ref{tab:relational_pattern} shows the per-relationship results on YAGO3-10. We observe that UltraE with rotation performs better on anti-symmetric relations while UltraE with reflection performs better on symmetric relations, suggesting that reflection is more suitable for representing symmetric patterns. On almost all relations including relations that are neither symmetric nor anti-symmetric, except for "wroteMusicFor", UltraE outperforms both rotation or reflection variants, showcasing that combining multiple operators can learn more expressive representations.

\begin{table}[t!]
    \centering
     \caption{Comparison of H@10 on YAGO3-10 relations. UltraE (Rot) and UltraE (Ref) are implemented by only rotation and reflection, respectively. We choose the best signature $(6,26)$. }
    \resizebox{\columnwidth}{!}{
    \begin{tabular}{lccccc}
        \hline 
        Relation & Anti-symmetric & Symmetric & UltraE (Rot) & UltraE (Ref) & UltraE \\
        \hline 
        hasNeighbor & $\times$ & $\checkmark$ & 75.3 & \textbf{100.0} & \textbf{100.0} \\
        isMarriedTo & $\times$ & $\checkmark$ & 94.0 & 94.4 & \textbf{100.0} \\
        actedIn & $\checkmark$ & $\times$ & 14.7 & 12.7 & \textbf{15.3} \\
        hasMusicalRole & $\checkmark$ & $\times$ & 43.5 & 37.0 & \textbf{46.0} \\
        directed & $\checkmark$ & $\times$ & 51.5 & 45.3 & \textbf{56.8} \\
        graduatedFrom & $\checkmark$ & $\times$ & 26.8 & 16.3 & \textbf{27.5} \\
        playsFor & $\checkmark$ & $\times$ & 67.2 & 64.0 & \textbf{66.8} \\
        wroteMusicFor & $\checkmark$ & $\times$ & \textbf{28.4} & 18.8 & 27.9  \\
        hasCapital & $\checkmark$ & $\times$ & \textbf{73.2} & 68.3 & \textbf{73.2}  \\
        dealsWith & $\times$ &  $\times$ & 30.4 & 29.7 & \textbf{43.6} \\
        isLocatedIn & $\times$ & $\times$ & 41.5 & 39.8 & \textbf{42.8} \\
        \hline
    \end{tabular}}
    \label{tab:relational_pattern}
\end{table}

\subsection{Conclusion}
We proposes UltraE, an ultrahyperbolic KG embedding method in a pseudo-Riemannian manifold that interleaves hyperbolic and spherical geometries, allowing for simultaneously modeling multiple hierarchical and non-hierarchical structures in KGs. 
We derive a relational embedding by exploiting the pseudo-orthogonal transformation, which is decomposed into various geometric operators including circular rotations/reflections and hyperbolic rotations, allowing for inferring complex relational patterns in KGs. We propose a Manhattan-like distance that measures the nearness of points in the ultrahyperbolic manifold.
The embeddings are optimized by standard gradient descent thanks to the differentiable and bijective mapping. 
We discuss theoretical connections of UltraE with other hyperbolic methods. 
On three standard KG datasets, UltraE outperforms many previous Euclidean and non-Euclidean counterparts, especially in  low-dimensional settings.

\cleardoublepage
\chapter{Geometric Embedding of Logical Structures }
\label{chap_ontological}

Recently, increasing efforts are put into learning continual representations for symbolic knowledge bases (KBs). 
However, these approaches either only embed the data-level knowledge or suffer from inherent limitations when dealing with concept-level knowledge, i.e., they cannot faithfully model the logical structure present in the KBs. 
We present BoxEL, a geometric KB embedding approach that allows for better capturing the logical structure in the description logic $\mathcal{EL}^{++}$. 
BoxEL models concepts in a KB as axis-parallel \emph{boxes} that are suitable for modeling concept intersection, entities as points inside boxes, and relations between concepts/entities as \emph{affine transformations}. 
We show theoretical guarantees (\textit{soundness}) of BoxEL for preserving logical structure. Namely, the learned model of BoxEL embedding with loss $0$ is a (logical) model of the KB. Experimental results on (plausible) subsumption reasonings and a real-world application for protein-protein prediction show that BoxEL outperforms traditional knowledge graph embedding methods as well as state-of-the-art $\mathcal{EL}^{++}$ embedding approaches.

\section{Motivation and Background}

Knowledge bases (KBs) provide a \textit{conceptualization} of objects and their relationships, which are of great importance in many applications like biomedical and intelligent systems~\cite{rector1996galen,gene2015gene}. 
KBs are often expressed using description logics (DLs)~\cite{baader2003description}, a family of languages allowing for expressing domain knowledge via logical statements (a.k.a axioms). 
These logical statements are divided into two parts: 1) an ABox consisting of \emph{assertions} over instances, i.e., factual statements like $\rel{isFatherOf}(\insta{John}, \insta{Peter})$;
2) a TBox consisting of logical statements constraining concepts, e.g., $\concept{Parent} \sqsubseteq \concept{Person}$.

KBs not only provide clear semantics in the application domains but also enable (classic) reasoners~\cite{DBLP:journals/ws/SteigmillerLG14,DBLP:journals/jar/KazakovKS14} to perform logical inference, i.e., making implicit knowledge explicit. Existing reasoners are highly optimized and scalable but they are limited to only computing classical logical entailment but not designed to perform inductive (analogical) reasoning and cannot handle noisy data. Embedding based methods, which map the objects in the KBs into a low dimensional vector space while keeping the similarity, have been proposed to complement the classical reasoners and shown remarkable empirical performances on performing (non-classical) analogical reasonings. 

Most KB embeddings methods~\cite{wang2017knowledge} focus on embedding data-level knowledge in ABoxes, a.k.a.,\ knowledge graph embeddings (KGEs). However, KGEs cannot preserve concept-level knowledge expressed in TBoxes. 
Recently, embedding methods for KBs expressed in DLs have been explored. Prominent examples include $\mathcal{EL}^{++}$ \cite{DBLP:conf/ijcai/KulmanovLYH19} that supports conjunction and full existential quantification, and $\mathcal{ALC}$ \cite{ozccep2020cone} that further supports logical negation.
We focus on $\mathcal{EL}^{++}$, an underlying formalism of the OWL2 EL profile of the Web Ontology Language \cite{graua2008web}, which has been used in expressing various biomedical ontologies \cite{gene2015gene,rector1996galen}. For embedding $\mathcal{EL}^{++}$ KBs, several approaches such as Onto2Vec \cite{DBLP:journals/bioinformatics/SmailiGH18} and OPA2Vec \cite{smaili2019opa2vec} have been proposed. These approaches require annotation data and do not model logical structure explicitly. Geometric representations, in which the objects are associated with geometric objects such as balls \cite{DBLP:conf/ijcai/KulmanovLYH19} and convex cones \cite{ozccep2020cone}, provide a high expressiveness on embedding logical properties. For $\mathcal{EL}^{++}$ KBs, ELEm \cite{DBLP:conf/ijcai/KulmanovLYH19} represents concepts as open $n$-balls and relations as simple translations. Although effective, ELEm suffers from several major limitations: 

\begin{figure}[t!]
    \centering
    \includegraphics[width=0.9\linewidth]{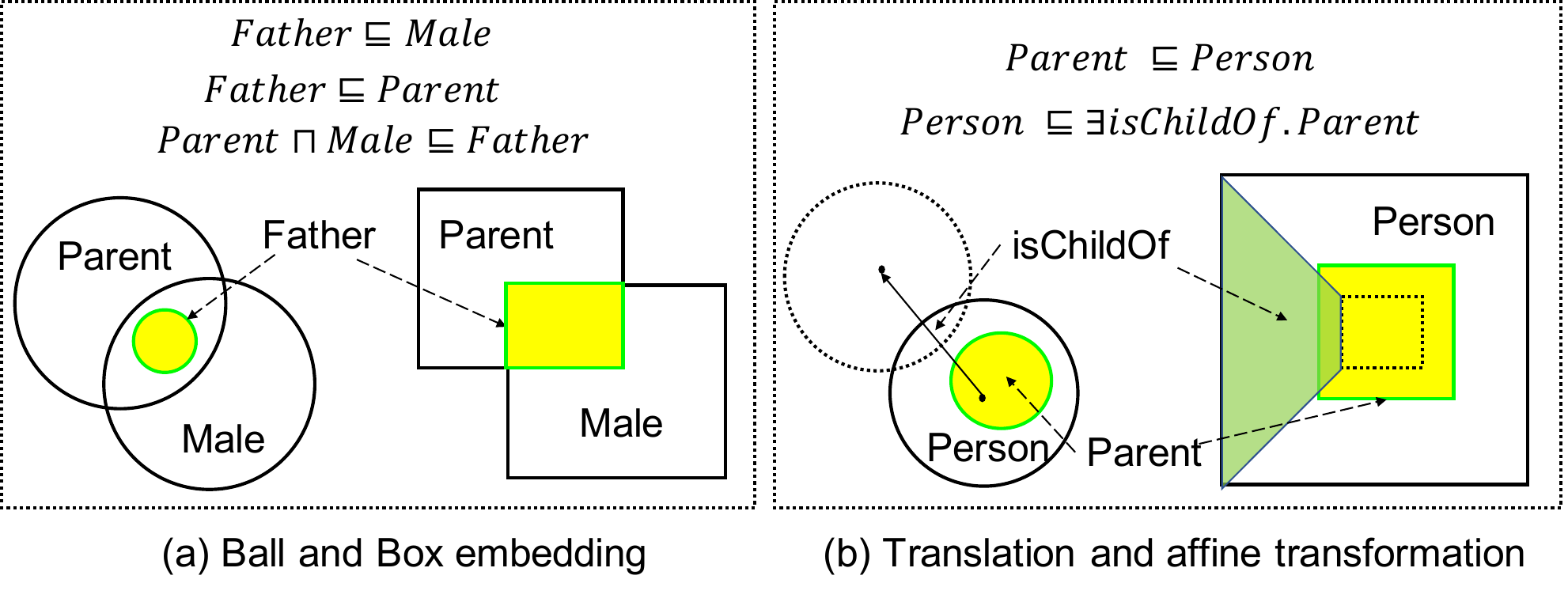}
    \caption{Two counterexamples of ball embedding and its relational transformation. (a) Ball embedding cannot express concept equivalence $\concept{Parent} \sqcap \concept{Male} \equiv \concept{Father}$ with intersection operator. (b) The \textit{translation} cannot model relation (e.g. $\rel{isChildOf}$) between $\concept{Person}$ and $\concept{Parent}$ when they should have different volumes. 
    These two issues can be solved by \textit{box} embedding and modelling relation as \textit{affine transformation} among boxes, respectively.}
    \label{fig:example_intersection}
\end{figure}

\begin{itemize}
    \item Balls are not closed under intersection and cannot faithfully represent concept intersections.
    For example, the intersection of two concepts $\concept{Parent} \sqcap \concept{Male}$, that is supposed to represent $\concept{Father}$, is not a ball (see Fig.~\ref{fig:example_intersection}(a)). 
    Therefore, the concept equivalence $\concept{Parent} \sqcap \concept{Male} \equiv \concept{Father}$ cannot be captured in the embedding.
    \item The relational embedding with simple \textit{translation} causes issues for embedding concepts with varying sizes. E.g., Fig.~\ref{fig:example_intersection}(b) shows the embeddings of the axiom  $\exists \rel{isChildOf}.\concept{Person} \sqsubseteq \concept{Parent}$ assuming the existence of another axiom $\concept{Parent} \sqsubseteq \concept{Person}$. In this case, it is impossible to \textit{translate} the larger concept $\concept{Person}$ to the smaller one $\concept{Parent}$,\footnote{Under the \textit{translation} setting, the embeddings will simply become $\concept{Parent} \equiv \concept{Person}$, which is obviously not what we want as we can express $\concept{Parent} \not\equiv \concept{Person}$ with $\mathcal{EL}^{++}$ by propositions like $\concept{Children} \sqcap \concept{Parent} \sqsubseteq \bot$, $\concept{Children} \sqsubseteq \concept{Person}$ and $\concept{Children}(a)$} as it does not allow for scaling the size.
    \item ELEm does not distinguish between entities in ABox and concepts in TBox, but rather regards ABox axioms as special cases of TBox axioms. This simplification cannot fully express the logical structure, e.g., an entity must have minimal volume. 
\end{itemize}

To overcome these limitations, we consider modeling concepts in the KB as \textit{boxes} (i.e., axis-aligned hyperrectangles), encoding entities as points inside the boxes that they should belong to, and the relations as the \textit{affine transformation} between boxes and/or points. Fig.~\ref{fig:example_intersection}(a) shows that the box embedding has closed form of intersection and the \textit{affine transformation} (Fig.~\ref{fig:example_intersection}(b)) can naturally capture the cases that are not possible in ELEm. 
In this way, we present BoxEL for embedding $\mathcal{EL}^{++}$ KBs, in which the interpretation functions of $\mathcal{EL}^{++}$ theories in the KB can be represented by the geometric transformations between boxes/points. 
We formulate BoxEL as an optimization task by designing and minimizing various loss terms defined for each logical statement in the KB. 
We show theoretical guarantee (\textit{soundness}) of BoxEL in the sense that if the loss of BoxEL embedding is $0$, then the trained model is a (logical) model of the KB.
Experiments on (plausible) subsumption reasoning over three ontologies and predicting protein-protein interactions show that BoxEL outperforms previous approaches.

\section{BoxEL for Embedding $\mathcal{EL}^{++}$ Knowledge Bases}

In this section, we first present the geometric construction process of $\mathcal{EL}^{++}$ with box embedding and affine transformation, followed by a discussion of the geometric interpretation. Afterward, we describe the BoxEL embedding by introducing proper loss function for each ABox and TBox axiom. Finally, an optimization method is described for the training of BoxEL. 

\subsection{Geometric Construction} 
We consider a KB $(\abox, \tbox)$ consisting
of an ABox $\abox$ and a TBox $\tbox$ where $\tbox$ has been normalized as explained before. 
Our goal is to associate entities (or individuals) with points and concepts with boxes in $\mathbb{R}^n$ such that the axioms in the KB are respected.

To this end, we consider two functions $m_w, M_w$ parameterized
by a parameter vector $w$ that has to be learned.
Conceptually, we consider points as boxes of volume $0$. This will be helpful later to encode the meaning of axioms for points and boxes in a uniform way.
Intuitively, $m_w: N_I \cup N_C \rightarrow \mathbb{R}^n$ maps individual and concept names to the lower left corner and  $M_w: N_I \cup N_C \rightarrow \mathbb{R}^n$ maps
them to the upper right corner of the box that represents them.
For individuals $a \in N_I$, we have $m_w(a) = M_w(a)$, so that it is sufficient 
to store only one of them.
The \emph{box associated with $C$} is defined as 
\begin{equation}
    \Box_w(C) = \{x \in \mathbb{R}^n \mid m_w(C) \leq x \leq M_w(C) \},
\end{equation}
where the inequality is defined component-wise. 

Note that boxes are closed under intersection, which allows us to compute the volume of the intersection of boxes. The lower corner of the box $\Box_w(C) \cap \Box_w(D)$
is $\max(m_w(C), m_w(D))$ and the upper corner is 
$\min\left(M_w(C), M_w(D)\right)$, where minimum and maximum are taken component-wise.
The volume of boxes can be used to encode axioms in a very concise way. 
However, as we will describe later, one problem is that points have volume $0$. This does not allow distinguishing empty boxes from points. To show that our encoding
correctly captures the logical meaning of axioms, we will consider a 
\emph{modified volume} that assigns a non-zero volume to points and some empty boxes. 
The (modified) volume of a box is defined as
\begin{equation}\label{eq:modi_vol}
    \Vol(\Box_w(C)) = \prod_{i=1}^n \max(0, M_w(C)_i - m_w(C)_i + \epsilon),
\end{equation}
where $\epsilon>0$ is a small constant. A point now
has volume $\epsilon^n$. Some empty boxes can actually have arbitrarily large modified volume.
For example the 2D-box with lower corner $(0,0)$ and upper corner $(-\frac{\epsilon}{2},N)$ has volume $\frac{\epsilon\cdot N}{2}$. 
While this is not meaningful geometrically, it does not cause any problems for our encoding because we only want to ensure that boxes with zero volume are empty (and not points). In practice, we will use \textit{softplus volume} as approximation (see Sec. \ref{sec:soft}). 

We associate every role name $r \in N_r$ with an affine transformation denoted by
$T^r_w(x) = D^r_w x + b^r_w$, where $D^r_w$ is an $(n \times n)$ diagonal matrix with non-negative entries and $b^r_w \in \mathbb{R}^n$ is a vector. 
In a special case where all diagonal entries of $D^r_w$ are $1$, $T^r_w(x)$ captures translations. 
Note that relations have been represented
by translation vectors analogous to TransE in \cite{DBLP:conf/ijcai/KulmanovLYH19}. However, this necessarily means
that the concept associated with the range of a role has the same size
as its domain. This does not seem very intuitive, in particular, for
N-to-one relationships like $has\_nationality$ or $lives\_in$ that map
many objects to the same object. 
Note that $T^r_w( \Box_w(C)) = \{T^r_w(x) \mid x \in \Box_w(C)\}$ is the box
with lower corner $T^r_w( m_w(C))$ and upper corner  $T^r_w( M_w(C))$.
To show this, note that $m_w(C) < M_w(C)$ implies
$D^r_w m_w(C)\leq D^r_w M_w(C)$ because $D^r_w$ is a diagonal matrix with
non-negative entries. Hence,
$T^r_w( m_w(C)) = 
D^r_w m_w(C) + b^r_w
\leq D^r_w M_w(C) + b^r_w =
T^r_w( M_w(C))$. For $m_w(C) \geq M_w(C)$, both 
$\Box_w(C)$ and 
$T^r_w( \Box_w(C))$ are empty.

Overall, we have the following parameters:
\begin{itemize}
    \item for every individual name $a \in N_I$, we have $n$ parameters  
    for the vector $m_w(a)$ (since $m_w(a)=M_w(A)$, we have to store only one
    of $m_w$ and $M_w$),
    \item for every concept name $C \in N_C$, we have $2n$ parameters for
    the vectors $m_w(C)$ and $M_w(C)$,
    \item for every role name $r \in N_r$, we have $2n$ parameters. $n$ parameters for the diagonal elements of $D^r_w$ and $n$ parameters
    for the components of $b^r_w$.
\end{itemize}
As we explained informally before, $w$ summarizes all parameters.
The overall number of parameters in $w$ is
$n \cdot (|N_I| + 2 \cdot |N_C| +  2 \cdot |N_r|)$.

\begin{figure}[t!]
    \centering
    \includegraphics[width=\textwidth]{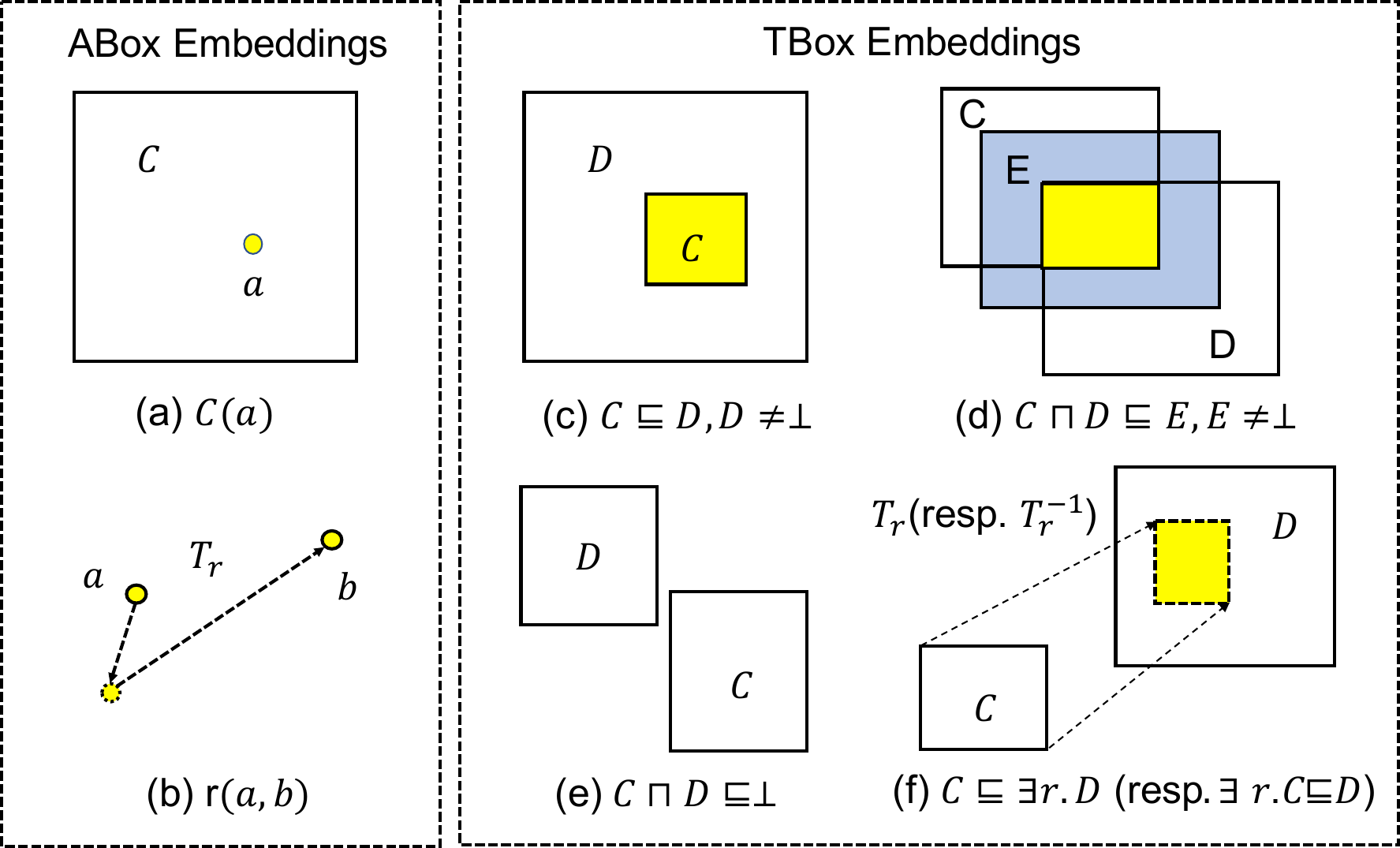}
    \caption{The geometric interpretation of logical statements in ABox (\textit{left}) and TBox (\textit{right}) expressed by DL $\mathcal{EL}^{++}$ with BoxEL embeddings. The concepts are represented by \textit{boxes}, entities are represented by \textit{points} and relations are represented by \textit{affine transformations}. $T_r$ and $T_r^{-1}$ denote the transformation function of relation $r$ and its inverse function, respectively.}
    \label{fig:box_el}
\end{figure}

\subsection{Geometric Interpretation}

The next step is to encode the axioms in our KB.
However, we do not want to do this in an arbitrary fashion, but, ideally, 
in a way that gives us some analytical guarantees.
\cite{DBLP:conf/ijcai/KulmanovLYH19} made an interesting first step by showing that their encoding is \emph{sound}.
In order to understand soundness, it is important to know that the parameters 
of the embedding are learnt by minimizing a loss function that contains a 
loss term for every axiom. Soundness then means that if the loss function
yields $0$, then the KB is satisfiable. Recall that satisfiability means
that there is an interpretation that satisfies all axioms in the KB. Ideally, we should be able to construct such an interpretation directly
from our embedding. This is indeed what the authors in \cite{DBLP:conf/ijcai/KulmanovLYH19} did. The idea is that points in the vector
space make up the domain of the interpretation, the points that lie in regions
associated with concepts correspond to the interpretation of this concept
and the interpretation of roles correspond to translations between points
like in TransE. In our context, geometric interpretation
can be defined as follows.

\begin{definition}[Geometric Interpretation]
Given a parameter vector $w$ representing 
an $\mathcal{EL}^{++}$ embedding, the corresponding \emph{geometric interpretation}  $\interpw = (\interpDomw, \interpMapw)$ is defined 
as follows:
\begin{enumerate}
    \item $\interpDomw = \mathbb{R}^n$,
    \item for every concept name $C \in N_C$, $C^{\interpw} =\Box_w(C)$,
    \item for every role $r \in N_R$, $r^{\interpw} = \{(x, y) \in \interpDomw \times \interpDomw \mid T^r_w(x) = y \}$,
    \item for every individual name $a \in N_I$, $a^{\interpw} = m_w(a)$.
\end{enumerate}
\end{definition}
We will now encode the axioms by designing one loss term for every axiom in a normalized $\mathcal{EL}^{++}$ KB, such that the axiom is satisfied by the geometric interpretation when the loss is $0$.

\subsubsection{ABox Embedding}
ABox contains concept assertions and role assertions. We introduce the following two loss terms that respect the geometric interpretations. 
\textbf{Concept Assertion} Geometrically, a concept assertion $C(a)$ asserts that the point $m_w(a)$ is inside the box $\Box_w(C)$ (see Fig.~\ref{fig:box_el}(a)). This can be expressed by demanding
$m_w(C) \leq m_w(a) \leq M_w(C)$ for every component. 
The loss $\mathcal{L}_{C(a)}(w)$ is defined by
\begin{align}\footnotesize
\mathcal{L}_{C(a)}(w) = & \sum_{i=1}^n \left\|\max(0, m_w(a)_i- M_w(C)_i)\right\|_2 \notag \\
& + \sum_{i=1}^n \left\|\max(0, m_w(C)_i- m_w(a)_i)\right\|_2.
\end{align}

\textbf{Role Assertion} Geometrically, a role assertion $r(a,b)$ means that the point $m_w(a)$ should be mapped to $m_w(b)$ by the transformation $T^r_w$ (see Fig.~\ref{fig:box_el}(b)). That is, we should have $T^r_w(m_w(a))=m_w(b)$. 
We define a loss term
\begin{equation}
    \mathcal{L}_{r(a,b) }(w) =\left\|T^r_w(m_w(a)) - m_w(b) \right\|_2.
\end{equation}
It is clear from the definition that when the loss terms are $0$, the axioms are satisfied in their geometric interpretation.
\begin{proposition}\label{prop:1}
We have 
\begin{enumerate}
    \item If $\mathcal{L}_{C(a)}(w) = 0$, then $\interpw \models C(a)$,
    \item If $\mathcal{L}_{r(a,b)}(w)=0$, then $\interpw \models r(a,b)$.
\end{enumerate}
\end{proposition}

\subsubsection{TBox Embedding}
For the TBox, we define loss terms for the four cases in
the normalized KB. Before doing so, we define an auxiliary function
that will be 
used inside these loss terms.
\begin{definition}[Disjoint measurement]
Given two boxes $B_1, B_2$, the disjoint measurement can be defined by the (modified) volumes of $B_1$ and the intersection box $B_1 \cap B_2$,
\begin{equation}
   \contains(B_1, B_2) = 1-\frac{\Vol(B_1 \cap B_2)}{\Vol(B_1)}. 
\end{equation}
\end{definition}
We have the following guarantees.
\begin{lemma}
\label{contains_lemma}
\begin{enumerate}
    \item $0 \leq \contains(B_1, B_2) \leq 1$,
    \item $\contains(B_1, B_2) = 0$ implies $B_1 \subseteq B_2$,
    \item $\contains(B_1, B_2) = 1$ implies $B_1 \cap B_2 = \emptyset$.
\end{enumerate}
\end{lemma}


\textbf{NF1: Atomic Subsumption} An axiom of the form $C \sqsubseteq D$ geometrically means that $\Box_w(C) \subseteq \Box_w(D)$ (see Fig.~\ref{fig:box_el}(c)). 
If $D\neq \bot$, we consider the loss term
\begin{equation}\label{eq:loss_nf1}
    \mathcal{L}_{C \sqsubseteq D}(w) = \contains(\Box_w(C), \Box_w(D)).
\end{equation}
For the case $D=\bot$ where $C$ is not a nominal, e.g., $C \sqsubseteq \bot$, we define the loss term 
\begin{equation}
    \mathcal{L}_{C \sqsubseteq \bot}(w) = 
    \max(0, M_w(C)_0 - m_w(C)_0 + \epsilon).
\end{equation}
If $C$ is a nominal, the axiom is inconsistent and our model can just return an error.
\begin{proposition}
If $\mathcal{L}_{C \sqsubseteq D}(w)=0$, 
then $\interpw \models C \sqsubseteq D$,
where we exclude the inconsistent case $C=\{a\}, D=\bot$.
\end{proposition}

\textbf{NF2: Conjunct Subsumption} An axiom of the form 
$C \sqcap D \sqsubseteq E$ means that $\Box(C) \cap \Box(D) \subseteq \Box(E)$
(see Fig.~\ref{fig:box_el}(d)).
Since $\Box(C) \cap \Box(D)$ is a box again, we can use the same idea as
for NF1. For the case $E\neq \bot$, we define the loss term as
\begin{equation}\label{eq:loss_nf2}
    \small
    \mathcal{L}_{C \sqcap D \sqsubseteq E}(w) = 
    \contains(\Box_w(C) \cap \Box_w(D), \Box_w(E)).
\end{equation}
For $E = \bot$, the axiom states that $C$ and $D$ must be disjoint. The disjointedness can be interpreted as the volume of the intersection of the associated boxes being $0$ (see Fig.~\ref{fig:box_el}(e)). 
However, just using the volume as a loss term may not work well
because a minimization algorithm may minimize the volume of the boxes
instead of the volume of their intersections. Therefore, we normalize the
loss term by dividing by the volume of the boxes. Given by
\begin{equation}\label{eq:loss_disjoint}
    \small
    \mathcal{L}_{C \sqcap D \sqsubseteq \bot}(w) = \frac{\Vol(\Box_w(C)\cap \Box_w(D))}{\Vol(\Box_w(C)) + \Vol(\Box_w(D))}.
\end{equation}
\begin{proposition}
If $\mathcal{L}_{C \sqcap D \sqsubseteq E}(w)=0$,
then $\interpw \models C \sqcap D \sqsubseteq E$,
where we exclude the inconsistent case ${a} \sqcap {a} \sqsubseteq \bot$ (that is, $C=D=\{a\}, E=\bot$).
\end{proposition}

\textbf{NF3: Right Existential}
Next, we consider axioms of the form $C \sqsubseteq \exists r.D$.
Note that $\exists r.D$ describes those entities that are in  
relation $r$ with an entity from $D$.
Geometrically, those are points that are mapped to points in
$\Box_w(D)$ by the affine transformation corresponding to $r$. 
$C \sqsubseteq \exists r.D$ then means that every point in
$\Box_w(C)$ must be mapped to a point in $\Box_w(D)$, that is
the mapping of $\Box_w(C)$ is contained in $\Box_w(D)$ (see Fig.~\ref{fig:box_el}(f)).
Therefore, the encoding comes again down to encoding a
subset relationship as before. The only difference to the first 
normal form is that $\Box_w(C)$ must be mapped by the affine
transformation $T^r_w$. These considerations lead
to the following loss term
\begin{equation}\label{eq:loss_nf3}
    \mathcal{L}_{C \sqsubseteq \exists r.D}(w) =
    \contains(T^r_w(\Box_w(C)), \Box_w(D)).
\end{equation}
\begin{proposition}
If $\mathcal{L}_{C \sqsubseteq \exists r.D}(w) = 0$, 
then $\interpw \models C \sqsubseteq \exists r.D$.
\end{proposition}
\par

\textbf{NF4: Left Existential}
Axioms of the form $\exists r.C \sqsubseteq D$ can be treated
symmetrically to the previous case (see Fig.~\ref{fig:box_el}(f)). We only consider the case $D\neq \bot$ and define the loss
\begin{equation}\label{eq:loss_nf4}
    \mathcal{L}_{\exists r.C \sqsubseteq D }(w) =
    \contains(T^{-r}_w(\Box_w(C)), \Box_w(D)),
\end{equation}
where $T^{-r}_w$ is the inverse function of $T^{r}_w$ that is defined by
$T^{-r}_w(x) =  D^{-r}_w x - D^{-r}_w b^r_w$, where $D^{-r}_w$ is 
obtained from $D^{r}_w$ by replacing all diagonal elements with  their
reciprocal. Strictly speaking, the inverse only exists 
if all diagonal entries of $D^{r}_w$ are non-zero. 
However, we assume that the entries that occur in a loss
term of the form $\mathcal{L}_{\exists r.C \sqsubseteq D }(w)$
remain non-zero in practice when we learn them iteratively. 
\begin{proposition}
If $\mathcal{L}_{\exists r.C \sqsubseteq D }(w) = 0$, 
then $\interpw \models \exists r.C \sqsubseteq D$.
\end{proposition}

\subsubsection{Optimization}

\textbf{Softplus Approximation}\label{sec:soft}
For optimization, while the computation of the volume of boxes is straightforward,
using a precise \emph{hard volume} is known to cause problems when learning the parameters using gradient descent algorithms, e.g. there is no training signal (gradient flow) when box embeddings that should overlap but become disjoint \cite{li2018smoothing,patel2020representing,dasgupta2020improving}.
To mitigate the problem, we approximate the volume of boxes by the \textit{softplus volume} \cite{patel2020representing} due to its simplicity.
\begin{equation}\small\label{eq:softplus_volume}
 \operatorname{SVol}\left( \Box_w\left(C\right)\right)= \prod_{i=1}^n \operatorname{Softplus}_{t} \left( M_w\left(C  \right)_i - m_w\left(C\right)_i\right)
\end{equation}
where $t$ is a temperature parameter. The softplus function is defined as $\operatorname{softplus}_{t}(x)=t \log \left(1+e^{x/t}\right)$, which can be regarded as a smoothed version of the ReLu function ($\max\{0,x\}$) used for calculating the volume of \textit{hard boxes}. 
In practice, the \textit{softplus volume} is used to replace the \textit{modified volume} in Eq.(\ref{eq:modi_vol}) as it empirically resolves the same issue that point has zero volume. 

\textbf{Regularization} We add a regularization term in Eq.(\ref{eq:regularizer}) to all non-empty boxes to encourage that the boxes lie in the unit box $[0,1]^n$. 
\begin{equation}\label{eq:regularizer}
    \lambda = \sum_{i=1}^n \max(0, M_w(C)_i -1 + \epsilon) + \max(0, -m_w(C)_i-\epsilon)
\end{equation}
In practice, this also avoids numerical stability issues. For example, to minimize a loss term, a box that should have a fixed volume could become very \textit{slim}, i.e. some side lengths be extremely large while others become extremely small.

\textbf{Negative Sampling} In principle, the embeddings can be optimized without negatives. However, we empirically find that the embeddings will be highly overlapped without negative sampling. e.g. for role assertion $r(a,b)$, $a$ and $b$ will simply become the same point. We generate negative samples for the role assertion $r(a,b)$ by randomly replacing one of the head or tail entity. 
Finally, we sum up all the loss terms, and learn the embeddings by minimizing the loss with Adam optimizer \cite{DBLP:journals/corr/KingmaB14}. 

\section{Empirical Evaluation}

\subsection{A Proof-of-concept Example}
We begin by first validating the model in modeling a toy ontology--family domain \cite{DBLP:conf/ijcai/KulmanovLYH19}, which is described by the following axioms:\footnote{Compared with the example given in \cite{DBLP:conf/ijcai/KulmanovLYH19}, we add additional concept assertion statements that distinguish entities and concepts:}
\begin{equation*}\label{eq:family_domain}
\small
\begin{array}{cl}
\concept {Male} \sqsubseteq \concept {Person} & \concept{Female} \sqsubseteq \concept { Person } \\
\concept { Father } \sqsubseteq \concept { Male } & \concept { Mother } \sqsubseteq \concept { Female } \\
\concept { Father } \sqsubseteq \concept { Parent } & \concept {Mother} \sqsubseteq \concept { Parent }\\
\concept { Female } \sqcap \concept { Male } \sqsubseteq \bot & \concept{ Female } \sqcap \concept { Parent } \sqsubseteq \concept { Mother } \\
\concept { Male } \sqcap \concept { Parent } \sqsubseteq \concept { Father } & \exists \rel {hasChild.Person } \sqsubseteq \concept { Parent }\\
\concept { Parent } \sqsubseteq \concept { Person } & \concept { Parent } \sqsubseteq \exists \rel{ hasChild.Person } \\
\concept{Father}(\text{Alex}) & \concept{Father}(\text{Bob}) \\
\concept{Mother}(\text{Marie}) & \concept{Mother}(\text{Alice})
\end{array}
\end{equation*}

We set the dimension to $2$ to visualize the embeddings. Fig.~\ref{fig:toy_example} shows that the generated embeddings accurately encode all of the axioms.
In particular, the embeddings of $\concept{Father}$ and $\concept{Mother}$ align well with the conjunction $\concept{Parent} \sqcap \concept{Male}$ and $\concept{Parent} \sqcap \concept{Female}$, respectively, which is impossible to be achieved by ELEm.

\begin{figure}
\centering
    \includegraphics[width=0.7\linewidth]{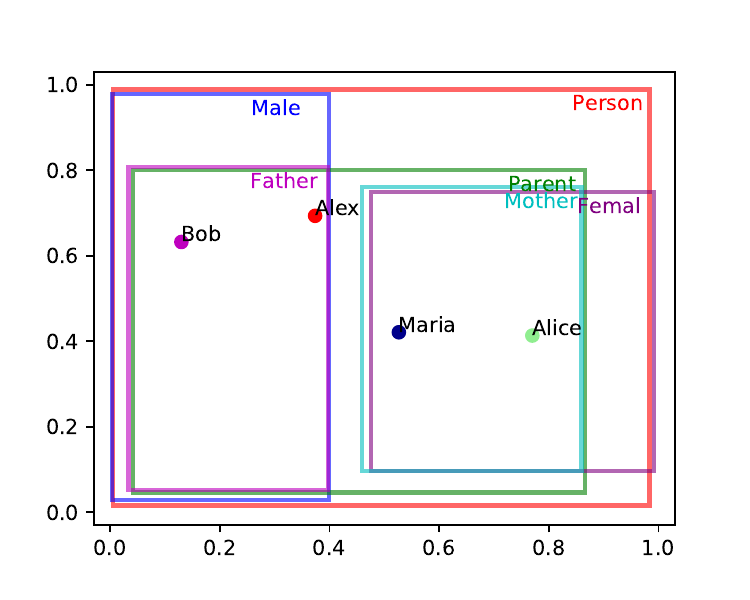}
    \caption{BoxEL embeddings in the family domain.
    }\label{fig:toy_example}
\end{figure}

\subsection{Subsumption Reasoning}\label{app:subsumption}
We evaluate the effectiveness of BoxEL on (plausible) subsumption reasoning (also known as ontology completion). The problem is to predict whether a concept is subsumed by another one.
For each subsumption pair $C \sqsubseteq D$, the scoring function can be defined by 
\begin{equation}
    P(C \sqsubseteq D) = \frac{\Vol(\Box(C) \cap \Box(D))}{\Vol(\Box(C))}.
\end{equation}
Note that such subsumption relations are not necessary to be 
(logically)
entailed by the input KB, e.g., a subsumption relation can be plausibly inferred by $P(C \sqsubseteq D)=0.9$, allowing for non-classical plausible reasoning.
While the subsumption reasoning does not need negatives, we add an additional regularization term for non-subsumption axiom. In particular, for each atomic subsumption axiom $C \sqsubseteq D$, we generate a non-subsumption axiom $C \not\sqsubseteq D^\prime$ or $C^\prime \not\sqsubseteq D$ by randomly replacing one of the concepts $C$ and $D$. Note that this does not produce regular negative samples as the generated concepts pair does not have to be disjoint. Thus, the loss term for non-subsumption axiom cannot be simply defined by $\mathcal{L}_{C \not\sqsubseteq D^\prime} = 1-\mathcal{L}_{C \sqsubseteq D^\prime}$. Instead, we define the loss term as $\mathcal{L}_{C \not\sqsubseteq D^\prime} = \phi(1-\mathcal{L}_{C \sqsubseteq D^\prime})$
by multiplying a small positive constant $\phi$ that encourages splitting the non-subsumption concepts while does not encourage them to be disjoint. If $\phi=1$, the loss would encourage the non-subsumption concepts to be disjoint. We empirically show that $\phi=1$ produces worse performance as we do not want non-subsumption concepts to be disjoint. 

\begin{table}[t!]
    \resizebox{\textwidth}{!}{
    \begin{minipage}{.47\linewidth}
        \caption{Summary of classes, relations and axioms in different ontologies. NF$_i$ represents the $i^{th}$ normal form.}
         \resizebox{\textwidth}{!}{
        \begin{tabular}{lccc}
        \hline Ontology & GO & GALEN & ANATOMY \\
        \hline
        Classes & 45895 & 24353 & 106363 \\
        Relations & 9 & 1010 & 157 \\
        NF1 & 85480 &  28890 & 122142 \\
        NF2 & 12131 & 13595 & 2121 \\
        NF3 & 20324 & 28118 & 152289 \\
        NF4 & 12129 & 13597 & 2143 \\
        Disjoint & 30 & 0 & 184 \\
        \hline
        \end{tabular}}
    \label{tab:dataset}
    \end{minipage}%
    \quad
    \begin{minipage}{.47\linewidth}
      \caption{The accuracies achieved by the embeddings of different approaches in terms of geometric interpretation of the classes in various ontologies.}
        \label{tab:el_syntax_seman}
        \resizebox{\textwidth}{!}{
            \begin{tabular}{ccccccc}
            \hline 
            & ELEm & EmEL$^{++}$ & BoxEL & & \\
            \hline 
            GO & 0.250 & 0.415 & \textbf{0.489 } \\
            GALEN & 0.480 & 0.345 & \textbf{0.788}  \\
            ANATOMY & 0.069 & 0.215 & \textbf{0.453}  \\
            \hline
    \end{tabular}}
    \label{tab:subsumption_accuracy}
    \end{minipage} 
    }
\end{table}

\textbf{Datasets} We use three biomedical ontologies 
as our benchmark. 1) \emph{Gene Ontology (GO)} \cite{harris2004gene} integrates the representation of genes and their functions across all species. 2) \emph{GALEN} \cite{rector1996galen} is a clinical ontology. 3) \emph{Anatomy} \cite{mungall2012uberon} is a ontology that represents linkages of different phenotypes to genes. Table \ref{tab:dataset} summarizes the statistical information of these datasets. The subclass relations are split into training set (70\%), validation set (20\%) and testing set (10\%), respectively.

\textbf{Evaluation protocol} 
Two strategies can be used to measure the effectiveness of the embeddings. 
1) Ranking based measures rank the probability of $C$ subsumed by all concepts. We evaluate and report four ranking based measures. Hits@10, Hits@100 describe the fraction of true cases that appear in the first $10$ and $100$ test cases of the sorted rank list, respectively. Mean rank computes the arithmetic mean over all individual ranks (i.e. $\mathrm{MR}=\frac{1}{|\mathcal{I}|} \sum_{\mathrm{rank} \in \mathcal{I}} \mathrm{rank}$, where $rank$ is the individual rank), while AUC computes the area under the ROC curve. 
2) Accuracy based measure is a stricter criterion, for which the prediction is true if and only if the subclass box is exactly inside the superclass box (even not allowing the subclass box slightly outside the superclass box). We use this measure as it evaluates the performance of embeddings on retaining the underlying characteristics of ontology in vector space. We only compare ELEm and EmEL$^{++}$ as KGE baselines fail in this setting (KGEs cannot preserve the ontology). 

\textbf{Implementation details} The ontology is normalized into standard normal forms, which comprise a set of axioms that can be used as the \emph{positive samples}. Similar to previous works \cite{DBLP:conf/ijcai/KulmanovLYH19}, we perform normalization using the OWL APIs and the APIs offered by the jCel reasoner [18]. The hyperparameter for negative sampling is set to $\phi=0.05$. 
For ELEm and EmEL$^{++}$, the embedding size is searched from $n=[50,100,200]$ and margin parameter is searched from $\gamma=[-0.1, 0, 0.1]$. Since box embedding has double the number of parameters of ELEm and EmEL$^{++}$, we search the embedding size from $n=[25,50,100]$ for BoxEL. We summarize the best performing hyperparameters in our supplemental material.
All experiments are evaluated with $10$ random seeds and the mean results are reported for numerical comparisons. 

\textbf{Baselines} We compare the state-of-the-art $\mathcal{EL}^{++}$ embeddings (ELEm) \cite{DBLP:conf/ijcai/KulmanovLYH19}, the first geometric embeddings of $\mathcal{EL}^{++}$, as well as the extension EmEL$^{++}$ \cite{mondala2021emel} that additionally considers the role inclusion and role chain embedding, as our major baselines. For comparison with classical methods, we also include the reported results of three classical KGEs in \cite{mondala2021emel}, including TransE \cite{DBLP:conf/nips/BordesUGWY13}, TransH \cite{DBLP:conf/aaai/WangZFC14} and DistMult \cite{yang2014embedding}.

\begin{table}[t!]
    \centering
    \caption{The ranking based measures of embedding models for sumbsumtion reasoning on the testing set. $*$ denotes the results from \cite{mondala2021emel}.}
      \small
    \resizebox{\columnwidth}{!}{
    \begin{tabular}{lllllllll}
    \hline 
    Dataset & Metric & TransE$*$ & TransH$*$ & DistMult$*$ & ELEm & EmEL$^{++}$ & \textbf{BoxEL}  \\
     \hline 
    \multirow{4}{*}{GO} 
    & Hits@10 &  0.00 & 0.00 & 0.00 &  0.09 & 0.10 & 0.03  \\
    & Hits@100 &  0.00 & 0.00 & 0.00  & 0.16 & 0.22 & 0.08 \\
    & AUC &  0.53 & 0.44 & 0.50  & 0.70 & 0.76 & 0.81  \\
    & Mean Rank & - & - & - & 13719 & 11050  & 8980 \\
    \hline 
    \multirow{4}{*}{GALEN} 
    & Hits@10 & 0.00 & 0.00 & 0.00 &  0.07 & 0.10 & 0.02  \\
    & Hits@100 & 0.00 & 0.00 & 0.00 &  0.14 & 0.17 & 0.03  \\
    & AUC &  0.54 & 0.48 & 0.51  & 0.64 & 0.65 & 0.85 \\
    & Mean Rank & - & - & - & 8321 & 8407 & 3584  \\
    \hline 
    \multirow{4}{*}{ANATOMY} 
    & Hits@10 & 0.00 & 0.00 & 0.00 & 0.18 & 0.18 & 0.03  \\
    & Hits@100 &  0.01 & 0.00 & 0.00 & 0.38 & 0.40  & 0.04  \\
    & AUC & 0.53 & 0.44 & 0.49 &  0.73 & 0.76  & 0.91  \\
    & Mean Rank & - & - & - & 28564 & 24421  & 10266  \\
    \hline 
    \end{tabular}
    }
    \label{tab:ranking_measure}
\end{table}

\textbf{Results} Table~\ref{tab:ranking_measure} summarizes the ranking based measures of embedding models. We first observe that both ELEm and EmEL$^{++}$ perform much better than the three standard KGEs (TransE, TransH, and DistMult) on all three datasets, especially on hits@k for which KGEs fail, showcasing the limitation of KGEs and the benefits of geometric embeddings on encoding logic structures. EmEL$^{++}$ performs slightly better than ELEm on all three datasets. 
Overall, our model BoxEL outperforms ELEm and EmEL$^{++}$. 
In particular, we find that for Mean Rank and AUC, our model achieves significant performance gains on all three datasets.
Note that Mean Rank and AUC have theoretical advantages over hits@k because hits@k is sensitive to any model performance changes while Mean Rank and AUC reflect the average performance, demonstrating that BoxEL achieves better average performance. 
Table~\ref{tab:subsumption_accuracy} shows the accuracies of different embeddings in terms of the geometric interpretation of the classes in various ontologies. 
It clearly demonstrates that BoxEL outperforms ELEm and EmEL$^{++}$ by a large margin, showcasing that BoxEL preserves the underlying ontology characteristics in vector space better than ELEm and EmEL$^{++}$ that use ball embeddings.

\subsection{Protein-Protein Interactions}

\textbf{Dataset} We use a biomedical knowledge graph built by \cite{DBLP:conf/ijcai/KulmanovLYH19} from Gene Ontology (TBox) and STRING database (ABox) to conduct this task. Gene Ontology contains information about the functions of proteins, while STRING database consists of the protein-protein interactions. We use the protein-protein interaction data of yeast and human organisms, respectively. For each pair of proteins $(P_1, P_2)$ that exists in STRING, we add a role assertion $\rel{interacts}(P1, P2)$. If protein $P$ is associated with the function $F$, we add a membership axiom $\{P\} \sqsubseteq \exists \rel{hasFunction}.F$, the membership assertion can be regarded as a special case of NF3, in which $P$ is a point (i.e. zero-volume box). The interaction pairs of proteins are split into training (80\%), testing (10\%) and validation (10\%) sets. To perform prediction for each protein pair $(P_1, P_2)$, we predict whether the role assertion $\rel{interacts}(P_1, P_2)$ hold. This can be measured by Eq.(\ref{eq:ppi}).
\begin{equation}\label{eq:ppi}
    P(\rel{interacts}(P_1, P_2)) = \left\|T^{\rel{interacts}}_w(m_w(P_1)) - m_w(P_2) \right\|_2.
\end{equation}
where $T^{\rel{interacts}}_w$ is the affine transformation function for relation $\rel{interacts}$. For each positive interaction pair $\rel{interacts}(P_1, P_2)$, we generate a corrupted negative sample by randomly replacing one of the head and tail proteins. 

\textbf{Baselines} We consider ELEm \cite{DBLP:conf/ijcai/KulmanovLYH19} and EmEL$^{++}$ \cite{mondala2021emel} as our two major baselines as they have been shown outperforming the traditional KGEs. We also report the result of Onto2Vec \cite{DBLP:journals/bioinformatics/SmailiGH18} that treats logical axioms as a text corpus and OPA2Vec \cite{smaili2019opa2vec} that combines logical axioms with annotation properties. Besides, we report the results of two semantic similarity measures: Resnik’s similarity and Lin’s similarity in \cite{DBLP:conf/ijcai/KulmanovLYH19}. For KGEs, we compare TransE \cite{BordesUGWY13}) and BoxE \cite{chheda2021box}.
We report the hits@10, hits@100, mean rank and AUC (area under the ROC curve) as explained before for numerical comparison. Both raw ranking measures and filtered ranking measures that ignore the triples that are already known to be true in the training stage are reported. Baseline results are taken from the standard benchmark developed by \cite{DBLP:journals/bib/KulmanovSGH21}.\footnote{https://github.com/bio-ontology-research-group/machine-learning-with-ontologies}

\textbf{Overall Results} Table~\ref{tab:ppi_result} and Table~\ref{tab:ppi_result_human} summarize the performance of protein-protein prediction in yeast and human organisms, respectively. 
We first observe that similarity based methods (SimResnik and SimLin) roughly outperform TransE, showcasing the limitation of classical knowledge graph embeddings. BoxE roughly outperforms TransE as it does encode some logical properties.  
The geometric methods ELEm and EmEL$^{++}$ fail on the hits@10 measures and does not show significant performance gains on the hits@100 measures in human dataset.
However, ELEm and EmEL$^{++}$ outperform TransE, BoxE and similarity based methods on Mean Rank and AUC by a large margin, especially for the Mean Rank, showcasing the expressiveness of geometric embeddings. Onto2Vec and OPA2Vec achieve relatively better results than TransE and similarity based methods, but cannot compete ELEm and EmEL$^{++}$. We conjecture that this is due to the fact that they mostly consider annotation information but cannot encode the logical structure explicitly. 
Our method, BoxEL consistently outperforms all methods in hits@100, Mean Rank and AUC in both datasets, except the competitive results of hits@10, showcasing the better expressiveness of BoxEL.

\begin{table}[t!]
    \centering
    \caption{Prediction performance on protein-protein interaction (yeast).}
    \resizebox{\textwidth}{!}{
    \begin{tabular}{ccccccccc}
        \hline 
        Method & \makecell[c]{Raw\\ Hits@10} & \makecell[c]{Filtered\\ Hits@10}  & \makecell[c]{Raw\\ Hits@100} & \makecell[c]{Filtered\\ Hits@100}  & \makecell[c]{Raw\\ Mean Rank}  & \makecell[c]{Filtered\\ Mean Rank}  & \makecell[c]{Raw\\ AUC} & \makecell[c]{Filtered\\ AUC} \\
        \hline 
        TransE    & 0.06 & 0.13 & 0.32 & 0.40	& 1125 & 1075 &	0.82 & 0.83 \\
        BoxE	    & 0.08 & 0.14 & 0.36 & 0.43	& 633 & 620 &	0.85 & 0.85 \\
        SimResnik  & \textbf{0.09} & 0.17 & 0.38 & 0.48	& 758  & 707  &	0.86 & 0.87 \\
        SimLin    & 0.08 & 0.15 & 0.33 & 0.41	& 875  & 825  &	0.8  & 0.85 \\
        ELEm        & 0.08 & 0.17 & 0.44 & 0.62	& 451  & 394  &	0.92 & 0.93 \\
        EmEL$^{++}$      & 0.08 & 0.16 & 0.45 & 0.63 & 451  & 397  & 0.90 & 0.91 \\
        Onto2Vec    & 0.08 & 0.15 & 0.35 & 0.48 & 641  & 588  & 0.79 & 0.80 \\
        OPA2Vec	    & 0.06 & 0.13 &	0.39 & 0.58 & 523  & 467  & 0.87 & 0.88 \\
        BoxEL    & \textbf{0.09} & \textbf{0.20} & \textbf{0.52} & \textbf{0.73} & \textbf{423} & \textbf{379} & \textbf{0.93} & \textbf{0.94} \\
        \hline 
    \end{tabular}}
    \label{tab:ppi_result}
\end{table}

\begin{table}[t!]
    \centering
    \caption{Prediction performance on protein-protein interaction (human).}
    \resizebox{\textwidth}{!}{
    \begin{tabular}{ccccccccc}
        \hline 
        Method & \makecell[c]{Raw\\ Hits@10} & \makecell[c]{Filtered\\ Hits@10}  & \makecell[c]{Raw\\ Hits@100} & \makecell[c]{Filtered\\ Hits@100}  & \makecell[c]{Raw\\ Mean Rank}  & \makecell[c]{Filtered\\ Mean Rank}  & \makecell[c]{Raw\\ AUC} & \makecell[c]{Filtered\\ AUC} \\
        \hline 
        TransE	    & 0.05 & \textbf{0.11} & 0.24 & 0.29	& 3960 & 3891 &	0.78 & 0.79 \\
        BoxE	    & 0.05 & 0.10 & 0.26 & 0.32	& 2121 & 2091 &	0.87 & 0.87 \\
        SimResnik	& 0.05 & 0.09 & 0.25 & 0.30	& 1934 & 1864 &	0.88 & 0.89 \\
        SimLin	    & 0.04 & 0.08 & 0.20 & 0.23	& 2288 & 2219 &	0.86 & 0.87 \\
        ELEm        & 0.01 & 0.02 & 0.22 & 0.26	& 1680 & 1638 &	0.90 & 0.90 \\
        EmEL$^{++}$       & 0.01 & 0.03 & 0.23 & 0.26 & 1671 & 1638 & 0.90 & 0.91 \\
        Onto2Vec    & 0.05 & 0.08 & 0.24 & 0.31	& 2435 & 2391 &	0.77 & 0.77 \\
        OPA2Vec	    & 0.03 & 0.07 & 0.23 & 0.26	& 1810 & 1768 &	0.86 & 0.88 \\
        BoxEL (Ours) & \textbf{0.07} & 0.10 & \textbf{0.42} & \textbf{0.63} & \textbf{1574} & \textbf{1530} & \textbf{0.93} & \textbf{0.93} \\
        \hline 
    \end{tabular}}
    \label{tab:ppi_result_human}
\end{table}

\subsection{Ablation Studies} 

\textbf{Transformation vs Translation} To study the contributions of using boxes for modeling concepts and using affine transformation for modeling relations, we conduct an ablation study by comparing relation embeddings with affine transformation (AffineBoxEL) and translation (TransBoxEL). The only difference of TransBox to the AffineBox is that TransBox does not associate a scaling factor for each relation. Table~\ref{tab:transform_translation} clearly shows that TransBoxEL outperforms EmEL$^{++}$, showcasing the benefits of box modeling compared with ball modeling. While AffineBoxEL further improves TransBoxEL, demonstrating the advantages of affine transformation. 
Hence, we could conclude that both of our proposed entity and relation embedding components boost the performance.

\textbf{Entities as Points vs Boxes} As mentioned before, distinguishing entities and concepts by identifying entities as points has better theoretical properties. Here, we study how this distinction influences the performance. For this purpose, we eliminate the ABox axioms by replacing each individual with a singleton class and rewriting relation assertions $r(a, b)$ and concept assertions $C(a)$ as $\{a\} \sqsubseteq \exists r.\{b\}$ and $\{a\} \sqsubseteq C$, respectively. In this case, we only have TBox embeddings and the entities are embedded as regular boxes. Table~\ref{tab:points_boxes} shows that for hits@k, there is marginal significant improvement of point entity embedding over boxes entity embedding, however, point entity embedding consistently outperforms box entity embedding on Mean Rank and AUC, showcasing the benefits of distinguishing entities and concepts.

\begin{table}[t!]
    \resizebox{\textwidth}{!}{
    \begin{minipage}{.47\linewidth}
      \caption{The performance of BoxEL with affine transformation (AffineBoxEL) and BoxEL with translation (TransBoxEL) on yeast protein-protein interaction.}
      \centering
      \resizebox{\textwidth}{!}{
      \begin{tabular}{ccccccc}
        \hline 
        Method & \multicolumn{2}{c}{EmEL} & \multicolumn{2}{c}{TransBoxEL} & \multicolumn{2}{c}{AffineBoxEL} \\
        & Raw & Filtered  & Raw & Filtered  & Raw & Filtered \\
        \hline 
        Hits@10 & 0.08 & 0.17  & 0.04 & 0.18 & \textbf{0.09} & \textbf{0.20}\\
        Hits@100 & 0.44 & 0.62 & 0.54 & 0.68 & \textbf{0.52} & \textbf{0.73}\\
        Mean Rank & 451 & 394 & 445 & 390 & \textbf{423} & \textbf{379} \\
        AUC & 0.92 & 0.93 & \textbf{0.93} & 0.93 & \textbf{0.93} & \textbf{0.94} \\
        \hline 
    \end{tabular}}
    \label{tab:transform_translation}
    \end{minipage}%
    \quad
    \begin{minipage}{.49\linewidth}
      \centering
        \caption{The performance of BoxEL with point entity embedding and box entity embedding on yeast protein-protein interaction dataset.}
        \resizebox{\textwidth}{!}{
        \begin{tabular}{ccccccc}
        \hline 
        Method & \multicolumn{2}{c}{EmEL} & \multicolumn{2}{c}{BoxEL (boxes)} & \multicolumn{2}{c}{BoxEL (points)} \\
        & Raw & Filtered  & Raw & Filtered  & Raw & Filtered \\
        \hline 
        Hits@10 & 0.08 & 0.17  & \textbf{0.09} & 0.19 & \textbf{0.09} & \textbf{0.20}\\
        Hits@100 & 0.44 & 0.62 & 0.48 & 0.68 &  \textbf{0.52} & \textbf{0.73}\\
        Mean Rank & 451 & 394 & 450 & 388 & \textbf{423} & \textbf{379} \\
        AUC & 0.92 & 0.93 & 0.92 & 0.93 & \textbf{0.93} & \textbf{0.94} \\
        \hline 
    \end{tabular}}
    \label{tab:points_boxes}
    \end{minipage} 
    }
\end{table}

\section{Conclusion}

This work proposes BoxEL, a geometric KB embedding method that explicitly models the logical structure expressed by the theories of $\mathcal{EL}^{++}$. 
Different from the standard KGEs that simply ignore the analytical guarantees, BoxEL provides \textit{soundness} guarantee for the underlying logical structure by incorporating background knowledge into machine learning tasks, offering a more reliable and logic-preserved fashion for KB reasoning. 
The empirical results further demonstrate that BoxEL outperforms previous KGEs and $\mathcal{EL}^{++}$ embedding approaches on subsumption reasoning over three ontologies and predicting protein-protein interactions in a real-world biomedical KB.

\cleardoublepage
\chapter{Geometric Embeddings of Structured Constraints in Machine Learning  }
\label{chap_logical}

In this chapter, we introduce a geometric embedding approach for encoding structured constraints in a multi-label prediction problem, where the labels are organized under implication and mutual exclusion constraints.
A major concern is to produce predictions that are logically consistent with these constraints.
To do so, we formulate this problem as an \emph{embedding inference} problem where the constraints are imposed onto the embeddings of labels by \emph{geometric construction}.
Particularly, we consider a hyperbolic Poincaré ball model in which we encode labels as Poincaré hyperplanes that work as linear decision boundaries. 
The hyperplanes are interpreted as convex regions such that the logical relationships (implication and exclusion) are geometrically encoded using \emph{insideness} and \emph{disjointedness} of these regions, respectively. 
We show theoretical groundings of the method for preserving logical relationships in the embedding space. 
Extensive experiments on $12$ datasets show 1) significant improvements in mean average precision; 2) lower number of constraint violations;  3) an order of magnitude fewer dimensions than baselines.

\section{Motivation and Background}

Structured multi-label prediction is a task aiming to associate every object with multiple labels that are semantically constrained in a structured manner (e.g., by implication and exclusion constraints). 
This task is of growing importance in many applications such as image annotation \cite{DBLP:conf/eccv/DengDJFMBLNA14,DBLP:conf/eccv/DengDJFMBLNA14}, text categorization, \cite{DBLP:journals/ijcv/KrishnaZGJHKCKL17,DBLP:conf/rep4nlp/LopezHS19} and functional genomics~\cite{vu2021protein, kulmanov2020deepgoplus}, where the labels are organized in a directed acyclic graph (DAG) or an ontology.
One of the central concerns of the task is to produce predictions that are \emph{logically consistent} with the constraints of the labels. 
For example, a protein must be labeled to have the function \emph{nucleic acid binding} if it is already labeled to have the function \emph{RNA binding} (i.e., implication) and must not have the function \emph{drug binding} (i.e., mutual exclusion). 

Various works have been proposed to improve the prediction consistency \cite{DBLP:journals/jcss/CerriBC14,DBLP:conf/nips/GiunchigliaL20,DBLP:conf/acl/McCallumVVMR18,DBLP:conf/icml/WehrmannCB18,patel2021modeling}. 
One line of work called \emph{label embedding} aims to represent labels as low-dimensional vectors~\cite{DBLP:conf/icml/BiK11, DBLP:conf/aaai/ChenHXCJ20}. 
A key disadvantage of the vector-based representations is that they only capture weak forms of correlation or ``similarity'  between labels, but do not strongly enforce the logical relationships. 
Another line of work~\cite{DBLP:journals/jcss/CerriBC14, DBLP:conf/acl/McCallumVVMR18,DBLP:conf/icml/WehrmannCB18, DBLP:conf/nips/GiunchigliaL20} imposes these logical constraints directly to the losses of neural networks. However, they do not explicitly learn the representations of labels and typically require a complete label taxonomy, which is not always available in and scalable to real-world settings \cite{patel2021modeling}.

Embedding-based inference \cite{DBLP:conf/nips/MirzazadehRDS15}, which imposes logical constraints directly to the label embeddings, is able to \emph{inductively} infer the underlying label relationships from incomplete labelings~\cite{mirzazadeh2017solving}. 
Once all embeddings are adhering to the constraints, each label can be predicted independently without accessing the label relationships, which significantly reduces the computation cost during inference~\cite{DBLP:conf/nips/MirzazadehRDS15}.
The key idea, which is inspired by the Venn diagram~\cite{mirzazadeh2017solving,venn1880diagrammatic} or set-theoretic semantics~\cite{van1977set}, is to represent each label as a convex region~\cite{DBLP:conf/nips/MirzazadehRDS15}. 
A prominent example is the multi-label box model (MBM)~\cite{patel2021modeling} that models label implications as box containments. However, MBM learns box-like decision boundaries, which are typically not compatible with standard classifiers (i.e., hyperplane margin-based models such as logistic regression \cite{ganea2018hyperbolic}). Besides, box models suffer from a theoretical limitation, i.e., lower-way intersections enforce higher-way interactions \cite{vilnis2021geometric}. Finally, current methods ignore the importance of constraining mutual exclusion, which is essential as otherwise, a model could trivially obtain zero implication violation by assigning the same score to all labels. 

We consider a structured multi-label prediction problem with \emph{implication} and \emph{mutual exclusion} constraints that are jointly described by a hierarchy and exclusion (HEX) graph (see \Cref{fig:labelspace}(a) for an example).
The key idea of our method is to transform the logical constraints into soft geometric constraints in the embedding space. 
In particular, we consider a hyperbolic Poincaré ball model that has demonstrated advantages in representing hierarchical data \cite{yang2022hicf} and assign each label a Poincaré hyperplane that has several favorable theoretical properties in classification. 
Each Poincaré hyperplane can be interpreted as a convex region such that the \emph{implication} and \emph{mutual exclusion} are modeled by geometric \emph{insideness} and \emph{disjointness} between the corresponding regions, respectively. In this way, a multi-label classifier can be defined by measuring the confidence of an instance having a label as geometric \emph{membership}.  
Unlike other hyperbolic region-based models such as hyperbolic cones \cite{ganea2018hyperbolic} and hyperbolic disks \cite{suzuki2019hyperbolic}, Poincaré hyperplane works as a linear decision boundary and can be seamlessly integrated into existing margin-based classifiers such as hyperbolic logistic regression~\cite{ganea2018hyperbolic}. 
\Cref{fig:labelspace}(b) shows an example of the learned label representations that respect all the constraints given in \Cref{fig:labelspace}(a). 
We show theoretical groundings of the proposed method on modeling \emph{implication} and \emph{mutual exclusion}. Extensive experiments on 12 multi-label classification tasks show the model’s capability to improve the mean average precision significantly while keeping the number of constraint violations low and requiring an order of magnitude fewer dimensions. 

\begin{figure}[t!]
    \centering 
    \subfloat[\centering ]{{\includegraphics[width=0.34\textwidth]{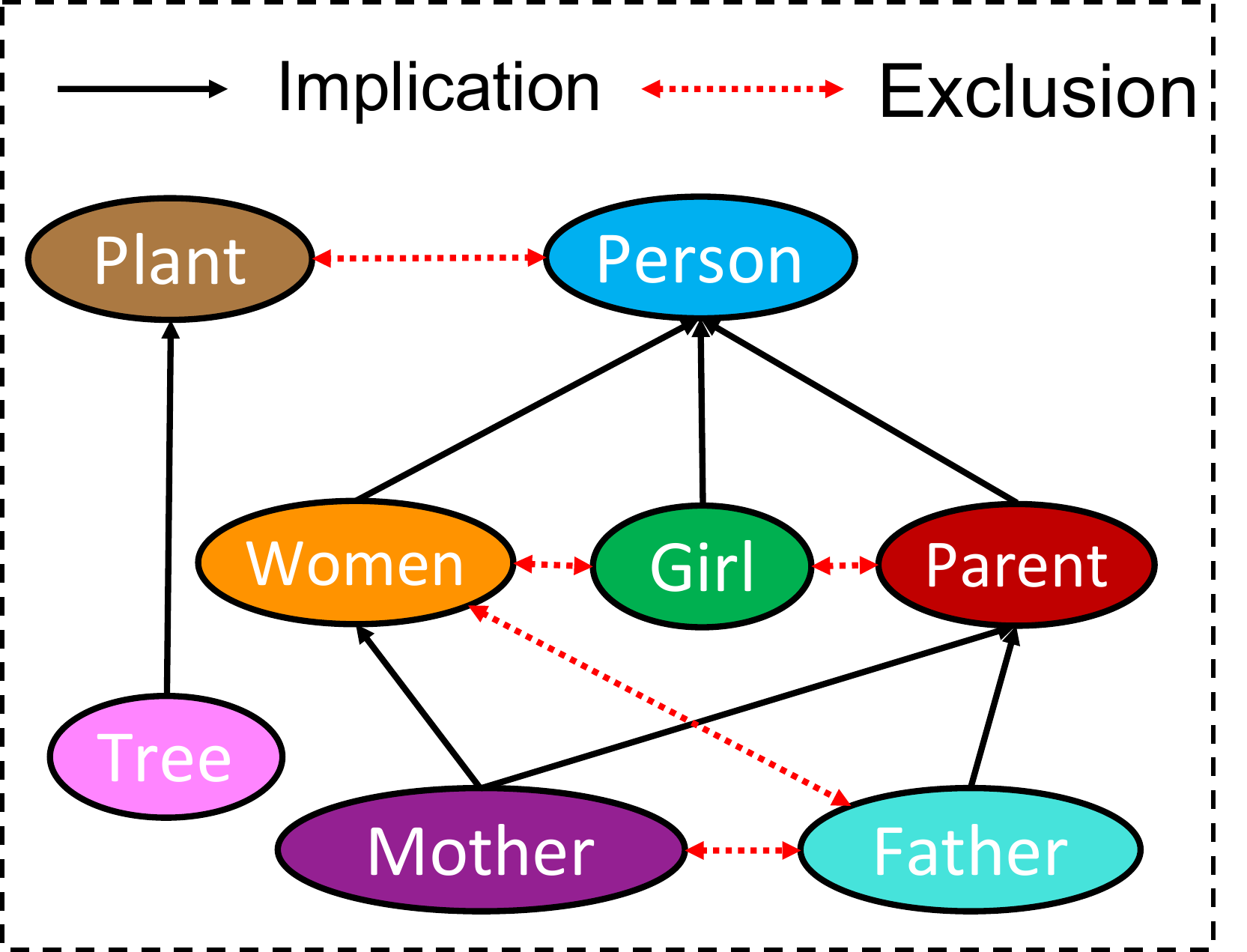}}}
    \subfloat[\centering ]{{\includegraphics[width=0.31\columnwidth]{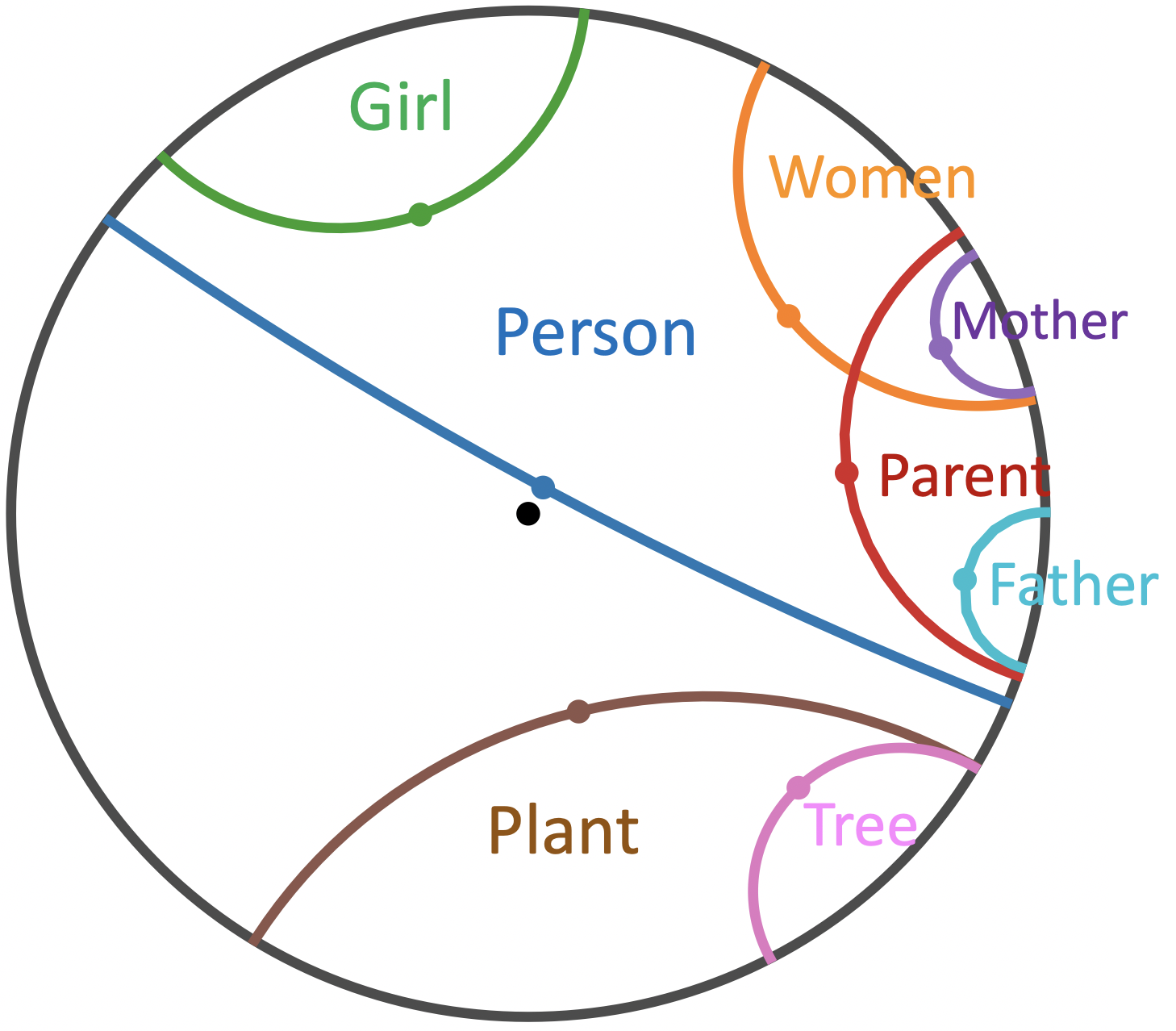}}}
    \subfloat[\centering ]
    {{\includegraphics[width=0.31\columnwidth]{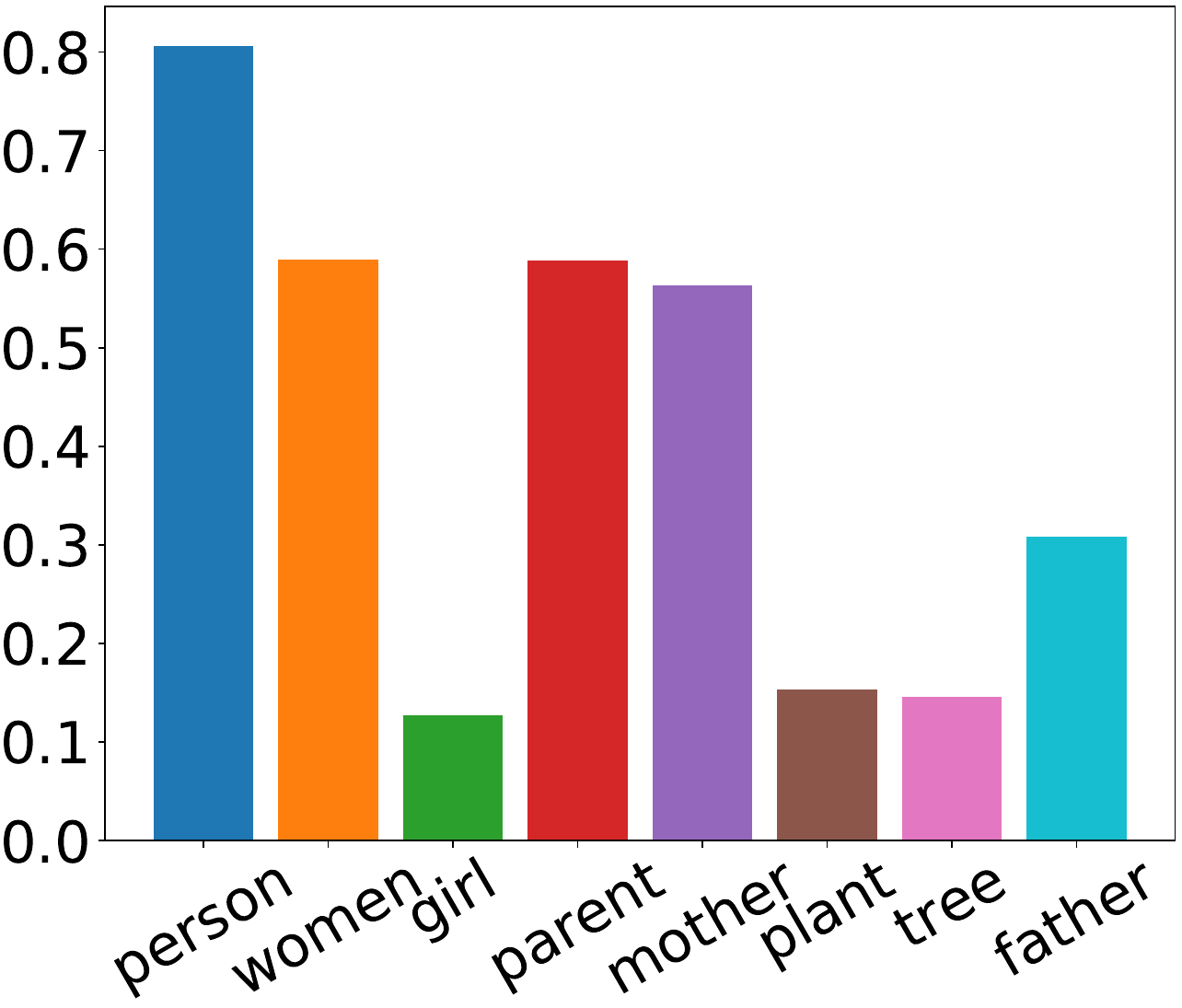}}}
    \caption{(a) A HEX graph describing the logical relationships (implication and exclusion) between different labels; (b) The learned label embeddings (linear decision boundaries) in the Poincaré ball, where all constraints in the HEX graph are respected; (c) The prediction scores of a given instance of \emph{mother} respect all constraints in the HEX graph, where each score is calculated as the confidence of the instance embedding being a member of the convex region of the corresponding label embedding.}
    \label{fig:labelspace}
     \vspace{-0.3cm}
\end{figure}




\section{Structured multi-label prediction} 

Let $\mathcal{X} \subseteq \mathbb{R}^{n}$ denote an $n$-dimensional instance space and $\mathcal{L}=\{l_1, l_2, $ $\dots\}, \left|\mathcal{L}\right| \geq 2$ denote the finite set of possible labels.
Given a set of $N$ training examples $\mathcal{D}=\left\{\left({x}_{i}, L_{i}\right) \mid 1 \leq i \leq N, {x}_{i} \in \mathcal{X}, L_i \subset \mathcal{L}\right\}$, \emph{multi-label prediction} aims to learn a labeling function $f: \mathcal{X} \rightarrow 2^{\mathcal{L}}$ mapping from the instance space to the \emph{powerset} of the label space, $f({x}) \subset \mathcal{L}$.

\emph{Structured} multi-label prediction additionally imposes a set of prior-known logical constraints over the labels, namely, the predictions must be logically consistent with these constraints. Analogous to Mirzazadeh et al.~\cite{DBLP:conf/nips/MirzazadehRDS15}, we consider two forms of logical constraints between labels: implication and mutual exclusion. 
Specifically, an \emph{implication} of the form $l_{a} \Rightarrow l_{b}$ imposes the constraint that whenever an instance is labeled as $l_{a}$ then it must also be labeled as $l_{b}$, i.e., $l_{a} \Rightarrow l_{b}$ is a shorthand notation for
$\forall x \in \mathcal{X}, l_{a} \in f(x) \Rightarrow l_{b} \in f(x)$. \emph{Mutual exclusions} are constraints of the form $\neg l_{a} \vee \neg l_{b}$, implying that an instance cannot be simultaneously labeled as $l_{a}$ and $l_{b}$, i.e., $\neg l_{a} \vee \neg l_{b}$ is a shorthand notation for $\forall x \in \mathcal{X}, l_{a} \notin f(x) \vee l_{b} \notin f(x)$.  
We can concisely represent a set of implication and exclusion constraints with a hierarchy and exclusion (HEX) graph~\cite{DBLP:conf/eccv/DengDJFMBLNA14}.

\begin{definition}[HEX graph~\cite{DBLP:conf/eccv/DengDJFMBLNA14}\footnote{Deng et al. \cite{DBLP:conf/eccv/DengDJFMBLNA14} use subsumption, which is the inverse relation of implication that we use here.}]
A HEX graph $G=\left(V, E_{h}, E_{e}\right)$ is a graph consisting of a set of nodes $V=\left\{v_{1}, \ldots, v_{n}\right\}$, directed (hierarchy) edges $E_{h} \subseteq V \times V$, and undirected (exclusion) edges $E_{e} \subseteq V \times V$, such that the subgraph $G_{h}=\left(V, E_{h}\right)$ is a DAG and the subgraph $G_{e}=\left(V, E_{e}\right)$ has no self loop.
Each node $v_i \in V$ represents the label $l_i$. 
A directed edge $\left(v_{i}, v_{j}\right) \in E_{h}$ represents the implication $l_{i} \Rightarrow l_{j}$, and an undirected edge $\left(v_{i}, v_{j}\right) \in E_{e}$ represents the exclusion $\neg l_{i} \vee \neg l_{j}$. 
\end{definition}

Note that an arbitrary HEX graph may contain redundant edges. 
A hierarchy edge $(v_i,v_j)$ is redundant when there is a path in $G_{h}$ from $v_i$ to $v_j$ which does not contain the edge $(v_i,v_j)$.
Similarly, an exclusion edge $(v_i,v_j)$ is redundant when there is another exclusion edge connecting their ancestors (or connecting one node’s ancestor to the other node).

We can transform a HEX graph into an equivalent HEX graph by adding or removing redundant edges.
In this paper, we only consider HEX graphs that have a minimal number of edges,
we call such HEX graph a \emph{minimal sparse} HEX graph (see \cref{fig:labelspace}(a) for an example). 
Given a minimal sparse HEX graph, we define the $\mathcal{\text{HEX}}$-property as

\begin{definition}[$\mathcal{\text{HEX}}$-property]
A labeling function $f$ has the HEX property with respect to a HEX graph $G$ if for all $x \in \mathcal{X}$, $f(x)$ respects all constraints represented by $G$.
\end{definition}

We also call such function $f$ \emph{logically consistent} w.r.t $G$. Given the HEX graph and the $\mathcal{\text{HEX}}$-property, structured multi-label prediction is formally defined as a constrained optimization problem.

\begin{definition}[Structured multi-label prediction]
The structured multi-label prediction task with respect to a 
training set $\mathcal{D}=\{\left({x}_{i}, L_{i}\right) \mid 1 \leq i \leq N, $ $ {x}_{i} \in \mathcal{X}, L_i \subset \mathcal{L}\}$, minimal HEX graph $G=\left(V, E_{h}, E_{e}\right)$, and multi-label prediction function $f$, is the task of learning $f$ such that the function $f$ minimizes $\sum_{ (x_i, L_{i}) \in \mathcal{D} } \operatorname{loss}(f({x_i}),L_{i})$, with $\operatorname{loss}$ a predefined function, while attempting to maintain the $\mathcal{\text{HEX}}$-property with respect to $G$.
\end{definition} 

Note that this definition allows for a \emph{soft} interpretation of the constraints, meaning that the goal is to adhere to all of them, but we do allow for loosening some if necessary. For example, a mutual exclusion constraint is allowed to loosen 
when an instance (e.g., image), though rarely happens, is simultaneously labeled as two mutual exclusive labels (e.g., \emph{dog} and \emph{cat}).

\section{Hyperbolic Embedding Inference}

We consider learning a real-valued ranking function $h: \mathcal{X} \times \mathcal{L} \mapsto [0,1] $, where the output is interpreted as the confidence of an instance $x \in \mathcal{X}$ having a label $l \in \mathcal{L}$. Afterward, a binary multi-label classifier $f: \mathcal{X} \rightarrow 2^{\mathcal{L}}$ can be simply obtained by thresholding the ranking function with a threshold $t$, i.e., $f(x)=\{l \mid h\left(x, l\right) \geq t, \forall l \in \mathcal{L}\}$. The objective of $h$ is to assign higher scores to positive instance-label pairs than that of negative instance-label pairs.

\subsection{Geometric construction}

Given an $n$-dimensional Poincaré ball $\mathbb{D}^n$, we associate each instance $x_i \in \mathcal{X}$ with a point in the Poincaré ball and associate each label $l_i \in \mathcal{L}$ with a Poincaré hyperplane, such that its corresponding positive and negative instances are correctly separated by the hyperplane.

\textbf{Poincaré hyperplanes} 

Let $\mathbb{B}^{n}$ denote the set of $n$-balls in $\mathbb{R}^{n}$ whose boundaries $\partial \mathbb{B}^{n}$ intersect the Poincaré ball $\mathbb{D}^{n}$ perpendicularly. 
Poincaré hyperplanes are defined by $\partial \mathbb{B}^{n} \cap \mathbb{D}^{n}$ (see \cref{fig:geocons}(a)) plus all linear subspaces going through the origin. 
For the former cases, a Poincaré hyperplane can be uniquely defined by its center point that has a minimal distance to the origin. 


\begin{definition}
Given a (center) point $\mathbf{c} \in \mathbb{D}^{n}$ where $\mathbf{c} \neq \mathbf{0}$, the Poincaré hyperplane is defined as
\begin{equation} 
\begin{array}{l}
H_{\mathbf{c}}=\left\{\mathbf{p} \in \mathbb{D}^{n}: g^{\mathbb{D}}\left( \log _{\mathbf{c}}\left(\mathbf{p}\right), \vec{\mathbf{c}} \right) =0\right\} 
\end{array}
\end{equation}
where $\mathbf{c}$ is the center point and $\vec{\mathbf{c}} \in T_{\mathbf{c}} \mathbb{D}^{n}$ \footnote{In this paper, we distinguish normal vectors from regular points by adding an arrow on top of its letters.} is the normal vector passing through the origin $\mathbf{0}$. 
\end{definition}
Intuitively, this corresponds to the union of all geodesics passing through $\mathbf{c}$ while orthogonal to the normal vector $\vec{\mathbf{c}} \in T_{\mathbf{c}} \mathbb{D}^{n}$. In the case where $\mathbf{c}$ is the center of the hyperplane, $\vec{\mathbf{c}}$ must simultaneously pass through $\mathbf{c}$ and the origin. 
Hence, $\vec{\mathbf{c}}$ can be simply taken as $\mathbf{c}$ without loss of generality. 
For the special case where $\mathbf{c} = \mathbf{0}$, the Poincaré hyperplanes are all linear subspaces (Euclidean planes) passing through the origin. In this paper, we exclude these special cases by assuming $\mathbf{c} \neq \mathbf{0}$. 




\textbf{Geometric intuition}
Essentially, the Poincaré hyperplane works as a linear decision boundary that separates the embedding space into two regions,\footnote{Note that by using the metric in the Poincaré ball, each region has infinite (exponentially growing) volume.} where the smaller region (i.e., convex hull) is interpreted as the space of positive samples while the other one is interpreted as the space of negative samples. 
Two reasons motivate us to model labels as Poincaré hyperplanes: 1) Modeling labels as hyperplanes has several desired theoretical advantages in margin-based classifiers. 
Our model shares the same philosophy as existing learning frameworks such as hyperbolic logistic regression~\cite{ganea2018hyperbolic} and hyperbolic SVM~\cite{cho2019large}; 
2) More importantly, unlike Euclidean space that is flat, hyperbolic Poincaré ball is a curved space in which there are infinitely many non-parallel hyperplanes which do not intersect, implying that linear decision boundaries in hyperbolic space can capture more complicated set-theoretic interactions, such as implication and mutual exclusion.


\textbf{Enclosing balls} 
Given a Poincaré hyperplane $H_{\mathbf{c}}$, we call the corresponding $n$-ball $\mathbb{B}_{\mathbf{c}}^n$ that encloses $H_{\mathbf{c}}$ its enclosing $n$-ball. 
Formally, an enclosing $n$-ball $\mathbb{B}^n\left(\mathbf{o}, r\right)$ is defined by $\mathbb{B}^n\left(\mathbf{o}, r\right)=\left\{\mathbf{p}: \left\|\mathbf{p}-\mathbf{o}\right\| \leq r\right\}$, where $\mathbf{o} \in \mathbb{R}^n$ and $r$ are the center point and the radius, respectively. Given $H_{\mathbf{c}}$, we have the following closed-form representation of $\mathbb{B}^{n}_{\mathbf{c}}$. 

\begin{proposition}\label{prop:p2ball}
Given a Poincaré hyperplane $H_\mathbf{c}$ where $\mathbf{c} \neq \mathbf{0}$, there exists an $n$-ball $\mathbb{B}_\mathbf{c}^n\left(\mathbf{o}_\mathbf{c}, r_\mathbf{c}\right)$ such that $H_\mathbf{c} \subset \mathbb{B}_\mathbf{c}^n\left(\mathbf{o}_\mathbf{c}, r_\mathbf{c}\right)$, i.e., $H_\mathbf{c}$ is a subset of $\mathbb{B}_\mathbf{c}^n\left(\mathbf{o}_\mathbf{c}, r_\mathbf{c}\right)$. $\mathbb{B}_\mathbf{c}^n$ is uniquely given by
\begin{equation}
    \mathbb{B}_\mathbf{c}^n = \mathbb{B}^n\left(  \frac{\left(1+\|\mathbf{c}\|^{2}\right)}{2\|\mathbf{c}\|}\mathbf{c},  \frac{1-\|\mathbf{c}\|^{2}}{2\|\mathbf{c}\|}\right)
\end{equation}
\end{proposition}
\textit{Proof sketch.} The key idea is to solve a quadratic equation given by the fact that the radius of $\mathbb{B}_\mathbf{c}^n$, the radius of $\mathbb{D}^n$, and the distance from the center of $\mathbb{D}^n$ to the center of $\mathbb{B}_\mathbf{c}^n$ must satisfy the Pythagorean theorem~\cite{kadison2002pythagorean}.
Full proof is in the supplementary material.



\subsection{Geometric interpretation}

Our main idea is to transform the logical relationships between labels into geometric relationships between their corresponding enclosing $n$-balls.
In particular, the implication is modeled by the geometric insideness while the mutual exclusion is modeled by the geometric disjointness. 

\begin{figure}
    \centering
    {\includegraphics[width=\textwidth]{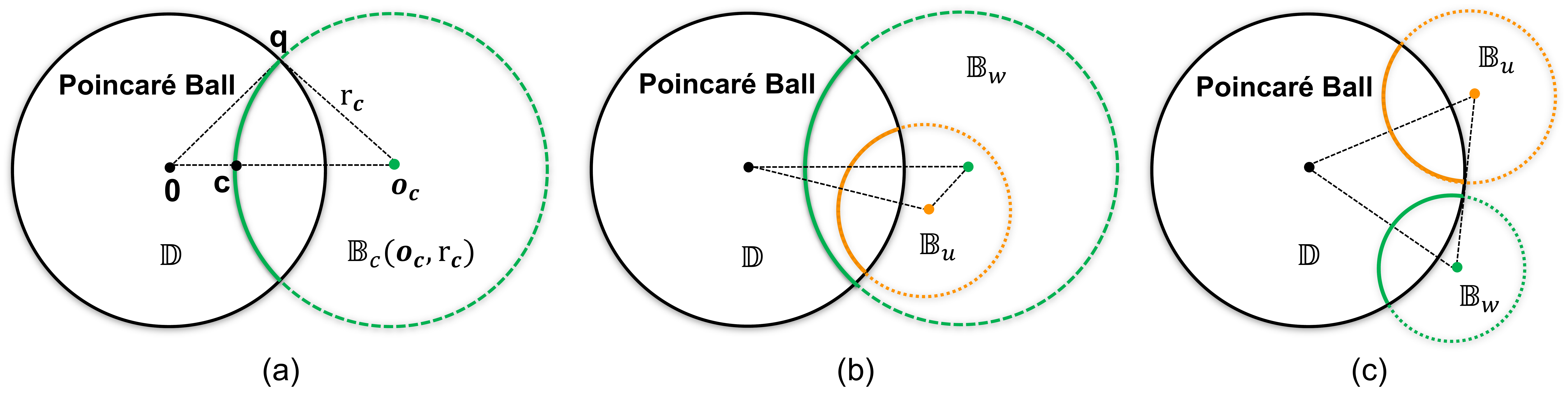}}
    \caption{(a) A Poincaré hyperplane is defined as the intersection between the Poincaré ball $\mathbb{D}$ and the boundary of an $n$-ball $\mathbb{B}_\mathbf{c}$. The Poincaré hyperplane is uniquely parameterized by a center point $\mathbf{c}$, and the corresponding $n$-ball (its radius and center) can be uniquely determined by \cref{prop:p2ball}.
    (b) Label implication is interpreted as $n$-ball insideness. (c) Mutual exclusion is interpreted as $n$-ball disjointedness.}
    \label{fig:geocons}
    \vspace{-0.4cm}
\end{figure}

\textbf{Implication} 
The logical implication between two labels is interpreted as geometric relations between $n$-balls, i.e., $n$-ball insideness illustrated in \cref{fig:geocons}(b).
In particular, an $n$-ball $\mathbb{B}_\mathbf{w}\left(\mathbf{o}_{\mathbf{w}}, r_{\mathbf{w}}\right)$ contains $\mathbb{B}_\mathbf{u}\left(\mathbf{o}_{\mathbf{u}}, r_{\mathbf{u}}\right)$ if and only if $\left\|\mathbf{o}_{\mathbf{u}}-\mathbf{o}_{\mathbf{w}}\right\|+r_{\mathbf{u}}<r_{\mathbf{w}}$, and thus we can create an insideness loss defined by

\begin{equation}\label{eq:implication}
\mathcal{L}_{\text {inside}}(\mathbb{B}_\mathbf{u}, \mathbb{B}_\mathbf{w}  )  = 
\max\{0, \left\|\mathbf{o}_{\mathbf{u}}-\mathbf{o}_{\mathbf{w}}\right\| + r_{\mathbf{u}} - r_{\mathbf{w}} \}.
\end{equation}
Clearly, the insideness loss term satisfies the properties of correctness and transitivity
\begin{lemma}[Correctness]
$\mathbb{B}_\mathbf{u}$ is inside of $\mathbb{B}_\mathbf{w} $ if and only if $\mathcal{L}_{\text {inside}}(\mathbb{B}_\mathbf{u}, \mathbb{B}_\mathbf{w} ) = 0$. 
\end{lemma}
\begin{lemma}[Transitivity]
If $\mathcal{L}_{\text {inside}}(\mathbb{B}_\mathbf{u}, \mathbb{B}_\mathbf{w} ) = 0$ and $\mathcal{L}_{\text {inside}}(\mathbb{B}_\mathbf{w}, \mathbb{B}_\mathbf{v} ) = 0$, we have $\mathcal{L}_{\text {inside}}(\mathbb{B}_\mathbf{u}, \mathbb{B}_\mathbf{v} ) \leq \mathcal{L}_{\text {inside}}(\mathbb{B}_\mathbf{u}, \mathbb{B}_\mathbf{w} ) + \mathcal{L}_{\text {inside}}(\mathbb{B}_\mathbf{w}, \mathbb{B}_\mathbf{v} ) \leq  \mathcal{L}_{\text {inside}}(\mathbb{B}_\mathbf{w} $ $, \mathbb{B}_\mathbf{v} ) = 0$.
\end{lemma}

\textbf{Mutual exclusion}
Similarly, we interpret mutual exclusion as geometric disconnectedness between $n$-balls illustrated in \cref{fig:geocons}(c).
$\mathbb{B}_\mathbf{u}$ disconnecting from $\mathbb{B}_\mathbf{w}$ can be measured by subtracting the distance between their center points from the sum of their radii. Inversely, the corresponding loss is
\begin{equation}\label{eq:exclusion}
    \mathcal{L}_{\text{disjoint}}(\mathbb{B}_\mathbf{u}, \mathbb{B}_\mathbf{w}) =  \max\{0, r_{\mathbf{w}} + r_{\mathbf{u}} - \|\mathbf{o}_{\mathbf{u}}- \mathbf{o}_{\mathbf{w}}\|\}
\end{equation}
Again, the disjointedness loss term satisfies the correctness property
\begin{lemma}[Correctness]
$\mathbb{B}_\mathbf{u}$ disconnects from $\mathbb{B}_\mathbf{w}$ if and only if $\mathcal{L}_{\text{disjoint}}(\mathbb{B}_\mathbf{u}, $ $\mathbb{B}_\mathbf{w}) = 0$.
\end{lemma}

\subsection{Classification and learning}\label{sec:classification}

Given the embeddings of instances and labels, an instance can be classified by measuring the \emph{geometric membership}, i.e., the confidence of a point $\mathbf{p} \in \mathbb{D}^n$ being inside the enclosing ball $\mathbb{B}$. 

\textbf{Membership and non-membership} Formally, given an instance embedding $\mathbf{p} \in \mathbb{D}^n$ and a label embedding associated with an enclosing $n$-ball $\mathbb{B}_{\mathbf{c}}$. The confidence of an instance $\mathbf{p}$ being inside the enclosing $n$-ball $\mathbb{B}_\mathbf{c}$ can be measured by subtracting the distance between the center point of $\mathbb{B}_\mathbf{c}$ and $\mathbf{p}$ from the radius of $\mathbb{B}_\mathbf{c}$. The corresponding loss is defined as the inverse of the measure, given by
\begin{equation}\label{eq:membership}
\mathcal{L}_{\text {membership}}\left(\mathbf{p}, \mathbb{B}_\mathbf{c}\left(\mathbf{o}_\mathbf{c}, r_\mathbf{c}\right) \right)  = \max\{0, \|\mathbf{o}_\mathbf{c}-\mathbf{p}\| - r_\mathbf{c}\}.
\end{equation}
Symmetrically, for negative instance-label relations, the loss of non-membership can be defined as 
\begin{equation}\label{eq:non_membership}
\mathcal{L}_{\text {non-membership}}\left(\mathbf{p}, \mathbb{B}_\mathbf{c}\left(\mathbf{o}_\mathbf{c}, r_\mathbf{c}\right) \right)  = \max\{0, r_\mathbf{c} - \|\mathbf{o}_\mathbf{c}-\mathbf{p}\|\}.
\end{equation}
Clearly, we have the following properties that follow directly from the definitions. 
\begin{lemma}\label{lemma:4}
A point $\mathbf{p}$ is a member of $\mathbb{B}_\mathbf{c}$ if and only if $\mathcal{L}_{\text {membership}}\left(\mathbf{p}, \mathbb{B}_\mathbf{c}\right) = 0$. 
\end{lemma}
\begin{lemma}\label{lemma:5}
A point $\mathbf{p}$ is not a member of $\mathbb{B}_\mathbf{c}$ if and only if $\mathcal{L}_{\text {non-membership}}\left(\mathbf{p}, \mathbb{B}_\mathbf{c} \right)$ $ = 0$. 
\end{lemma}

\textbf{Lemma 1-2, Lemma 3, Lemma 4-5} immediately follow the definitions of geometric insideness, disjointedness, and membership, respectively.

We aim to learn an encoder $E_{\theta}$ (i.e., a hyperbolic neural network whose designs depend on the datasets), where $\theta$ is the trainable parameter, and a function $\mathcal{C}$ which maps labels to the center points of the corresponding Poincaré hyperplanes in the Poincaré ball. 

Now, we define
\begin{equation}
    h(x, l) =  \sigma \left( \mathcal{L}_{\text {non-membership}}\left(E_{\theta}(x), C(l)\right) - \mathcal{L}_{\text {membership}}\left(E_{\theta}(x), C(l)\right) \right)
\end{equation}
as our ranking function, where $\sigma$ is the sigmoid function. 
The final classification function is then defined by $f(x) = \{l \mid h\left(x, l\right) \geq 0.5\}$. 
We call our classifier \emph{hyperbolic multi-label embedding inference} (HMI). Given a HEX graph, HMI has the following guarantee. 

\begin{proposition}[HEX-property]
The classification function $f$ of HMI has the HEX property with respect to $G$ if for every constraint in $G$, the corresponding loss term is $0$. 
\end{proposition}

\textbf{Learning with soft constraints}

Let $\mathcal{D}^{\text{+}}=\{ (x_i, l_n) |(x_i,L_i) \in \mathcal{D}, l_n \in L_i \}$ be the set of positive instance-label pairs and  $\mathcal{D}^{\text{-}}=\{ (x_i, l_n) |(x_i,L_i) \in \mathcal{D}, l_n \in \mathcal{L}, l_n \notin L_i \}$ be the set of negative instance-label pairs.
By combining the loss functions of membership, non-membership, insideness and disjointedness, the learning objective can be formulated as 


\small
\begin{align} 
  \displaystyle 
  \min_{\theta, \mathcal{C}} & 
  \sum_{ \left(x_i, l_n\right) \in \mathcal{D}^{+}  } 
  \mathcal{L}_{\text {membership}}
  \left(E_{\theta}\left(x_i\right), \mathbb{B}_{\mathcal{C}\left(l_n\right)}\right) \\
  & + 
  \sum_{ \left(x_i, l_n\right) \in \mathcal{D}^{-}  } 
  \mathcal{L}_{\text {non-membership}}
  \left(E_{\theta}\left(x_i\right), \mathbb{B}_{\mathcal{C}\left(l_n\right)}\right) \nonumber \\
  & + \lambda 
  \left( 
        \sum_{\left(v_i,v_j\right) \in E_h } 
            \mathcal{L}_{inside}\left(
                \mathbb{B}_{\mathcal{C}\left(l_i\right)}, \mathbb{B}_{\mathcal{C}\left(l_j\right)}
            \right) + 
        \sum_{\left(v_i,v_j\right) \in E_e } 
            \mathcal{L}_{disjoint}\left(
                \mathbb{B}_{\mathcal{C}\left(l_i\right)}, \mathbb{B}_{\mathcal{C}\left(l_j\right)}
            \right) 
    \right)
\end{align}\label{eq:overall_objective}
The first two terms are losses for positive and negative samples while the last two terms are implication and exclusion constraints, respectively, with $\lambda$ being the penalty weight of the constraints.

The following corollary shows that our model has a strong inductive bias for preserving consistency.

\begin{corollary}
Given a HEX graph $G$ of labels, if the loss terms $\mathcal{L}_{inside}$ and $\mathcal{L}_{disjoint}$ are $0$, then the learned prediction function is logically consistent.  
\end{corollary}

\textbf{Classification via hyperbolic logistic regression}
A key advantage of our method is that the losses of constraints are compatible with other (margin-based) hyperbolic classifiers such as hyperbolic logistic regression (HLR)~\cite{ganea2018hyperbolic} and hyperbolic support vector machine (HSVM)~\cite{cho2019large}.
In our experiment we explore HLR, 
which formulates the \emph{logits} as the distances from an instance to a Poincaré hyperplane of a label. That is, $h(x,l) = d\left(E_{\theta}\left(x\right), H_C\left(l\right)\right)$. $d(\mathbf{p}, H_\mathbf{c})$ has the following closed form: 
\begin{equation}
    d(\mathbf{p}, H_\mathbf{c}) =\sinh ^{-1}\left(\frac{2|\langle(-\mathbf{c}) \oplus \mathbf{p}, \mathbf{c}\rangle|}{\left(1-\|(-\mathbf{c}) \oplus \mathbf{p}\|^{2}\right)\|\mathbf{c}\|}\right)
\end{equation}
where $\oplus$ is the Möbius addition \cite{ganea2018hyperbolic}. The classifier is defined by $f(x) = \{l | \operatorname{\sigma}\left(h\left(x,l\right) \right) \geq 0.5, \forall l \in \mathcal{L} \}$ where $\sigma$ is the sigmoid function. We dub such classifier combined with HMI as HMI+HLR. 

\section{Evaluation}

\subsection{Experiment setup}
\textbf{Datasets} 
We consider $12$ datasets that have been used for evaluating multi-label prediction methods~\cite{patel2021modeling, DBLP:conf/nips/GiunchigliaL20, DBLP:conf/icml/WehrmannCB18}. These consist of $8$ functional genomic datasets \cite{clare2003machine}, $3$ image annotation datasets \cite{DBLP:journals/pr/DimitrovskiKLD11,DBLP:journals/ecoi/DimitrovskiKLD12}, and $1$ text classification dataset \cite{DBLP:conf/ecml/KlimtY04}. All input features are pre-processed in the same way as described by Patel et al.~\cite{patel2021modeling}. 
For all datasets, the implication constraints (label taxonomy) are given. 
Following Mirzazadeh et al.~\cite{DBLP:conf/nips/MirzazadehRDS15} we add exclusion constraints between sibling nodes whenever this does not create a contradiction (i.e., they share no common descendant nodes). 
We also explore other strategies for deriving exclusions, but no significant difference was observed (see the supplement for an analysis). 
Similar to MBM~\cite{patel2021modeling} and its baselines, we sample $30\%$ of the implications and exclusions constraints for training the model. 

\textbf{Hyperbolic encoder} 
We adopt a simple hyperbolic linear layer as the instance encoder for all datasets. 
A single-layer hyperbolic fully-forward linear layer is defined by $f_{ \theta= \{\mathbf{W},\mathbf{b}\}}(\mathbf{x})=\tanh ^{\otimes}\left(\mathbf{W} \otimes \mathbf{x} \oplus \mathbf{b}\right)$, with $\otimes$ being Möbius matrix-vector multiplication defined by 
$M \otimes \mathbf{x}=\tanh \frac{\|M \mathbf{x}\|}{\|\mathbf{x}\|} \tanh ^{-1} $ $ (\|\mathbf{x}\|)) \frac{M \mathbf{x}}{\|M \mathbf{x}\|}$, where $\mathbf{W} \in \mathbb{R}^{n \times d}$ is a trainable matrix and $\mathbf{x}$ is a point $\mathbf{x} \in$ $\mathbb{D}^{n}, M \mathbf{x} \neq 0$. 
$\oplus$ denotes Möbius addition given by
\begin{equation}
    \mathbf{x} \oplus \mathbf{y}=\frac{\left(1+2\langle \mathbf{x}, \mathbf{y}\rangle+\|\mathbf{y}\|^{2}\right) \mathbf{x}+\left(1-\|\mathbf{x}\|^{2}\right) \mathbf{y}}{1+2\langle \mathbf{x}, \mathbf{y}\rangle+\|\mathbf{x}\|^{2}\|\mathbf{y}\|^{2}}
\end{equation}
and $\tanh ^{\otimes}$ denotes an Möbius version of pointwise non-linearity given by $\tanh ^{\otimes}(\mathbf{x})= \exp_{0}\left(\tanh\left(\log_{0}(\mathbf{x})\right)\right)$, with $\exp_{0}$ and $\log_{0}$ being the exponential and logarithmic maps, see \cite{ganea2018hyperbolic} for more details.

\textbf{Baselines} 
We compare our approach with both classical vector-based and state-of-the-art region-based embedding methods. 
In particular, we consider two vector-based models: 1) The multi-label vector model (MVM)~\cite{DBLP:conf/acl/HenaoLCSWWZZ18}, which  encodes both inputs and labels as Euclidean vectors; 2) the multi-label hyperbolic model (MHM) used by Chen et al.~\cite{DBLP:conf/aaai/ChenHXCJ20}, which represents inputs and labels as hyperbolic points; and two box models: 3) the non-probabilistic box model (BoxE)~\cite{abboud2020boxe} and 4) the probabilistic multi-label box model (MBM)~\cite{patel2021modeling} that encodes both instances and labels as axis-parallel hyper-rectangles. 
Besides, we compare with 5) hyperbolic logistic regression (HLR)~\cite{ganea2018hyperbolic} since it also encodes labels as Poincaré hyperplanes (but does not use geometric constraints). 
Furthermore, we compare with 6) C-HMCNN, a state-of-the-art non-embedding based method that injects hierarchy constraints directly into the loss function without embedding labels. A notable difference is that C-HMCNN needs the full hierarchy constraints as its input. 
Finally, we also implement HMI+HLR, a combination of our proposed constraints with HLR for an ablation study.

\textbf{Implementation details} 
We implement HMI, HLR and HMC-HLR using PyTorch~\cite{paszke2019pytorch} and train the models on NVIDIA A100 with 40GB memory.
We train HMI, HLR and HMI+HLR using Riemannian Adam~\cite{DBLP:conf/iclr/BecigneulG19} optimizer implemented by the Geoopt library \cite{DBLP:journals/corr/abs-2005-02819} with a batch size of $4$.
We also explore some larger batch sizes but it does not yield better results, which is also observed in~Wehrmann et al.\cite{DBLP:conf/icml/WehrmannCB18}. 
We set the dropout rate to $0.6$ suggested by~\cite{DBLP:conf/icml/WehrmannCB18} to avoid the case that the model overfits the small training sets. We employ an early-stopping strategy with patience $20$ to save training time. 
The results of other baselines are as reported by Patel et al.\cite{patel2021modeling} that we closely follow. 
The learning rate is searched from \{$1e-4$, $5e-4$, $1e-3$, $5e-3$, $1e-2$\}. 
The penalty weight of the violation is searched from \{$1e-5$, $5e-4$, $1e-4$, $5e-3$, $1e-2$\} and we also show its impact in an ablation. The best dimension per dataset is searched from $\{32, 64, 128, 256\}$, which is one order of magnitude lower than that used by Patel et al. \cite{patel2021modeling} ($\{250, 500,1000,1750\}$). 
All methods have been run $10$ times with random seeds and the average results are reported. We omit the standard deviations since they are in a very small range ($[2 \times 10^{-4}, 2.3 \times 10^{-3}]$). 

\textbf{Evaluation protocols}
In line with Patel et al.~\cite{patel2021modeling}, we consider \emph{Mean Average Precision (mAP)},\footnote{\url{https://scikit-learn.org/stable/modules/generated/sklearn.metrics.average_precision_score.html}} 
which summarizes the information of precisions and recalls with varied thresholds. We also report two metrics that additionally take the constraints into account: 1) Constrained mAP (CmAP) is a variant of mAP that replaces the score of each label with the maximum scores of its descendant labels in the hierarchy \cite{patel2021modeling}. 
2) \emph{Hierarchy Constraint Violation (HCV)} \cite{patel2021modeling} measures the extent to which the label scores violate the implication constraints regardless of true labels for the instances. $\mathrm{HCV}$ is computed as $\mathrm{HCV}=\frac{1}{|\mathcal{D}||E_h|} \sum_{k=1}^{|\mathcal{D}|} \sum_{\left(l_{i}, l_{j}\right) \in E_h} \mathbbm{1}\left(h_{i}^{k}-h_{j}^{k} > 0\right)$, where $h_i$ means the prediction score of label $l_i$.
Clearly, a lower value of $\mathrm{HCV}$ implies higher consistency in the predictions. 

\subsection{Main results}

\begin{table}
    \vspace{-0.1cm}
    \centering
      \caption{Comparison of performance and consistency on $12$ datasets, where \underline{underline} indicates the best results over embedding-based methods, and \textbf{boldface} indicates the best results over all methods. 
      We implemented HMI, HLR and HMI+HLR. Other results are taken from Patel et al.~\cite{patel2021modeling}. All metrics are averaged across $10$ runs with random seeds (standard deviations are relatively small (in range $[2 \times 10^{-4}, 2.3 \times 10^{-3}]$) and are hence omitted).}
    \resizebox{\textwidth}{!}{
    \begin{tabular}{cc|cc|ccccc|c}
    \hline
    \multirow{2}{*}{Dataset} & \multirow{2}{*}{Metric} & \multicolumn{2}{c}{ Ours }  & \multicolumn{5}{c}{ Embeddings }  & Non-embedding \\
    &  & \textbf{HMI} & \textbf{HMI+HLR} & MVM & MHM & BoxE & MBM & HLR & C-HMCNN \\
    \hline \multirow{3}{*}{ ExprFUN } 
    & mAP $\uparrow$ & \textbf{\underline{38.53}} & 38.50 &  37.94 & 31.90 & 37.30 & 38.42 & 37.98 & 38.41   \\
    & CmAP $\uparrow$ & \textbf{\underline{38.72}} & 38.62 &  37.41 & 32.02 & 37.92 & 38.67 & 37.44 & 38.41   \\
    & HCV $\downarrow$ & \underline{0.92} & 1.07 & 1.97 & 1.92 & 4.79 & 1.87 & 2.17 & \textbf{0}   \\
    \hline \multirow{3}{*}{ CellcycleFUN } 
    & mAP $\uparrow$ & 34.82 & \textbf{\underline{34.84}} & 31.61 & 28.74 & 31.96 & 34.61 & 34.05 & 34.35  \\
    & CmAP $\uparrow$ & 34.90 & \textbf{\underline{35.00}} & 31.33 & 28.89 & 32.70 & 34.78 & 34.11 & 34.35  \\
    & HCV $\downarrow$ & \underline{\underline{1.30}} & 1.32 & 3.45 & 1.78 & 4.02 & 1.35 & 2.30 & \textbf{0}   \\
    \hline \multirow{3}{*}{ DerisiFUN } 
    & mAP $\uparrow$ & \textbf{\underline{36.71}} & \textbf{\underline{36.71}} &  24.16 & 24.40 & 26.66 & 28.71 & 26.65 & 28.19  \\
    & CmAP $\uparrow$ & \textbf{\underline{36.94}} & 36.89 & 24.35 & 24.52 & 26.96 & 28.88 & 26.83 & 28.19 \\
    & HCV $\downarrow$ & \underline{0.73} & 0.87 & 4.01 & 0.85 & 2.27 & 1.43 & 2.30 & \textbf{0} \\
    \hline \multirow{3}{*}{ SpoFUN } 
    & mAP $\uparrow$  & \textbf{\underline{36.47}} & 36.44 &  24.21 & 26.57 & 27.97 & 29.62 & 28.29 & 29.18  \\
    & CmAP $\uparrow$ & 36.43 & \textbf{\underline{36.54}} & 24.55 & 26.79 & 28.38& 29.78 & 28.31 & 29.18  \\
    & HCV $\downarrow$ & \underline{0.92} & 1.05 & 4.73 & 1.69 & 2.75 & 1.53 & 1.98  & \textbf{0}  \\
    \hline \multirow{3}{*}{ ExprGO } 
    & mAP $\uparrow$ & 48.63 & 48.50 &  44.97 & 40.52 & 46.75 & 48.45 & \textbf{\underline{48.65}} & 48.61  \\
    & CmAP $\uparrow$ & \textbf{\underline{48.68}} & 48.61 & 41.84 & 40.70 & 47.28& 48.56 & 48.65 & 48.61  \\
    & HCV $\downarrow$ & 1.37 & 1.45 & 7.05 & 5.19 & 5.74 & 1.91 & \underline{1.35} & \textbf{0} \\
    \hline \multirow{3}{*}{ CellcycleGO } 
    & mAP $\uparrow$ & \underline{45.58} & 45.51 &  44.19 & 39.74 & 43.08 & 44.93 & 40.28 & \textbf{45.61}  \\
    & CmAP $\uparrow$ & \underline{45.58} & 45.53 & 41.02 & 39.76 & 43.79 & 45.01 & 40.30 & \textbf{45.61}  \\
    & HCV $\downarrow$ & 1.19 & \underline{1.12}  & 3.03 & 2.49 & 5.06 & 2.16 & 3.26 & \textbf{0}  \\
    \hline \multirow{3}{*}{ DerisiGO } 
    & mAP $\uparrow$ & \textbf{\underline{42.31}} & 42.12 &  41.13 & 40.10 & 40.44 & 42.02 & 40.33 & 42.24 \\
    & CmAP $\uparrow$ & \textbf{\underline{42.38}} & 42.28 & 38.21 & 40.20 & 40.73 & 42.12 & 40.35 & 42.24  \\
    & HCV $\downarrow$ & \underline{0.86} & 0.99 & 3.46 & 2.02 & 3.16 & 1.13 & 2.31 & \textbf{0}  \\
    \hline \multirow{3}{*}{ SpoGO } 
    & mAP $\uparrow$ & 42.70 & \underline{42.74} &  42.20 & 39.70 & 40.88 & 41.74 & 39.22 & \textbf{42.77}  \\
    & CmAP $\uparrow$ & 42.76 & \textbf{\underline{42.77}} & 39.04 & 39.77 & 41.27 & 41.54 & 39.26 & \textbf{42.77}  \\
    & HCV $\downarrow$ & \underline{0.95} & 1.20 & 2.77 & 1.90 & 3.89 & 1.80  & 2.33  & \textbf{0}\\
      \hline \multirow{3}{*}{  Enron} 
    & mAP $\uparrow$ & 80.43 & 80.43 &  73.68 & 75.62 & \textbf{\underline{80.44}} & 80.06 & 78.87 & 80.04  \\
    & CmAP $\uparrow$ & \textbf{\underline{80.50}} & 80.47 & 66.87 & 75.68 & 80.46 & 80.05 & 78.94 & 80.04  \\
    & HCV $\downarrow$ & \textbf{\underline{0}} & \textbf{\underline{0}} & 2.53 & 0.36 & 0.20 & 0.03 & 0.04 & \textbf{0}  \\
    \hline \multirow{3}{*}{ Diatoms } 
    & mAP $\uparrow$  & \textbf{\underline{79.19}} & 79.10 &  72.65 & 56.86 & 43.71 & 79.14 & 77.90 & 76.23  \\
    & CmAP $\uparrow$ & \textbf{\underline{79.40}} & 79.36 & 72.18 & 56.07  & 45.16& 79.23 & 78.07 & 76.23   \\
    & HCV $\downarrow$ & \underline{0.17} & 0.18 & 19.20 & 5.55 & 6.39 & 0.34 & 6.36 & \textbf{0} \\
     \hline \multirow{3}{*}{ Imclef07a } 
    & mAP $\uparrow$ & \textbf{\underline{90.67}} & 89.60 &  78.22 & 65.30 & 83.71 & 69.26 & 88.33 & 90.26 \\
    & CmAP $\uparrow$ & \textbf{\underline{90.89}} & 89.71 & 77.46 & 66.01 & 84.73 & 69.48 & 88.45 & 90.26  \\
    & HCV $\downarrow$ & 0.20 & \underline{0.19} & 22.86 & 4.75 & 12.73 & 2.40 & 1.77 &  \textbf{0}  \\
     \hline \multirow{3}{*}{ Imclef07d } 
    & mAP $\uparrow$ & 89.19 & 89.20 &  88.59 & 75.69 & 87.95 & \textbf{\underline{89.56}} & 88.91 & 89.22  \\
    & CmAP $\uparrow$ & 90.00 & 90.02 & 86.87 & 76.95 & 88.93 & \textbf{\underline{90.07}} & 87.38 & 89.22  \\
    & HCV $\downarrow$ & 0.37 & \underline{0.36} & 11.02 & 7.56 & 11.93 & 5.66  & 6.88 & \textbf{0}  \\
    \hline \multirow{3}{*}{ \textbf{Avg. Rank $\downarrow$} } 
    & mAP  & \textbf{\underline{1.75}} & 2.42 & 6.33 & 7.58 & 5.75 & 3.5 & 5.25 & 3.08   \\
    & CmAP  & \textbf{\underline{1.58}} & 2.08 & 7.16 & 7.41 & 5.25 & 3.58 & 5.41 & 3.33 \\
    & HCV & \underline{2.25} & 2.75 & 7.42 & 5.25 & 7.25 & 4.25 & 5.58 & \textbf{1.00} \\
    \hline
    \end{tabular}
    }
    \label{tab:overall_results}
    \vspace{-0.3cm}
\end{table}

As \cref{tab:overall_results} shows, our method HMI either achieves the best (7-8/12 datasets) or competitive (4-5/12 datasets) performance (mAP and CmAP) over all compared methods. 
HMI outperforms all methods w.r.t the average ranking of mAP/CmAP, showcasing the advantages of HMI. We observed that the CmAP is close to mAP, indicating that the model is adhering to the label constraints \cite{patel2021modeling}. In terms of predictive consistency (HCV), HMI consistently achieves the best or the second-best results. Note that C-HMCNN always gets zero HCV because it exploits the complete hierarchy. HMI achieves competitive HCV, despite only using 30\% of the hierarchy. 

\begin{figure}
    \centering
    \subfloat[\centering mAP]{{\includegraphics[width=0.33\textwidth]{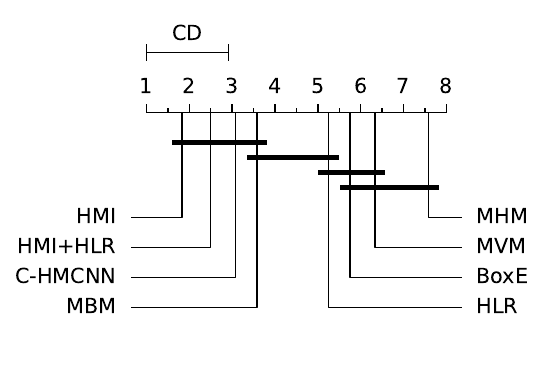}}}
    \subfloat[\centering CmAP]{{\includegraphics[width=0.33\textwidth]{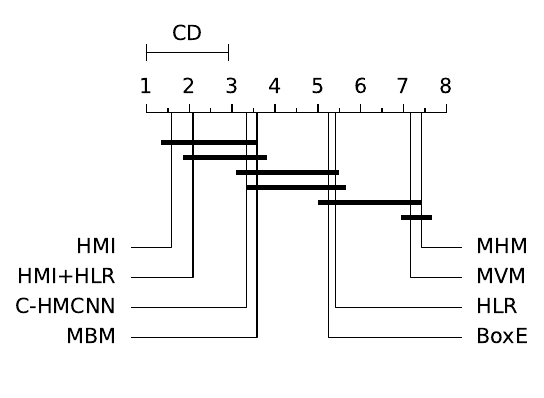}}}
    \subfloat[\centering HCV]{{\includegraphics[width=0.33\textwidth]{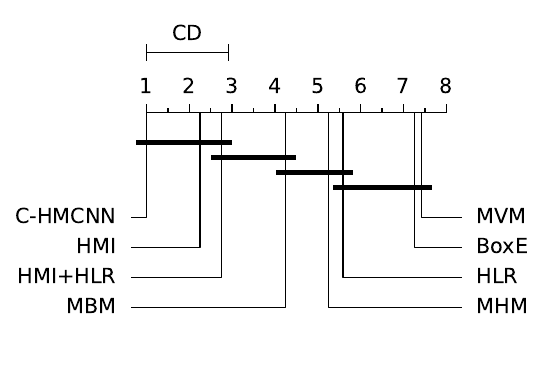}}}
    \caption{Critical diagrams of the post-hoc Nemenyi test across all 12 datasets.}
    \label{fig:critical_diagrams}
    \vspace{-0.3cm}
\end{figure}

\begin{table}{r}
\centering
 \caption{Results of Wilcoxon test over HMI against baselines.}
    \begin{tabular}{cc|ccc}
        \hline
        & Method & mAP & CmAP & CV \\
        \hline
        & 
        HMI vs C-HMCNN & $5.8 \times 10^{-4}$ & $4.4 \times 10^{-4}$ & $5.0 \times 10^{-3}$ \\
        & HMI vs MBM & $3.3 \times 10^{-4}$ & $2.4 \times 10^{-4}$ & $4.9 \times 10^{-4}$ \\
        & HMI vs HMI+HLR & $2.3 \times 10^{-2}$ & $3.8 \times 10^{-2}$ & $9.7 \times 10^{-1}$ \\
        \hline
    \end{tabular}
    \label{tab:wilcoxon_differnece}
\end{table}

\textbf{Statistical significance} Following~Patel et al.~\cite{patel2021modeling} and Giunchiglia and
Lukasiewicz~\cite{DBLP:conf/nips/GiunchigliaL20}, we test the statistical significance of the performance across all datasets. 
First, we perform the Friedman test~\cite{DBLP:journals/jmlr/Demsar06} and show that there exists a significant difference w.r.t. all metrics with p-values $\ll 0.05$. Next, we conduct the post-hoc Nemenyi test to verify the statistical differences of the average ranking. The critical diagram w.r.t the average ranking of mAP/CmAP is shown in \cref{fig:critical_diagrams}, in which the methods that have no significant differences (significance level $0.05$) are connected by a horizontal line. As shown in the diagrams, it is clear to conclude that there is a statistically significant difference w.r.t mAPs/CmAPs of HMI and HMI+HLR against MVM, BoxE, MHM, and HLR but not the two strong baselines (MBM and C-HMCNN). 
We further perform the Wilcoxon test that considers not only the differences in rankings but also the numerical differences in the performance. 
The Wilcoxon test results show that there is a statistically significant difference between the mAPs/CmAPs of HMI and the two strong baselines with p-value $\ll 0.05$.
In terms of HCV, our statistical significance test in \Cref{fig:critical_diagrams} and \cref{tab:wilcoxon_differnece} shows that HMI and HMI+HLR significantly outperform MVM, BoxE, MHM, HLR, and MBM but not C-HMCNN since it has zero HCV.
However, we observed that the predictive performance (mAP, CmAP) is not fully proportional to the HCV, e.g., HMI outperforms C-HMCNN w.r.t. mAP/CmAP on many of the datasets even though C-HMCNN has zero CV.

\textbf{Classification via hyperbolic logistic regression} To validate whether our proposed geometric constraints are able to improve hyperbolic logistic regression (HLR) \cite{ganea2018hyperbolic}, we implement HMI+HLR, a combination of our proposed constraints with HLR as described in \Cref{sec:classification}. \Cref{tab:overall_results} show that HMI+HLR outperforms HLR with statistical confidence, showcasing that HMI is able to improve the predictive performance and consistency of HLR. However, there is no significant difference (with $p$-value larger than $0.05$ in \Cref{tab:wilcoxon_differnece}) between the two variants of our method (HMI and HMI+HLR).




\subsection{Ablation studies \& parameter sensitivity. } 

For further ablation, we introduce one additional metric. Exclusion Constraint Violation (ECV) measures, analogous to HCV, the  fraction of the exclusion constraints violated by the predictions i.e., $\mathrm{ECV}=\frac{1}{|\mathcal{D}||E_e|} \sum_{k=1}^{|\mathcal{D}|} \sum_{\left(l_{i}, l_{j}\right) \in E_e} $ $ \mathbbm{1}\left( f_{i}^{k} \wedge f_{j}^{k}\right)$. 
We introduce this because HCV can be made zero trivially by associating all labels with the same score. Hence, in the ablation study, we will show how the exclusion constraints (the results of ECV) complement HCV and influence the overall performance. 

\begin{wraptable}{r}{0.62\linewidth}
\vspace{-0.3cm}
\centering
    \caption{Impact of violation penalty weight $\lambda$ on CellcycleFUN and CellcycleGO dataset.}
    \resizebox{\linewidth}{!}{
    \begin{tabular}{ccccccc}
    \hline
    Dataset & Metric & $\lambda=0.0$ & $\lambda=0.001$ & $\lambda=0.005$ & $\lambda=0.01$ &  $\lambda=0.1$ \\
    \hline \multirow{4}{*}{CellcycleFUN} 
    & mAP & 33.87 & 34.78 & \textbf{34.82} & 34.76 & 32.28 \\
    & CmAP & 34.03 & 34.83 & \textbf{34.90} & 34.85 & 33.75 \\
    & HCV & 2.33 & 1.87 & 1.30 & 1.04 & \textbf{0.75} \\
    & ECV & 4.33 & 3.77 & 2.40 & 1.67 & \textbf{1.35} \\
    \hline
    \multirow{4}{*}{CellcycleGO} 
    & mAP & 40.26 & 41.47 & \textbf{45.58} & 45.56 & 41.28 \\
    & CmAP & 39.87 & 42.05 & 45.58 & \textbf{45.60} & 40.75 \\
    & HCV & 2.28 & 1.57 & 1.19 & 0.99 & \textbf{0.86} \\
    & ECV & 3.98 & 3.27 & 2.17 & 1.71 & \textbf{1.34} \\
     \hline
    \end{tabular}
    }
    \label{tab:penalty_weight}
\end{wraptable}

\textbf{Impact of penalty weight} 
\Cref{tab:penalty_weight} shows the results of HMI on "CellcycleFUN" and "CellcycleGO" dataset. We observed that with different penalty weights, the obtained results are slightly different. Even without penalty ($\lambda=0$), the model already achieves acceptable results, in particular, it outperforms MVM, MHM, and BoxE, indicating that our hyperbolic model, to some extent, is capable of capturing label hierarchies without any explicit constraints. However, as \cref{tab:penalty_weight} shows, a proper $\lambda=0.001$, $\lambda=0.005$ and $\lambda=0.01$ indeed improves the performance and consistency.
Finally, we observed that increasing $\lambda$ to $0.1$, though further improves consistency (HCV and ECV), does not further improve mAP and CmAP. We conjecture that this is because a large $\lambda$ would encourage the model to "overfit" the given constraints while "underfitting" the classification loss.

\begin{wraptable}{r}{0.62\linewidth}
\vspace{-0.3cm}
\centering
    \caption{Impact of implication and exclusion constraints on CellcycleFUN and CellcycleGO dataset.}
     \resizebox{\linewidth}{!}{
    \begin{tabular}{ccccccc}
    \hline
    Dataset & Metric & HMI & w/o implication & w/o exclusion & non constraints \\
    \hline \multirow{2}{*}{CellcycleFUN} 
    & mAP & \textbf{34.82} & 34.70 & 34.74 & 33.87\\
    & CmAP & \textbf{34.90} & 34.75 & 34.82 & 34.03\\
    & HCV & 1.30 & 2.34 & 1.45 & 2.33\\
    & ECV & 2.40 & 2.67 & 3.63 & 4.33 \\
    \hline
    \multirow{2}{*}{CellcycleGO} 
    & mAP & \textbf{45.58} & 42.56 & 44.50 & 40.26 \\
    & CmAP & \textbf{45.58} & 42.56 & 45.31 & 39.87\\
    & HCV & 1.19 & 2.16 &  1.73 & 2.28\\
    & ECV & 2.17 & 3.68 & 3.07 & 3.98 \\
     \hline
    \end{tabular}
    }
    \label{tab:hierarchy_exclusion}
\end{wraptable}

\textbf{Impact of implication \& exclusion} To study the roles of implication and exclusion. We implemented three variants of HMI by removing either implication or exclusion, or removing both of them. \Cref{tab:hierarchy_exclusion} depicts the results of these variants. It is clear that both implication and exclusion constraints improve the base model that has no constraints. When implication and exclusion are jointly constrained, the performance is significantly improved again. 
We also observed that implication and exclusion constraints, to some extent, do complement each other, e.g., by only using implication (resp. exclusion), the model archives lower ECV (resp. HCV). Finally, we observed that even without exclusion, our model still slightly outperforms MBM, showcasing the advantages of hyperbolic space for modeling hierarchies.

\textbf{Impact of sampling ratio} 
To study whether our method is able to preserve logical constraints from incomplete label constraints we compare the performance of HMI with different ratios for sampling the training constraints. 
As \Cref{fig:ablation_sampling_dimension}(a) depicts, with zero sampling ratio, our method already achieves acceptable results. We conjecture that this is because some constraints can be learned from the data. However, \Cref{fig:ablation_sampling_dimension}(a) clearly shows that including constraints indeed helps to improve the performance. Making the sampling ratio larger than $30$-$40\%$ does not lead to a significant performance gain. We conjecture that this is because certain ratio of training constraints is sufficient for inferring the full set of constraints. 

\begin{wrapfigure}{r}{0.66\linewidth}
    \vspace{-0.5cm}
    \centering 
    \resizebox{\linewidth}{!}{
    \subfloat[\centering ]{{\includegraphics[width=0.46\textwidth]{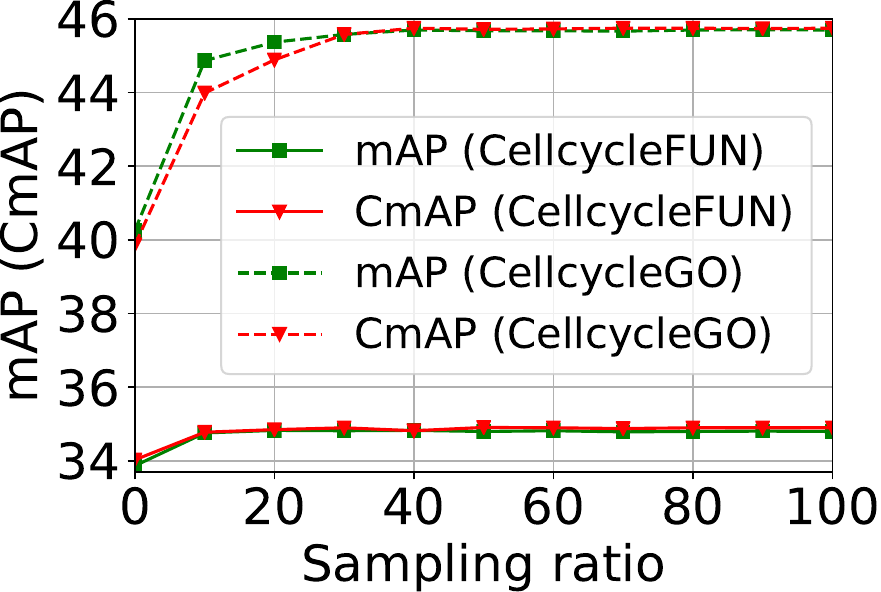}}}
    \qquad
    \subfloat[\centering ]{{\includegraphics[width=0.46\columnwidth]{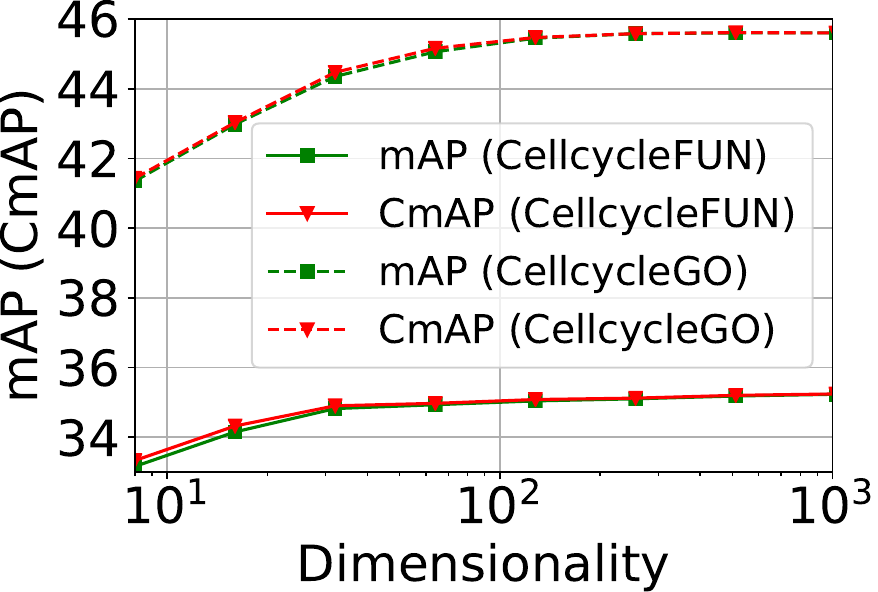}}}
    }
    \caption{(a) The variation of performance w.r.t the sampling ratio. (b) The variation of performance w.r.t the embedding dimensions. }
    \label{fig:ablation_sampling_dimension}
\end{wrapfigure}

\textbf{Impact of embedding dimensionality} 
We study how the choice of dimensionality affects performance. 
As \Cref{fig:ablation_sampling_dimension}(b) depicts, HMI achieves acceptable results even in a very low dimension ($n \leq 100)$. When increasing the dimension an order of magnitude ($n=1000$), the performance grows only slightly. Note that all reported baselines achieved acceptable results with dimensions in $[500,1000,1750]$ (see hyperparameter settings in the Appendix of Patel et al.~\cite{patel2021modeling}). We conjecture that the reason we can achieve good performance with fewer dimensions is that the hyperbolic hyperplane is more suitable for representing hierarchical decision boundaries.

\textbf{Comparison with MBM with only implication or without any constraint} To faithfully study the advantages of hyperbolic hyperplane on modeling label relations than that of the box model (MBM), we also implement two versions of HMI by considering only (30\%) implication constraints and without any constraint (sampling ratio$=0$), respectively. 
Our Wilcoxon test in Table \ref{tab:wilcoxon_differnece_hmi} shows that HMI with only implication and HMI without any constraint still outperform their corresponding counterparts of MBM on CmAP and HCV (with p-value $<0.05$) while achieving comparable results on mAP (i.e., with better average ranks but without statistical significance, we believe this is because mAP is less sensitive to the constraints than CmAP).

\begin{table}[h!]
\centering
\caption{
Results of Wilcoxon test on HMI against MBM in the settings where only implications are available and without any constraint. 
$-$ means no statistical difference between the compared methods. 
}
\resizebox{\textwidth}{!}{
    \begin{tabular}{cc|ccc}
        \hline
        & Method & mAP & CmAP & CV \\
        \hline
        & 
        HMI (impl.) vs MBM (impl.) & $-$ & $2.4 \times 10^{-4}$ & $1.2 \times 10^{-3}$ \\
        & HMI (no conts.) vs MBM (no conts.) & $-$ &$1.3 \times 10^{-2} $ & $6.1 \times 10^{-3}$ \\
        \hline
    \end{tabular}
    }
    \label{tab:wilcoxon_differnece_hmi}
\end{table}

\section{Conclusion}
In this work, we focus on a structured multi-label prediction task whose output is supposed to respect the implication and exclusion constraints. We show that such a problem can be formulated in a hyperbolic Poincaré ball space whose linear decision boundaries (Poincaré hyperplanes) can be interpreted as convex regions. The implication and exclusion constraints are geometrically interpreted as insideness and disjointedness, respectively. 
Experiments on $12$ datasets show significant improvements in mean average precision and lower constraint violations, even with an order of magnitude fewer dimensions than baselines.

\cleardoublepage
\chapter{Geometric Embeddings of High-Order Structures   }
\label{chap_hyper}

In this chapter, we introduce two geometric embeddings for high-order relational knowledge graphs. 
In Section \ref{sec:shrinking}, we introduce ShrinkE, a shrinking embedding for hyper-relational knowledge graphs. 
In Section \ref{sec:facte}, we introduce FactE, an embedding model for knowledge graphs with nested structures.   

\section{Shrinking Embeddings for Hyper-Relational Knowledge Graphs}\label{sec:shrinking}

Link prediction on knowledge graphs (KGs) has been extensively studied on binary relational KGs, wherein each fact is represented by a triple. 
A significant amount of important knowledge, however, is represented by hyper-relational facts where each fact is composed of a primal triple and a set of qualifiers comprising a key-value pair that allows for expressing more complicated semantics.
Although some recent works have proposed to embed hyper-relational KGs, these methods fail to capture essential inference patterns of hyper-relational facts such as qualifier monotonicity, qualifier implication, and qualifier mutual exclusion, limiting their generalization capability. 
To unlock this, we present \emph{ShrinkE}, a geometric hyper-relational KG embedding method aiming to explicitly model these patterns. ShrinkE models the primal triple as a spatial-functional transformation from the head into a relation-specific box. 
Each qualifier ``shrinks'' the box to narrow down the possible answer set and, thus, realizes qualifier monotonicity. 
The spatial relationships between the qualifier boxes allow for modeling core inference patterns of qualifiers such as implication and mutual exclusion. Experimental results demonstrate ShrinkE's superiority on three benchmarks of hyper-relational KGs. 

\subsection{Motivation and Background}

Link prediction on knowledge graphs (KGs) is a central problem for many KG-based applications \cite{zhang2016collaborative,lukovnikov2017neural,DBLP:conf/sigir/jiayinglu,DBLP:conf/semweb/XiongPTNS22,DBLP:conf/wsdm/00010ZZMHK22}. 
Existing works \cite{DBLP:conf/iclr/SunDNT19,DBLP:conf/nips/BordesUGWY13}s have mostly studied link prediction on binary relational KGs, wherein each fact is represented by a triple, e.g., (\emph{Einstein}, \emph{educated\_at}, \emph{University of Zurich}). 
In many popular KGs such as Freebase \cite{bollacker2007freebase}, however, a lot of important knowledge is not only expressed in triple-shaped facts, but also via facts about facts, which taken together are called hyper-relational facts. 
For example, ((\emph{Einstein}, \emph{educated\_at}, \emph{University of Zurich}), \{(\emph{major}:\emph{physics}), (\emph{degree}:\emph{PhD})\}) is a hyper-relational fact, where the primary triple (\emph{Einstein}, \emph{educated\_at}, \emph{University of Zurich}) is contextualized by a set of key-value pairs \{(\emph{major}:\emph{physics}),(\emph{degree}:\emph{PhD})\}. 
Like much other related work, we follow the terminology established for Wikidata \cite{vrandevcic2014wikidata} and use the term \emph{qualifiers} to refer to the key-value pairs.\footnote{Synonyms include statement-level \emph{metadata} in RDF-star \cite{rdf-star} and triple \emph{annotation} in provenance communities \cite{green2007provenance}.} 
The qualifiers play crucial roles in avoiding ambiguity issues. For instance, \emph{Einstein} was \emph{educated\_at} several universities and the qualifiers for \emph{degree} and \emph{major} help distinguish them. 
\par
In order to predict links in hyper-relational KGs, pioneering works represent each hyper-relational fact as either an $n$-tuple in the form of $r(e_1,e_2,\cdots,e_n)$ \cite{abboud2020boxe, DBLP:conf/ijcai/FatemiTV020, DBLP:conf/www/0016Y020} or a set of key-value pairs in the form of $\{(k_i:v_i)\}_{i=1}^m$ \cite{DBLP:conf/www/GuanJWC19,guan2021link,DBLP:conf/www/LiuYL21}. However, these modelings lose key structure information and are incompatible with the RDF-star schema \cite{rdf-star} used by modern KGs, where both primal triples and qualifiers constitute the fundamental data structure. 
Recent works \cite{DBLP:conf/acl/GuanJGWC20,DBLP:conf/www/RossoYC20} represent each hyper-relational fact as a primary triple coupled with a set of qualifiers that are compatible with RDF-star standards \cite{rdf-star}. 
Link prediction is then achieved by modeling the validity of the primary triple and its compatibility with each annotated qualifier \cite{DBLP:conf/acl/GuanJGWC20, DBLP:conf/www/RossoYC20}. More complicated graph encoders and decoders \cite{DBLP:conf/emnlp/GalkinTMUL20,DBLP:journals/corr/abs-2104-08167,DBLP:conf/acl/WangWLZ21,DBLP:journals/corr/abs-2208-14322} are proposed to further boost the performance. However, they require a relatively huge number of parameters that make them prone to overfitting. 

To encourage generalization capability, KG embeddings should be able to model inference patterns, i.e., specifications of logical properties that may exist in KGs, which, if learned, empowers further principled inferences \cite{abboud2020boxe}. 
This has been extensively studied for binary relational KG embeddings \cite{DBLP:conf/icml/TrouillonWRGB16,DBLP:conf/iclr/SunDNT19} but ignored for hyper-relational KGs in which not only primal triples but also qualifiers matter. 
One of the most important properties is qualifier monotonicity.
Given a query, the answer set shrinks or at least does not expand as more qualifiers are added to the query expression. 
For example, a query $\left(\emph{Einstein}, \emph{educated\_at}, ?x\right)$ with a variable $?x$ corresponds to two answers $\{\emph{University of Zurich}, \emph{ETH Zurich}\}$, but a query $\left(\left(\emph{Einstein}, \emph{educated\_at},?x\right),\{\left(\emph{degree}:\emph{B.Sc.}\right)\} \right)$ extended by a qualifier for \emph{degree} will only respond with $\{\emph{ETH Zurich}\}$. Besides, different qualifiers might form logical relationships that the model must respect during inference including qualifier implication (e.g., adding a qualifier that is implicitly implied in the existing qualifiers does not change the truth of a fact) and qualifier mutual exclusion (e.g., adding any two mutually exclusive qualifiers to a fact leads to a contradiction). 
\par
In light of this, we propose ShrinkE, a hyper-relational embedding model that allows for modeling these inference patterns. 
ShrinkE embeds each entity as a point and models a primal triple as a spatio-functional transformation from the head entity to a relation-specific box that entails the possible tails. Each qualifier is modeled as a shrinking of the primal box to a qualifier box. 
The shrinking of boxes simulates the ``monotonicity'' of hyper-relational qualifiers, i.e., attaching qualifiers to a primal triple may only narrow down but never enlarges the answer set. 
The plausibility of a given fact is measured by a point-to-box function that judges whether the tail entity is inside the intersection of all qualifier boxes.
Moreover, since each qualifier is associated with a box, the spatial relationships between the qualifier boxes allow for modeling core inference patterns such as qualifier implication and mutual exclusion. We theoretically show the capability of ShrinkE on modeling various inference patterns including (fact-level) monotonicity, triple-level, and qualifier-level inference patterns. Empirically, ShrinkE achieves competitive performance on three benchmarks.

\subsection{Shrinking Embeddings for Hyper-Relational KGs }

We aim to design a scoring function $f(\cdot)$ taking the embeddings of facts as input so that the output values respect desired logical properties. 
To this end, we introduce primal triple embedding and qualifier embedding, respectively, as illustrated in Fig. \ref{fig:example_bigmap}.

\begin{figure}[t!]
    \centering
    \subfloat[\centering ]{{\includegraphics[width=.46\columnwidth]{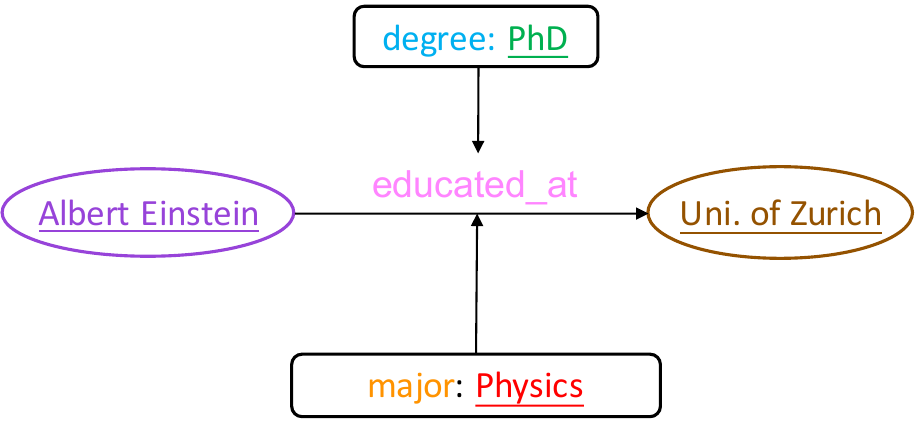}}}
    \subfloat[\centering ]{{\includegraphics[width=.54\columnwidth]{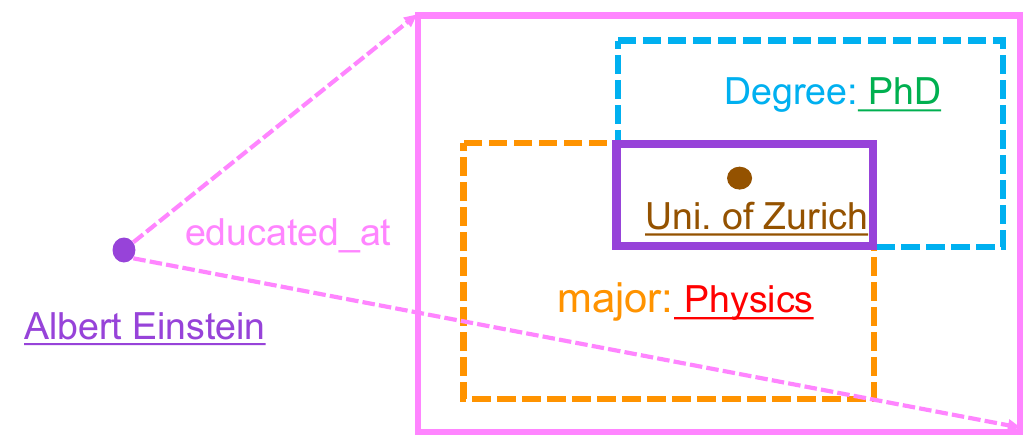} }}
    \caption{An illustration of the proposed idea. 
    (a) A hyper-relational fact is composed of a primal triple and two key-value qualifiers, in which entities (values) are \underline{underlined} while relations (keys) are not. 
    (b) An illustration of the proposed hyper-relational KG embedding model ShrinkE. ShrinkE models the primal triple as a relation-specific transformation from the head entity to a query box (\emph{purple}) that entails the possible answer entities. Each qualifier is modeled as a shrinking of the query box (\emph{orange} and \emph{cyan}) such that the shrinking box is a subset of the query box. The shrinking of the box can be viewed as a geometric interpretation of the monotonicity assumption that we follow. The final answer entities are supposed to be in the intersection box of all shrinking boxes.} 
    \label{fig:example_bigmap}
\end{figure}

\subsubsection{Primal Triple Embedding}
We represent each entity as a point $\mathbf{e} \in \mathbb{R}^d$. Each primal relation $r$ is modeled as a spatio-functional transformation $\mathcal{B}_r:\mathbb{R}^d \rightarrow \Box(d)$ that maps the head $\mathbf{e}_h \in \mathbb{R}^d$ to a $d$-dimensional box in $\Box(d)$ with $\Box(d)$ being the set of boxes in $\mathbb{R}^d$. 
Each box can be parameterized by a lower left point $\mathbf{m} \in \mathbb{R}^d$ and an upper right point $\mathbf{M} \in \mathbb{R}^d$, given by
\begin{equation}
\begin{split}
    &\Box^d(\mathbf{m}, \mathbf{M}) = \\
    &\{\mathbf{x} \in \mathbb{R}^d \mid \mathbf{m}_i \leq \mathbf{x}_i \leq \mathbf{M}_i, \: i=1,\cdots,d\}.
\end{split}
\end{equation}
We leave the superscript of $\Box^d$ away if it is clear from context. 
and call the transformed box a query box. Intuitively, all points in the query box correspond to the possible answer tail entities. Hence, the query box can be viewed as a geometric embedding of the answer set. 
Note that a query could result in an empty answer set. In order to capture such property, we do not exclude empty boxes that correspond to queries with empty answer set. Empty boxes are covered by the cases where there exists a dimension $i$ such that $\mathbf{m}_i \geq \mathbf{M}_i$. 

\textbf{Point-to-box transform}
The spatio-functional point-to-box transformation $\mathcal{B}$ is composed of a relation-specific point transformation $\mathcal{H}_r:  \mathbb{R}^d \rightarrow \mathbb{R}^d$ that transforms the head point $\mathbf{e}_h$ to a new point, and a relation-specific spanning that spans the transformed point to a box, formally given by
\begin{equation}
    \mathcal{B}_r(\mathbf{e}_h) = \Box( \mathcal{H}_r(\mathbf{e}_h) - \tau(\mathbf{\boldsymbol\delta}_r),  \mathcal{H}_r(\mathbf{e}_h) + \tau(\mathbf{\boldsymbol\delta}_r)),
    \label{eq:span}
\end{equation}
where $\mathbf{\boldsymbol\delta}_r \in \mathbb{R}^n$ is a relation-specific spanning/offset vector, and $\tau_{t}(\mathbf{x})=t \log \left(1+e^{\mathbf{x}/t}\right)$ with $t$ being a temperature hyperparameter, is a \emph{softplus} function that enforces the spanned box to be non-empty. 

The point transformation function $\mathcal{H}_r$ could be any functions that are used in other KG embedding models such as translation used in TransE \cite{DBLP:conf/nips/BordesUGWY13} and rotations used in RotatE \cite{DBLP:conf/iclr/SunDNT19}. 
Hence, our model is highly flexible and effective at embedding primal triples. To allow for capturing multiple triple-level inference patterns such as symmetry, inversion, and composition, we combine translation and rotation, and formulate $\mathcal{H}_r$ as
\begin{equation}
    \mathcal{H}_r(\mathbf{e}_h) = \Theta_{r} \mathbf{e}_h + \mathbf{b}_{r}
\end{equation}
where $\Theta_{r}$ is a rotation matrix and $\mathbf{b}_{r}$ is a translation vector.
We parameterize the rotation matrix by a block diagonal matrix $ \Theta_{r} =\operatorname{diag}\mathbf{G}\left(\theta_{r, 1}\right), \ldots,$ $\mathbf{G}\left(\theta_{r, \frac{d}{2}}\right)$, 
where 
\begin{equation}
    \mathbf{G}(\theta)=\left[\begin{array}{cc}
    \cos (\theta) &  \sin (\theta) \\
    \sin (\theta) &  \cos (\theta)
\end{array}\right].
\end{equation}

\begin{figure}[t!]
    \centering
    \includegraphics[width=\linewidth]{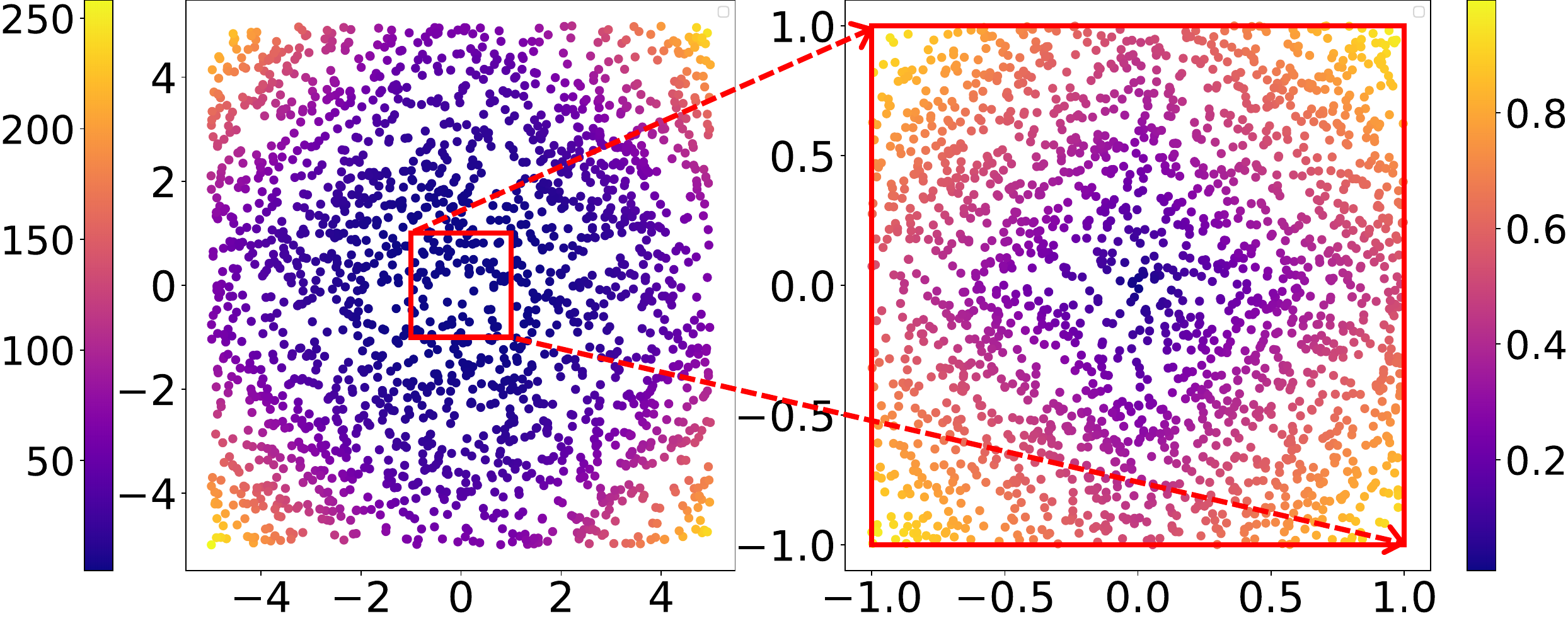}
    \caption{An illustration of the point-to-box distance. The distance (visualized by color maps) grows slowly when the point is inside of the box (\emph{right}) while growing faster when the point is outside of the box (\emph{left}). } 
    \label{fig:point2box}
\end{figure}

\textbf{Point-to-box distance}
The validity of a primal triple $(h,r,t)$ is then measured by judging whether the tail entity point $\mathbf{e}_t$ is geometrically inside of the query box. 
Given a query box $\Box^n(\mathbf{m}, \mathbf{M})$ and an entity point $\mathbf{e} \in \mathbb{R}^{d}$, we denote the center point as $\mathbf{c}=\frac{\mathbf{m} + \mathbf{M}}{2}$. Let $|\cdot|$ denote the L1 norm and $\max()$ denote an element-wise maximum operation. The point-to-box distance is given by
\begin{equation}\small
\begin{split}
    &D(\mathbf{e}, \Box(\mathbf{m}, \mathbf{M})) =  \frac{|\mathbf{e}-\mathbf{c}|_1}{|\max(\mathbf{0}, \mathbf{M}-\mathbf{m})|_1} \\
    &+ \left(|\mathbf{e}-\mathbf{m}|_1 + |\mathbf{e}-\mathbf{M}|_1  - |\max(\mathbf{0}, \mathbf{M}-\mathbf{m})|_1 \right)^2.
\end{split}
\end{equation}
Fig. \ref{fig:point2box} visualizes the distance function. Intuitively, in cases where the point is in the query box, the distance grows relatively slowly and inversely correlates with the box size. In cases where the point is outside the box, the distance grows fast. 

\subsubsection{Qualifier Embedding}

Conceptually, qualifiers add information to given primary facts potentially allowing for additional inferences, but never for the retraction of inferences, reflecting the monotonicity of the representational paradigm. 
Corresponding to the non-declining number of inferences, the number of possible models for this representation shrinks, which can be intuitively reflected by a reduced size of boxes incurred by adding qualifiers. 

\textbf{Box Shrinking}
To geometrically mimic this property in the embedding space, we model each qualifier $(k:v)$ as a "shrinking" of the query box.
Given a box $\Box(\mathbf{m}, \mathbf{M})$, a shrinking is defined as a box-to-box transformation $\mathcal{S}: \Box \rightarrow \Box$ that potentially shrinks the volume of the box while not moving the resulting box outside of the source box. Let $\mathbf{L} = \left(\mathbf{M}-\mathbf{m}\right)$ denote the side length vector, box shrinking is defined by
\begin{equation}\normalsize
\begin{split}
&\mathcal{S}_{r,k,v}\left( \Box\left(\mathbf{m}, \mathbf{M}\right) \right) =\\
&\Box \left( \mathbf{m} + 
\sigma\left(\mathbf{s}_{r,k,v}\right) \odot \mathbf{L} ,
\mathbf{M} - \sigma\left(\mathbf{S}_{r,k,v}\right) \odot \mathbf{L} \right),
\end{split}
\end{equation}
where $\mathbf{s}_{r,k,v} \in \mathbb{R}^n$ and $\mathbf{S}_{r,k,v} \in \mathbb{R}^n$ are the "shrinking" vectors for the lower left corner and the upper right corner, respectively. $\sigma$ is a \emph{sigmoid} function and $\odot$ is element-wise vector multiplication. 
The resulting box, including the case of empty box, is always inside the query box, i.e.,  $\mathcal{S}( \Box^n(\mathbf{m}, \mathbf{M}) ) \subseteq \Box^n(\mathbf{m}, \mathbf{M})$, which exactly resembles the qualifier monotonicity.

We use $r,k,v$ as the indices of the shrinking vectors because the shrinking of the box should depend on the relatedness between the primal relation and the qualifier. 
For example, if a qualifier $({\emph{degree}:\emph{bachelor}})$ is highly related to the primal relation $\emph{educated\_at}$, the scale of the shrinking vectors should be small as it adds a weak constraint to the triple. 
If the qualifier is unrelated to the primal relation, e.g., $({\emph{degree}:\emph{bachelor}})$ and $\emph{born\_in}$, the shrinking might even enforce an empty box. 

To learn the shrinking vectors, we leverage an MLP layer that takes the primal relation and key-value qualifier as input and outputs the shrinking vectors defined by $\mathbf{s}_{r,k,v}, \mathbf{S}_{r,k,v} = \operatorname{MLP}\left(\operatorname{concat} \left( r_\theta, k_\theta, v_\theta \right) \right)$ where $r_\theta, k_\theta, v_\theta$ are the embeddings of $r,k,v$, respectively. 

\subsubsection{Scoring function and learning.}

\textbf{Scoring function} 

The score of a given hyper-relational fact is defined by 
\begin{equation}
    f\left( \left( \left(h,r,t\right),\mathcal{Q}\right) \right) = D(\mathbf{e}_t, \Box_\mathcal{Q}(\mathbf{m}, \mathbf{M})),
\end{equation}
where $\Box_\mathcal{Q}(\mathbf{m}, \mathbf{M})$ denotes the target box that is calculated by the intersection of all shrinking boxes of the qualifier set $\mathcal{Q}$. 
The intersection of $n$ boxes can be calculated by taking the maximum of lower left points of all boxes and taking the minimum of upper right points of all boxes, given by
\begin{equation}\small
\begin{split}
    \mathcal{I}(\Box_1, \cdots, \Box_n)  = 
    \Box\left( \max_{i \in 1, \cdots,n} \mathbf{m}_i, \min_{i \in 1, \cdots,n} \mathbf{M}_i \right).
\end{split}
\end{equation}

Note that if there is no intersection between boxes, this intersection operation still works as it results in an empty box. The intersection of boxes is a permutation-invariant operation, implying that perturbing the order of qualifiers does not change the plausibility of the facts. 

\textbf{Learning}
As a standard data augmentation strategy, we add reciprocal relations $\left(t^{\prime}, r^{-1}, h^{\prime}\right)$ for the primary triple in each hyper-relational fact.
For each positive fact in the training set, we generate $n_{\operatorname{neg}}$ negative samples by corrupting a subject/tail entity with randomly selected entities from $\mathcal{E}$. 
We adopt the cross-entropy loss to optimize the model via the Adam optimizer, which is given by
\begin{equation}\small
    \mathcal{L}=-\frac{1}{N} \sum_{i=1}^{N}\left(y_{i} \log \left(p_{i}\right)+ \sum_{i=1}^{n_{\operatorname{neg}}} \left(1-y_{i}\right) \log \left(1-p_{i}\right)\right),
\end{equation}
where $N$ denotes the total number of facts in the training set. $y_{i}$ is a binary indicator denoting whether a fact is true or not. $p_{i}=\sigma(f(\mathcal{F}))$ is the predicted score of a fact $\mathcal{F}$ with $\sigma$ being the \emph{sigmoid} function. 

\subsection{Theoretical Analysis}
\label{sec:theorem}

Analyzing and modeling inference patterns is of great importance for KG embeddings because it enables generalization capability, i.e., once the patterns are learned, new facts that respect the patterns can be inferred. 
An inference pattern is a specification of a logical property that may exist in a KG, Formally, an inference pattern is a logical form $\psi \rightarrow \phi$ with $\psi$ and $\phi$ being the body and head, implying that if the body is satisfied then the head must also be satisfied. 

In this section, we analyze the theoretical capacity of ShrinkE for modeling inference patterns. All proofs of propositions are in Appendix \ref{app:proof}. 

\textbf{Fact-level inference pattern (monotonicity)} The following proposition shows that ShrinkE is able to model monotonicity. 
\begin{proposition}
Given any two facts $\mathcal{F}_1=\left(\mathcal{T},\mathcal{Q}_1\right)$ and $\mathcal{F}_2=\left(\mathcal{T},\mathcal{Q}_2\right)$ where $\mathcal{Q}_2 \subseteq \mathcal{Q}_1$, i.e., $\mathcal{F}_2$ is a partial fact of $\mathcal{F}_1$, the output of the scoring function $f(\cdot)$ of ShrinkE satisfy the constraint $f(\mathcal{F}_2) \geq f(\mathcal{F}_1)$. 
\label{prop:qualifier_monotonicity}
\end{proposition}

\textbf{Triple-level inference patterns} 
Prominent triple-level inference patterns include symmetry 
$(h,r,t) \rightarrow (t,r,h)$, anti-symmetry $(h,r,t) \rightarrow \neg (h,r,t)$, inversion $(h,r_1,t) \rightarrow (t,r_2,h)$, composition $(e_1,r_1,$ $ e_2) \wedge (e_2,r_2,e_3)  \rightarrow (e_1,r_3,e_3)$, relation implication $(h,r_1,t) \rightarrow (h,r_2,t)$, relation intersection $(h,r_1, t) \wedge (h, r_2, t) \rightarrow (h, r_3, t)$, and relation mutual exclusion $(h,r_1, t) \wedge (h, r_2, t) \rightarrow \bot $. All these triple-level inference patterns also exist in hyper-relational facts when their qualifiers are the same, e.g., hyper-relational symmetry means $\left(\left(h,r,t\right), \mathcal{Q}\right) \rightarrow \left(\left(t,r,h\right), \mathcal{Q}\right)$. 
Proposition \ref{prop:triple_inference_pattern} states that ShrinkE is able to infer all of them. 
\begin{proposition}
    ShrinkE is able to infer hyper-relational symmetry, anti-symmetry, inversion, composition, relation implication, relation intersection, and relation exclusion.  
\label{prop:triple_inference_pattern}
\end{proposition}

\textbf{Qualifier-level inference pattern}
In hyper-relational KGs, inference patterns not only exist at the triple level but also at the level of qualifiers. 

\begin{definition}[qualifier implication]
Given two qualifiers $q_i$ and $q_j$, $q_i$ is said to imply $q_j$, i.e., $q_i \rightarrow q_j$ iff for any fact $\mathcal{F} = \left(\mathcal{T}, \mathcal{Q}\right)$, if attaching $q_i$ to $\mathcal{Q}$ results in a true (resp. false) fact, then attaching $q_j$ to $\mathcal{Q}\cup \left\{q_i\right\}$ also results in a true (resp. false) fact. Formally, $q_i \rightarrow q_j$ implies
\begin{equation}
    \forall \ \mathcal{T},\mathcal{Q}: \left(\mathcal{T}, \mathcal{Q}\cup \left\{q_i\right\}\right) \rightarrow \left(T, Q \cup \{q_i, q_j\}\right).
\end{equation}
\end{definition}

\begin{definition}[qualifier exclusion] 
Two qualifiers $q_i,q_j$ are said to be mutually exclusive iff for any fact $\mathcal{F} = \left(\mathcal{T}, \mathcal{Q}\right)$, 
by attaching $q_i,q_j$ to the qualifier set of $\mathcal{F}$, 
the new fact $\mathcal{F}^{\prime}=\left(\mathcal{T}, \mathcal{Q} \cup \left\{q_i,q_j\right\} \right)$ is false, meaning that they lead to a contradiction, i.e.,  $q_i \wedge q_j \rightarrow \bot$. Formally, $q_i \wedge q_j \rightarrow \bot$ implies
\begin{equation}
    \forall \ \mathcal{T},\mathcal{Q} \: : \left(T, Q \cup \left\{q_i,q_j\right\} \right) \rightarrow \bot 
\end{equation}
\end{definition}

Note that if two qualifiers $q_i,q_j$ are neither mutually exclusive nor forming implication pair, then $q_i,q_j$ are said to be overlapping, a state between implication and mutual exclusion. 
Qualifier overlapping, in our case, can be captured by box intersection/overlapping. 
Qualifier overlapping itself does not form any logical property in the form of $\psi \rightarrow \phi$. 
However, when involving three qualifiers and two of them overlap, qualifier intersection can be modeled. 

\begin{definition}[qualifier intersection] 

A qualifier $q_k$ is said to be an intersection of two qualifiers $q_i,q_j$ iff for any fact $\mathcal{F} = \left(\mathcal{T}, \mathcal{Q}\right)$, if attaching $q_i,q_j$ to $\mathcal{Q}$ results in a true (resp. false) fact, then by replacing $\{q_i,q_j\}$ with $q_k$, the truth value of the fact does not change. Namely, $q_i \wedge q_j \rightarrow q_k$ implies
\begin{equation}
    \forall \ \mathcal{T},\mathcal{Q}: \left(T, Q \cup \{q_i,q_j\}  \right) \rightarrow  \left(\mathcal{T}, \mathcal{Q} \cup \left\{q_k\right\} \right).
\end{equation}
\end{definition}

Apparently, qualifier intersection $ q_i \wedge q_j \rightarrow  q_k$ necessarily implies qualifier implications $q_i \rightarrow q_k$ and $q_j \rightarrow q_k$. 
Hence, qualifier intersection can be viewed as a combination of two qualifier implications, and this can be generalized to $ q_1 \wedge q_2 \wedge \cdots  \rightarrow  q_k$.
Proposition \ref{prop:qualifier_inference_pattern} shows that ShrinkE is able to infer qualifier implication, exclusion, and composition. 

\begin{proposition}
ShrinkE is able to infer qualifier implication, mutual exclusion, and intersection.
\label{prop:qualifier_inference_pattern}
\end{proposition}

\subsection{Evaluation}

In this section, we evaluate the effectiveness of ShrinkE on hyper-relational link prediction tasks.

\begin{table}[t!]
    \centering
      \caption{Dataset statistics, where the columns indicate the number of all facts, hyper-relational facts with the number of qualifiers $m>0$, entities, relations, and facts in train/dev/test sets, respectively.}
      \resizebox{\columnwidth}{!}{
        \begin{tabular}{cccccccccc}
            \hline
            & All facts & Higher-arity facts (\%) & Entities & Relations & Train & Dev & Test \\ 
            \hline
            JF17K & 100,947 & 46,320 (45.9\%) & 28,645 & 501 & 76,379 & – & 24,568 \\
            WikiPeople & 382,229 & 44,315 (11.6\%) & 47,765 & 193 & 305,725 & 38,223 & 38,281 \\
            WD50k & 236,507 & 32,167 (13.6\%) & 47,156 & 532 & 166,435 & 23,913 & 46,159 & \\
            WD50K(33) & 102,107 & 31,866 (31.2\%) & 38,124 & 475 & 73,406 & 10,568 & 18,133\\
            WD50K(66) & 49,167  & 31,696 (64.5\%) & 27,347 & 494 & 35,968 & 5,154 & 8,045\\
            WD50K(100) & 31,314 & 31,314 (100\%)  & 18,792 & 279 & 22,738 & 3,279 & 5,297\\
            \hline
        \end{tabular}
        }
        \label{tab:shrinke_dataset}
\end{table}

\subsubsection{Experimental Setup}

\textbf{Datasets.} 
We conduct link prediction experiment on three hyper-relational KGs: JF17K \cite{DBLP:conf/ijcai/WenLMCZ16}, WikiPeople \cite{DBLP:conf/www/GuanJWC19}, and WD50k \cite{DBLP:conf/emnlp/GalkinTMUL20}. 
JF17K is extracted from Freebase while WikiPeople and WD50k are extracted from Wikidata. In WikiPeople and WD50k, only $11.6\%$ and $13.6\%$ of the facts, respectively, contain qualifiers, while the remaining facts contain only triples (after dropping statements containing literals in WikiPeople, only 2.6\% facts contain qualifiers). 
For better comparison, we also consider three splits of WD50K that contain a higher percentage of triples with qualifiers. 
The three splits are WD50K(33), WD50K(66), and WD50K(100), which contain 33\%, 66\%, and 100\% facts with qualifiers, respectively. 
Statistics of the datasets are given in Table \ref{tab:shrinke_dataset}.
We conjecture that the performance on WikiPeople and WD50k will be dominated by the scores of triple-only facts while the performance on the variants of WD50k will be dominated by the modeling of qualifiers. 
\cite{DBLP:conf/emnlp/GalkinTMUL20} detected a data leakage issue of JF17K, i.e., about $44.5\%$ of the test facts share the same primal triple as the train facts. To alleviates this issue, WD50K removes all facts from train/validation sets that share the same primal triple with test facts. 
We conjecture that WD50K will be a more challenging benchmark than JF17K and WikiPeople. 
Besides, WD50K still contains only a small percentage (13.6\%) of facts that contain qualifiers.
Since JF17K does not provide a validation set, we split $20\%$ of facts from the training set as the validation set. 

\textbf{Environments and hyperparameters}
We implement ShrinkE with Python 3.9 and Pytorch 1.11, and train our model on one Nvidia A100 GPU with 40GB of VRAM. 
We use Adam optimizer with a batch size of $128$ and an initial learning rate of $0.0001$.  
For negative sampling, we follow the strategy used in StarE \cite{DBLP:conf/emnlp/GalkinTMUL20} by randomly corrupting the head or tail entity in the primal triple. 
Different from HINGE \cite{DBLP:conf/www/RossoYC20} and NeuInfer \cite{DBLP:conf/acl/GuanJGWC20} that score all potential facts one by one that takes an extremely long time for evaluation, ShrinkE ranks each target answer against all candidates in a single pass and significantly reduces the evaluation time. 
We search the dimensionality from $[50, 100,200,300]$ and the best one is $200$. We set the temperature parameter to be $t=1.0$. We use the label smoothing strategy and set the smoothing rate to be $0.1$. 
We repeat all experiments for $5$ times with different random seeds and report the average values, the error bars are relatively small and are omitted. 

\textbf{Baselines} 
We compare ShrinkE against various models, including m-TransH \cite{DBLP:conf/ijcai/WenLMCZ16}, RAE \cite{DBLP:conf/www/ZhangLMM18}, NaLP-Fix \cite{DBLP:conf/www/RossoYC20}, HINGE \cite{DBLP:conf/www/RossoYC20}, NeuInfer \cite{DBLP:conf/acl/GuanJGWC20}, BoxE \cite{boxE}, Transformer and StarE \cite{DBLP:conf/emnlp/GalkinTMUL20}. Note that we exclude Hy-Transformer \cite{DBLP:journals/corr/abs-2104-08167}, GRAN 
\cite{DBLP:conf/acl/WangWLZ21} and QUAD \cite{DBLP:journals/corr/abs-2208-14322} for comparison because 
1) they are heavily based on StarE and Transformer; and
2) they leverage auxiliary training tasks, which can also be incorporated into our framework and we leave as one future work.

\textbf{Evaluation} 
We strictly follow the settings of \cite{DBLP:conf/emnlp/GalkinTMUL20}, where the aim is to predict a missing head/tail entity in a hyper-relational fact. We consider the widely used ranking-based metrics for link prediction: mean reciprocal rank (MRR) and H@K (K=1,10). For ranking calculation, we consider the filtered setting by filtering the facts 
existing in the training and validation sets \cite{DBLP:conf/nips/BordesUGWY13}.

\begin{table}
    \centering
    \caption{Link prediction results on three benchmarks with the number in the parentheses denoting the ratio of facts with qualifiers. Baseline results are taken from \cite{DBLP:conf/emnlp/GalkinTMUL20}. }
    \resizebox{\columnwidth}{!}{
    \begin{tabular}{lcccccccccccc}
    \hline 
    \multirow{2}{*}{Method} & \multicolumn{3}{c}{WikiPeople (2.6)}  & & \multicolumn{3}{c}{JF17K (45.9)} & & \multicolumn{3}{c}{WD50K (13.6)} \\
    \cline {2-4} \cline {6-8} \cline{10-12} & MRR & H@1 & H@10 & & MRR & H@ 1& H@ 10 & & MRR & H@ 1 & H@ 10 \\
    \hline 
    m-TransH & $0.063$ & $0.063$ & $0.300$ & & $0.206$ & $0.206$ & $0.463$ & & $-$ & $-$ & $-$\\
    RAE & $0.059$ & $0.059$ & $0.306$ & & $0.215$ & $0.215$ & $0.469$ & & $-$ & $-$ & $-$\\
    NaLP-Fix & $0.420$ & $0.343$ & $0.556$ & & $0.245$ & $0.185$ & $0.358$ & & 0.177 & 0.131 & 0.264 \\
    NeuInfer & $0.350$ & $0.282$ & $0.467$ & & $0.451$ & $0.373$ & $0.604$ & & $-$ & $-$ & $-$ \\
    HINGE & $0.476$ & $0.415$ & $0.585$ & & $0.449$ & $0.361$ & $0.624$ & & $0.243$ & $0.176$ & $0.377$ \\
    Transformer & 0.469 & 0.403 & 0.586 & & 0.512 & 0.434 & 0.665 & & 0.264 & 0.194 & 0.401 \\
    BoxE & 0.395 & 0.293 & 0.503 & & 0.560 & 0.472 & 0.722 & & $-$ & $-$ & $-$ \\
    StarE & \textbf{0.491} & 0.398 & \textbf{0.648} & & 0.574 & 0.496 & 0.725 & & \textbf{0.349} & \underline{0.271} & \textbf{0.496} \\
    \hline
    ShrinkE & \underline{0.485} & \textbf{0.431} & \underline{0.601} & & \textbf{0.589} &\textbf{ 0.506} & \textbf{0.749} &  & \underline{ 0.345} & \textbf{0.275} & \underline{0.482} \\
    \hline
    \end{tabular}}
    \label{tab:main_results}
\end{table}

\subsubsection{Main Results and Analysis}

Table \ref{tab:main_results} and Table \ref{tab:wd50k-variants} summarize the performances of all approaches on the six datasets. Overall, ShrinkE achieves either the best or the second-best results against all baselines, showcasing the expressivity and capability of ShrinkE on hyper-relational link prediction. 
In particular, We observe that ShrinkE outperforms all baselines on JF17K and the three variants of WD50K with a high ratio of facts containing qualifiers while achieving highly competitive results on WikiPeople and the original version of WD50K that contain fewer facts with qualifiers. 
Interestingly, we find that the performance gains increase when increasing the ratio of facts containing qualifiers. On WD50K (100) where 100\% facts contain qualifiers, the performance gain of ShrinkE is most significant across all metrics (6.2\%, 6.9\%, and 4.7\% improvements over MRR, H@1, and H@10, respectively). 
We believe this is because that ShrinkE is excellent at modeling qualifiers due to its explicit modeling of inference patterns, while other models relying on tremendous parameters tend to be overfitting. 

\begin{table}
    \centering
    \caption{Link prediction results on WD50K splits with the number in the parentheses denoting the ratio of facts with qualifiers. Baseline results are taken from \cite{DBLP:conf/emnlp/GalkinTMUL20}. }
    \resizebox{\columnwidth}{!}{
    \begin{tabular}{lccccccccccc}
    \hline 
    \multirow{2}{*}{Method} & \multicolumn{3}{c}{WD50K (33)} & & \multicolumn{3}{c}{WD50K (66)} & & \multicolumn{3}{c}{WD50K (100)} \\
    \cline {2-4} \cline {6-8} \cline{10-12} & MRR & H@1 & H@10 & & MRR & H@ 1 & H@ 10 & & MRR & H@ 1 & H@ 10 \\
    \hline 
    NaLP-Fix & 0.204 & 0.164 & 0.277 & & 0.334 & 0.284 & 0.423 & & 0.458 & 0.398 & 0.563 \\
    HINGE & 0.253 & 0.190 & 0.372 & & 0.378 & 0.307 & 0.512 & & 0.492 & 0.417 & 0.636 \\
    Transformer & 0.276 & 0.227 & 0.371 & & 0.404 & 0.352 & 0.502 & & 0.562 & 0.499 & 0.677 \\
    StarE & \underline{0.331} & \underline{0.268} & \textbf{0.451} & & \underline{0.481} & \underline{0.420} & \underline{0.594} & & \underline{0.654} & \underline{0.588} & \underline{0.777} \\
    \hline
    ShrinkE & \textbf{0.336} & \textbf{0.272} & \underline{0.449} & & \textbf{0.511} & \textbf{0.422} & \textbf{0.611} & & \textbf{0.695} & \textbf{0.629} & \textbf{0.814}   \\
    \hline
    \end{tabular}}
    \label{tab:wd50k-variants}
\end{table}



\begin{table}
    \centering
    \caption{Example pairs of qualifiers with implication relations (body $\rightarrow$ head). $X \in $ [\emph{Eric Schmidt}, \emph{Mark Zuckerberg}, \emph{Dustin Moskovitz}, \emph{Larry Page}] denotes a CEO name of a company. $Y \in [112,115,113, \cdots]$ and $Z \in [912,18,192, \cdots]$ are emergency numbers involving \emph{police} and \emph{fire department}, respectively. 
Qualifier exclusion pairs are ubiquitous and are hence omitted. }
    \begin{tabular}{cc}
\hline
body & head \\
\hline
 (residence: Monte Carlo) & (country, Monaco)  \\
 (residence: Belgrade) & (country, Serbia) \\
 (owned\_by: X) & (of, voting interest) \\
 (emergency phone number: Y) & (has\_use, police) \\
 (emergency phone number: Z) & (has\_use, fire department) \\
 (used\_by: software) & (via, operating\_system) \\
\hline
\end{tabular}
\label{tab:qualifier_implication}
\end{table}

\textbf{Case analysis}
Table \ref{tab:qualifier_implication} shows some examples of qualifier implication pairs recovered by our learned embeddings. Note that exclusions pairs are ubiquitous (i.e., most of the random qualifiers are mutually exclusive) and hence we do not analyze them. We find that some qualifier implications happen when they are about geographic information and involve geographic inclusion such as \emph{Monte Carlo} is in \emph{Monaco}. Interestingly, we find that qualifiers associated with key \emph{owned\_by} imply (\emph{of, voting interest}), and qualifiers with key \emph{emergency phone number} imply (\emph{has\_use, police}) or (\emph{has\_use, \emph{file department}}), which conceptually make sense.

\subsubsection{Ablations and Parameter Sensitivity}

\begin{table}
    \centering
    \begin{tabular}{cccc}
    \hline
    Method & MRR & H@ 1 & H@ 10 \\
    \hline
    ShrinkE (w/o translation) & 0.583 & 0.495 & 0.729 \\
    ShrinkE (w/o rotation) & 0.581 & 0.497 & 0.724 \\
    ShrinkE (w/o shrinking) & 0.571 & 0.490 & 0.711 \\
    \hline
    ShrinkE  & \textbf{0.589} & \textbf{0.506} & \textbf{0.749} \\
    \hline
    \end{tabular}
    \caption{The performance of ShrinkE by removing one relational component on JF17K. }
    \label{tab:ablation_ShrinkE}
\end{table}

\textbf{Impact of relational components} 
To determine the importance of each component in relational modeling, we conduct an ablation study by considering three versions of ShrinkE in which one of the components (translation, rotation, and shrinking) is removed. 
Table \ref{tab:ablation_ShrinkE} shows that the removal of each component of the relational transformation leads to a degradation in performance, validating the importance of each component.
In particular, by removing the qualifier shrinking, which is the main contribution of our framework, the performance reduces 3\% and 5\% in  MRR and H@10, respectively, showcasing the usefulness of modeling qualifiers as shrinking. The removals of translation and rotation both result in around 1\% and 2\% reduction in MRR and H@10, respectively.

\begin{figure}
    \centering
    \includegraphics[width=0.6\textwidth]{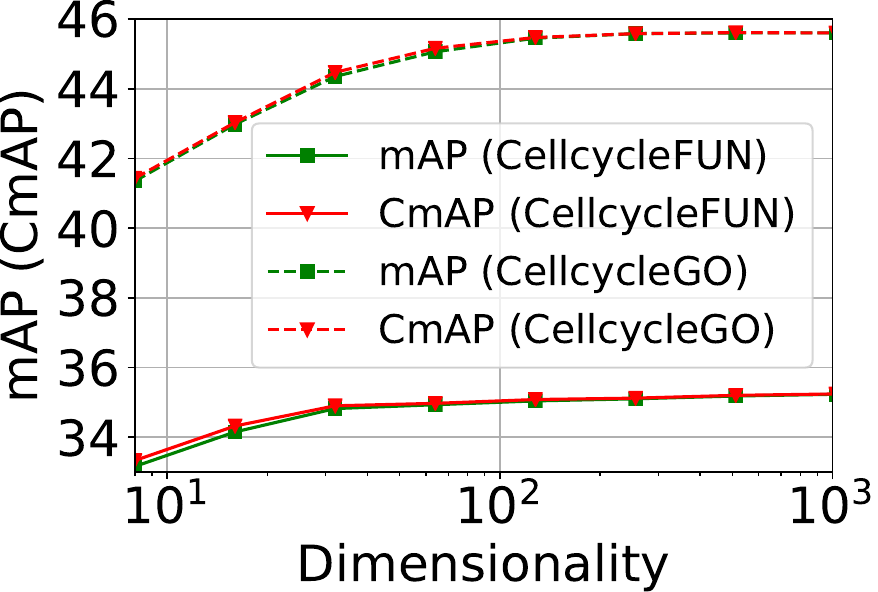}
    \caption[Performance of ShrinkE with different dimensions]{Performance of ShrinkE with different dimensions $d=[4,8,16,32,64,128,256]$ on JF17K. }
    \label{fig:dimension}
\end{figure}

\textbf{Impact of dimensionality} 
We conduct experiments on JF17K under a varied number of dimensions $d=[4,8,16,32,64,128,256]$. As Fig. \ref{fig:dimension} depicts,  the performance increases when increasing the number of dimensions. 
However, the growth trend gradually flattens with the increase of dimensions and it achieves comparable performance when the dimension is higher than $128$.

\subsubsection{Discussion}

\textbf{Comparison with neural network models} 
Heavy neural network models such as GRAN 
\cite{DBLP:conf/acl/WangWLZ21} and QUAD \cite{DBLP:journals/corr/abs-2208-14322} are built on relational GNNs and/or Transformers and require a large number of parameters. In contrast, ShrinkE is a neuro-symbolic model that requires only one MLP layer and a much smaller number of parameters. The logical modelling of ShrinkE makes it more explainable than GNN-based and Transformer-based methods. 

\textbf{Comparison with other box embeddings in KGs} 
ShrinkE is the first to not only represent hyper-relational facts, but also explicitly model the logical properties of these facts. SrinkE is different from previous box embedding methods \cite{abboud2020boxe} of KGs in three key modules: 1) our point-to-box transform function modelling triple inference patterns; 2) a new point-to-box distance function; and 3) we introduce box shrinking to model qualifier-level inference patterns. Moreover, we provide a comprehensive theoretical analysis of ShrinkE on modelling various logical properties.

\subsection{Conclusion}
We present a novel hyper-relational KG embedding model ShrinkE. ShrinkE models a primal triple as a spatio-functional transformation while modeling each qualifier as a shrinking that monotonically narrows down the answer set. We proved that ShrinkE is able to spatially infer core inference patterns at different levels including triple-level, fact-level, and qualifier-level. Experimental results on three benchmarks demonstrate the advantages of ShrinkE in predicting hyper-relational links.

\section{Modeling Relationships between Facts for Knowledge Graph Reasoning}
\label{sec:facte}

Reasoning with knowledge graphs (KGs) has primarily focused on triple-shaped facts. Recent advancements have been explored to enhance the semantics of these facts by incorporating more potent representations, such as hyper-relational facts. However, these approaches are limited to \emph{atomic facts}, which describe a single piece of information. This paper extends beyond \emph{atomic facts} and delves into \emph{nested facts}, represented by quoted triples where subjects and objects are triples themselves (e.g., ((\emph{BarackObama}, \emph{holds\_position}, \emph{President}), \emph{succeed\_by}, (\emph{DonaldTrump}, \emph{holds\_position}, \emph{President}))). These nested facts enable the expression of complex semantics like \emph{situations} over time and \emph{logical patterns} over entities and relations. 
In response, we introduce FactE, a novel KG embedding approach that captures the semantics of both atomic and nested factual knowledge. 
FactE represents each atomic fact as a $1\times3$ matrix, and each nested relation is modeled as a $3\times3$ matrix that rotates the $1\times3$ atomic fact matrix through matrix multiplication. 
Each element of the matrix is represented as a complex number in the generalized 4D hypercomplex space, including (spherical) quaternions, hyperbolic quaternions, and split-quaternions. 
Through thorough analysis, we demonstrate the embedding's efficacy in capturing diverse logical patterns over nested facts, surpassing the confines of first-order logic-like expressions. Our experimental results showcase FactE's significant performance gains over current baselines in triple prediction and conditional link prediction. 
The code is attached as supplemental material and will be made publicly available.

\subsection{Motivation and Background}

Knowledge graphs (KGs) depict relationships between entities, commonly through triple-shaped facts such as (\emph{JoeBiden}, \emph{holds\_position}, \emph{VicePresident}). KG embeddings map entities and relations into a lower-dimensional vector space while retaining their relational semantics. This empowers the effective inference of missing relationships between entities directly from their embeddings.
Prior research \cite{DBLP:conf/nips/BordesUGWY13, DBLP:conf/iclr/SunDNT19, DBLP:conf/icml/TrouillonWRGB16} has primarily centered on embedding triple-shaped facts and predicting the missing elements of these triples. Yet, to augment the triple-shaped representations, recent endeavors explore knowledge that extends beyond these triples.
For instance, $n$-ary facts \cite{DBLP:conf/www/0016Y020,DBLP:conf/ijcai/FatemiTV020} describe relationships between multiple entities, and hyper-relational facts \cite{DBLP:conf/emnlp/GalkinTMUL20,DBLP:conf/www/RossoYC20} augment primal triples with key-value qualifiers that provide contextual information. These approaches allow for expressing complex semantics and enable answering more sophisticated queries with additional knowledge \cite{DBLP:conf/iclr/AlivanistosBC022}.

\begin{figure}[t]
\centering
\includegraphics[width = 0.9\columnwidth]{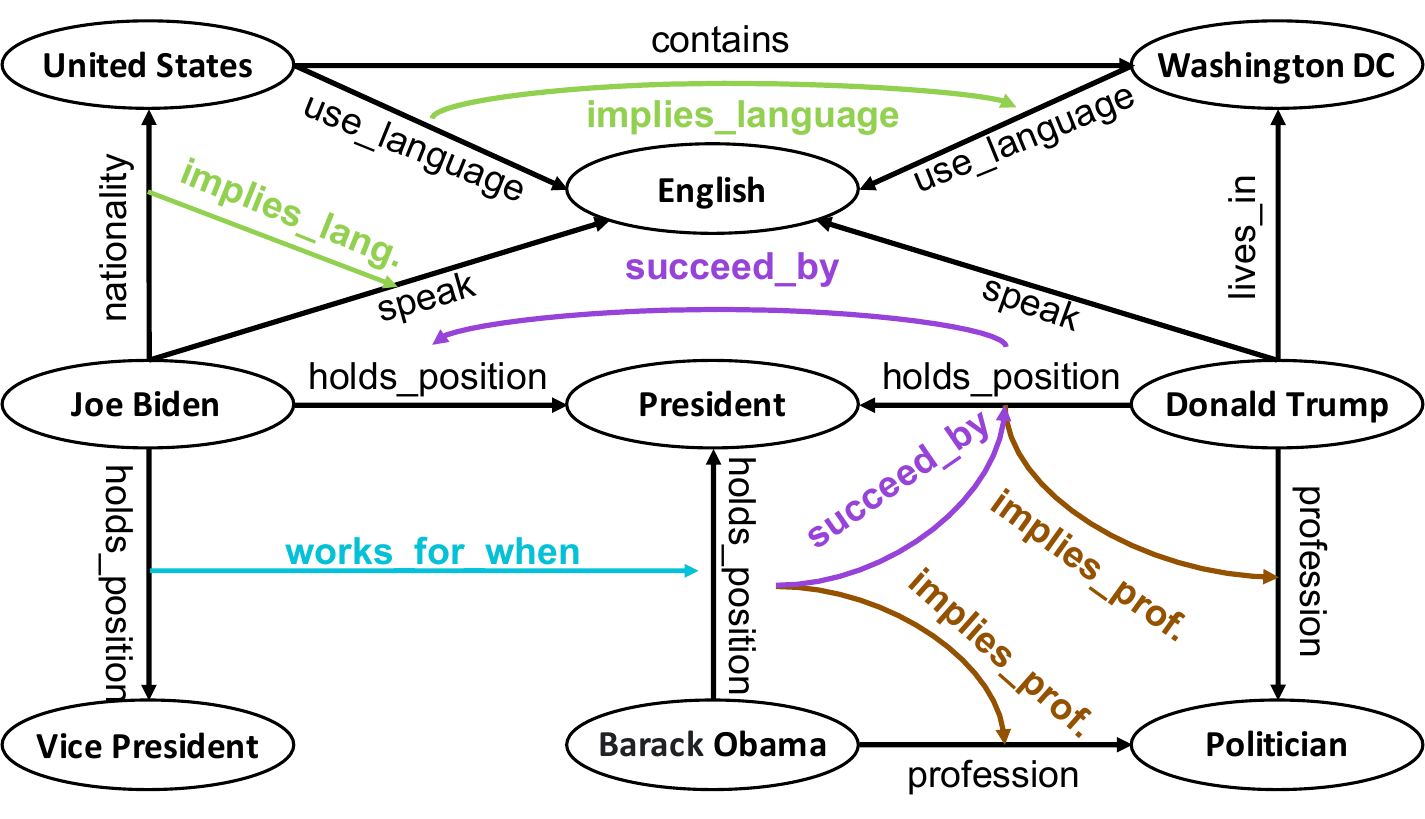}
\caption{An example of a nested factual KG consisting of 1) a set of atomic facts describing the relationship between entities and 2) a set of nested facts describing the relationship between atomic facts. Nested factual relations are colored and they either describe situations in/over time (e.g., \emph{succeed\_by} and \emph{works\_for\_when}) or logical patterns (e.g., \emph{implies\_profession} and \emph{implies\_language}). 
}
\label{fig:bikg}
\end{figure}

However, these beyond-triple representations typically focus only on relationships between entities that jointly define an \emph{atomic fact}, 
overlooking the significance of relationships that describe multiple facts together.
Indeed, within a KG, each atomic fact may have a relationship with another atomic fact. Consider the following two atomic facts:
$T_1$=(\emph{JoeBiden}, \emph{holds\_position}, \emph{VicePresident}) and $T_2$=(\emph{BarackObama}, \emph{holds\_position}, \emph{President}).
We can depict the scenario where \emph{JoeBiden} held the position of \emph{VicePresident} under the \emph{President} \emph{BarackObama} using a triple $(T_1, \emph{works\_for\_when}, T_2)$.
Such a fact about facts is referred to as a \emph{nested fact} \footnote{This is also called a \emph{quoted triple} in RDF star \cite{DBLP:journals/ercim/Champin22}.}
and the relation connecting these two facts is termed a \emph{nested relation}.
Fig. \ref{fig:bikg} provides an illustration of a KG containing both atomic and nested facts.

These nested relations play a crucial role in expressing complex semantics and queries in two ways: 1)
\textbf{ Expressing situations involving facts in or over time}. This facilitates answering complex queries that involve multiple facts.
For example, KG embeddings face challenges when addressing queries like "\emph{Who was the president of the USA after \emph{DonaldTrump}?}" because the query about the primary fact (?, \emph{holds\_position}, \emph{President}) depends on another fact (\emph{DonaldTrump}, \emph{holds\_position}, \emph{President}). As depicted in Fig. \ref{fig:bikg}, \emph{succeed\_by} conveys such temporal situation between these two facts, allowing the direct response to the query through conditional link prediction; 2)
\textbf{Expressing logical patterns (implications) using a non-first-order logical form $\psi \stackrel{\widehat{r}}{\rightarrow}\phi$}.
As illustrated in Fig. \ref{fig:bikg}, (\emph{Location A}, \emph{uses\_language}, \emph{Language B}) $\stackrel{\emph{implies\_language}}{\rightarrow} $ (\emph{Location in A}, \emph{uses\_language}, \emph{Language B}) represents a logical pattern, as it holds true for all pairs of (\emph{Location A}, \emph{Location in A}).
Modeling such logical patterns is crucial as it facilitates generalization. Once these patterns are learned, new facts adhering to these patterns can be inferred.
A recent study \cite{DBLP:conf/aaai/Chanyoung} explored link prediction over nested facts.\footnote{In their work, the KG is referred to as a bi-level KG, and the term "high-level facts" is synonymous with nested facts.}
However, their method embeds facts using a multilayer perceptron (MLP), which fails to capture essential logical patterns and thus has limited generalization capabilities.

In this paper, we introduce FactE, an innovative approach designed to embed the semantics of both atomic facts and nested facts that enable representing temporal situations and logical patterns over facts. 
FactE represents each atomic fact as a $1\times3$ hypercomplex matrix, with each element signifying a component of the atomic fact. Furthermore, each nested relation is modeled through a $3\times3$ hypercomplex matrix that rotates the $1\times3$ atomic fact matrix via a matrix-multiplicative Hamilton product. 
Our matrix-like modeling for facts and nested relations demonstrates the capacity to encode diverse logical patterns over nested facts. The modeling of these logical patterns further enables efficient modeling of logical rules that extend beyond the first-order-logic-like expressions (e.g., Horn rules).
Moreover, we propose a more general hypercomplex embedding framework that extends the quaternion embedding \cite{DBLP:conf/nips/0007TYL19} to include hyperbolic quaternions and split-quaternions. This generalization of hypercomplex space allows for expressing rotations over hyperboloid, providing more powerful and distinct inductive biases for embedding complex structural patterns (e.g., hierarchies). 
Our experimental findings on triple prediction and conditional link prediction showcase the remarkable performance gain of FactE.

\subsection{Related Work}


\textbf{Describing relationships between facts} 
Rule-based approaches \cite{DBLP:conf/aaai/NiuZ0CLLZ20,DBLP:conf/ijcai/MeilickeCRS19,DBLP:conf/emnlp/DemeesterRR16,DBLP:conf/emnlp/GuoWWWG16,DBLP:conf/nips/YangYC17,DBLP:conf/nips/SadeghianADW19} consider relationships between facts, but they are confined to first-order-logic-like expressions (i.e., Horn rules), i.e., $\forall e_1,e_2,e_3: (e_1,r_1,e_2)\wedge (e_2,r_2,e_3) \Rightarrow (e_1,r_3,e_3)$, where there must exist a path connecting $e_1$, $e_2$, and $e_3$ in the KG. 
Notably, \cite{DBLP:conf/aaai/Chanyoung} marked an advancement by examining KG embeddings with relationships between facts as nested facts, denoted as $(x, r_1, y) \xRightarrow{\hat{r}} (p, r_2, q)$. The proposed embeddings (i.e., BiVE-Q and BiVE-B\footnote{Note that BiVE-B, despite being described as based on the biquaternian--BiQUE, employs quaternion space with an additional translation component based on our analysis of the code.}) concatenate the embeddings of the head, relation, and tail, subsequently embedding them via an MLP. However, such modeling does do not explicitly capture crucial logical patterns over nested facts, which bear significant importance in KG embeddings \cite{DBLP:conf/icml/TrouillonWRGB16, DBLP:conf/iclr/SunDNT19}. 

\textbf{Algebraic and geometric embeddings} 
Algebraic embeddings like QuatE \cite{DBLP:conf/nips/0007TYL19} and BiQUE \cite{DBLP:conf/emnlp/GuoK21} represent relations as algebraic operations and score triples using inner products. They can be viewed as a unification of many earlier functional \cite{DBLP:conf/nips/BordesUGWY13} and multiplication-based \cite{DBLP:conf/icml/TrouillonWRGB16} models. 
Geometric embeddings like hyperbolic embeddings \cite{chami2020low,DBLP:conf/nips/BalazevicAH19} further extend the functional models to non-Euclidean hyperbolic space, enabling the representation of hierarchical relations.

\subsection{FactE: Embedding Atomic and Nested Facts }

\subsubsection{Unified Hypercomplex Embeddings}

We first extend QuatE \cite{DBLP:conf/nips/0007TYL19}, a KG embedding in 4D hypercomplex quaternion space, into a more general 4D hypercomplex number system including three variations: (spherical) quaternions, hyperbolic quaternions, and split quaternions. Each of these 4D hypercomplex numbers is composed of one real component and three imaginary components denoted by $s+x\textit{i}+y\textit{j}+z\textit{k}$ with $s,x,y,z \in \mathbb{R}$ and $i,j,k$ being the three imaginary parts. 
The distinctive feature among these hyper-complex number systems lies in their multiplication rules of the imaginary components.

\noindent\textbf{(Spherical) quaternions $\mathcal{Q}$} follow the multiplication rules:
\begin{equation}
\begin{array}{l}
\textit{i}^2 = \textit{j}^2 = \textit{k}^2 = 1, \\
\textit{ij} = \textit{k}=-ji, \textit{jk} = \textit{i}=-kj, \textit{ki} = \textit{j}=-ik. \
\end{array}
\end{equation}

\noindent\textbf{Hyperbolic quaternions $\mathcal{H}$} follow the multiplication rules:
\begin{equation}
\begin{array}{l}
\textit{i}^2 = -1, \textit{j}^2 = \textit{k}^2 = 1, \\
\textit{ij} = \textit{k}, \textit{jk} = -\textit{i}, \textit{ki} = \textit{j}, \
\textit{ji} = -\textit{k}, \textit{kj} = \textit{i}, \textit{ik} = -\textit{j}.
\end{array}
\end{equation}

\noindent\textbf{
Split quaternions $\mathcal{S}$} follow the multiplication rules:
\begin{equation}
\begin{array}{l}
\textit{i}^2 = -1, \textit{j}^2 = \textit{k}^2 = 1, \\
\textit{ij} = \textit{k}, \textit{jk} = -\textit{i}, \textit{ki} = \textit{j}, \
\textit{ji} = -\textit{k}, \textit{kj} = \textit{i}, \textit{ik} = -\textit{j}.
\end{array}
\end{equation}

\textbf{Geometric intuitions}
The distinctions in the multiplication rules of various hypercomplex numbers give rise to different geometric spaces that provide suitable inductive biases for representing different types of relations. Specifically, spherical quaternions, hyperbolic quaternions, and split quaternions with the same norm $c$ correspond to 4D hypersphere, Lorentz model of hyperbolic space (i.e., the upper part of the double-sheet hyperboloid), and pseudo-hyperboloid (i.e., one-sheet hyperboloid, with curvature $\sqrt{c}$, respectively. These are denoted as follows:
\begin{equation}\small
\begin{array}{l}
|\mathcal{Q}| = s^2 + x^2 + y^2 + z^2 = c > 0 \ (\text{hypersphere}) \\
|\mathcal{H}| = s^2 - x^2 - y^2 - z^2 = c > 0 \ (\text{Lorentz hyperbolic space}) \\
|\mathcal{S}| = s^2 + x^2 - y^2 - z^2 = c > 0 \ (\text{pseudo-hyperboloid}). \\
\end{array}
\end{equation}
The 3D versions of them are shown in Fig. \ref{fig:space}. These spaces have well-known characteristics: spherical spaces are adept at modeling cyclic relations \cite{DBLP:conf/www/WangWSWNAXYC21}, hyperbolic spaces provide geometric inductive biases for hierarchical relations \cite{chami2020low}, and the pseudo-hyperboloid \cite{DBLP:conf/kdd/XiongZNXP0S22} offers a balance between spherical and hyperbolic spaces, making it suitable for embedding both cyclic and hierarchical relations.
Moreover, by representing relations as geometric rotations over these spaces (i.e., Hamilton product), fundamental logical patterns such as symmetry, inversion, and compositions can be effectively inferred \cite{DBLP:conf/nips/0007TYL19,chami2020low,DBLP:conf/kdd/XiongZNXP0S22}. 
Our proposed embeddings can be viewed as a unification of previous approaches that leverages these geometric inductive biases in these geometric spaces within a single geometric algebraic framework.

For convenience, we parameterize each entity and relation as a Cartesian product of $d$ 4D hypercomplex numbers $\mathbf{s} + \mathbf{x}\textit{i} + \mathbf{y}\textit{j} + \mathbf{z}\textit{k}$, where $\mathbf{s}, \mathbf{x}, \mathbf{y}, \mathbf{z} \in \mathbb{R}^d$.
This enables us to define all algebraic operations involving these hypercomplex vectors in an element-wise manner.

\begin{figure}
    \centering
    \includegraphics[width=\columnwidth]{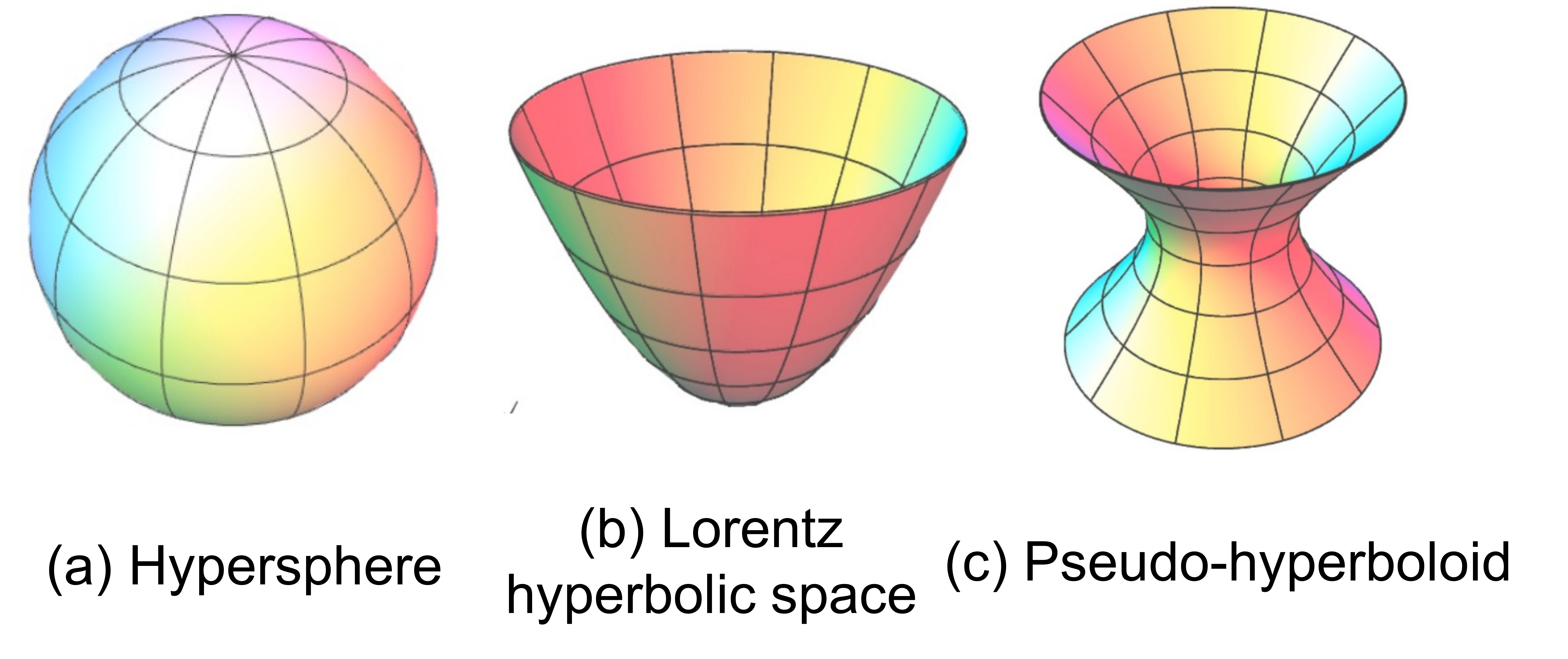}
     \vspace{-2em}
    \caption{Visualization of the hypersphere, Lorentz hyperbolic space, and pseudo-hyperboloid in 3D space. 
    }
    \label{fig:space}
    \vspace{-1em}
\end{figure}

\subsubsection{Atomic Fact Embeddings}

Each atomic relation is represented by a rotation hypercomplex vector $\mathbf{r}_\theta$ and a translation hypercomplex vector $\mathbf{r}_b$. For a given triple $(h, r, t)$, we apply the following operation:
\begin{equation}
\mathbf{h}^\prime = ( \mathbf{h} \oplus \mathbf{r}_b ) \otimes \mathbf{r}_\theta,
\end{equation}
where $\oplus$ and $\otimes$ stand for addition and Hamilton product between hypercomplex numbers, respectively. The addition involves an element-wise sum of each hypercomplex component. The Hamilton product rotates the head entity. To ensure proper rotation on the unit sphere, we normalize the rotation hypercomplex number $\mathbf{r}_\theta = \mathbf{s}_r^\theta+\mathbf{x}_r^\theta\textit{i}+\mathbf{y}_r^\theta \textit{j}+\mathbf{z}_r^\theta \textit{k}$ by $\mathbf{r}_\theta=\frac{\mathbf{s}_r^\theta+\mathbf{x}_r^\theta \boldsymbol{i}+\mathbf{y}_r^\theta \boldsymbol{j}+\mathbf{z}_r^\theta \boldsymbol{k}}{\sqrt{ {\mathbf{s}_r^\theta}^2+{\mathbf{x}_r^\theta}^2+{\mathbf{y}_r^\theta}^2+{\mathbf{z}_r^\theta}^2}}$. 
Hamilton product is defined by combining the components of the hypercomplex numbers. 
\begin{equation}
\begin{split}
& \mathbf{h}^\prime = \mathbf{h} \otimes \mathbf{r}_\theta \\
& =\left(\mathbf{s}_h \circ \mathbf{s}_r^\theta \circ 1 + \mathbf{x}_h \circ \mathbf{x}_r^\theta \circ i^2  + \mathbf{y}_h \circ \mathbf{y}_r^\theta  \circ j^2 + \mathbf{z}_h \circ \mathbf{z}_r^\theta \circ k^2 \right) \\
& +\left(\mathbf{s}_h \circ \mathbf{x}_r^\theta \circ i + \mathbf{x}_h \circ \mathbf{s}_r^\theta  \circ i  + \mathbf{y}_h \circ \mathbf{z}_r^\theta \circ jk + \mathbf{z}_h \circ \mathbf{y}_r^\theta \circ kj \right) \\
& +\left(\mathbf{s}_h \circ \mathbf{y}_r^\theta \circ j + \mathbf{x}_h \circ \mathbf{z}_r^\theta \circ ik + \mathbf{y}_h \circ \mathbf{s}_r^\theta \circ j  + \mathbf{z}_h \circ \mathbf{x}_r^\theta \circ ik \right)  \\
& +\left(\mathbf{s}_h \circ \mathbf{z}_r^\theta \circ k + \mathbf{x}_h \circ \mathbf{y}_r^\theta \circ ij + \mathbf{y}_h \circ \mathbf{x}_r^\theta \circ ij + \mathbf{z}_h \circ \mathbf{s}_r^\theta \circ k \right) \\ 
& = \mathbf{s}_{h^\prime}+\mathbf{x}_{h^\prime}\textit{i}+\mathbf{y}_{h^\prime} \textit{j}+\mathbf{z}_{h^\prime} \textit{k},
\end{split}
\end{equation}
where the multiplication of imaginary components follows the rules (Eq.1-3) of the chosen hypercomplex systems.



The scoring function $\phi(h, r, t)$ is defined as:
\begin{equation}\small
\phi(h, r, t)= \langle \mathbf{h}^\prime, \mathbf{t} \rangle =\left\langle \mathbf{s}_{h^\prime}, \mathbf{s}_t\right\rangle+\left\langle \mathbf{x}_{h^\prime}, \mathbf{x}_t\right\rangle+\left\langle \mathbf{y}_{h^\prime}, \mathbf{y}_t\right\rangle+\left\langle \mathbf{z}_{h^\prime}, \mathbf{z}_t\right\rangle,
\end{equation}
where $\langle \cdot, \cdot \rangle$ represents the inner product. 

\subsubsection{Nested Fact Embeddings}

To represent an atomic fact $(h, r, t)$ without losing information, we embed each atomic triple as a $1\times3$ matrix, where each column corresponds to the embedding of the respective element. Consequently, we have $\mathbf{T}_i = \left[\mathbf{h}_i,\mathbf{r}_i,\mathbf{t}_i\right]$.

To embed various shapes of nested relations between $T_i$ and $T_j$ with relation $\widehat{r}$, 
we model each nested relation using a $1\times3$ translation matrix and a $3\times3$ rotation matrix, where each element is a 4D hypercomplex number. Specifically, we first translate the head triple $\mathbf{T}_i$ with $\mathbf{\widehat{r}}_b$, followed by applying a matrix-like rotation of $\mathbf{\widehat{r}}_\theta$, as defined by
\begin{equation}
\mathbf{T}i^\prime = ( \mathbf{T}i \oplus_{1 \times 3} \mathbf{\widehat{r}}_b ) \otimes_{3 \times 3} \mathbf{\widehat{r}}_\theta,
\end{equation}
where the matrix addition $\oplus_{1 \times 3}$ is performed through an element-wise summation of the hypercomplex components within the matrices. The matrix-like Hamilton product $\otimes_{3 \times 3}$ is defined as a product akin to matrix multiplication:
\begin{equation}\small
\begin{split}
\mathbf{T}_i^\prime =  \mathbf{T}_i \otimes_{3 \times 3} \mathbf{ \widehat{r} }_\theta = 
\left[
\begin{array}{c}
   \mathbf{h}_i\\
   \mathbf{r}_i\\
   \mathbf{t}_i\\
\end{array}
\right]^\top \times
\left[ 
\begin{array}{ccc}
   \mathbf{\widehat{r}}_{11}^\theta & \mathbf{\widehat{r}}_{12}^\theta & \mathbf{\widehat{r}}_{13}^\theta\\
   \mathbf{\widehat{r}}_{21}^\theta & \mathbf{\widehat{r}}_{22}^\theta & \mathbf{\widehat{r}}_{23}^\theta\\
   \mathbf{\widehat{r}}_{31}^\theta & \mathbf{\widehat{r}}_{32}^\theta & \mathbf{\widehat{r}}_{33}^\theta\\
\end{array}
\right] = \\
\left[
\begin{array}{c}
   \mathbf{h}_i \otimes \mathbf{\widehat{r}}_{11}^\theta + \mathbf{r}_i \otimes \mathbf{\widehat{r}}_{21}^\theta + \mathbf{t}_i \otimes \mathbf{\widehat{r}}_{31}^\theta   \\
   \mathbf{h}_i \otimes \mathbf{\widehat{r}}_{12}^\theta + \mathbf{r}_i \otimes \mathbf{\widehat{r}}_{22}^\theta + \mathbf{t}_i \otimes \mathbf{\widehat{r}}_{32}^\theta   \\
   \mathbf{h}_i \otimes \mathbf{\widehat{r}}_{13}^\theta + \mathbf{r}_i \otimes \mathbf{\widehat{r}}_{23}^\theta + \mathbf{t}_i \otimes \mathbf{\widehat{r}}_{33}^\theta   \\
\end{array}
\right]^\top = 
\left[
\begin{array}{c}
   \mathbf{h}_i^\prime\\
   \mathbf{r}_i^\prime\\
   \mathbf{t}_i^\prime\\
\end{array}
\right]^\top,
\end{split}
\end{equation}
where $\otimes$ is the Hamilton product. 

\textbf{Remarks} 
This matrix-like modeling of nested facts provides flexibility to capture diverse shapes of logical patterns inherent in nested relations. In essence, different shapes of situations or patterns can be effectively modeled by manipulating the $3\times3$ rotation matrix. For instance, relational implications can be represented using a diagonal matrix, while inversion can be captured using an anti-diagonal matrix. See \textbf{theoretical justification} for further analysis.

To assess the plausibility of the nested fact $(T_i, \widehat{r}, T_j)$, we calculate the inner product between the transformed head $\mathbf{T}_i^\prime$ fact and the tail fact $\mathbf{T}_j$ as:
\begin{equation}
\rho(T_i, \widehat{r}, T_j)= \langle\mathbf{T}_i^\prime, \mathbf{T}_j \rangle,
\end{equation}
where $\langle \cdot, \cdot \rangle$ denotes the matrix inner product.

\textbf{Learning objective}

We sum up the loss of atomic fact embedding $\mathcal{L}_{\text{atomic}}$, the loss of nested fact embedding $\mathcal{L}_{\text{meta}}$, and additionally the loss term $\mathcal{L}_{\text{aug}}$ for augmented triples generated by random walking as used in \cite{DBLP:conf/aaai/Chanyoung}. The overall loss is defined as
\begin{equation}
    \mathcal{L} = \mathcal{L}_{\text{atomic}} + \lambda_1 \mathcal{L}_{\text{nested}} + \lambda_2 \mathcal{L}_{\text{aug}}
\end{equation}
where $\lambda_1$ and $\lambda_2$ are the weight hyperparameters indicating the importance of each loss. 
Negative sampling is applied by randomly replacing one of the head or tail entity/triple. 
These losses are defined as follows:
\begin{equation}\footnotesize 
    \begin{array}{@{} *{6}{>{\displaystyle}c} @{}}
    \mathcal{L}_{\text{atomic}} = \sum_{(h,r,t) \in \mathcal{T}} g\left(-\phi\left(h,r,t\right) \right) + \sum_{g\left(\left(h^\prime,r^\prime,t^\prime\right)\right) \notin \mathcal{T}} g\left(\phi\left(h^\prime,r^\prime,t^\prime\right)\right)\\
     \mathcal{L}_{\text{aug}} = \sum_{(h,r,t) \in \mathcal{T}^\prime}g\left(-\phi\left(h,r,t\right)\right) + \sum_{(h^\prime,r^\prime,t^\prime) \notin \mathcal{T}^\prime}g\left(\phi(h^\prime,r^\prime,t^\prime)\right)\\
    \mathcal{L}_{\text{nested}} = \sum_{(T_i,\widehat{r},T_j) \notin \sTT } g\left(-\rho(T_i,\widehat{r},T_j)\right) + \sum_{(T_i^\prime,\widehat{r},T_j^\prime) \notin \sTT} g\left(\rho(T_i^\prime,\widehat{r},T_j^\prime)\right),\\
    \end{array}
\end{equation}
where $g=\log(1+\exp(x))$ and $\mathcal{T}^\prime$ is the set of augmented triples. 

\subsection{Theoretical Justification}
\label{sec:theory}

Modeling logical patterns is of great importance for KG embeddings because it enables generalization, i.e., once the patterns are learned, new facts that respect the patterns can be inferred. A logical pattern is a logical form $\psi \rightarrow \phi$ with $\psi$ and $\phi$ being the body and head, implying that if the body is satisfied then the head must also be satisfied. 

\textbf{First-order-logic-like logical patterns} 
Existing KG embeddings studied logical patterns expressed in the first-order-logic-like form. 
Prominent examples include symmetry $\forall h,t\colon (h,r,t) \rightarrow (t,r,h)$, anti-symmetry $\forall h,t\colon (h,r,t) \rightarrow \neg (h,r,t)$, inversion $\forall h,t\colon (h,r_1,t) \rightarrow (t,r_2,h)$ and composition $\forall e_1,e_2,e_3\colon (e_1,r_1,e_2) \wedge (e_2,r_2,e_3)  \rightarrow (e_1,r_3,e_3)$.

\begin{proposition}
FactE can infer symmetry, anti-symmetry, inversion, and composition, regardless of the specific choices of hypercomplex number systems.
\end{proposition}
This proposition holds because FactE subsumes ComplEx \cite{DBLP:conf/icml/TrouillonWRGB16} (i.e., 4D complex numbers generalize 2D complex numbers).



\textbf{Logical patterns over nested facts}
We extend the vanilla logical patterns in KGs to include nested facts. 
This can be expressed in a non-first-order-logic-like form $\psi \stackrel{\widehat{r}}{\rightarrow}\phi$.

\begin{figure}[t!]
    \centering
    \includegraphics[width=\linewidth]{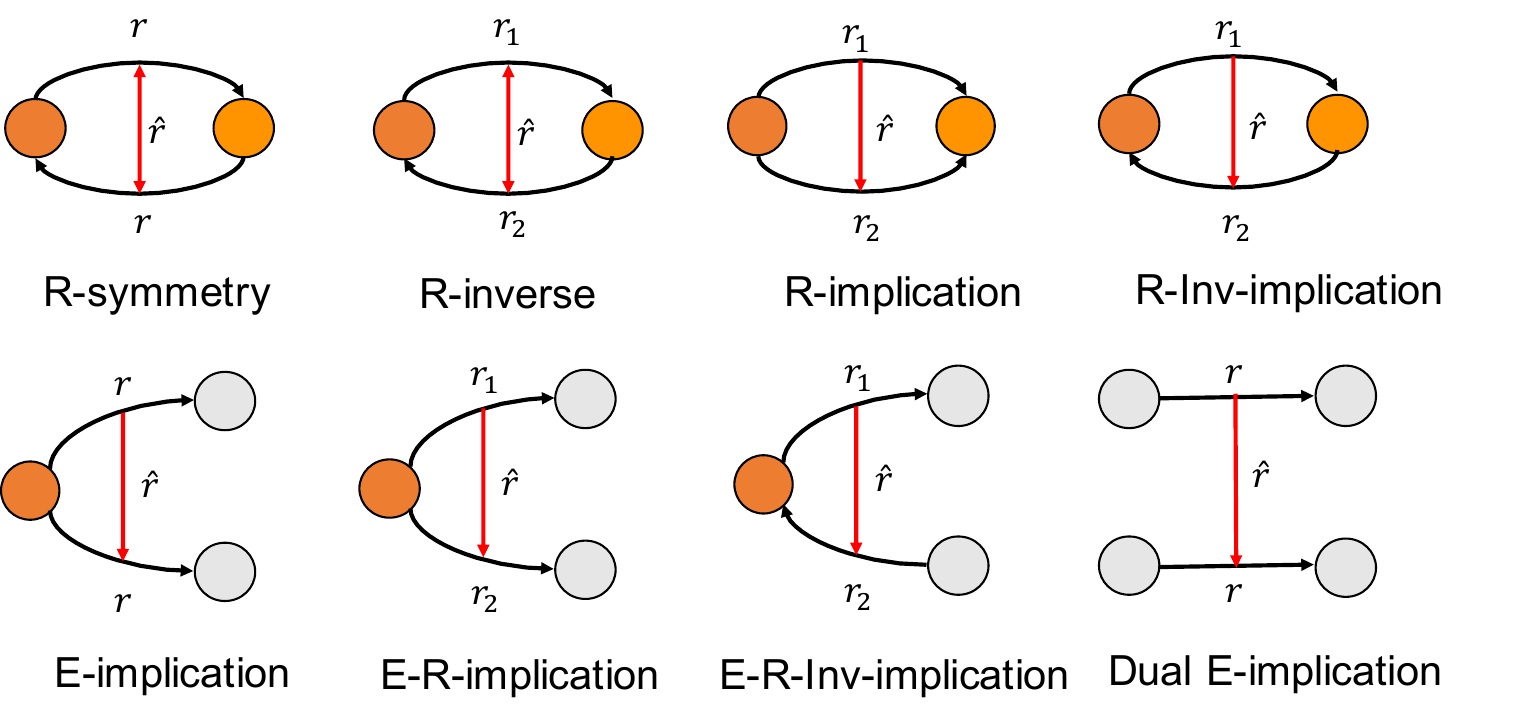}
    \vspace{-0.6cm}
    \caption{A structural illustration of different shapes of logical patterns, where the colored circles are free variables. 
    }
    \vspace{-0.1cm}
    \label{fig:pattern}
\end{figure}

\begin{itemize}
    \item \textbf{Relational symmetry (R-symmetry):} an atomic relation $r$ is symmetric w.r.t a nested relation $\widehat{r}$ if $\forall x,y \in \mathcal{E},  \langle x, r, y\rangle  \stackrel{\widehat{r}}{\leftrightarrow} \langle y, r, x\rangle$. 

    \item \textbf{Relational inverse (R-inverse):} two atomic relations $r_1$ and $r_2$ are inverse w.r.t a nested relation $\widehat{r}$ if $\forall x,y \in \mathcal{E}, $ $(\langle x, r_1, y\rangle \stackrel{\widehat{r}}{\leftrightarrow} \langle y, r_2, x\rangle)$.

    \item \textbf{Relational implication (R-implication):} an atomic relation $r_1$ implies a atomic relation $r_2$ w.r.t a nested relation $\widehat{r}$ if $\forall x,y \in \mathcal{E}, (\langle x, r_1, y\rangle \stackrel{\widehat{r}}{\rightarrow} \langle x, r_2, y\rangle)$.

    \item \textbf{Relational inverse implication (R-Inv-implication):} an atomic relation $r_1$ inversely implies an atomic relation $r_2$ w.r.t a nested relation $\widehat{r}$ if $\forall x,y \in \mathcal{E}, (\langle x, r_1, y\rangle \stackrel{\widehat{r}}{\rightarrow} \langle y, r_2, x\rangle)$.

    \item \textbf{Entity implication (E-implication):} an entity $x_1$ (resp. $y_1$) implies entity $x_2$ (resp. $y_2$) w.r.t an atomic relation $r$ and a nested relation $\widehat{r}$ if  $\forall y \in \mathcal{E}, (\langle x_1, r, y\rangle \stackrel{\widehat{r}}{\rightarrow} \langle x_2, r, y\rangle)$ (resp. $\forall x \in \mathcal{E}, (\langle x, r, y_1\rangle \stackrel{\widehat{r}}{\rightarrow} \langle x, r, y_2\rangle)$ ).

    \item \textbf{Entity relational implication (E-R-implication):} an entity $x_1$ and relation $r_1$ (resp. $y_1$ and relation $r_1$) implies entity $x_2$ and relation $r_2$ (resp. $y_2$ and relation $r_2$) 
    w.r.t a nested relation $\widehat{r}$ if  $\forall y \in \mathcal{E}, (\langle x_1, r_1, y\rangle \stackrel{\widehat{r}}{\rightarrow} \langle x_2, r_2, y\rangle)$ (resp. $\forall x \in \mathcal{E}, (\langle x, r_1, y_1\rangle \stackrel{\widehat{r}}{\rightarrow} \langle x, r_2, y_2\rangle)$).

    \item \textbf{Entity relational inverse implication (E-R-Inv-implication):}  an entity $x_1$ and relation $r_1$ (resp. $y_1$ and relation $r_1$) inversely implies entity $x_2$ and relation $r_2$ (resp. $y_2$ and relation $r_2$) w.r.t a nested relation $\widehat{r}$ if  $\forall y \in \mathcal{E}, (\langle x_1, r_1, y\rangle \stackrel{\widehat{r}}{\rightarrow} \langle y, r_2, x_2\rangle)$ (resp. $\forall x \in \mathcal{E}, (\langle x, r_1, y_1\rangle \stackrel{\widehat{r}}{\rightarrow} \langle y_2, r_2, x\rangle)$).
     \item \textbf{Dual Entity implication (Dual E-implication):} an entity pair ($x_1$, $x_2$) implies another entity pair ($y_1$, $y_2$)  iff both ($x_1$, $y_1$) and ($x_2$, $y_2$) satisfy E-implication. 

\end{itemize}

Fig. \ref{fig:pattern} illustrates the structure of the introduced patterns and Table \ref{tab:pattern} presents exemplary patterns of nested facts.

\begin{table}[]
    \centering
    \setlength{\tabcolsep}{0.1em}
     \caption{Exemplary logical patterns of nested facts.  }
    \resizebox{0.9\columnwidth}{!}{
    \begin{tabular}{c|ccccc}
        Pattern & $\widehat{r}$ & triple template & \\
        \midrule
        \multirow{2}{*}{R-Symmetry} & \multirow{2}{*}{EquivalentTo} & (Person A, \underline{IsMarriedTo}, Person B)  \\
         & & (Person B, \underline{IsMarriedTo}, Person A) \\
        \multirow{2}{*}{R-Inverse} & \multirow{2}{*}{EquivalentTo} & (Location A, \underline{UsesLanguage}, Language B)  \\
         & & (Language B, \underline{IsSpokenIn}, Location A) \\
         \multirow{2}{*}{R-Implication} & \multirow{2}{*}{ImpliesLocation} & (Country A, \underline{CapitalIsLocatedIn}, City B)  \\
         & & (Country A, \underline{Contains}, City B) \\
           \multirow{2}{*}{R-Inv-Implication} & \multirow{2}{*}{ImpliesLocation} & (Organization A, \underline{Headquarter}, Location B)  \\
         & & (Location B, \underline{Contains}, Organization A) \\
        \multirow{2}{*}{E-Implication} & \multirow{2}{*}{ImpliesTimeZone} & (\underline{Location A}, TimeZone, Time Zone B)  \\
         & & (\underline{Location in A}, TimeZone, Time Zone B) \\
         \multirow{2}{*}{E-R-Implication} & \multirow{2}{*}{ImpliesProfession} & (Person A, \underline{HoldsPosition, Government Position B})  \\
         & & (Person A, \underline{IsA, Politician}) \\
          \multirow{2}{*}{E-R-Inv-Implication} & \multirow{2}{*}{ImpliesProfession} & (\underline{Work A, CinematographyBy}, Person B)  \\
         & & (Person B, \underline{IsA, Cinematographer}) \\
           \multirow{2}{*}{Dual E-Implication} & \multirow{2}{*}{ImpliesLocation} & (\underline{Location A}, Contains, \underline{Location B})  \\
         & & (\underline{Location containing A}, Contains, \underline{Location in B}) \\
        \bottomrule
    \end{tabular}
    }
    \label{tab:pattern} 
\end{table}

\begin{proposition}
    FactE can infer R-symmetry, R-inverse, R-implication, R-Inv-implication, E-implication, E-R-implication, E-R-Inv-implication, and Dual E-implication.
\end{proposition}

\begin{proof}
To infer different logical patterns via different free variables, we can set some elements of the relation matrix to be zero-valued or one-valued complex numbers. For example, the implication and inverse implication relations can be inferred by setting the matrix to be diagonal or anti-diagonal. See Appendix for details.
\end{proof}


\subsection{Experimental Results}

\subsubsection{Experiment Setup}

\textbf{Datasets} 
We utilize three benchmark KGs: FBH, FBHE, and DBHE, that contain nested facts and are constructed by \cite{DBLP:conf/aaai/Chanyoung}. 
FBH and FBHE are based on FB15K237 from Freebase \cite{DBLP:conf/sigmod/BollackerEPST08} while DBHE is based on DB15K from DBpedia \cite{DBLP:conf/semweb/AuerBKLCI07}. 
FBH contains only nested facts that can be inferred from the triple facts, e.g., \emph{prerequisite\_for} and \emph{implies\_position}, while FBHE and DBHE further contain externally-sourced knowledge crawled from Wikipedia articles, e.g., \emph{next\_almaMater} and \emph{transfers\_to}. 
The authors of \cite{DBLP:conf/aaai/Chanyoung} spent six weeks manually defining these nested facts and adding them to the KGs. 
Besides, we employ the same data augmentation strategies as used in the original paper, i.e., adding reverse relations and reversed triple to the training set, and using random walks to augment plausible triples. 
The dataset details are presented in Table~\ref{tb:data}. We split $\sT$ and $\sTT$ into training, validation, and test sets in an 8:1:1 ratio. 

\begin{table}[t]
\small
\centering
\setlength{\tabcolsep}{0.65em}
\caption{Statistics of $\Ghat=(V, R, \sT, \Rhat, \sTT)$. $|\sT|^\prime$ denotes the number of atomic triples involved in the nested triples.}
\begin{tabular}{ccccccc}
\toprule
 & $|V|$ & $|R|$ & $|\sT|$ & $|\Rhat|$ & $|\sTT|$ & $|\sT|^\prime$ \\
\midrule
FBH & 14,541 & 237 & 310,117 & 6 & 27,062 & 33,157 \\
FBHE & 14,541 & 237 & 310,117 & 10 & 34,941 & 33,719 \\
DBHE & 12,440 & 87 & 68,296 & 8 & 6,717 & 8,206 \\
\bottomrule
\end{tabular}
\label{tb:data}
\end{table}

\textbf{Baselines} 
We consider BiVE-Q and BiVE-B \cite{DBLP:conf/aaai/Chanyoung} as our major baselines as they are specifically designed for KGs with nested facts and have demonstrated significant improvements over triple-based methods. 
We also compare some rule-based approaches as they indirectly consider relations between facts in first-order-logic-like expression, including Neural-LP \cite{DBLP:conf/nips/YangYC17}, DRUM \cite{DBLP:conf/nips/SadeghianADW19}, and AnyBURL \cite{DBLP:conf/ijcai/MeilickeCRS19}.
We further include QuatE \cite{DBLP:conf/nips/0007TYL19} and BiQUE \cite{DBLP:conf/emnlp/GuoK21} as they are the SoTA triple-based methods and they are also based on 4D hypercomplex numbers. 
However, these triple-based methods do not directly apply to the nested facts. 
Following \cite{DBLP:conf/aaai/Chanyoung}, we create a new triple-based KG $G_T$ where the atomic facts are converted into entities and nested facts are converted into triples (see Appendix for details).
For our approach, we implement three variants of FactE: FactE-Q (using quaternions), FactE-H (using hyperbolic quaternions), FactE-S (split quaternions), as well as their counterparts with translations: FactE-QB, FactE-HB, and FactE-SB. 
In the Appendix, we also extend BiVE-Q and BiVE-B to other hypercomplex numbers: BiVE-H, BiVE-HB BiVE-S, and BiVE-SB for further comparison. 
We employ three standard metrics: Filtered MR (Mean Rank), MRR (Mean Reciprocal Rank), and Hit@$10$. 
We report the mean performance over $10$ random seeds for each method, and the relatively small standard deviations are omitted. 


\textbf{Implementation details} 
We implement the framework based on OpenKE \footnote{https://github.com/thunlp/OpenKE} and the code  \footnote{https://github.com/bdi-lab/BiVE/}.
We train our methods on triple prediction and evaluate them on other tasks. The dimensionality is set to be $d=200$. 
We reuse the hyperparameters as used in \cite{DBLP:conf/aaai/Chanyoung} and do not perform further hyperparameter searches. The detailed hyperparameter settings can be found in the Appendix.

\begin{table}[t!]
\small
\centering
\setlength{\tabcolsep}{0.3em}
\renewcommand{\arraystretch}{0.9} 
\caption{Results of triple prediction. Shaded numbers are better results than the best baseline. The best scores are boldfaced and the second best scores are underlined. * denotes results taking from \cite{DBLP:conf/aaai/Chanyoung}. }
\resizebox{\columnwidth}{!}{
\begin{tabular}{cccccccccc}
\toprule
 & \multicolumn{3}{c}{FBH} & \multicolumn{3}{c}{FBHE} & \multicolumn{3}{c}{DBHE} \\
 & MR ($\downarrow$) & MRR  ($\uparrow$) & Hit@10 ($\uparrow$) & MR ($\downarrow$) & MRR ($\uparrow$) & Hit@10 ($\uparrow$) & MR ($\downarrow$) & MRR ($\uparrow$) & Hit@10 ($\uparrow$) \\
\midrule
QuatE*  & 145603.8 & 0.103 & 0.114 & 94684.4 & 0.101 & 0.209 & 26485.0 & 0.157 & 0.179 \\
BiQUE*  & 81687.5 & 0.104 & 0.115 & 61015.2 & 0.135 & 0.205 & 19079.4 & 0.163 & 0.185 \\
\midrule
Neural-LP* & 115016.6 &  0.070 &  0.073 &  90000.4 &  0.238 &  0.274 &  21130.5 &  0.170 & 0.209 \\
DRUM* & 115016.6 & 0.069 & 0.073 & 90000.3 & 0.261 &  0.274 & 21130.5 & 0.166 & 0.209 \\
AnyBURL* & 108079.6 & 0.096 & 0.108 & 83136.8 & 0.191 & 0.252 & 20530.8 & 0.177 & 0.214 \\
\midrule
BiVE-Q & 6.20 & 0.855 &	0.941 & 8.35 & 0.711 & 0.866 & 3.63 & 0.687 & 0.958 \\
BiVE-B & 8.63 & 0.833 & 0.924 & 9.53 & 0.705 & 0.860 & 4.66 & 0.718 & 0.945\\
\midrule
FactE-Q (Ours) & 6.56 & \cellcolor{gray!25}0.863 & \cellcolor{gray!25}0.953 & \cellcolor{gray!25}5.77 & \cellcolor{gray!25}0.811 & \cellcolor{gray!25}0.943 & \cellcolor{gray!25}3.51 & \cellcolor{gray!25}0.809 & \cellcolor{gray!25}0.960  \\
FactE-H (Ours) & \cellcolor{gray!25}4.69 & \cellcolor{gray!25}0.858 & \cellcolor{gray!25}0.964 & \cellcolor{gray!25}3.99 & \cellcolor{gray!25}0.781 & \cellcolor{gray!25}0.943 & \cellcolor{gray!25}2.65 & \cellcolor{gray!25}0.806 & \cellcolor{gray!25}0.969 \\
FactE-S (Ours) & \cellcolor{gray!25}\underline{3.87} & \cellcolor{gray!25}0.867 & \cellcolor{gray!25}\underline{0.977} & \cellcolor{gray!25}3.60 & \cellcolor{gray!25}0.795 & \cellcolor{gray!25}0.947 & \cellcolor{gray!25}2.55 & \cellcolor{gray!25}0.809 & \cellcolor{gray!25}0.966  \\
\midrule
FactE-QB (Ours) & \cellcolor{gray!25}6.04 & \cellcolor{gray!25}0.898 & \cellcolor{gray!25}0.958 & \cellcolor{gray!25}5.55 & \cellcolor{gray!25}\underline{0.845} & \cellcolor{gray!25}0.947 & \cellcolor{gray!25}\underline{2.54} & \cellcolor{gray!25}\underline{0.847} & \cellcolor{gray!25}\underline{0.973}  \\
FactE-HB (Ours) & \cellcolor{gray!25}3.82 & \cellcolor{gray!25}\underline{0.899} & \cellcolor{gray!25}0.971 & \cellcolor{gray!25}\underline{3.53} &  \cellcolor{gray!25}0.828 & \cellcolor{gray!25}\underline{0.955} & \cellcolor{gray!25}2.62 & \cellcolor{gray!25}0.842 & \cellcolor{gray!25}0.972 \\
FactE-SB (Ours) & \cellcolor{gray!25}\textbf{3.34} & \cellcolor{gray!25}\textbf{0.922} & \cellcolor{gray!25}\textbf{0.982} & \cellcolor{gray!25}\textbf{3.05} & \cellcolor{gray!25}\textbf{0.851} & \cellcolor{gray!25}\textbf{0.962} & \cellcolor{gray!25}\textbf{2.07} & \cellcolor{gray!25}\textbf{0.862} & \cellcolor{gray!25}\textbf{0.984} \\
\bottomrule
\end{tabular}
}
\vspace{-0.1cm}
\label{tb:tp}
\end{table}

\begin{table}[t!]
\small
\centering
\setlength{\tabcolsep}{0.3em}
\renewcommand{\arraystretch}{0.9} 
\caption{Results of conditional link prediction. Shaded numbers are better results than the best baseline. The best scores are boldfaced and the second best scores are underlined. * denotes results taking from \cite{DBLP:conf/aaai/Chanyoung}. }
\resizebox{\columnwidth}{!}{
\begin{tabular}{cccccccccc}
\toprule
 & \multicolumn{3}{c}{FBH} & \multicolumn{3}{c}{FBHE} & \multicolumn{3}{c}{DBHE} \\
 & MR ($\downarrow$) & MRR ($\uparrow$) & Hit@10 ($\uparrow$) & MR ($\downarrow$) & MRR ($\uparrow$) & Hit@10 ($\uparrow$) & MR ($\downarrow$) & MRR ($\uparrow$) & Hit@10 ($\uparrow$) \\
\midrule
QuatE* & 163.7 & 0.346 & 0.494 & 1546.4 & 0.124 & 0.189 & 551.6 & 0.208 & 0.309 \\
BiQUE* & 111.0 & 0.423 & 0.641 & 90.1 & 0.387 & 0.617 & 29.5 & 0.378 & 0.677 \\
\midrule
Neural-LP* & 185.9 & 0.433 & 0.648 & 146.2 & 0.466 & 0.716 & 32.2 & 0.517 & 0.756\\
DRUM* & 262.7 & 0.394 & 0.555 & 207.6 & 0.413 & 0.620 & 49.0 & 0.470 & 0.732\\
AnyBURL* & 228.5 & 0.380 & 0.563 & 166.0 & 0.418 & 0.607 & 81.7 & 0.403 & 0.594\\
\midrule
BiVE-Q & 4.33 & 0.826 & 0.948 & 6.56 & 0.761 & 0.886 & 2.69 & 0.852 & 0.971  \\
BiVE-B & 5.34 & 0.836 & 0.940 & 7.49 & 0.761 & 0.872 & 2.91 & 0.858 & 0.967  \\
\midrule
FactE-Q (Ours)  & \cellcolor{gray!25}1.70 & \cellcolor{gray!25}0.930 & \cellcolor{gray!25}0.986 & \cellcolor{gray!25}2.89 & \cellcolor{gray!25}0.863 & \cellcolor{gray!25}0.948 & \cellcolor{gray!25}\textbf{1.68} & \cellcolor{gray!25}\underline{0.930} & \cellcolor{gray!25}0.987  \\

FactE-H (Ours) & \cellcolor{gray!25}1.68 & \cellcolor{gray!25}0.909 &	\cellcolor{gray!25}0.987 & \cellcolor{gray!25}2.87 & \cellcolor{gray!25}0.843 & \cellcolor{gray!25}0.945 & \cellcolor{gray!25}1.82 & \cellcolor{gray!25}0.912 & \cellcolor{gray!25}0.986 \\
FactE-S (Ours)  & \cellcolor{gray!25}\underline{1.54} & \cellcolor{gray!25}0.925 & \cellcolor{gray!25}\textbf{0.991} & \cellcolor{gray!25}3.04 & \cellcolor{gray!25}0.850 & \cellcolor{gray!25}0.941 & \cellcolor{gray!25}1.76 & \cellcolor{gray!25}0.910 & \cellcolor{gray!25}\underline{0.988}  \\
\midrule
FactE-QB (Ours)  & \cellcolor{gray!25}1.71 & \cellcolor{gray!25}\textbf{0.935} & \cellcolor{gray!25}0.987 & \cellcolor{gray!25}3.00 & \cellcolor{gray!25}\underline{0.865} & \cellcolor{gray!25}0.949 & \cellcolor{gray!25}\underline{1.70} & \cellcolor{gray!25}\textbf{0.931} & \cellcolor{gray!25}0.986  \\
Fact-HB (Ours) & \cellcolor{gray!25}1.60 & \cellcolor{gray!25}0.924 & \cellcolor{gray!25}\underline{0.989} & \cellcolor{gray!25}\underline{2.76} & \cellcolor{gray!25}0.855 & \cellcolor{gray!25}\underline{0.950} & \cellcolor{gray!25}1.92 & \cellcolor{gray!25}0.918 & \cellcolor{gray!25}0.981 \\
FactE-SB (Ours) & \cellcolor{gray!25}\textbf{1.52} & \cellcolor{gray!25}\underline{0.934} & \cellcolor{gray!25}\textbf{0.991} & \cellcolor{gray!25}\textbf{2.61} & \cellcolor{gray!25}\textbf{0.867} & \cellcolor{gray!25}\textbf{0.951} & \cellcolor{gray!25}1.72 & \cellcolor{gray!25}0.919 & \cellcolor{gray!25}\textbf{0.990}  \\
\bottomrule
\end{tabular}
}
\vspace{-0.2cm}
\label{tb:clp}
\vspace{-0.2cm}
\end{table}

\subsubsection{Main Results}

\textbf{Triple prediction} 

Table \ref{tb:tp} presents the results of triple prediction. 
First, it shows that all triple-based approaches yield relatively modest results compared to BiVE-Q and BiVE-B, designed specifically for KGs with nested facts. Our approach, FactE-Q, the quaternionic version, already outperforms the baselines across most metrics. 
Particularly notable are the pronounced enhancements in FBHE and DBHE, with MRR improvements of 14.1\% and 17.7\% respectively, underscoring the efficacy of the proposed FactE model.
Furthermore, FactE-H and FactE-S demonstrate heightened performance over FactE-Q across various evaluation metrics, particularly in terms of MR. This highlights the advantages that hyperbolic quaternions and split quaternions offer over standard quaternions. Impressively, the split quaternionic version attains the highest performance, followed closely by the hyperbolic quaternionic variant.
Moreover, through the incorporation of a hypercomplex translation component, FactE-QB, Fact-HB, and FactE-SB consistently outperform their non-translation counterparts, illustrating the advantages of combining multiple transformations (rotation and translation) within the hypercomplex space.

\textbf{Conditional link prediction} 
Table \ref{tb:clp} shows the outcomes of conditional link prediction. It is evident that all three FactE variants substantially outperform the two SoTA baselines, BiVE-Q and BiVE-B, across all datasets. Notably, the best FactE variant surpasses the baselines by 11.8\%, 13.9\%, and 8.5\% in terms of MRR for FBH, FBHE, and DBHE, respectively. This remarkable performance gain underscores the effectiveness of the proposed method.
Similar to the trends observed in triple prediction, the incorporation of translation components in FactE-QB, Fact-HB, and Fact-SB leads to further improvements over their counterparts without translation components. This reaffirms the advantages gained from the integration of multiple hypercomplex transformations. Intriguingly, we noticed that varying hypercomplex number systems yield the best performance on different datasets, contrasting the observations from triple prediction. We conjecture that this stems from the inherent variance in inductive biases offered by different hypercomplex number systems, making them more suitable for certain datasets over others.
We believe the choices of spaces can be linked to a hyperparameter that offers flexibility in adapting to diverse dataset characteristics.

\begin{table}
\small
\centering
\setlength{\tabcolsep}{0.2em}
\renewcommand{\arraystretch}{0.8} 
\caption{Results of base link prediction. The best scores are boldfaced and the second best scores are underlined. * denotes results taking from \cite{DBLP:conf/aaai/Chanyoung}. }
\resizebox{\columnwidth}{!}
{
\begin{tabular}{ccccccc}
\toprule
 & \multicolumn{3}{c}{FBHE} & \multicolumn{3}{c}{DBHE} \\
 & MR ($\downarrow$) & MRR ($\uparrow$) & Hit@10 ($\uparrow$) & MR ($\downarrow$) & MRR ($\uparrow$) & Hit@10 ($\uparrow$) \\

\midrule
QuatE* & 139.0 & 0.354 & 0.581 & 409.6 & 0.264 & 0.440\\
BiQUE* &  134.9 & 0.356 & 0.583 & \textbf{376.6} & 0.274 & \textbf{0.446}\\
\midrule
Neural-LP* & 1942.5 & 0.315 & 0.486 & 2904.8 & 0.233 & 0.357 \\ 
DRUM* & 1945.6 & 0.317 & 0.490 & 2904.7 & 0.237 & 0.359 \\
AnyBURL* & 342.0 & 0.310 & 0.526 & 879.1 & 0.220 & 0.364\\
\midrule
BiVE-Q &  136.13 & 0.369 & 0.603 & 827.18 & 0.271 & 0.428  \\
BiVE-B & 136.54 & \underline{0.370} & \underline{0.607} & 795.59 & 0.274 & 0.422 \\
\midrule
FactE-Q (Ours) &  \underline{131.72} & 0.365 & 0.605  & 749.75 & \underline{0.284} & \textbf{0.446}  \\
Fact-H (Ours) & 153.00 & 0.349 & 0.593 & 868.82 & 0.266 & 0.423  \\
FactE-S (Ours) & 149.64 & 0.350 & 0.592 & 895.85 & 0.272 & 0.432   \\
\midrule
FactE-QB (Ours) & \textbf{130.13} & \textbf{0.371} & \textbf{0.608} & 751.18 & \textbf{0.289} & \underline{0.443}  \\
Fact-HB (Ours) & 155.74 & 0.353 & 0.594 & 801.76 & 0.271 & 0.423 \\
FactE-SB (Ours) & 149.73 & 0.355 & 0.594 & 827.89 & 0.273 & 0.431  \\
\bottomrule
\end{tabular}
}
\vspace{-0.2cm}
\label{tb:blp}
\end{table}

\textbf{Base link prediction}
Table \ref{tb:blp} illustrates the results of base link prediction. Among our approaches, namely FactE-Q, FactE-H, and FactE-S, we observe competitive or improved results in comparison to SoTA embedding-based and rule-based methods on the FBHE and DBHE datasets. The best performance is achieved by FactE-QB, which outperforms the baselines across a majority of metrics. This outcome substantiates the fact that the incorporation of nested facts into triple-based KGs indeed enhances the inference capabilities for base link prediction.

\subsubsection{Ablation Analysis}

\textbf{Embedding analysis of logical patterns} 
To verify whether the learned embeddings capture the inference of logical patterns over nested facts, 
we visualized the real part of the embeddings of the $8$ relations in DBHE. 
The analysis of the embeddings yields insightful observations. As shown in Fig. \ref{fig:patternanalysis}, the lower left element and upper right element of the embedding \emph{EquivalentTo} are $1$, showcasing that \emph{EquivalentTo} predominantly adheres to R-symmetry or R-inverse. On the other hand, the upper left element and lower right element of the \emph{ImpliesLang.} are $1$, affirming its alignment with the R-implication rule. Similarly, the embeddings of \emph{NextAlmaM.}, \emph{TransfersTo}, and \emph{ImpliesGenre} indicate high adherence to E-implicationsas as only one of the corners is $1$. 
We find that the embedding of relation \emph{ImpliesProf.} does not have a significant pattern. We conjecture that this is because \emph{ImpliesProf.} follows many rule patterns and there exists no global solution that satisfies all rules. 
See the Appendix for the statistics of the logical patterns in the datasets.

\begin{figure}
    \centering
    \includegraphics[width=\linewidth]{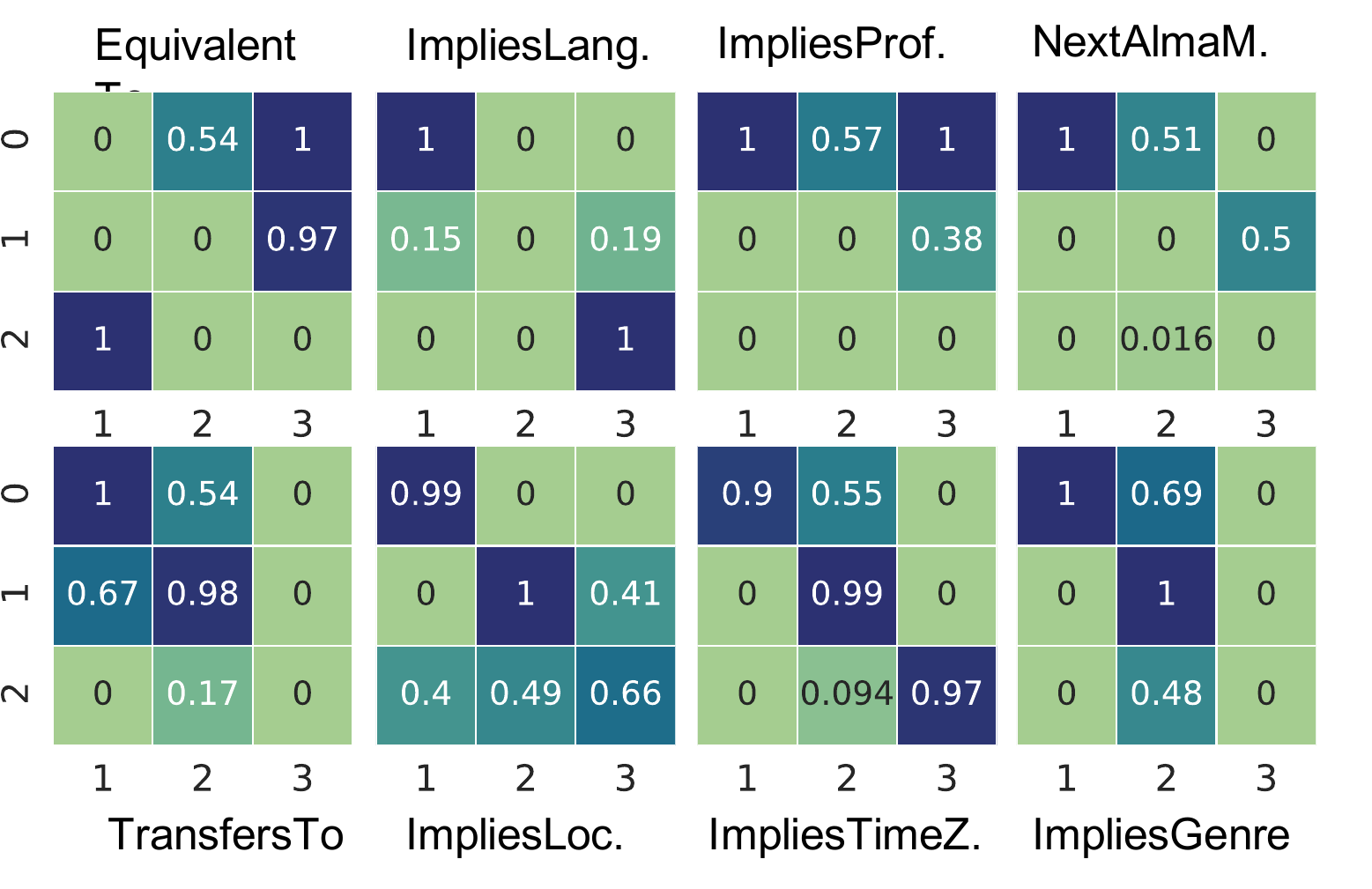}
     \vspace{-0.7cm}
    \caption{The visualization of the average of the real component embeddings of the $8$ nested relations in DBHE. }
    \vspace{-0.5cm}
    \label{fig:patternanalysis}
\end{figure}

\textbf{Influence of nested fact embeddings}
To evaluate the influence of nested fact embeddings, we perform a comparison by excluding the loss associated with nested fact embeddings (i.e., setting $\lambda_1=0$). The outcomes presented in Table \ref{tb:impact_nested} underscore the significant enhancements achieved by incorporating nested fact embeddings, particularly evident in the improvements in MR and H@10 for DBHE.

\begin{table}
\small
\centering
\setlength{\tabcolsep}{0.2em}
\renewcommand{\arraystretch}{0.9} 
\caption{Ablation study on the nested fact embeddings for base link prediction. 
Best results are boldfaced. 
}
\resizebox{\columnwidth}{!}
{
\begin{tabular}{ccccccc}
\toprule
 & \multicolumn{3}{c}{FBHE} & \multicolumn{3}{c}{DBHE} \\
 & MR ($\downarrow$) & MRR ($\uparrow$) & Hit@10 ($\uparrow$) & MR ($\downarrow$) & MRR ($\uparrow$) & Hit@10 ($\uparrow$) \\
\midrule
FactE-Q ($\lambda_1=0$) &  135.38 & \textbf{0.368} & 0.604 & 799.97 & 0.281 & 0.431  \\
FactE-Q ($\lambda_1=0.5$) &  \textbf{131.72} & 0.365 & \textbf{0.605}  & \textbf{749.75} & \textbf{0.284} & \textbf{0.446}   \\
\midrule
FactE-H ($\lambda_1=0$) &  154.75 &  0.347 & 0.589 & 922.34 & 0.267 & 0.420 \\
FactE-H ($\lambda_1=0.5$) & \textbf{153.00} & \textbf{0.349} & \textbf{0.593} & \textbf{868.82} & 0.266 & \textbf{0.423}  \\
\midrule
FactE-S ($\lambda_1=0$) & 151.84 & 0.347 & 0.589 & 910.16 & \textbf{0.272} & 0.427    \\
FactE-S ($\lambda_1=0.5$)& \textbf{149.64} & \textbf{0.350} & \textbf{0.592} & \textbf{895.85} & \textbf{0.272} & \textbf{0.432}  \\
\bottomrule
\end{tabular}
}
\vspace{-0.1cm}
\label{tb:impact_nested}
\end{table}

\begin{table}[t]
\scriptsize
\centering
\setlength{\tabcolsep}{0.2em}
\renewcommand{\arraystretch}{0.9} 
\caption{Performance per relation on triple prediction. Freq. indicates the number of nested facts in the test set.  }
\resizebox{\columnwidth}{!}
{
\begin{tabular}{cccccccc}
\toprule
 $\rhat$ & Freq. & FactE-Q & FactE-QB & FactE-H & FactE-HB  & FactE-S & FactE-SB  \\
\midrule
Equiv.To & 98 & 0.994 & \textbf{0.997} & 0.992 & \textbf{0.997} & \textbf{0.997} & \underline{0.995} \\
ImpliesLang. & 29 & \underline{0.671} & 0.602 & \textbf{0.680} & 0.662 & 0.614 & 0.622 \\
ImpliesProf. & 210 & 0.807 & 0.916 & 0.830 & \underline{0.935} & 0.832 & \textbf{0.936} \\
ImpliesLocat. & 163 & \underline{0.929} & 0.869 & 0.893 & 0.810 & \underline{0.929} & \textbf{0.958} \\
ImpliesTime. & 44 & 0.305 & 0.297 & \underline{0.307} & \textbf{0.329} & 0.290 & 0.293 \\
ImpliesGenre & 84 & 0.726 & \underline{0.762} & 0.719 & 0.741 & 0.742 & \textbf{0.796} \\
NextAlmaM. & 14 & 0.770 & \textbf{0.812} & 0.689 & 0.688 & 0.751 & \underline{0.795} \\
Transf.To & 29 &  \textbf{0.977} & \underline{0.952} & 0.964 & 0.949 & 0.921 & 0.953 \\
\bottomrule
\end{tabular}
}

\label{tb:rel}
\vspace{-0.5cm}
\end{table}

\textbf{Relation-specific performance}
In Table \ref{tb:rel}, we present the performance results for each relation within the DBHE dataset. Notably, the diverse hypercomplex number systems lead to optimal performance for different relations. This reiterates our conjecture that distinct benefits are offered by varying hypercomplex number systems, catering to the specific characteristics of different relation types.
Remarkably, our findings reveal that the incorporation of a hypercomplex translation component (as seen in FactE-QB, FactE-HB, and FactE-SB) notably enhances the embeddings of relations such as \emph{ImpliesProf.} and \emph{ImpliesGenre} across all variants of hypercomplex number systems. 
However, this does not extend to relations like \emph{ImpliesLocat.} and \emph{ImpliesLang.}, suggesting a more complex relationship between these specific relations and the hypercomplex translation components.

\subsection{Conclusion}

This paper considers a novel perspective by extending traditional atomic factual knowledge representation to include nested factual knowledge. This enables the representation of both temporal situations and logical patterns that go beyond conventional first-order logic expressions (Horn rules). Our proposed approach, FactE, presents a family of hypercomplex embeddings capable of embedding both atomic and nested factual knowledge. This framework effectively captures essential logical patterns that emerge from nested facts. Empirical evaluation demonstrates the substantial performance enhancements achieved by FactE compared to existing baseline methods. 
Additionally, our generalized hypercomplex embedding framework unifies previous algebraic (e.g., quaternionic) and geometric (e.g., hyperbolic) embedding methods, offering versatility in embedding diverse relation types. 

\cleardoublepage

\chapter{Conclusion, Limitations, and Future works}
\label{chap_conclusion}

\section{Conclusion}

Relational data offer a structured representation of real-world knowledge that has been applied in a wide range of applications. 
Many types of relational data, such as social networks, knowledge graphs, and ontologies, are incomplete and noisy due to the process of human curation. 
Relational representation learning aims to address these issues by mapping these relational objects into continual low-dimensional vector space such that the relational structure is preserved and missing relational knowledge can be inferred by the analogical or the similarity structure among the embedding vectors. 

However, mapping relational data to a continual vector space is more challenging than the embeddings of image and text data that typically require only preservation of "similarity" between data objects. Relational data, besides requiring capturing similarity, exhibit various discrete properties that cannot be easily captured by the plain vectors in Euclidean space. 
In this dissertation, we consider geometric relational embeddings that map relational objects as geometric objects that faithfully model the  discrete properties inherent in the relational data. In particular, we made the following contributions. 

\begin{description}[leftmargin=2.30cm]

\item[Chapter~\ref{chap_structral}] introduces pseudo-Riemannian manifold embeddings.  We propose a QGCN, pseudo-Riemannian graph convolutional networks (GCNs) that generalizes GCN in the pseudo-Riemannian manifolds. This GCN defines node embeddings in the pseudo-Riemannian manifolds allowing for capturing both hierarchical and cyclic structural patterns. Besides, we propose a pseudo-Riemannian knowledge graph embedding method that models entities in pseudo-Riemannian manifold and relations as pseudo-orthogonal transformation. The proposed KG embedding allows for simultaneous modeling of graph structural patterns (hierarchices and cycles) and relational patterns (symmstry, inversion, composition).
  
\item[Chapter~\ref{chap_ontological}] introduces BoxEL, a geometric embedding that allows for better capturing the logical structure in a knowledge base that is expressed in the Description Logic EL++. BoxEL models concepts as boxes and models relations as affine transformation. This modeling faithfully models the intersectional closure and the complex relational mapping between concepts.

\item[Chapter~\ref{chap_logical}] introduces HMI, a hyperbolic embedding method that explicitly encodes the class hierarchy and exclusion relations for structured multi-label prediction. Embedding such relational constraints of classes onto embedding space is useful as it encourages the predictions to be coherent to the given relational constraints.

\item[Chapter~\ref{chap_hyper}] introduces two high-order relational embeddings: 1) We propose ShrinkE, a hyper-relational embedding model that allows for modeling logical patterns over qualifiers, including monotonicity and qualifier-level logical patterns (ie..g, qualifier implication and qualifier exclusion); 2) We propose NestE, an innovative approach designed to embed the semantics of both atomic facts and nested facts that enable representing temporal situations and logical patterns over facts.

\end{description}

\section{Limitations and Future Works}

The current geometric embeddings still suffer from several limitations.

\begin{itemize}

\item \textbf{Shallow nature of geometric embeddings:} Although geometric embeddings have found applications in various relational reasoning tasks, many of these methods learn embeddings in a shallow manner. This becomes problematic when the input relational objects contain rich multi-modal features, such as images and text descriptions. In such cases, geometric embeddings should learn a neural network encoder that takes images or text as input and outputs the corresponding geometric objects.

\item \textbf{Lack of freely chosen geometric inductive biases:} Different geometric inductive biases may be suitable for various inference tasks. Typically, geometric inductive biases are chosen based on prior human knowledge about the tasks. However, when a task requires the joint consideration of multiple properties, developing a geometric relational embedding that simultaneously captures these diverse properties remains highly challenging.

\item \textbf{Increased computational and parameter requirements:} Geometric embeddings often necessitate additional computation or memory. For instance, hyperbolic embeddings involve the computation of exponential and logarithmic maps. Box embeddings require two embedding vectors, storing the lower-left corner and the upper-right corner, respectively.

\end{itemize}

In the future, we plan to explore the following directions to address and enhance our research:

\begin{itemize}

    \item \textbf{Geometric inductive biases for imperfect learning settings.} 
    Real-world data rarely align with perfection, and many existing methods assume balanced training data, leading to biased predictions, often referred to as minority collapse. This poses challenges in generalizing to real-world scenarios. 
    We plan to develop geometric relational biases explicitly incorporating inductive biases to promote equality. 
    An example of such a bias is the maximum class separation. 
    Generalizing these biases to geometric embeddings is non-trivial and requires specific geometric design considerations.

    \item \textbf{Heterogeneous hierarchies} Most current geometric embedding methods can only encode one hierarchy relation (e.g., \emph{is\_a}). However,  a real-world KG might simultaneously contain multiple hierarchical relations (e.g., \emph{is\_a} and \emph{has\_part}) \cite{patel2020representing}. We plan to develop embedding models that simultaneously encode multiple hierarchies. 

    \item \textbf{Injecting relational constraints into Large Language Models (LLMs).} 
    Relational knowledge graphs offer explicit factual knowledge that complements the inherent lack of factual knowledge in LLMs. 
    By learning explicit geometric KG embeddings, complex relational knowledge, such as logical rules, can be preserved. 
    However, LLMs do not guarantee the presence of such factual knowledge, leading to hallucination, especially in the case of domain-specific data like biomedical and healthcare information. In the future, we plan to inject relational constraints into LLMs.

    \item \textbf{Learning vector-symbolic representations.} 
    The fusion of knowledge representation (symbolic methods) and deep learning (neural methods) is particularly intriguing, leveraging the strengths of both approaches—the reliability of knowledge representations and the effectiveness of deep learning. 
    We aim to develop vector-symbolic representations that carry symbolic meanings while serving as continuous vectors suitable for input into neural networks. An example of such an approach is the Vector Symbolic Architecture (VSA) \cite{DBLP:journals/air/SchlegelNP22,DBLP:journals/csur/KleykoROR23}, which has seen limited exploration in the machine learning community.

    \item \textbf{Hybrid semantic search.} LLMs are shaping the future of data management by transforming unstructured data into numerical vectors stored in a Vector Database. However, these embeddings are structureless and do not carry any explicit symbolic information (e.g., item attributes) that plays essential roles in querying unstructured data. With geometric embedding, we may perform hybrid query that combines similarity search with structured attribute filtering. This can be done by transforming attribute information into geometric vector space. 

    \item \textbf{Efficiency improvements.} We plan to investigate strategies for enhancing the efficiency of geometric embeddings by reducing additional computational and memory requirements. This includes exploring optimizations for specific geometric embedding methods, potentially streamlining processes and making them more scalable for real-world applications.

    \item \textbf{Flexible geometric inductive biases.} We may explore to devise methods that allow for the incorporation of freely chosen geometric inductive biases. This involves exploring techniques that dynamically adapt to different inference tasks, especially those requiring the joint consideration of multiple properties, to improve the flexibility and adaptability of geometric relational embeddings.

    \item \textbf{Applications in biomedical sciences.}  Geometric embeddings benefit biomedical science in modeling various aspects of biomedical networks: 1) biomedical networks exhibit hierarchical structures, e.g., the structures of ICD-9 codes, and MeSH terms; 2) biomedical knowledge graphs are heterogeneous and many of the relations exhibits complex relational patterns (e.g., symmetry of drug-drug interaction); 3) biomedical ontologies and taxonomies have logical structures and the prediction models must be consistent with these logical structures. My future work would be exploring novel relational embeddings that are suitable for biomedical scenarios.

\end{itemize}

\addtocontents{toc}{\protect\setcounter{tocdepth}{0}}
\cleardoublepage
\appendix

\newpage

\chapter{Proof of Theorems}

\setcounter{proposition}{0}
\setcounter{corollary}{0}
\setcounter{theorem}{0}
\setcounter{lemma}{0}

\section{Proof of Theorems of QGCN}\label{app:proof}

\subsection{Proof of Theorem 1}\label{app:3.1}
\begin{theorem}[Theorem 4.1 in \cite{law2020ultrahyperbolic}]
\label{appendix_theorem_law_diff}
For any point $\mathbf{x} \in \mathcal{Q}_{\beta}^{s, t}$, there exists a diffeomorphism $\psi: \mathcal{Q}_{\beta}^{s, t} \rightarrow  \mathbb{S}_{1}^{t} \times \mathbb{R}^{s}$ that maps $\mathbf{x}$ into the product manifolds of an unit sphere and the Euclidean space, the mapping and its inverse are given by, 
\begin{equation}
\psi(\mathbf{x})=\left(\begin{array}{c}
\frac{1}{\|t\|} \mathbf{t} \\
\frac{1}{\sqrt{|\beta|}} \mathbf{s}
\end{array}\right), \\
\quad \psi^{-1}(\mathbf{z})=\sqrt{|\beta|}\left(\begin{array}{c}
\sqrt{1+\|\mathbf{v}\|^{2}} \mathbf{u} \\
\mathbf{v}
\end{array}\right),
\end{equation}
where $\mathbf{x}=\left(\begin{array}{c}
\mathbf{t} \\
\mathbf{s}
\end{array}\right) \in \mathcal{Q}_{\beta}^{s, t}$ with $\mathbf{t} \in \mathbb{R}_*^{t+1}$ and $\mathbf{s} \in \mathbb{R}^s$. $\mathbf{z}=\left(\begin{array}{c}
\mathbf{u} \\
\mathbf{v}
\end{array}\right) \in \mathbb{S}_{1}^{t} \times \mathbb{R}^{s}$ with $\mathbf{u} \in \mathbb{S}_{1}^t$ and $\mathbf{v} \in \mathbb{R}^{s}$. 
\end{theorem}
Please refer Appendix C.5 in \cite{law2020ultrahyperbolic} for proof of this theorem.

\subsection{Proof of Theorem 2}\label{app:lm:sbr}
\begin{theorem}
For any point $\mathbf{x} \in \mathcal{Q}_{\beta}^{s, t}$, there exists a diffeomorphism $\psi: \mathcal{Q}_{\beta}^{s, t} \rightarrow  \mathbb{S}_{-\beta}^{t} \times \mathbb{R}^{s}$ that maps $\mathbf{x}$ into the product manifolds of a sphere and the Euclidean space, the mapping and its inverse are given by,
\begin{equation}\label{eq:map_to_sbr}
\psi(\mathbf{x})=\left(\begin{array}{c}
\sqrt{|\beta|} \frac{\mathbf{t}}{\|\mathbf{t}\|} \\
 \mathbf{s}
\end{array}\right) \quad \text {,} \quad \psi^{-1}(\mathbf{z})=\left(\begin{array}{c}
\frac{\sqrt{|\beta|+\|\mathbf{v}\|^{2}}}{\sqrt{|\beta|}} \mathbf{u} \\
\mathbf{v}
\end{array}\right),
\end{equation}
where $\mathbf{x}=\left(\begin{array}{c}
\mathbf{t} \\
\mathbf{s} 
\end{array}\right) \in \mathcal{Q}_{\beta}^{s, t}$ with  $\mathbf{t} \in \mathbb{R}_*^{t}$ and $\mathbf{s} \in \mathbb{R}^s$. $\mathbf{z}=\left(\begin{array}{c}
\mathbf{u} \\
\mathbf{v}
\end{array}\right) \in \mathbb{S}_{-\beta}^{t} \times \mathbb{R}^{s}$ with $\mathbf{u} \in \mathbb{S}_{-\beta}^{t}$ and $\mathbf{v} \in \mathbb{R}^{s}$.
\end{theorem}
\begin{proof}
It is easy to show that the $\psi$ and $\psi^{-1}$ are smooth functions as they only involve with a linear mapping with a constant scaling vector. Hence, we only need to show that $\psi\left(\psi^{-1}(\mathbf{z})\right)=\mathbf{z}$ and $\psi^{-1}(\psi(\mathbf{x}))=\mathbf{x}$. Here, we consider space dimensions and time dimensions separately.

For space dimensions, the mapping of the space dimensions of $\mathbf{x}$ to the space dimensions of $\mathbf{z}$ is an identity function (i.e. $\mathbf{v}=\mathbf{s}, \mathbf{s}=\mathbf{v}$). Thus, we only need to show the invertibility of the mappings taking time dimensions as inputs. 
For time dimensions, we first show that:
\begin{align}
\psi^{-1}(\psi(\mathbf{t})) &= \frac{\sqrt{|\beta|+\|\mathbf{v}\|^{2}}}{\sqrt{|\beta|}} \sqrt{|\beta|} \frac{\mathbf{t}}{\|\mathbf{t}\|} 
= \sqrt{|\beta|+\|\mathbf{v}\|^{2}}\frac{\mathbf{t}}{\|\mathbf{t}\|} \nonumber\\
& = \sqrt{|\beta|+\|\mathbf{s}\|^{2}}\frac{\mathbf{t}}{\|\mathbf{t}\|} 
= \|\mathbf{t}\|\frac{\mathbf{t}}{\|\mathbf{t}\|}=\mathbf{t}.
\end{align}
Note that the last equality can be inferred by the fact that $\mathbf{x} \in \mathcal{Q}_{\beta}^{s, t}$ and $\beta < 0$. 
We then show that:
\begin{equation}
\psi(\psi^{-1}(\mathbf{u})) = \sqrt{|\beta|} \frac{\sqrt{|\beta|+\|\mathbf{v}\|^{2}}}{\sqrt{|\beta|}} \frac{\mathbf{u}}{\| \frac{\sqrt{|\beta|+\|\mathbf{v}\|^{2}}}{\sqrt{|\beta|}} \mathbf{u}\|}  
= \sqrt{|\beta|}\frac{\mathbf{u}}{\|\mathbf{u}\|}=\mathbf{u}. \end{equation}
Note that the last equality can be inferred by the fact that $\mathbf{u} \in \mathbb{S}_{-\beta}^{t}$ and $\beta < 0$. 
\end{proof}

\subsection{Proof of Theorem 3}\label{app:invariance_space}
\begin{theorem}\label{theory:tangent_sharing_appendix}
For any reference point $\mathbf{x}=\left(\begin{array}{c}
\mathbf{t}\\
\mathbf{s}
\end{array}\right) \in \mathcal{Q}_{\beta}^{s, t}$ with space dimension $\mathbf{s}=\mathbf{0}$,
the induced tangent space of $\mathcal{Q}_{\beta}^{s, t}$ is equal to the tangent space of its diffeomorphic manifold $\mathbb{S}_{-\beta}^{t} \times \mathbb{R}^{s}$, namely, $\mathcal{T}_{\psi(\mathbf{x})}({\mathbb{S}_{-\beta}^{t} \times \mathbb{R}^{s}}) = \mathcal{T}_\mathbf{x} \mathcal{Q}_{\beta}^{s, t}$.
\end{theorem}
\begin{proof}
For any point $\mathbf{x}=\left(\begin{array}{c}
\mathbf{t} \\
\mathbf{s}
\end{array}\right) \in \mathcal{Q}_{\beta}^{s, t}$ with $\mathbf{s}=\mathbf{0}$, the corresponding point in the diffeomorphic manifold is $\psi(\mathbf{x})=\left(\begin{array}{c}
\sqrt{|\beta|} \frac{\mathbf{t}}{\|\mathbf{t}\|} \\
\mathbf{0}
\end{array}\right)$. Based on the definition of tangent space, for any tangent vector $\bm{\xi} \in \mathcal{T}_\mathbf{x}\mathcal{Q}_{\beta}^{s, t}$, $\langle \mathbf{x}, \bm{\xi} \rangle_t =0$ implies $-\mathbf{t}\bm{\xi}_t +\mathbf{0}\bm{\xi}_s=0$, which means $\mathbf{t}\bm{\xi}_t=0$. 
Thus, $\langle \psi(\mathbf{x}), \bm{\xi} \rangle_t =0$. 
Based on the definition of tangent space, $\bm{\xi} \in \mathcal{T}_{\psi(\mathbf{x})}({\mathbb{S}_{-\beta}^{t} \times \mathbb{R}^{s}})$.
\end{proof}

\subsection{Proof of Theorem 4}\label{app:theorem_nn}
\begin{theorem}
For any point $\mathbf{x} \in \mathcal{Q}_{\beta}^{s, t}$, the union of the normal neighborhood of $\mathbf{x}$ and the normal neighborhood of its antipodal point $\mathbf{-x}$ cover the entire manifold. Namely, $\mathcal{U}_{\mathbf{x}} \cup \mathcal{U}_{\mathbf{-x}} = \mathcal{Q}_{\beta}^{s, t}$.
\end{theorem}
\begin{proof}
For any point $\mathbf{x} \in \mathcal{Q}_{\beta}^{s, t}$ and $\mathbf{y} \notin \mathcal{U}_x$. Based on the definition of normal neighborhood, $\langle \mathbf{x}, \mathbf{y} \rangle_t \geq |\beta| \rightarrow \langle \mathbf{-x}, \mathbf{y} \rangle_t \leq |\beta|$. Thus, $\mathbf{y} \in \mathcal{U}_{-x}$, and $\mathcal{U}_{\mathbf{x}} \cup \mathcal{U}_{\mathbf{-x}} = \mathcal{Q}_{\beta}^{s, t}$.
\end{proof}

\section{Proof of Pattern Inference of UltraE }

UltraE can naturally infer relation patterns including symmetry, anti-symmetry, inversion and composition. As discussed above, the defined relation transformation $f_{r}=U_{\theta_r} B_{\mu_r} V_{\Phi_r}$ consists of three operations, including a circular rotation, a hyperbolic rotation, and a circular reflection.
The three operation matrices can all be identified as identity matrices.
Therefore, there are several different combinations of parameter settings to meet the above inferred requirements, demonstrating the comprehensive capability of the proposed UltraE on encoding relational patterns. 
For the sake of proof, we assume $B_{\mu_r}$ is an identity matrix $\mathbf{I}$, and $\Theta_r,\Phi_r\in[-\pi,\pi)$.
\begin{proposition}
Let $r$ be a symmetric relation such that for each triple $(e_h, r, e_t)$, its symmetric triple $(e_t, r, e_h)$ also holds. This symmetric property of $r$ can be encoded into UltraE.
\end{proposition}
\begin{proof}
If $r$ is a symmetric relation, by taking the $B_{\mathbf{b}_r}=\mathbf{I}$ and $U_{\Theta_r}=\mathbf{I}$, we have
\begin{align}
   \mathbf{e}_{h} = f_{r}\left(\mathbf{e}_{t}\right) =  V_{\Phi_r} \mathbf{e}_{t}, \ 
   \mathbf{e}_{t} = f_{r}\left(\mathbf{e}_{h}\right) =  V_{\Phi_r} \mathbf{e}_{h}
   \Rightarrow V_{\Phi_r}^2 = \mathbf{I} \nonumber
\end{align}
which holds true when $\Phi_r=0$ or $\Phi_r=-\pi$.
\end{proof}

\begin{proposition}
Let $r$ be an anti-symmetric relation such that for each triple $(e_h, r, e_t)$, its symmetric triple $(e_t, r, e_h)$ is not true. This anti-symmetric property of $r$ can be encoded into UltraE.
\end{proposition}
\begin{proof}
If $r$ is an anti-symmetric relation, by taking the $B_{\mathbf{b}_r}=\mathbf{I}$ and $U_{\Theta_r}=\mathbf{I}$, we have
\begin{align}
   \mathbf{e}_{h} = f_{r}\left(\mathbf{e}_{t}\right) =  V_{\Phi_r} \mathbf{e}_{t}, \ 
   \mathbf{e}_{t} = f_{r}\left(\mathbf{e}_{h}\right) =  V_{\Phi_r} \mathbf{e}_{h}
   \Rightarrow \mathbf{e}_{h}  = \mathbf{e}_{t} \nonumber
\end{align}
which holds true when $\Phi_r\neq0$ and $\Phi_r\neq-\pi$.
\end{proof}

\begin{proposition}
Let $r_1$ and $r_2$ be inverse relations such that for each triple $(e_h, r_1, e_t)$, its inverse triple $(e_t, r_2, e_h)$ is also true. This inverse property of $r_1$ and $r_2$ can be encoded into UltraE.
\end{proposition}
\begin{proof}
If $r_1$ and $r_2$ are inverse relations, by taking the $B_{\mathbf{b}_{r_1}}=B_{\mathbf{b}_{r_2}}=\mathbf{I}$ and $V_{\Phi_{r_1}}=V_{\Phi_{r_2}}=\mathbf{I}$, we have
\begin{align}
   \mathbf{e}_{h} = f_{r_1}\left(\mathbf{e}_{t}\right) =  U_{\Theta_{r_1}} \mathbf{e}_{t}, \ 
   \mathbf{e}_{t} = f_{r_2}\left(\mathbf{e}_{h}\right) =  U_{\Theta_{r_2}} \mathbf{e}_{h}
   \Rightarrow U_{\Theta_{r_1}}U_{\Theta_{r_2}} = \mathbf{I} \nonumber
\end{align}
which holds true when $\Theta_{r_1}+\Theta_{r_2}=0$.
\end{proof}

\begin{proposition}
Let relation $r_1$ be composed of $r_2$ and $r_3$ such that triple $(e_h, r_1, e_t)$ exists when $(e_h, r_2, e_t)$ and $(e_h, r_3, e_t)$ exist. This composition property can be encoded into UltraE.
\end{proposition}
\begin{proof}
If $r_1$ is composed of $r_2$ and $r_3$, by taking the $B_{\mathbf{b}_{r_1}}=B_{\mathbf{b}_{r_2}}=\mathbf{I}$ and $V_{\Phi_{r_1}}=V_{\Phi_{r_2}}=\mathbf{I}$, we have
\begin{align}
   &\mathbf{e}_{h} = f_{r_1}\left(\mathbf{e}_{t}\right) =  U_{\Theta_{r_1}} \mathbf{e}_{t}, \ 
   \mathbf{e}_{h} = f_{r_2}\left(\mathbf{e}_{t}\right) =  U_{\Theta_{r_2}} \mathbf{e}_{t}, \\
   &\mathbf{e}_{h} = f_{r_3}\left(\mathbf{e}_{t}\right) =  U_{\Theta_{r_3}} \mathbf{e}_{t} \
   \Rightarrow U_{\Theta_{r_1}} = U_{\Theta_{r_2}}U_{\Theta_{r_3}} \nonumber
\end{align}
which holds true when $\Theta_{r_1}=\Theta_{r_2}+\Theta_{r_3}$ or $\Theta_{r_1}=\Theta_{r_2}+\Theta_{r_3}+2\pi$ or $\Theta_{r_1}=\Theta_{r_2}+\Theta_{r_3}-2\pi$.
\end{proof}

\textbf{Connections with Hyperbolic Methods.}
UltraE has close connections with some existing hyperbolic KG embedding methods, including HyboNet \cite{DBLP:journals/corr/abs-2105-14686}, RotH/RefH \cite{chami2020low}, and MuRP \cite{DBLP:conf/nips/BalazevicAH19}. To show this, we first introduce Lorentz transformation. 
\begin{definition}\label{def:lorent}
Lorentz transformation is a pseudo-orthogonal transformation with signature $(p,1)$.
\end{definition}

\noindent
\textbf{HyboNet} \cite{DBLP:journals/corr/abs-2105-14686} embeds entities as points in a Lorentz space and models relations as Lorentz transformations. According to Definition \ref{def:lorent}, we have the following proposition.

\begin{proposition}
UltraE, if parameterized by a full $J$-orthogonal matrix, generalizes HyboNet to support arbitrary signatures. 
\end{proposition}
That is, HyboNet is the case of UltraE (with full $J$-orthogonal matrix parameterization) where $q=1$.

By exploiting the polar decomposition \cite{ratcliffe1994foundations}, a Lorentz transformation matrix $\mathbf{T}$ can be decomposed into $\mathbf{T}=\mathbf{R}_{\mathbf{U}} \mathbf{R}_{\mathbf{b}}$, where
\begin{equation}
    \mathbf{R_{U}}=\left[\begin{array}{cc}
\mathbf{U} & \mathbf{0} \\
\mathbf{0} & 1
\end{array}\right], \mathbf{R_{b}}=\left[\begin{array}{cc}
\left(\mathbf{I}+\mathbf{b} \mathbf{b}^{\top}\right)^{\frac{1}{2}} & \mathbf{b}^{\top} \\
\mathbf{b} & \sqrt{1+\|\mathbf{b}\|_{2}^{2}}
\end{array}\right],
\end{equation}
where $\mathbf{R_{U}}$ is an orthogonal matrix. 
In Lorentz geometry, $\mathbf{R_{U}}$ and $\mathbf{R_{b}}$ are called Lorentz rotation and Lorentz boost, respectively. $\mathbf{R_{U}}$ represents rotation or reflection in space dimension (without changing the time dimension), while $\mathbf{R_{b}}$ denotes a hyperbolic rotation across the time dimension and each space dimension. 
\cite{DBLP:journals/spl/TabaghiD21} established an equivalence between Lorentz boost and Möbius addition (or hyperbolic translation). Hence, HyboNet inherently models each relation as a combination of a rotation/reflection and a hyperbolic translation. 

\noindent
\textbf{RotH/RefH} \cite{chami2020low}, interestingly, also models each relation as a combination of a rotation/reflection and a hyperbolic translation that is implemented by Möbius addition. 
Hence, HyboNet subsumes RotH/RefH,\footnote{Note that RotH/RefH consider Poincaré Ball while HyboNet considers Lorentz model. The subsumption still holds since Poincaré Ball is isometric to the Lorentz model.} where the equivalence cannot hold because the rotation/reflection of RotH/RefH is parameterized by the Givens rotation/reflection \cite{chami2020low}. 

\noindent
\textbf{MuRP} \cite{DBLP:conf/nips/BalazevicAH19} models relations as a combination of Möbius multiplication (with diagonal matrix) and Möbius addition. Note that \cite{DBLP:journals/corr/abs-2105-14686} established a fact that a Lorentz rotation is equivalent to Möbius multiplication, and \cite{DBLP:journals/spl/TabaghiD21} proved that Lorentz boost is equivalent to Möbius addition. Hence, HyboNet subsumes MuRP, where the equivalence cannot hold because the Möbius multiplication in MuRP is parameterized by a diagonal matrix.

To sum up, UltraE generalizes HyboNet to allow for arbitrary signature $(p,q)$, while HyboNet subsumes RotH/RefH and MuRP. 




\section{Proofs of Theorems of BoxEL}

\begin{proposition}\label{prop:1}
We have 
\begin{enumerate}
    \item If $\mathcal{L}_{C(a)}(w) = 0$, then $\interpw \models C(a)$,
    \item If $\mathcal{L}_{r(a,b)}(w)=0$, then $\interpw \models r(a,b)$.
\end{enumerate}
\end{proposition}
Proposition \ref{prop:1} follows directly from the definitions.

\begin{lemma}
\label{contains_lemma}
\begin{enumerate}
    \item $0 \leq \contains(B_1, B_2) \leq 1$,
    \item $\contains(B_1, B_2) = 0$ implies $B_1 \subseteq B_2$,
    \item $\contains(B_1, B_2) = 1$ implies $B_1 \cap B_2 = \emptyset$.
\end{enumerate}
\end{lemma}
\begin{proof}
1. Since $B_1 \cap B_2 \subseteq B_1$, $\frac{\Vol(B_1 \cap B_2)}{\Vol(B_1)}\leq 1$.
The modified volume is also non-negative, so that the fraction is non-negative and
$0 \leq \contains(B_1, B_2) \leq 1$.

2. If $\contains(B_1, B_2) = 0$ then we must have $\Vol(B_1 \cap B_2)=1$ and
therefore $\Vol(B_1 \cap B_2) = \Vol(B_1) \neq 0$.
Since $B_1 \cap B_2 \subseteq B_1$, this is only possible if $B_1 \cap B_2 = B_1$,
but this implies that $B_1 \subseteq B_2$.

3. If $\contains(B_1, B_2) = 1$, we must have $\Vol(B_1 \cap B_2)=0$. By definition
of the modified volume this is only possible if $B_1 \cap B_2 = \emptyset$.
\end{proof}

\begin{proposition}
If $\mathcal{L}_{C \sqsubseteq D}(w)=0$, 
then $\interpw \models C \sqsubseteq D$,
where we exclude the inconsistent case $C=\{a\}, D=\bot$.
\end{proposition}
\begin{proof}
For $D\neq \bot$, the claim follows from Lemma \ref{contains_lemma}. 
For $D=\bot$, $\mathcal{L}_{C \sqsubseteq \bot}(w)=0$ implies that 
$\Box_w(C) = \emptyset$ and the claim is trivially true.
\end{proof}

\begin{proposition}
If $\mathcal{L}_{C \sqcap D \sqsubseteq E}(w)=0$,
then $\interpw \models C \sqcap D \sqsubseteq E$,
where we exclude the inconsistent case ${a} \sqcap {a} \sqsubseteq \bot$ (that is, $C=D=\{a\}, E=\bot$).
\end{proposition}
\begin{proof}
We have to show that for every $x \in \Box_w(C)$, there is 
a $y \in \Box_w(D)$ such that $T^r_w(x) = y$. Note that
$\contains(T^r_w(\Box_w(C)), \Box_w(D))=0$
implies that $T^r_w(\Box_w(C)) \subseteq \Box_w(D)$
according to Lemma \ref{contains_lemma}. 
Since $x \in \Box_w(C)$,
we have $T^r_w(x) = y \in \Box_w(D)$.
\end{proof}

\begin{proposition}
If $\mathcal{L}_{C \sqsubseteq \exists r.D}(w) = 0$, 
then $\interpw \models C \sqsubseteq \exists r.D$.
\end{proposition}
The proof of Proposition 4 is analogous to the proof of Proposition 3.

\begin{proposition}
If $\mathcal{L}_{\exists r.C \sqsubseteq D }(w) = 0$, 
then $\interpw \models \exists r.C \sqsubseteq D$.
\end{proposition}
\begin{proof}
We have to show that if $T^{-r}_w(x)=y$ and
$y \in \Box_w(C)$, then $x \in \Box_w(D)$. $\contains(T^{-r}_w(\Box_w(C)), \Box_w(D))=0$
implies that $T^{-r}_w(\Box_w(C)) \subseteq \Box_w(D)$
according to Lemma \ref{contains_lemma}.
Since $y \in \Box_w(C)$, we have 
$x = T^{-r}_w(y) \in \Box_w(D)$.
\end{proof}




\section{Proof of theorems of HMI }

\begin{proposition}\label{prop:p2ball}
Given a Poincaré hyperplane $H_\mathbf{c}$ where $\mathbf{c} \neq \mathbf{0}$, there exists an $n$-ball $\mathbb{B}_\mathbf{c}\left(\mathbf{o}_\mathbf{c}, r_\mathbf{c}\right)$ such that $H_\mathbf{c} \subset \mathbb{B}_\mathbf{c}\left(\mathbf{o}_\mathbf{c}, r_\mathbf{c}\right)$, i.e., $H_\mathbf{c}$ is a subset of $\mathbb{B}_\mathbf{c}\left(\mathbf{o}_\mathbf{c}, r_\mathbf{c}\right)$. $\mathbb{B}_\mathbf{c} $ is uniquely given by
\begin{equation}
    \mathbb{B}_\mathbf{c}^n = \mathbb{B}^n\left(  \frac{\left(1+\|\mathbf{c}\|^{2}\right)}{2\|\mathbf{c}\|}\mathbf{c},  \frac{1-\|\mathbf{c}\|^{2}}{2\|\mathbf{c}\|}\right)
\end{equation}
\end{proposition}
\begin{proof} 
Since $c$ is the center point of the Poincaré hyperplane, the vector $\overrightarrow{c}$ must be a normal vector of the tangent space $T_{c} \mathbb{B}^{n}$ of $\mathbb{B}^{n}$ at $c$. Let $q$ be one of the point that the Poincaré hyperplane and the Poincaré ball intersect at. Then, the radius of $\mathbb{B}_\mathbf{c}\left(\mathbf{o}_\mathbf{c}, r_\mathbf{c}\right)$, the radius of $\mathbb{D}^n$, and the distance from the centers of $\mathbb{D}^n$ to the center of $\mathbb{B}_\mathbf{c}\left(\mathbf{o}_\mathbf{c}, r_\mathbf{c}\right)$ must satisfy the Pythagorean theorem, i.e., the three Euclidean distances $d(\mathbf{0},q)$, $d(q,\mathbf{o}_\mathbf{c})$ and $d(\mathbf{o}_\mathbf{c},\mathbf{0})$ must satisfy
\begin{equation}
d(\mathbf{0},\mathbf{q})^2+d(\mathbf{q},\mathbf{o}_\mathbf{c})^2=d(\mathbf{o}_\mathbf{c},\mathbf{0})^2 =\left(d\left(\mathbf{0},\mathbf{c}\right)+d\left(\mathbf{c},\mathbf{o}_\mathbf{c}\right)\right)^2.
\end{equation} 
Since we have $d\left(\mathbf{c},\mathbf{o}_\mathbf{c}\right)=d(\mathbf{q},\mathbf{o}_\mathbf{c})=r_\mathbf{c}$, by solving this quadratic equation, we have $r_\mathbf{c} =  \frac{1-\|\mathbf{c}\|^{2}}{2\|\mathbf{c}\|}$. Since $\mathbf{o}_\mathbf{c} = \mathbf{c} (1+ \frac{r_\mathbf{c}}{d(\mathbf{0},\mathbf{c})})$,
we have $\mathbf{o}_\mathbf{c} = \mathbf{c}\frac{\left(1+\|\mathbf{c}\|^{2}\right)}{2\|\mathbf{c}\|}$. 
Thus, $\mathbb{B}_\mathbf{c} = \mathbb{B}\left( \mathbf{o}_\mathbf{c} =  \mathbf{c}\frac{\left(1+\|\mathbf{c}\|^{2}\right)}{2\|\mathbf{c}\|}, r_\mathbf{c} = \frac{1-\|\mathbf{c}\|^{2}}{2\|\mathbf{c}\|}\right)$.
\end{proof}

\begin{proposition}[HEX-property]
The classification function $f$ has the HEX property with respect to $G$ if and only if for any constraint in $G$, the corresponding loss term is $0$. 
\end{proposition}
\begin{proof}
Note that the loss term of the constraint being $0$ implies that the corresponding constraint is respected. Our loss terms clearly connect the HEX property. That is, for any point $\mathbf{p} \in D^n$ and a pair of enclosing $n$-balls $(\mathbb{B}_w, \mathbb{B}_u)$,
$\mathcal{L}_{\text {membership}}\left(p, \mathbb{B}_w \right) \geq \mathcal{L}_{\text {membership}}\left(p, \mathbb{B}_u \right)$ for all $(\mathbb{B}_w, \mathbb{B}_u)$ where $\mathcal{L}_{\text {inside}}(\mathbb{B}_w, \mathbb{B}_u) =0 $ and $\neg \mathcal{L}_{\text {membership}}\left(p, \mathbb{B}_w \right) \vee \neg \mathcal{L}_{\text {membership}}\left(p, \mathbb{B}_u \right)$ for all $(\mathbb{B}_w, \mathbb{B}_u)$ where $\mathcal{L}_{\text{disjoint}}(\mathbb{B}_u, \mathbb{B}_w) = 0$. According to the definition of HEX-property, $f$ has the HEX property with respect to $G$ if and only if the corresponding loss term of the corresponding constraint is $0$.
\end{proof}

\begin{corollary}
Given a HEX graph $G$ of labels and if the loss of the embeddings is $0$, then the learned prediction function is logically consistent with respect to $G$.  
\end{corollary}
\begin{proof}
Note that the loss terms $\mathcal{L}_{\text {inside}}, \mathcal{L}_{\text {disjoint}},\mathcal{L}_{\text {membership}}, \mathcal{L}_{\text {non-membership}}$ in Eq.7 are all non-negative. Hence, the loss being $0$ implies that all losses are zeros (all constraints are satisfied). According to the definition of consistency, the prediction function is consistent.
\end{proof}

\section{Proof of propositions of ShrinkE}
\label{app:proof}

\begin{proposition}
Given any two facts $\mathcal{F}_1=\left(\mathcal{T},\mathcal{Q}_1\right)$ and $\mathcal{F}_2=\left(\mathcal{T},\mathcal{Q}_2\right)$ where $\mathcal{Q}_2 \subseteq \mathcal{Q}_1$, i.e., $\mathcal{F}_2$ is a partial fact of $\mathcal{F}_1$, the output of the scoring function $f(\cdot)$ of ShrinkE satisfy the constraint $f(\mathcal{F}_2) \geq f(\mathcal{F}_1)$, which implies Eq.(\ref{eq:momonotonicity}). 
\end{proposition}

\begin{proof}
We first prove that the resulting box of $\mathcal{F}_2$ subsumes the resulting box of $\mathcal{F}_2$.  Since the primal triple of $\mathcal{F}_1$ and $\mathcal{F}_2$ are the same (let assume it is $\mathcal{T}=(h,r,t)$ ), the spanned boxes of the two facts are $\mathcal{H}_r(\mathbf{e}_h)$. Since $\mathcal{Q}_2 \subseteq \mathcal{Q}_1$, the final shrunken box of $\mathcal{F}_1$ must be a subset of the shrunken box of $\mathcal{F}_2$.  Hence, we have,
\begin{equation}
    \Box_{\mathcal{F}_2} \subseteq \Box_{\mathcal{F}_1}. 
\end{equation}
Given the tail entity $t$ whose embedding is denoted by $\mathbf{e}_t$,
we consider three cases of its position. 

1) If $\mathbf{e}_t$ is inside the small box $\Box_{\mathcal{F}_2}$, then $\mathbf{e}_t$ must also be inside $\Box_{\mathcal{F}_1}$ since $\Box_{\mathcal{F}_2} \subseteq \Box_{\mathcal{F}_1}$. Note that our point-to-box function is monotonically increasing w.r.t. the increase of distance from the tail point to the center of box. Hence, we will have $D(\mathbf{e}, \Box_{\mathcal{F}_2}) \geq D(\mathbf{e}, \Box_{\mathcal{F}_1})$, implying $f(\mathcal{F}_2) \geq f(\mathcal{F}_1)$. 

2) If $\mathbf{e}_t$ is outside the small box $\Box_{\mathcal{F}_2}$ but inside in the larger $\Box_{\mathcal{F}_1}$, according to the definition of the point-to-box distance function, we immediately have $D(\mathbf{e}, \Box_{\mathcal{F}_2}) \geq D(\mathbf{e}, \Box_{\mathcal{F}_1})$, implying $f(\mathcal{F}_2) \geq f(\mathcal{F}_1)$. 

3) If $\mathbf{e}_t$ is outside the larger box $\Box_{\mathcal{F}_1}$,, then $\mathbf{e}_t$ must also be outside $\Box_{\mathcal{F}_2}$ since $\Box_{\mathcal{F}_2} \subseteq \Box_{\mathcal{F}_1}$. Note that our point-to-box function is monotonically decreasing w.r.t. the increase of volume of box. Hence, we will have $D(\mathbf{e}, \Box_{\mathcal{F}_2}) \geq D(\mathbf{e}, \Box_{\mathcal{F}_1})$, implying $f(\mathcal{F}_2) \geq f(\mathcal{F}_1)$. 
\end{proof}

\begin{proposition}
    ShrinkE is able to infer hyper-relational symmetry, anti-symmetry, inversion, composition, hierarchy, intersection, and exclusion.  
\end{proposition}

We first prove that ShrinkE is able to infer symmetry, anti-symmetry, inversion, and composition. For the sake of proof, we assume $\theta_r \in [-\pi,\pi)$. We prove them by proving Lemma B.1-4 one by one.  

\begin{lemma}[Symmetry]
Let $r$ be a symmetric relation such that for each triple $(e_h, r, e_t)$, its symmetric triple $(e_t, r, e_h)$ also holds. This symmetric property of $r$ can be modeled by ShrinkE.
\end{lemma}
\begin{proof}
If $r$ is a symmetric relation, by taking the $\mathbf{\boldsymbol\delta}_r=\mathbf{0}$, $\mathbf{b}_r=\mathbf{0}$, and $\mathbf{\Theta}_r=\operatorname{diag}\left(\mathbf{G}\left(\mathbf{\theta}_{r, 1}\right), \ldots, \mathbf{G}\left(\mathbf{\theta}_{r, \frac{d}{2}}\right)\right)$, where $\mathbf{G}(\theta)$ is a $2\times2$ diagonal matrix, we have
\begin{align}
\begin{split}
    & \mathbf{e}_{h} = f_{r}\left(\mathbf{e}_{t}\right) =  \mathbf{\Theta}_r \mathbf{e}_{t}, \ 
   \mathbf{e}_{t} = f_{r}\left(\mathbf{e}_{h}\right) =  \mathbf{\Theta}_r \mathbf{e}_{h} \\
   & \Rightarrow \mathbf{\Theta}_r^2 = \mathbf{I} \nonumber
\end{split}
\end{align}
which holds true when $\mathbf{\theta}_{r,i}=\mathbf{0}$ or $\mathbf{\theta}_{r,i}=-\mathbf{\pi}$ for $i =  1,\cdots,\frac{d}{2}$.
\end{proof}

\begin{lemma}[Anti-symmetry]
Let $r$ be an anti-symmetric relation such that for each triple $(e_h, r, e_t)$, its symmetric triple $(e_t, r, e_h)$ is not true. This anti-symmetric property of $r$ can be modeled by ShrinkE.
\end{lemma}

\begin{proof}
If $r$ is a anti-symmetric relation, by taking the $\mathbf{\boldsymbol\delta}_r=\mathbf{0}$, $\mathbf{b}_r=\mathbf{0}$, and $\mathbf{\Theta}_r=\operatorname{diag}\left(\mathbf{G}\left(\mathbf{\theta}_{r, 1}\right), \ldots, \mathbf{G}\left(\mathbf{\theta}_{r, \frac{d}{2}}\right)\right)$, where $\mathbf{G}(\theta)$ is a $2\times2$ diagonal matrix, we have
\begin{align}
\begin{split}
    & \mathbf{e}_{h} \neq f_{r}\left(\mathbf{e}_{t}\right) =  \mathbf{\Theta}_r \mathbf{e}_{t}, \ 
   \mathbf{e}_{t} = f_{r}\left(\mathbf{e}_{h}\right) =  \mathbf{\Theta}_r \mathbf{e}_{h} \\
   & \Rightarrow \mathbf{\Theta}_r^2 \neq \mathbf{I} \nonumber
\end{split}
\end{align}
which holds true when $\mathbf{\theta}_{r,i} \neq \mathbf{0}$ or $\mathbf{\theta}_{r,i} \neq -\mathbf{\pi}$ for $i =  1,\cdots,\frac{d}{2}$.
\end{proof}

\begin{lemma}[Inversion]
Let $r_1$ and $r_2$ be inverse relations such that for each triple $(e_h, r_1, e_t)$, its inverse triple $(e_t, r_2, e_h)$ is also true. This inverse property of $r_1$ and $r_2$ can be modeled by ShrinkE.
\end{lemma}
\begin{proof}
If $r_1$ and $r_2$ are inverse relations, by taking the $\mathbf{\boldsymbol\delta}_r=\mathbf{0}$, $\mathbf{b}_r=\mathbf{0}$, and $\mathbf{\Theta}_r=\operatorname{diag}\left(\mathbf{G}\left(\mathbf{\theta}_{r, 1}\right), \ldots, \mathbf{G}\left(\mathbf{\theta}_{r, \frac{d}{2}}\right)\right)$, where $\mathbf{G}(\theta)$ is a $2\times2$ diagonal matrix, we have
\begin{equation}
\begin{split}
    &\mathbf{e}_{t} = f_{r_1}\left(\mathbf{e}_{h}\right) =  \Theta_{r_1} \mathbf{e}_{h}, \ 
   \mathbf{e}_{h} = f_{r_2}\left(\mathbf{e}_{t}\right) =  \Theta_{r_2} \mathbf{e}_{h} \\
   & \Rightarrow \Theta_{r_1} \Theta_{r_2} = \mathbf{I} \nonumber
\end{split}
\end{equation}
which holds true when for $\theta_{r_1,i}{r_1}+\theta_{r_2,i}=0$ for $i =  1,\cdots,\frac{d}{2}$.
\end{proof}

\begin{lemma}[Composition]
Let relation $r_1$ be composed of $r_2$ and $r_3$ such that triple $(e_1, r_1, e_3)$ exists when $(e_1, r_2, e_2)$ and $(e_2, r_3, e_3)$ exist. This composition property can be modeled by ShrinkE.
\end{lemma}
\begin{proof}
If $r_1$ is composed of $r_2$ and $r_3$, by taking the $\mathbf{\boldsymbol\delta}_r=\mathbf{0}$, $\mathbf{b}_r=\mathbf{0}$, and $\mathbf{\Theta}_r=\operatorname{diag}\left(\mathbf{G}\left(\mathbf{\theta}_{r, 1}\right), \ldots, \mathbf{G}\left(\mathbf{\theta}_{r, \frac{d}{2}}\right)\right)$, where $\mathbf{G}(\theta)$ is a $2\times2$ diagonal matrix, we have
\begin{equation}
\begin{split}
    &\mathbf{e}_{3} = f_{r_1}\left(\mathbf{e}_{1}\right) =  \Theta_{r_1} \mathbf{e}_{1}, \ 
   \mathbf{e}_{2} = f_{r_2}\left(\mathbf{e}_{1}\right) =  \Theta_{r_2} \mathbf{e}_{1}, \\
   &\mathbf{e}_{3} = f_{r_3}\left(\mathbf{e}_{2}\right) =  \Theta_{r_3} \mathbf{e}_{2} \
   \Rightarrow \Theta_{r_1} = \Theta_{r_2}\Theta_{r_3} \nonumber
\end{split}
\end{equation}
which holds true when $\theta_{r_1,i}=\theta_{r_2,i}+\theta_{r_3,i}$ or $\theta_{r_1,i}=\theta_{r_2,i}+\theta_{r_3,i}+2\pi$ or $\theta_{r_1,i}=\theta_{r_2,i}+\theta_{r_3,i}-2\pi$ for $i =  1,\cdots,\frac{d}{2}$.
\end{proof}

We now prove that ShrinkE is able to infer relation implication, exclusion and intersection.

\begin{lemma}[Relation implication]
Let $r_1 \rightarrow r_2$ form a hierarchy such that for each triple $(e_h, r_1, e_t)$, $(e_h, r_2, e_t)$ also holds. This hierarchy property $r_1 \rightarrow r_2$ can be modeled by ShrinkE.
\end{lemma}
\begin{proof}
If $r_1 \rightarrow r_2$, by taking $\mathcal{T}_{r_1} =\mathcal{T}_{r_2}$, i.e.,  $\boldsymbol\delta_{r_1}=\boldsymbol\delta_{r_2}$ and $\Theta_{r_1}=\Theta_{r_2}$, we have,
$(e_h, r_1, e_t) \rightarrow (e_h, r_2, e_t)$ implies that the spanning box of query $(e_h, r_1, x?)$ is subsumed by the spanning box of query $(e_h, r_2, x?)$. i.e., 
$\Box( \mathcal{H}_{r_1}(e_h)-\sigma(\boldsymbol\delta_{r_1}), \mathcal{H}_{r_1}(e_h)+\sigma(\boldsymbol\delta_{r_1}) ) 
\subseteq \Box( \mathcal{H}_{r_1}(e_h)-\sigma(\boldsymbol\delta_{r_2}), \mathcal{H}_{r_1}(e_h)+\sigma(\boldsymbol\delta_{r_2}))$,
which holds true when $\boldsymbol\delta_{r_1} \leq  \boldsymbol\delta_{r_2}$.
\label{lem:hierarchy}
\end{proof}

\begin{lemma}[Relation exclusion]
Let $r_1, r_2$ be mutually exclusive, that is, $(e_h, r_1, e_t)$, $(e_h, r_2, e_t)$ can not be simultaneously hold. This mutual exclusion property $r_1 \wedge r_2 \rightarrow \bot$ can be modeled by ShrinkE.
\end{lemma}
\begin{proof}
If $r_1 \wedge r_2 \rightarrow \bot$, we have
$(e_h, r_1, e_t) \wedge (e_h, r_2, e_t) \rightarrow \bot$, which implies that the spanning box of query $(e_h, r_1, x?)$ and the spanning box of query $(e_h, r_2, x?)$ are mutually exclusive, i.e., 
$\Box( \mathcal{H}_{r_1}(e_h)-\sigma(\boldsymbol\delta_{r_1}), \mathcal{H}_{r_1}(e_h)+\sigma(\boldsymbol\delta_{r_1}) ) 
\cap \Box( \mathcal{H}_{r_1}(e_h)-\sigma(\boldsymbol\delta_{r_2}), \mathcal{H}_{r_1}(e_h)+\sigma(\boldsymbol\delta_{r_2})) \rightarrow \bot$
\end{proof}

\begin{lemma}[Relation intersection]
Let $r_3$ be a intersection of $r_1, r_2$, that is, if $(e_h, r_1, e_t)$ and $(e_h, r_2, e_t)$ hold, then  $(e_h, r_3, e_t)$ also holds. This intersection property $r_1 \wedge r_2 \rightarrow r_3$ can be modeled by ShrinkE.
\end{lemma}
\begin{proof}
Note that box is closed under intersection and this property can be view as a combination of two pairs of relation implication. Hence, the proof is similar to the proof of Lemma \ref{lem:hierarchy}. 
\end{proof}

\begin{proposition}
ShrinkE is able to infer qualifier implication, mutual exclusion, and intersection.
\end{proposition}
\begin{proof}
Since each qualifier is associated with a box, the implication and mutual exclusion relationships between qualifiers can be modeled by their geometric relationships, i.e., box entailment and box disjointedness, respectively, between their corresponding boxes.  
Qualifier intersection can be modeled by enforcing the box of one qualifier to be inside the intersection of the boxes of another two qualifiers.  
\end{proof}

\section{Theoretical Justifications of NestE}

We extend the vanilla logical patterns in KGs to include nested facts. 
This can be expressed in a non-first-order-logic-like form $\psi \stackrel{\widehat{r}}{\rightarrow}\phi$.

\begin{itemize}
    \item \textbf{Relational symmetry (R-symmetry):} an atomic relation $r$ is symmetric w.r.t a nested relation $\widehat{r}$ if $\forall x,y \in \mathcal{E},  \langle x, r, y\rangle  \stackrel{\widehat{r}}{\leftrightarrow} \langle y, r, x\rangle$. 

    \item \textbf{Relational inverse (R-inverse):} two atomic relations $r_1$ and $r_2$ are inverse w.r.t a nested relation $\widehat{r}$ if $\forall x,y \in \mathcal{E}, $ $(\langle x, r_1, y\rangle \stackrel{\widehat{r}}{\leftrightarrow} \langle y, r_2, x\rangle)$.

    \item \textbf{Relational implication (R-implication):} an atomic relation $r_1$ implies a atomic relation $r_2$ w.r.t a nested relation $\widehat{r}$ if $\forall x,y \in \mathcal{E}, (\langle x, r_1, y\rangle \stackrel{\widehat{r}}{\rightarrow} \langle x, r_2, y\rangle)$.

    \item \textbf{Relational inverse implication (R-Inv-implication):} an atomic relation $r_1$ inversely implies an atomic relation $r_2$ w.r.t a nested relation $\widehat{r}$ if $\forall x,y \in \mathcal{E}, (\langle x, r_1, y\rangle \stackrel{\widehat{r}}{\rightarrow} \langle y, r_2, x\rangle)$.

    \item \textbf{Entity implication (E-implication):} an entity $x_1$ (resp. $y_1$) implies entity $x_2$ (resp. $y_2$) w.r.t an atomic relation $r$ and a nested relation $\widehat{r}$ if  $\forall y \in \mathcal{E}, (\langle x_1, r, y\rangle \stackrel{\widehat{r}}{\rightarrow} \langle x_2, r, y\rangle)$ (resp. $\forall x \in \mathcal{E}, (\langle x, r, y_1\rangle \stackrel{\widehat{r}}{\rightarrow} \langle x, r, y_2\rangle)$ ).

    \item \textbf{Entity relational implication (E-R-implication):} an entity $x_1$ and relation $r_1$ (resp. $y_1$ and relation $r_1$) implies entity $x_2$ and relation $r_2$ (resp. $y_2$ and relation $r_2$) 
    w.r.t a nested relation $\widehat{r}$ if  $\forall y \in \mathcal{E}, (\langle x_1, r_1, y\rangle \stackrel{\widehat{r}}{\rightarrow} \langle x_2, r_2, y\rangle)$ (resp. $\forall x \in \mathcal{E}, (\langle x, r_1, y_1\rangle \stackrel{\widehat{r}}{\rightarrow} \langle x, r_2, y_2\rangle)$).

    \item \textbf{Entity relational inverse implication (E-R-Inv-implication):}  an entity $x_1$ and relation $r_1$ (resp. $y_1$ and relation $r_1$) inversely implies entity $x_2$ and relation $r_2$ (resp. $y_2$ and relation $r_2$) w.r.t a nested relation $\widehat{r}$ if  $\forall y \in \mathcal{E}, (\langle x_1, r_1, y\rangle \stackrel{\widehat{r}}{\rightarrow} \langle y, r_2, x_2\rangle)$ (resp. $\forall x \in \mathcal{E}, (\langle x, r_1, y_1\rangle \stackrel{\widehat{r}}{\rightarrow} \langle y_2, r_2, x\rangle)$).
     \item \textbf{Dual Entity implication (Dual E-implication):} an entity pair ($x_1$, $x_2$) implies another entity pair ($y_1$, $y_2$)  iff both ($x_1$, $y_1$) and ($x_2$, $y_2$) satisfy E-implication. 

\end{itemize}

\begin{proposition}
\label{prop:facte}
    NestE can infer R-symmetry, R-inverse, R-implication, R-Inv-implication, E-implication, E-R-implication, E-R-Inv-implication, and Dual E-implication.
\end{proposition}

To infer different logical patterns via different free variables, we can set some elements of the relation matrix to be zero-valued or one-valued complex numbers. 
We prove proposition \ref{prop:facte} by showing the following special solutions for each case of pattern. 

\begin{proposition}
    Relational symmetry can be inferred by setting
    \begin{equation}
        \mathbf{R} = 
         \left[ 
              {\begin{array}{ccc}
               \mathbf{0} & \mathbf{0} & \mathbf{1}\\
               \mathbf{0} & \mathbf{R}_{22} & \mathbf{0}\\
               \mathbf{1} & \mathbf{0} & \mathbf{0}\\
              \end{array} } 
        \right]
    \end{equation}  
    where $\mathbf{r} \otimes \mathbf{R}_{22} = \mathbf{r}$ implying $\mathbf{R}_{22}=\mathbf{1}$. 
\end{proposition}

\begin{proposition}
    Relational inversion can be inferred by setting
    \begin{equation}
        \mathbf{R} = 
         \left[ 
              {\begin{array}{ccc}
               \mathbf{0} & \mathbf{0} & \mathbf{1}\\
               \mathbf{0} & \mathbf{R}_{22} & \mathbf{0}\\
               \mathbf{1} & \mathbf{0} & \mathbf{0}\\
              \end{array} } 
        \right]
    \end{equation}  
    where $\mathbf{r}_1 \otimes \mathbf{R}_{22} = \mathbf{r}_2$ and $\mathbf{r}_2 \otimes \mathbf{R}_{22} = \mathbf{r}_1$ implying $\mathbf{R}_{22} = \mp 1$.

\end{proposition}

\begin{proposition}
    Relational implication can be inferred by setting
    \begin{equation}
        \mathbf{R} = 
         \left[ 
              {\begin{array}{ccc}
               \mathbf{1} & \mathbf{0} & \mathbf{0}\\
               \mathbf{0} & \mathbf{R}_{22} & \mathbf{0}\\
               \mathbf{0} & \mathbf{0} & \mathbf{1}\\
              \end{array} } 
        \right]
    \end{equation}  
    where $\mathbf{r}_1 \otimes \mathbf{R}_{22} = \mathbf{r}_2$.
    
\end{proposition}

\begin{proposition}
    Relational inverse implication can be inferred by setting
    \begin{equation}
        \mathbf{R} = 
         \left[ 
              {\begin{array}{ccc}
               \mathbf{0} & \mathbf{0} & \mathbf{1}\\
               \mathbf{0} & \mathbf{R}_{22} & \mathbf{0}\\
               \mathbf{1} & \mathbf{0} & \mathbf{0}\\
              \end{array} } 
        \right]
    \end{equation}  
    where $\mathbf{r}_1 \otimes \mathbf{R}_{22} = \mathbf{r}_2$.
\end{proposition}

\begin{proposition}
    Entity implication can be inferred by setting
    \begin{equation}
        \mathbf{R} = 
         \left[ 
              {\begin{array}{ccc}
               \mathbf{R}_{11} & \mathbf{R}_{12} & \mathbf{0}\\
               \mathbf{R}_{21} & \mathbf{R}_{21} & \mathbf{0}\\
               \mathbf{0} & \mathbf{0} & \mathbf{1}\\
              \end{array} } 
        \right]
    \end{equation}  
    where   
    \begin{align}
    \mathbf{x}_1 \otimes \mathbf{R}_{11} + \mathbf{r} \otimes  \mathbf{R}_{21} = \mathbf{x}_2 \\
    \mathbf{x}_1 \otimes \mathbf{R}_{12} + \mathbf{r} \otimes  \mathbf{R}_{22} = \mathbf{r}
      \end{align}
    or 
    
    \begin{equation}
        \mathbf{R} = 
         \left[ 
              {\begin{array}{ccc}
               \mathbf{1} & \mathbf{0} & \mathbf{0}\\
               \mathbf{0} & \mathbf{R}_{22} & \mathbf{R}_{23}\\
               \mathbf{0} & \mathbf{R}_{32} & \mathbf{R}_{33}\\
              \end{array} } 
        \right]
    \end{equation}  
    where 
    \begin{align}
        \mathbf{r} \otimes  \mathbf{R}_{32} + \mathbf{y}_1 \otimes  \mathbf{R}_{33} = \mathbf{y}_2\\
        \mathbf{r} \otimes  \mathbf{R}_{22} + \mathbf{y}_1 \otimes \mathbf{R}_{32}  = \mathbf{r}
    \end{align}
    
\end{proposition}

\begin{proposition}
    Entity inverse implication can be inferred by setting
    \begin{equation}
        \mathbf{R} = 
         \left[ 
              {\begin{array}{ccc}
               \mathbf{0} & \mathbf{0} & \mathbf{R}_{13}\\
               \mathbf{0} & \mathbf{1} & \mathbf{0}\\
               \mathbf{R}_{31} & \mathbf{0} & \mathbf{1}\\
              \end{array} } 
        \right]
    \end{equation}  
    where $\mathbf{x}_1 \otimes \mathbf{R}_{31} = \mathbf{x}_2$. Or,
    \begin{equation}
        \mathbf{R} = 
         \left[ 
              {\begin{array}{ccc}
               \mathbf{0} & \mathbf{0} & \mathbf{R}_{13}\\
               \mathbf{0} & \mathbf{1} & \mathbf{0}\\
               \mathbf{R}_{31} & \mathbf{0} & \mathbf{0}\\
              \end{array} } 
        \right]
    \end{equation}  
    where $\mathbf{y}_1 \otimes \mathbf{R}_{33} = \mathbf{y}_2$.
\end{proposition}

\renewcommand{\chapter}{\OOOchapter}
\renewcommand{\section}{\OOOsection}

\newgeometry{outer=2.8cm, inner=3.7cm, top=68pt}
\KOMAoptions{headwidth=412pt}

\cleardoublepage
\phantomsection
\addcontentsline{toc}{chapter}{\tocEntry{\nomname}}
\printnomenclature
\label{nomenclature}




\cleardoublepage
\begingroup
  \phantomsection
  \addcontentsline{toc}{chapter}{\tocEntry{\listfigurename}}
  \listoffigures
\endgroup

\cleardoublepage

\begingroup
  \phantomsection
  \addcontentsline{toc}{chapter}{\tocEntry{\listtablename}}
  \listoftables
\endgroup

\cleardoublepage
\defbibheading{bibintoc}[\bibname]{%
  \phantomsection
  \manualmark
  \markboth{\spacedlowsmallcaps{#1}}{\spacedlowsmallcaps{#1}}%
  \addcontentsline{toc}{chapter}{\tocEntry{#1}}%
  \chapter*{#1}%
}
\printbibliography[heading=bibintoc]

\cleardoublepage
\pdfbookmark[0]{Declaration}{declaration}
\chapter*{Declaration}
\thispagestyle{empty}

\textbf{Erklärung über die Eigenständigkeit der Dissertation}

Ich versichere, dass ich die vorliegende Arbeit mit dem Titel 
\glqq Geometric relational embeddings\grqq{} selbständig verfasst und keine anderen als die angegebenen Quellen und Hilfsmittel benutzt habe; aus fremden Quellen entnommene Passagen und Gedanken sind als solche kenntlich gemacht.

\bigskip
\textbf{Declaration of authorship}

I hereby certify that the dissertation entitled
``Geometric relational embeddings''
is entirely my own work except where otherwise indicated. Passages and ideas from other sources have been clearly indicated. 

\bigskip

\noindent\textit{Stuttgart, \submissiondate}

\smallskip

\begin{flushright}
    \begin{tabular}{m{5cm}}
        \\ \hline
        \centering Bo Xiong \\
    \end{tabular}
\end{flushright}

\end{document}